\documentclass[MAL,biber]{nowfnt} 

\usepackage[utf8]{inputenc}

\usepackage{amsmath,amssymb}
\usepackage{xltabular} 
\usepackage{caption}
\usepackage{subcaption} 
\usepackage{enumerate}
\usepackage{multirow}

\usepackage{imakeidx}
\makeindex[intoc]

\renewcommand{\eqref}[1]{Eq. (\ref{#1})}
\newcommand{\figref}[1]{Figure \ref{#1}}
\newcommand{\tbref}[1]{Table \ref{#1}}
\newcommand{\secref}[1]{Section \ref{#1}}
\newcommand{\appref}[1]{Appendix \ref{#1}}
\newcommand{\thref}[1]{Theorem \ref{#1}}
\newcommand{\algref}[1]{Algorithm \ref{#1}}

\newtheorem{prop}{Proposition}

\DeclareMathOperator*{\argmax}{arg\,max}

\usepackage{algorithmic}
\usepackage{algorithm,float}

\def\bx{x}

\usepackage{scalerel,stackengine}
\stackMath
\newcommand\reallywidehat[1]{%
	\savestack{\tmpbox}{\stretchto{%
			\scaleto{%
				\scalerel*[\widthof{\ensuremath{#1}}]{\kern-.6pt\bigwedge\kern-.6pt}%
				{\rule[-\textheight/2]{1ex}{\textheight}}
			}{\textheight}%
		}{0.5ex}}%
	\stackon[1pt]{#1}{\tmpbox}%
}

\title{Energy-Based Models with Applications to Speech and Language Processing}

\maintitleauthorlist{
Zhijian Ou\\
Tsinghua University, Beijing, China\\
ozj@tsinghua.edu.cn
}

\issuesetup
{%
 copyrightowner={Zhijian Ou},
 volume        = xx,
 issue         = xx,
 pubyear       = 2022,
 isbn          = xxx-x-xxxxx-xxx-x,
 eisbn         = xxx-x-xxxxx-xxx-x,
 doi           = 10.1561/XXXXXXXXX,
 firstpage     = 1, 
 lastpage      = 200
 }

\usepackage{mwe}

\author{Ou,Zhijian}

\affil{Tsinghua University, Beijing, China; ozj@tsinghua.edu.cn}

\articledatabox{\nowfntstandardcitation}

\begin{document}

\listofalgorithms
\newpage

\listoffigures
\newpage

\listoftables
\newpage

\makeabstracttitle

\begin{abstract}
Energy-Based Models (EBMs) are an important class of probabilistic models, also known as random fields and undirected graphical models. EBMs are un-normalized and thus radically different from other popular self-normalized probabilistic models such as hidden Markov models (HMMs), autoregressive models, generative adversarial nets (GANs) and variational auto-encoders (VAEs). During these years, EBMs have attracted increasing interests not only from core machine learning but also from application domains such as speech, vision, natural language processing (NLP) and so on, with significant theoretical and algorithmic progress. To the best of our knowledge, there are no review papers about EBMs with applications to speech and language processing. The sequential nature of speech and language also presents special challenges and needs treatment different from processing fix-dimensional data (e.g., images).

The purpose of this monograph is to present a systematic introduction to energy-based models, including both algorithmic progress and applications in speech and language processing, which is organized into four chapters. First, we will introduce basics for EBMs, including classic models, recent models parameterized by neural networks, sampling methods, and various learning methods from the classic learning algorithms to the most advanced ones. The next three chapters will present how to apply EBMs in three different scenarios, i.e., for modeling marginal, conditional and joint distributions, respectively.
1) EBMs for sequential data with applications in language modeling, where we are mainly concerned with the marginal distribution of a sequence itself;
2) EBMs for modeling conditional distributions of target sequences given observation sequences, with applications in speech recognition, sequence labeling and text generation;
3) EBMs for modeling joint distributions of both sequences of observations and targets, and their applications in semi-supervised learning and calibrated natural language understanding.
In addition, we will introduce some open-source toolkits to help the readers to get familiar with the techniques for developing and applying energy-based models.

\end{abstract}

\chapter{Introduction}
\label{ch:intro}

\section{The probabilistic approach}
\label{sec:probabilistic_approach}
As a community we seem to have embraced the fact that dealing with \emph{uncertainty} is crucial for machine intelligence tasks such as speech recognition and understanding, speech synthesis, natural language labeling, machine translation, text generation, computer vision, signal denoising, decision making, and so on.
Uncertainty arises because of limitations in our ability to observe the world, limitations in our ability to model it, and possibly even because of innate nondeterminism \citep{koller2009probabilistic}. 
In face of such uncertainty, we use probabilistic models to describe the random phenomena. Indeed, many tasks in intelligent signal processing and machine learning are solved in the \emph{probabilistic approach}\index{Probabilistic approach}, which generally involves probabilistic modeling, inference and learning, as shown in \figref{fig:prob_approach}. Such probabilistic approach has been introduced in textbooks with sufficient details \citep{koller2009probabilistic,murphy2012machine,bishop2006pattern,hastie2009elements}, and thus in this paper we only give a brief overview as the background material.

\begin{figure}[t]
	\centering
	\includegraphics[scale=0.5]{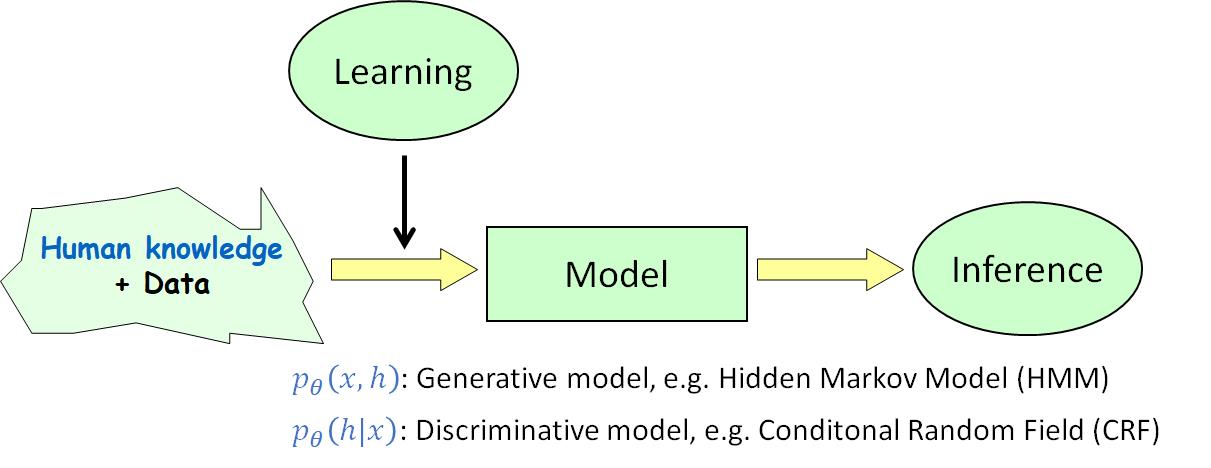}
	\caption{The probabilistic approach}
	\label{fig:prob_approach}
\end{figure}

A \emph{probabilistic model}\index{Probabilistic model} is, in mathematical terms, a distribution over a set of random variables, which are assumed to characterise the random phenomena in the specific task. The set of variables can generally be divided into observations $x$ and (optionally) hidden variables $h$, according to their roles in the task. Hidden variables, or called latent variables, are variables that are part of the model, but which we do not observe, and are therefore not part of the data. 
Remarkably, the observability of some variables may change, depending on what phase (training or testing) the model is used.
A most common example is the \emph{target variable}\index{Target variable} in prediction tasks, such as the class label in classification or the response variable in regression, which is observed in training but becomes unknown in testing.
To avoid clutter in this paper, such variable is viewed as part of the hidden variables and usually denoted by $y$. 

We will typically denote a variable by a lower case letter such as $x$, $h$ and $y$.
Whether $x$ denotes the value that the variable takes or represents the variable itself would be clear from the context.
Further, for notational simplicity, we also use lower case letter (e.g., $x$) to denote a set of random variables, i.e., flattened and concatenated such that the set is represented as a single vector.
So if $x$ is a vector or a sequence, its components can be accessed by subscripts $x_i$.
Here, we are using the terminology \emph{distribution} or \emph{density} loosely, typically denoted by $p$. 
Our notation $p$ should be understood as a mass function (density with respect to counting measure) in the discrete case, and a density function with respect to Lebesgue measure in the continuous case.
See \secref{sec:notations} for more on notations.

Given the form of the probabilistic model, namely the distribution $p_\theta(x,h)$ with parameters $\theta$, there are two crucial problems that must be solved in applying the model in real-world tasks:
\begin{itemize}
    \item Inference: how to reason in the presence of uncertainty;
    \item Learning: how to learn from experience.
\end{itemize}
The former problem is often referred to \emph{probabilistic inference}\index{Probabilistic inference} with a fully-specified model, or inference for short; and the later problem sometimes referred to \emph{statistical inference}\index{Statistical inference} (or more often to say, \emph{learning} in machine learning terminology) for model parameters \citep{neal1993probabilistic}.

Put in a more straightforward way, \emph{learning}\index{Learning} is to find the most appropriate model with parameters, using both data and human knowledge. Human knowledge is implicitly employed to specify the family of parametric distributions, and data are used to estimate the parameters.
Given a fully-specified model, i.e., fully-determined with fixed parameters, \emph{inference} is to infer the unknown from the observation $x$.
There are several typical classes of inference problems: 
\begin{itemize}
    \item Computing conditional probabilities, e.g., $p_\theta(h|x)$.
    This amounts to computing the posterior probability of some variables given the values of other variables (i.e., given evidence on others).
    \item Computing marginal probabilities, including the likelihood $p_\theta(x)$.
    \item Computing modes, e.g., $\arg\max_{h} p_\theta(h|x)$.
    \item Sampling from the model \citep{neal1993probabilistic,liu2001monte}.
\end{itemize}

We provide two more points for readers to appreciate the importance of the inference problems. First, the inference problems themselves are often taken as the means to use the model. For example, speech recognition is generally to find the mode of the posterior distribution on state sequences given observed speech.
Second, learning algorithms often make use of some inference problem as a subroutine.
For example, algorithms that maximize likelihood for learning latent variable models, e.g., the \emph{expectation-maximization} (EM) algorithm \citep{Dempster1977MaximumLF}, call the calculation of $p_\theta(h|x)$ as a subroutine.
Seeking computational efficient algorithms to solve these inference problems for increasingly complex models has been an enduring challenge for our research community.


\subsection{Generative models and discriminative models}

One major division in the probabilistic approach is generative versus discriminative modeling.
In generative modeling, one aims to learn the joint distribution $p_\theta(x,h)$ over all the variables. 
In discriminative modeling, one only models the conditional distribution $p_\theta(h|x)$ over the target variable (denoted by $h$ for convenience) given the observation $x$. 
In discriminative modeling, the observation and the target variable are also called the input and output, respectively.

The generative-discriminative distinction has received much attention in machine learning \citep{ng2001discriminative,liang2008asymptotic}.
When a discriminative model follows the induced form of the conditional distribution $p_\theta(h|x)$ from a generative model $p_\theta(x,h)$, the two models are called a \emph{generative-discriminative pair}\index{Generative-discriminative pair} (i.e., under the same parametric family of models) \citep{ng2001discriminative}.
For example, naive Bayes classifier and logistic regression, \emph{hidden Markov model} (HMM) \citep{rabiner1989tutorial} and \emph{conditional random field} (CRF) \citep{lafferty2001conditional,sutton2012introduction}, form Generative-Discriminative pairs, respectively.
To compare generative and discriminative learning, it seems natural to focus on such pairs.
Basically, there are different regimes of performance as the training set size in increased.
Taking naive Bayes and logistic regression as a case study, it is shown in \citep{ng2001discriminative} that ``while discriminative learning has lower asymptotic error, a generative classifier may also approach its (higher) asymptotic error much faster''.
The comparison of HMM and CRF is further studied in \citep{liang2008asymptotic}, and it is found that generative modeling (modeling more of the data) tends to reduce asymptotic variance, but at the cost of being more sensitive to model misspecification.
These previous results, including \citep{ng2001discriminative,liang2008asymptotic}, to name a few, strengthen our basic intuitions about generative-discriminative distinction.

Given that the generative and discriminative estimators are complementary, one natural question is how to interpolate between the two to get the benefits of both. There have studies for \emph{hybrid generative-discriminative}\index{Hybrid generative-discriminative} methods (see \cite{bouchard2007bias} and the references therein). Notably, those hybrid models have been applied for \emph{semi-supervised learning} (SSL)\index{Semi-supervised learning (SSL)}, where one may have few labeled examples and many more unlabeled examples, but mostly based on traditional generative models like naive Bayes.

In recent years, generative modeling techniques have been greatly advanced by inventing new models with new learning algorithms under the umbrella of \emph{deep generative models} (DGMs)\index{Deep generative model (DGM)}, which are characterized by using \emph{multiple layers of stochastic or deterministic variables} in modeling and are much more expressive than classic generative models such as naive Bayes and HMM. See \citep{ou2018review} for a systematic introduction to DGMs from perspective of graphical modeling.
The generative-discriminative discussion continues with new points, when more types of generative models have constantly emerged and become studied.
Here we provide two examples with the new points.
\begin{itemize}
    \item A type of DGMs, \emph{variational autoencoders} (VAEs)\index{Variational autoencoder (VAE)} \citep{kingma2019introduction}, has been successfully applied in the setting of semi-supervised learning.
    \item It is concurrently shown in \citep{song2018learning,grathwohl2019your} that a standard discriminative classifier $p_\theta(y|x)$ can be used to directly define an \emph{energy-based model} (EBM)\index{Energy-based model (EBM)} for the joint distribution $p_\theta(x,y)$.
    It is shown in \citep{song2018learning} that energy-based semi-supervised training of the joint distribution produces strong classification results on par with state-of-art DGM-based semi-supervised methods.
    It is demonstrated in \citep{grathwohl2019your} that energy based training of the joint distribution improves calibration, robustness, and out-of-distribution detection while also generating samples rivaling the quality of recent \emph{generative adversarial network} (GAN)\index{Generative adversarial network (GAN)} \citep{goodfellow2014generative} approaches.
\end{itemize}

\subsection{Conditional models}
\label{sec:cond}

Discriminative models are a kind of conditional models for discriminative tasks.
However, conditional modeling is a more general modeling concept than discriminative modeling.
Basically, a \emph{conditional model}\index{Conditional model} is, in probability terms, a conditional distribution of a random variable of interest, when another variable $c$ is known to take a particular value. In this case, $c$ is often called the \emph{input} of the model. The variable of interest generally can still consist of observable and (optionally) hidden components, denoted by $x$ and $h$ respectively.
Thus, a conditional model can generally be denoted by $p_\theta(x,h|c)$.

Many real-world applications are solved by conditional modeling. Some examples from discriminative tasks are as follows.
\begin{itemize}
    \item First, by abuse of notation, discriminative modeling of image classification involves the conditional model $p_\theta(y|x)$, where $x$ is the input image and $y$ is the images's class.
    \item A more complicated example is the \emph{recurrent neural network transducer} (RNN-T)\index{Recurrent neural network transducer (RNN-T)} model \citep{graves2012sequence} for speech recognition. 
     Let $x$ denote the input speech, $y$ the label sequence (e.g., word transcription), and $\pi$ the hidden state sequence (or say, a path) which realizes the alignment of $x$ and $y$. Then the RNN-T model involves the conditional model $p_\theta(y,\pi|x)$.
     See \secref{sec:rnnt} for more details on RNN-T.
\end{itemize}

Apart from discriminative tasks, conditional models can also be used for \emph{conditional generation}\index{Conditional generation} tasks.
One example is the reverse of the image classification problem: prediction of a distribution over images, conditioned on the class
label.

Importantly, one should keep in mind that the learning and inference methods introduced in unconditional modeling are in theory equally applicable to conditional models. So the basics introduced in Chapter \ref{ch:basics} lay the foundation for both (unconditional) EBMs in Chapter \ref{ch:lm} and conditional EBMs in Chapter \ref{ch:conditional}. On the other hand, the unconditional and conditional settings have their own characteristics, and thus needs different treatments, as we will detail in Chapter \ref{ch:lm} and \ref{ch:conditional} respectively.

\section{Features of EBMs}

In the probabilistic approach, the family of models chosen in real-world applications clearly plays a crucial role.
In terms of graphical modeling terminology \citep{koller2009probabilistic}, probabilistic models can be broadly classified into two classes - directed and undirected.
\begin{itemize}
    \item In \emph{directed graphical models} (DGMs)\index{Directed graphical model (DGM)}, also known as (a.k.a.)\emph{Bayesian networks} (BNs)\index{Bayesian network (BN)} or called \emph{locally-normalized models}\index{Locally-normalized model}, the distribution is factorized into a product of local conditional density functions.
    \item In contrast, in \emph{undirected graphical models} (UGMs)\index{Undirected graphical model (UGM)}, also known as \emph{Markov random fields} (MRFs)\index{Markov random field (MRF)} or \emph{energy-based models} (EBMs)\index{Energy-based model (EBM)} or called \emph{globally-normalized models}\index{Globally-normalized model}, the distribution is defined to be proportional to the product of local potential functions. The three terms, UGMs, MRFs and EBMs, are exchangeable in this monograph.
\end{itemize}
Simply speaking, an easy way to tell an undirected model from a directed model is that an undirected model is un-normalized and involves the normalizing constant (also called the partition function in physics), while the directed model is self-normalized.

In general, directed models and undirected models
make different assertions of conditional independence.
Thus, there are families of probability distributions that
are captured by a directed model and are not captured
by any undirected model, and vice versa \citep{pearl1988probabilistic}.
Therefore, undirected models, though less explored, provide an important complementary choice to directed models for various real-world applications.

During these years, EBMs have attracted increasing interests not only from core machine learning but also from application domains such as speech, vision, natural language processing and so on, with significant theoretical and algorithmic progress. There have emerged a dedicated workshop at ICLR 2021, which is a broad forum about EBM researches, and a tutorial at CVPR 2021, which focuses on computer vision tasks.
\begin{itemize}
\item ICLR2021 Workshop - Energy Based Models: Current Perspectives, Challenges, and Opportunities, \url{https://sites.google.com/view/ebm-workshop-iclr2021}
\item CVPR 2021 Tutorial: Theory and Application of Energy-Based Generative Models, \url{https://energy-based-models.github.io/}
\end{itemize}

To the best of our knowledge, there are no review papers about EBMs with applications to speech and language processing. 
The sequential nature of speech and language also presents special challenges and needs treatment different from processing fix-dimensional images that was described in the CVPR 2021 tutorial.
The aim of this monograph is to present a systematic introduction to energy-based models, including both algorithmic progress and applications in speech and language processing.
We hope it will also be of general interests to the artificial intelligence and signal processing communities.

Before delving into the specific content, we first point out \emph{five key features of EBMs}, which may motivate you to pursue the study and application of EBMs.
\begin{itemize}
    \item Flexibility in modeling. 
Compared to modeling a self-normalized density function, learning EBMs relaxes the normalization constraint and thus allows much greater flexibility in the parameterization of the energy function. 
Moreover, undirected modeling is more natural for certain domains, where fixing the directions of edges is awkward in a graphical model.
    \item Computation efficiency in likelihood evaluation, since the negative log likelihood of an EBM (by ignoring an additive constant) can be easily evaluated, without incurring any calculation for normalization.
    \item Naturally overcoming label bias and exposure bias suffered by locally-normalized models (\secref{sec:bias}). 
    \item Superiority for hybrid generative-discriminative and semi-supervised learning (\secref{ch:jem}).
    \item Challenge in model training. Both computation of the exact likelihood and exact sampling from EBMs are generally intractable, which makes training especially difficult. 
\end{itemize}

\section{Organization of this monograph}

\begin{figure}[t]
	\centering
	\includegraphics[scale=0.3]{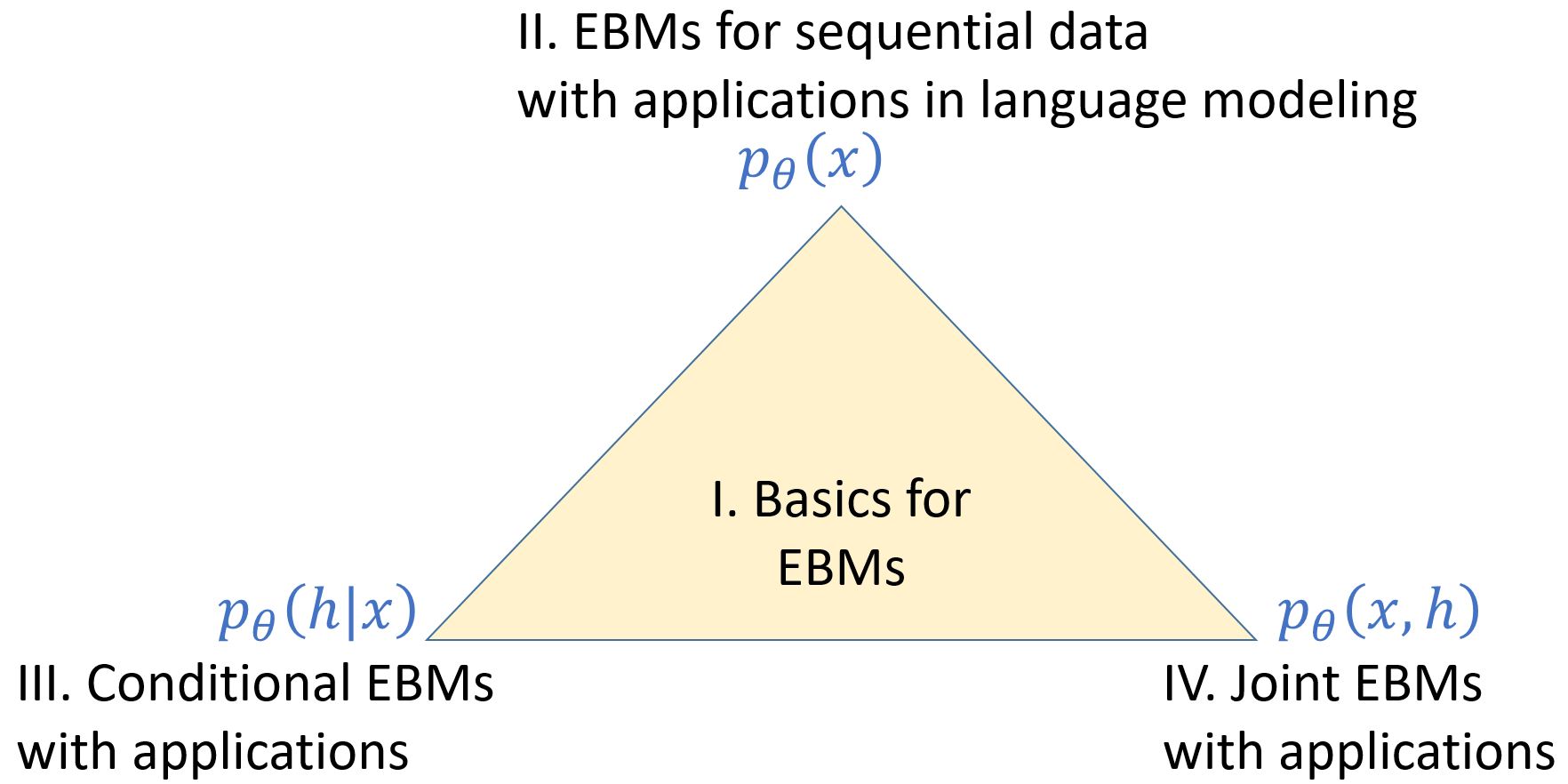}
	\caption{Outline of this monograph}
	\label{fig:outline}
\end{figure}

The rest of the monograph is organized as follows.

In Chapter \ref{ch:basics}, we present basics for EBMs. We start with a brief introduction to probabilistic graphical models (PGMs), because basically we introduce EBMs as undirected graphical models (UGMs).
Then, we present EBM model examples, including both classic ones (such as Ising model and restricted Boltzmann machines) and modern ones parameterized by neural networks.
Next, basic algorithms for learning EBMs are described, which covers the two most widely used classes of methods - Monte Carlo based maximum likelihood methods and noise-contrastive estimation (NCE) methods.
Finally, we present a dedicated section to introduce how to sample/generate from EBMs, since sampling is not only a critical step in maximum likelihood learning of EBMs, but also itself forms as an important class of applications in speech and language processing.

The basics for inference and learning with EBMs are general for both discrete and continuous data modeling.
Remarkably, most applications covered in this monograph is discrete data modeling (text in natural language processing, discrete labels in speech recognition), but in some places, we also present examples and applications in images. For example, Ising model is introduced for readers to get the abstract concepts conveyed by EBMs.
EBM based joint-training for semi-supervised image classification is a fixed-dimensional counterpart of the more complicated sequence setting, which is for semi-supervised natural language labeling.

The next three chapters are devoted to introduce how to develop EBMs in three different scenarios respectively.
\begin{itemize}
    \item Note that the sequential nature of speech and language presents special challenges and needs treatment different from processing fix-dimensional data (e.g., images).
In Chapter \ref{ch:lm}, we introduce EBMs for sequential data with applications in language modeling. In this scenario, we are mainly concerned with learning the (\emph{marginal}) distribution of an observation sequence $x$ itself, e.g., a natural language sentence as in language modeling.
\item In Chapter \ref{ch:conditional}, we introduce EBMs for modeling \emph{conditional} distributions of target sequences given observation sequences. Conditional EBMs have been successfully applied in speech recognition, sequence labeling in natural language processing (NLP)\index{Natural language processing (NLP)}, and various forms of conditional text generation (e.g., controlled text generation, factual error correction).
\item In Chapter \ref{ch:jem}, we introduce EBMs for modeling \emph{joint} distributions of both sequences of observations and targets.
We first introduce the fixed-dimensional case, then move on to the sequential case, and finally present the applications in semi-supervised natural language labeling and calibrated natural language understanding.
\end{itemize}

Finally, conclusions are given in Chapter \ref{ch:conclusion} to summarize the monograph and to discuss future challenges and directions.

We visualize the content of this monograph in \figref{fig:outline}.
At the center is the basic knowledge for EBM modeling and learning. 
The basic theory can be applied to model different types of distributions – the distribution of the observation itself, the conditional distribution, and the joint distribution. In different applications or scenarios, we are concerned with different types of distributions.
In Chapter \ref{ch:lm}, \ref{ch:conditional}, and \ref{ch:jem}, we in fact show how to develop EBMs for the three different types of distributions in three different scenarios, respectively, as described above.

This monograph contains the material expanded from the tutorial that the author gave at ICASSP 2022 in May 2022. Substantial updates have been
made to incorporate more recent work and cover wider areas of research activities.

\chapter{Basics for EBMs}
\label{ch:basics}

Basically we introduce EBMs as undirected graphical models. We begin with background on probabilistic graphical models, which would be beneficial for readers to intuitively appreciate the differences between directed graphical models (DGMs) and undirected graphical models (UGMs).

\section{Probabilistic graphical models (PGMs)}
\label{sec:pgm}

Probabilistic graphical models provide a general framework for describing and applying probabilistic models in the \emph{probabilistic approach}. Many ideas developed in the probabilistic approach can be understood, unified, and generalized within the formalism of graphical models. 
\begin{quote}
``A graphical model is a family of probability distributions defined in terms of a directed or undirected graph.
The nodes in the graph are identified with random variables, and joint probability distributions are defined by taking products over functions defined on connected subsets of nodes.'' \citep{jordan2004graphical}
\end{quote}

\begin{figure}[t]
	\centering
	\includegraphics[scale=0.8]{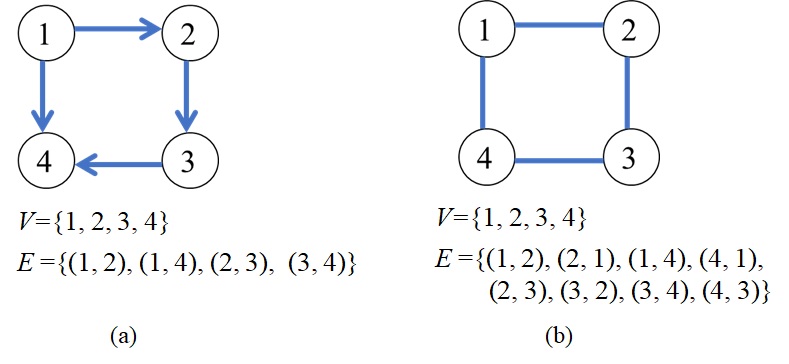}
	\caption{(a) A simple directed graphical model with four variables $(x_1, x_2, x_3, x_4)$. (b) A simple undirected graphical model with four variables $(x_1, x_2, x_3, x_4)$. For both types of graphs, $V$ denotes the set of nodes and $E$ the set of edges. If both ordered pairs $(\alpha, \beta)$ and $(\beta, \alpha)$ belong to $E$, we say that we have an undirected edge between $\alpha$ and $\beta$. A nice introduction of graph theory in the context of graphical models could be found in Chapter 4 of \cite{Cowell1999ProbabilisticNA}.
 }
\label{fig:GM_examples}
\end{figure}

Consider a graph $G = (V, E)$ where $V$ is a set of vertices (also called nodes) and the set of edges $E$ is a subset of the set $V \times V$. 
See \figref{fig:GM_examples} for an illustration of these concepts from graph theory.
Let\footnote{As described in \secref{sec:notations}, we allow sets of indices to appear wherever a single index appears.} $x_V \triangleq \{x_v : v \in V\}$ be a collection of random variables indexed by the nodes of the graph.
A \emph{graphical model}\index{Graphical model} in terms of $G$ describes a family of probability
distributions $p(x_V)$ over the variables $x_V$. \emph{A variable can either be scalar- or vector-valued}, where in the latter case the vector variable implicitly corresponds to a sub-graphical model over the elements of the vector. 

The edges $E$ specifies the connections between the nodes and, according to the graph semantics (see below), plays a crucial role in defining the graphical model distribution.
One view is that the edges $E$, depending on the graph semantics, determines a particular factorized form of the distribution.
Another view is that the edges $E$, of course still depending on the graph semantics, determines a particular set of conditional independence (CI) assumptions over the random variables.
In both views, the properties, either \emph{the factorized form} or \emph{the CI properties}, implied by the graphical model are true for all members of its associated distribution family.
As we will see in \secref{sec:dgm_two_views} and \secref{sec:ugm_two_views}, the two views are, in a strong sense, equivalent.

Graphical models can be defined over different types of graphs, directed, undirected or mixed, each with differing semantics. The \emph{semantics} specifies what a graphical model means and tells how the family of distributions is defined \citep{russell2010artificial,koller2009probabilistic}.
The set of CI properties specified by a particular graphical model, and therefore the family of probability distributions it represents, will be different depending on the type of graphical model currently being considered. 

The two most common forms of graphical model (GM) are directed graphical models (DGMs) and undirected graphical models (UGMs), based on directed acylic graphs and undirected graphs, respectively.
In general, directed graphs and undirected graphs
make different assertions of conditional independence.
Thus, there are families of probability distributions that
are captured by a directed graph and are not captured
by any undirected graph, and vice versa \citep{pearl1988probabilistic,koller2009probabilistic}. 

\subsection{Directed graphical models} 
\label{sec:dgm}

Let us begin with the directed case. Continuing with the notations in \secref{sec:pgm}, let $G = (V, E)$ be a directed acyclic graph (DAG), and $x_V$ be a collection of random variables indexed by the nodes of $G$.
For each node $v \in V$, let $\text{pa}(v)$ denote the subset of indices of its parents; thus $x_{\text{pa}(v)}$ denotes the vector of random variables indexed by the parents of $v$. 

\begin{definition}[DGM]
A \emph{directed graphical model}\index{Directed graphical model (DGM)} in terms of $G$ consists of a family of distributions that factorize in the following way: 
\begin{equation} \label{eq:dgm}
p(x_V) = \prod_{v \in V} p(x_v | x_{\text{pa}(v)})
\end{equation}
We then also say that $p(x_V)$ has the directed factorization property (DF) according to $G$, or simply, $p(x_V)$ factorizes according to $G$.
\end{definition}

Remarkably, the notation in \eqref{eq:dgm} is self-consistent, because it can be verified that the joint distribution $p(x_V)$ defined by the factorization \eqref{eq:dgm} has $\{p(x_v | x_{\text{pa}(v)})\}$ as its conditionals.
For the simple directed graphical model shown in \figref{fig:GM_examples}(a), the joint distribution that it describes is:
\begin{displaymath}
p(x_1,x_2,x_3,x_4) = p(x_1) p(x_2|x_1) p(x_3|x_2) p(x_4|x_1,x_3)    
\end{displaymath}

General speaking, the nodes in a graphical model correspond to the random variables, and the edges indicate some direct probabilistic interactions between the nodes. In directed graphical models, every edges is directed and, intuitively, correspond to \emph{direct influence} of the parent node on the child node. Thus, DGMs are suitable for modeling clear influence relationships between random variables, which are expressed through conditional distributions.

\subsubsection{Factorization and Markov properties in directed graphical models}
\label{sec:dgm_two_views}
The factorization in \eqref{eq:dgm} implies a set of conditional independence statements among the variables $x_V$.
In an opposite way, we could define a set of conditional independence statements in terms of $G$, which is often referred to as a \emph{Markov property} over $G$. A range of Markov properties could be defined, relative to $G$, such as the directed global Markov property (DG), the directed local Markov property (DL), the directed ordered Markov property (DO) (See Section 5.3 of \cite{Cowell1999ProbabilisticNA}). 
It can be shown that the three Markov properties, (DG), (DL) and (DO), are equivalent, and further, they are equivalent to the (DF) property. 
\emph{These properties collectively characterize a graphical model} (i.e., a family of distributions defined in terms of a graph). 
The Markov properties of a distribution are precisely what allow it to be expressed compactly in a factorized form. Conversely, a particular factorization of a distribution guarantees that certain independencies hold. 

\subsubsection{DGM example - HMM}
\label{sec:hmm}

Many classic probabilistic models in speech and language processing can be easily understood in terms of graphical models. \figref{fig:hmm} shows the graphical model representation of an \emph{hidden Markov model} (HMM)\index{Hidden Markov model (HMM)} \citep{rabiner1989tutorial}, which has been widely used in speech recognition and various natural language processing tasks.

\begin{figure}[t]
	\centering
	\includegraphics[scale=0.8]{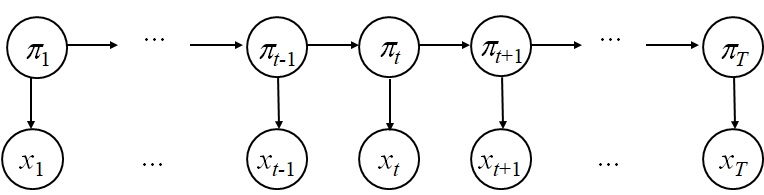}
	\caption{Graphical model representation of a hidden Markov model (HMM).}
	\label{fig:hmm}
\end{figure}

In an HMM, there is an underlying hidden Markov chain, corresponding to the state sequence $\pi_{1:T} \triangleq \pi_1 \cdots \pi_T$. At each time frame, depending on the state $\pi_t$, the model probabilistically emit an output $x_t$, which can be observed. 
The joint distribution is given by
\begin{equation} \label{eq:hmm}
p(\pi_{1:T}, x_{1:T}) = p(\pi_1) \prod_{t=1}^{T-1} p(\pi_{t+1} | \pi_{t} ) \prod_{t=1}^{T} p(x_{t} | \pi_{t} )
\end{equation}
where $p(\pi_{t+1} | \pi_{t} )$ and $p(x_{t} | \pi_{t} )$ are often called state-transition distribution and state-output distribution, respectively.
It is easy to verify that the graphical model as shown in \figref{fig:hmm} exactly describes the joint distribution \eqref{eq:hmm}, by following the DGM semantics - \emph{the joint distribution is defined as the product of local conditionals of each variable given its parents}. Through this example, readers can appreciate the naturalness of the graphical model approach in formulating probabilistic models of complex phenomena.

\subsubsection{DGM example - Neural network based classifier}
\label{sec:logistic_regression}

Traditionally, each conditional probability distribution $p(x_v | x_{\text{pa}(v)})$ is parameterized as a lookup table or a linear model \citep{koller2009probabilistic}. 
A more flexible way to parameterize such conditional distributions is with neural networks. In this case, a neural networks takes as input the parents of a variable in a directed graph, and produces the \emph{distributional parameters} over the variable. 
\begin{align*}
\eta &= \text{NeuralNet}(x_{\text{pa}(v)}) \\
p(x_v | x_{\text{pa}(v)}) &= \text{PDF}(x_v|\eta)
\end{align*}
where we use $\text{NeuralNet}(\cdot)$ and $\text{PDF}(\cdot|\eta)$ to generally denote a neural network (NN) function and a probability density function (PDF) parameterized by $\eta$, respectively.
For example, if $x_v$ is a continuous variable, $\eta$ could denote the mean and variance parameters; if $x_v$ is a discrete variable, $\eta$ could denote the logits (as explained below).

\begin{figure}[t]
	\centering
	\includegraphics[scale=0.5]{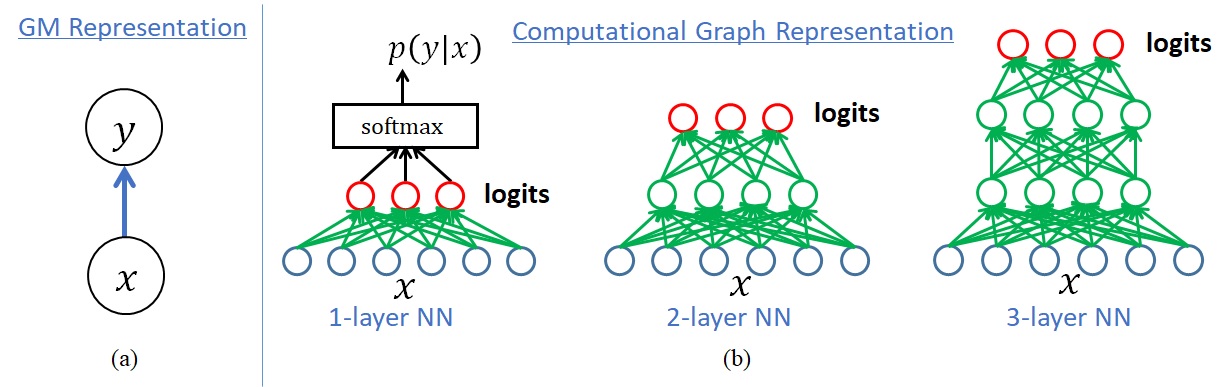}
	\caption{Neural network based classifier. (a) GM representation; (b) Computational graph representation.}
	\label{fig:nn_classifier}
\end{figure}

Basically, we can employ the above concept to build arbitrary directed graphical models, which is not the main focus of this monograph. For illustration, let us examine the widely used NN based classifier, which could be cast as a simple two-node directed model for observation $x \in \mathbb{R}^D$ and class label $y \in \{1,\cdots,K\}$, as shown in \figref{fig:nn_classifier}. It is important to differentiate the GM representation and computational graph representation.

The classic \emph{multi-class logistic regression}\index{Multi-class logistic regression} \citep{bishop2006pattern,murphy2012machine} basically is to use a single linear layer to obtain the \emph{logits} $z_k$'s, which are then fed to a \emph{softmax}\index{Softmax} layer to calculate the class posterior:
\begin{equation}
\label{eq:softmax}
    p(y=k|x) = \frac{\exp(z_k)}{\sum_{j=1}^K \exp(z_j)} \triangleq \text{softmax}(z_{1:K})_k    
\end{equation}
where 
\begin{equation}\label{eq:linear_layer}
z_k = w_k^T x + b_k, k=1,\cdots,K
\end{equation}
are often called the \emph{logits}\index{Logits}\footnote{Presumably because the argument of the sigmoid function is often called the logit, so analogously, the argument of the softmax function (as a multi-class generalization of the sigmoid) is also called the logit.}, and $w_k \in \mathbb{R}^D, b_k \in \mathbb{R}$ denote the weight vector and bias of the linear layer. A simple notation to describe the \emph{linear layer}\index{Linear layer} \eqref{eq:linear_layer} is
\begin{displaymath}
z_k = \text{Linear}(x|w_k, b_k)
\end{displaymath}
or denoted as $z_k = \text{Linear}(x)$ when the parameters are suppressed.

A recent advance in deep learning is that we can use a multi-layer neural network, often referred to as a \emph{deep neural network}\index{Deep neural network (DNN)} (DNN),  to calculate the logits, and then still use the softmax function to obtain the probability vector from the logits. In this way, the multi-layer NN could be viewed as a non-linear \emph{feature extractor}\index{Feature extractor}, which hopefully can be trained to extract features, more discriminative than the raw observation $x$.
Here in describing an NN based classifier, we show a directed model which has shallow stochastic connections but use deep deterministic layers in implementing the conditional distributions. \emph{Variational autoencoder} (VAE)\index{Variational autoencoder (VAE)} \citep{kingma2019introduction} is also such a model.
Further, directed models with deep stochastic connections have also been examined such as such as \emph{Sigmoid belief Networks} (SBNs)\index{Sigmoid belief network (SBN)} \citep{Neal1992ConnectionistLO,Saul1996}, \emph{Helmholtz machines} (HMs)\index{Helmholtz machine (HM)} \citep{Hinton1995the,dayan1995helmholtz}.

\subsection{Undirected graphical models}
\label{sec:ugm}

Let us now consider the undirected case. Given an undirected graph $G=(V,E)$, we again let $x_V$ be a collection of random variables indexed by the nodes of the graph and let $\mathcal{C}$ denote the set of  cliques\footnote{A subset of nodes $C$ is called a clique\index{Clique}, if every pair of nodes in $C$ is joined.} of the graph.
Associated with each clique $C \in \mathcal{C}$, let $\phi_C(x_C)$ denote a \emph{potential function}, which is a non-negative function of its arguments.

\begin{definition}[UGM]
With the above notation, an \emph{undirected graphical model}\index{Undirected graphical model (UGM)} in terms of $G$ consists of a family of distributions that factorize as:
\begin{equation} \label{eq:ugm}
p(x_V) = \frac{1}{Z} \prod_{C \in \mathcal{C}} \phi_C(x_C)
\end{equation}
where $Z$ is the normalizing constant (also known as the partition function) given by
\begin{equation} \label{eq:partition_func}
Z = \sum_{x_V} \prod_{C \in \mathcal{C}} \phi_C(x_C)
\end{equation}
We then also say that $p(x_V)$ has the factorization property (F) according to $G$, or simply, $p(x_V)$ factorizes according to $G$.
\end{definition}

For the simple undirected graphical model shown in \figref{fig:GM_examples}(b), the joint distribution that it describes is\footnote{Note that factorization over the set of cliques can be easily shown to be equivalent to factorization over the set of maximal cliques, i.e., the set of all cliques that are not properly contained within any other clique. Therefore, the joint distribution in this example can be written as the product of 4 potentials over the 4 maximum cliques, divided by the normalizing constant.}:
\begin{displaymath}
p(x_1,x_2,x_3,x_4) = \frac{1}{Z} \phi(x_1, x_2) \phi(x_2, x_3) \phi(x_3, x_4) \phi(x_1, x_4)    
\end{displaymath}

Remarkably, the potentials do not need to be self-normalized, only required to be non-negative functions. In contrast, the conditionals in directed models are required to be normalized. Thus, undirected models generally offer more flexibility in modeling than directed models. Moreover, undirected models do not require us to specify edge orientations, and are well suited to be used in problems in which there is little directional structure to guide the construction of a directed graph.

\subsubsection{Factorization and Markov properties in undirected graphical models}
\label{sec:ugm_two_views}

We have discussed the equivalence of factorization and conditional independence in directed models in \secref{sec:dgm_two_views}. Similarly, a range of Markov properties could be defined, relative to a undirected graph $G$, such as the global Markov property (G), the pairwise Markov property (P), the local Markov property (L) (See Section 5.2 of \cite{Cowell1999ProbabilisticNA}). However, in contrast to the discussion in the case of directed graphical models, the four properties, (F), (G), (L) and (P), are different in general. In general, we have (F)$\Rightarrow$(G)$\Rightarrow$(L)$\Rightarrow$(P).
In the case where $p(x_V)$ has a positive density (i.e., never zero or negative for any value of $x_V$), it can be shown that (P) implies (F), and thus the four properties become equivalent. This result is known as the Hammersley-Clifford theorem.

\begin{figure}[t]
\centering
\includegraphics[scale=0.8]{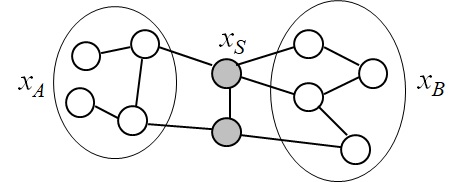}
\caption{Illustration of the global Markov property in UGMs.}
\label{fig:ugm-G}
\end{figure}

When we use the term undirected graphical model without further qualification, we shall always mean one that factorizes, hence satisfies all of the properties. 
In the following, we will detail the (G) property, which is important because it enable us to easily decide when two groups of variables $A$ and $B$ are conditionally independent give a third group of variables $S$ in an undirected model.

A probability distribution $p(x_V)$ is said to obey the global Markov property (G), relative to a undirected graph $G$ , if for any triple $(A, B, S)$ of disjoint subsets of $V$ such that $S$ separates $A$ from $B$\footnote{We say $S$ separates $A$ from $B$ if all trails from $A$ to $B$ intersect $S$.}, we have $x_A \perp x_B | x_S$.
Simply put, the (G) property means: separation between nodes in the graph implies CI between variables in the distribution, as illustrated in \figref{fig:ugm-G}.

\subsubsection{Energy-based models and Gibbs distributions}
\label{sec:ugm_ebm}

\begin{definition}[EBM]
When we are restricted to potential functions which are strictly positive, it is convenient to express them as exponentials, so that
\begin{displaymath}
    \phi_C(x_C) = \exp[-E_C(x_C)]
\end{displaymath}
where $E_C(x_C)$ is called an \emph{energy function}\index{Energy function}. 
Hence, negative log-potential is often called energy, and high probability states correspond to low energy configurations.
Distributions of this exponential form are called \emph{energy-based models} (EBMs)\index{Energy-based model (EBM)}, also known as the Gibbs (or Boltzmann) distributions, originating from statistical physics:
\begin{equation} \label{eq:ebm}
 p(x_V) = \frac{1}{Z} \exp\left[ \sum_{C \in \mathcal{C}} -E_C(x_C) \right]
\end{equation}
\end{definition}
A benefit of the form of EBMs is that unlike the potential functions, the log-potential functions (or after negating, the energy functions) are not constrained to be non-negative and can be very flexibly parameterized.

\subsubsection{Log-linear models and maximum entropy models}
\label{sec:ugm_loglinear}
\begin{definition}[Log-linear model]\index{Log-linear model}
A classic approach to implement log-potentials is to define them as a linear function of the parameters:
\begin{displaymath}
   \log \phi_C(x_C) = \theta_C^T f_C(x_C) 
\end{displaymath}
where $f_C(x_C)$ is a feature vector derived from (the values of) the variable $x_C$, $\theta_C$ is the associated feature weight vector. The resulting distribution has the form
\begin{equation}
\label{eq:log-linear}
p(x_V) = \frac{1}{Z(\theta)} \exp\left[ \sum_{C} \theta_C^T f_C(x_C) \right] 
\end{equation}
where $\theta = \{\theta_C ~|~ C \in \mathcal{C}\}$ collectively denotes the model parameters and we explicit the dependence of the normalizing constant on $\theta$. This is known as a \emph{log-linear model} or an exponential-family
model \citep{wainwright2008graphical,murphy2012machine}. 
\end{definition}

\begin{example}[Word morphology] 
\label{eg:word_morph}
\index{Word morphology}
As an illustrative example, suppose we are interested in making a probabilistic model of English spelling. This is known as the \emph{word morphology} problem, which aims to model English words as letter sequences, $x_1,x_2,\cdots$, by assigning probabilities \citep{Wang2017LearningTR,inducingfeatures,murphy2012machine}.
Since certain letter combinations occur together quite frequently (e.g., ``ing''), we will need higher order cliques to capture this. Suppose we limit ourselves to letter trigrams. 
Since the variables in the sequence here are discrete, we can represent the potential functions as tables of (non-negative) numbers, which are often called \emph{tabular potentials}\index{Tabular potential}.
A tabular potential has $26^3 = 17,576$ parameters in it. However, most of these triples will never occur.
An alternative approach is to define indicator functions (as \emph{feature functions}\index{Feature function}) that look for certain special triples, such as ``ing'', ``qu-'', etc. Then we can define the potential at position $t$ as follows:
\begin{displaymath}
\phi_t(x_{t-1}, x_t, x_{t+1}) = \exp \left( \sum_k \theta_k f_k(x_{t-1}, x_t, x_{t+1}) \right)
\end{displaymath}
where $k$ indexes the different features, corresponding to ``ing'', ``qu-'', etc., and $f_k$ is the corresponding binary feature function\footnote{For example, $f_{\text{``ing''}}(x_{t-1}, x_t, x_{t+1})$ equals to 1, if $x_{t-1}=\text{``i''}, x_t=\text{``n''}, x_{t+1}=\text{``g''}$, and equals to 0, otherwise.}. By tying the parameters across positions, we can define the probability of a word of any length $T$ using
\begin{equation}
\label{eq:word_morph}
p(x_{1:T}|\theta) = \exp \left( \sum_{t=1}^T \sum_k \theta_k f_k(x_{t-1}, x_t, x_{t+1}) \right)
\end{equation}
It has been shown in \citep{Wang2017LearningTR} that a (trans-dimensional) log-linear model outperforms a traditional $n$-gram model, using the same set of features.
\end{example}

Modeling with \eqref{eq:word_morph} raises the question of where these feature functions come from. Traditionally, the features $f_C$ are mainly hand-crafted to reflect domain knowledge, and people take efforts for feature engineering. In fact, log-linear models with careful feature engineering were mainstream in NLP before the revival of neural approaches.
Recently, neural networks have been successfully used to implement more general form of EBMs, where, unlike in classic log-linear models, the log-potentials are defined by NN-based non-linear functions, as we will show later in \secref{sec:ebm_nn}. 
Notably, if NN-based log-potential uses a linear final layer, \emph{these NN-based EBMs could in some sense still be viewed as log-linear models}, but on top of the learned features extracted by the trainable neural networks.

Interestingly, log-linear models are closely connected to \emph{maximum entropy (maxent) models}\index{Maximum entropy model (maxent)}.
A classic conclusion is that the maxent distribution (i.e., the distribution with the maximum entropy subject to the constraints that empirical expectation of features equal to model expectation of features) is the same as the maximum likelihood distribution from the closure of the set of log-linear distributions \citep{mackay2003information,inducingfeatures} (see \appref{sec:maxent} for details).
Hence EBMs, which could be viewed as log-linear models as we describe above, also enjoy such a connection to maxent models. 
Remarkably, the EBM distribution $p(x)$ is only known up to a normalizing constant $Z$ and the potentials are not probabilities, it may be hard for us to understand EBM models. Such \emph{maximum entropy property of EBM models} allows us to gain intuitive understanding of EBM models.


\section{EBM model examples}

In this section, we first introduce classic EBM model examples including Ising model and restricted Boltzmann machine, to familiarize readers with basic concepts. Then, we focus on modern ones parameterized by neural networks.

\subsection{Ising model in statistical physics}\index{Ising model}

\begin{figure}[h]
\begin{subfigure}{0.5\textwidth}
\centering
\includegraphics[scale=0.48]{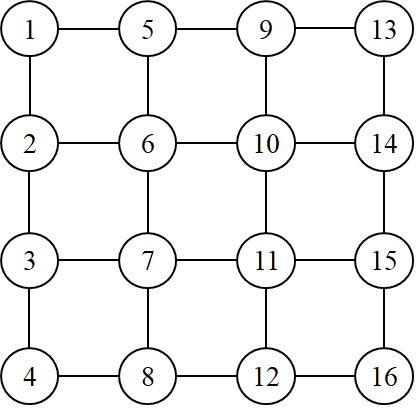}
\caption{}  
\end{subfigure}
\begin{subfigure}{0.5\textwidth}
\includegraphics[scale=0.6]{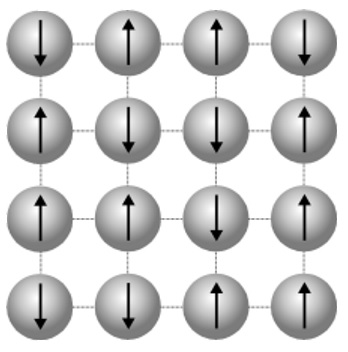}
\centering
\caption{}  
\end{subfigure}
\caption{Ising model: (a) The undirected graph representation, (b) A sample.}  
\label{fig:Ising} 
\end{figure}

A well-known example of EBMs is the Ising model, which arose from statistical physics, and has been a classic probabilistic model for binary images. It was originally used for modeling the behavior of magnets. 
In particular, the spin of an atom can either be spin down or up.
The spins are arranged in a lattice, allowing each spin to interact with its neighbors. Neighboring spins that agree have a lower energy than those that disagree; the system tends to the lowest energy but heat disturbs this tendency, thus creating the possibility of different structural phases. 
The two-dimensional square-lattice Ising model is one of the simplest statistical models to show a phase transition.

Consider a lattice of variables $x_{1:N}$, each component $x_i$ taking values $-1$/$+1$ to model a spin down/up or a pixel black/white, as shown in \figref{fig:Ising}. This $\sqrt{N} \times \sqrt{N}$ lattice, as an undirected graph, defines an Ising model $p(x_{1:N}) \propto \exp[-E(x_{1:N})]$ with the energy function
\begin{equation} \label{eq:Ising}
E(x_{1:N}) = - \beta \left[ \sum_{i \sim j} J x_i x_j + \sum_i H x_i \right]
\end{equation}
where $i \sim j$ denotes that two spins $i$ and $j$ are neighbours.

The energy for an Ising model includes two contributions: the interaction between neighboring spins and the effect of an applied external magnetic field on each individual spin. Consider the case of ferromagnetism. The interaction between neighboring spins tends to induce parallel alignment of the neighbors, so it should be favorable (negative energy) when the neighbors are both $+1$ or both $-1$, and unfavorable (positive energy) when the neighbors are $+1$ next to $-1$. Hence, for each pair of neighbors $i$ and $j$, the interaction energy can be written as $-J x_i x_j$, where $J$ is a positive coefficient giving the interaction strength.
If the applied magnetic field is pointing up, it favors each spin pointing up; if the field is pointing down, it favors each spin pointing down. Hence, for each site $i$, the field energy can be written as $-H x_i$, where $H$ denotes the magnetic moment of the field.
Putting these pieces together, the total energy for the system becomes \eqref{eq:Ising}. 

It is usual to include in \eqref{eq:Ising} the \emph{inverse temperature parameter} $\beta = \frac{1}{k_B T}$, where $k_B$ is Boltzmann's constant, and $T$ the temperature.
$\beta$ measures how much neighboring spins take identical values is favored. The larger $\beta$ (equivalently the lower temperature $T$) is, the more favorable of neighboring spins to take identical values.
\figref{fig:Ising samples} shows a sequence of typical samples from
the simulation of $N = 4096$ spins at a sequence of decreasing temperatures. 
At infinite temperature ($\beta = 0$), each spin is completely independent of any other, and if typical states at infinite temperature are plotted, they look like television snow. For high, but not infinite temperature, there are small correlations between neighboring positions, the snow tends to clump a little bit, but the screen stays randomly looking. When the temperature decreases ($\beta$ increases), it is more favored for neighboring spins to take identical values, so large patches of black or white become to appear.

\begin{figure}[t]
\centering
\includegraphics[scale=0.5]{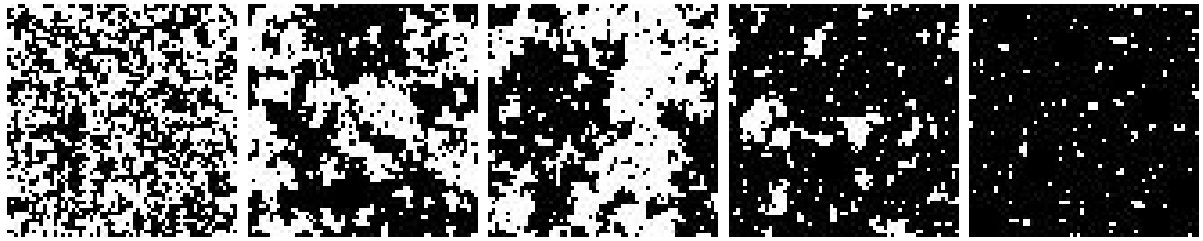}
\caption{Sample states of square Ising models with $J=1, H=0, k_B=1, N=4096$ at a sequence of temperatures $T=5, 2.5, 2.4, 2.3, 2$. \citep{mackay2003information}}
\label{fig:Ising samples}
\end{figure}

Through this example, we could get some sense of the characteristics of EBM modeling - \emph{EBMs are natural for modeling interactions (mutual influences), where the directions of edges cannot be clearly defined}.
For example, here it is hard to say one pixel determines another pixel, even in a probabilistic sense. It is better to use undirected edges to model interactions between pixels through energy functions. Particularly in this example, for each pair of nodes connected by an edge in the lattice, there is an clique potential, which is implemented as follow:
\begin{displaymath}
    \phi_{ij}(x_i,x_j) =
\left\{ \begin{array}{cc}
e^{\beta}, & x_i = x_j=\pm 1\\
e^{-\beta}, & x_i = - x_j=\pm 1
\end{array} \right.
\end{displaymath}
where we set $J=1, H=0, k_B=1$.
So when the two neighboring pixels take the same value, it will contribute $e^{\beta}$ to the un-normalized density; otherwise, contribute $e^{-\beta}$.

\subsection{Restricted Boltzmann Machines (RBMs)}
\label{sec:rbm}
\index{Restricted Boltzmann machine (RBM)}

\emph{Restricted Boltzmann machines} (RBMs) are main building blocks of \emph{deep belief networks} (DBNs) \citep{hinton2006a}, which ignite deep learning.
A RBM is a classic undirected graphical model with hidden variables.
It is defined over a bipartite graph, in which the visible, binary stochastic variables $v \in \{0, 1\}^D$ are connected to hidden binary stochastic variables $h \in \{0, 1\}^H$, as shown in \figref{fig:rbm}. The energy of the state $\{v,h\}$ is defined over cliques\footnote{Each node is a clique, and for each edge connecting $v_i$ and $h_j$, there is a clique.}:
\begin{displaymath}
\begin{split}
   E_\theta(v,h) &= - v^T W h - b^T v - a^T h\\
   &= - \sum_{i=1}^D \sum_{j=1}^H v_i W_{ij} h_j - \sum_{i=1}^D b_i v_i  - \sum_{j=1}^H a_j h_j
\end{split}
\end{displaymath}
where $\theta = \{W,b,a\}$ are the model parameters: $W_{ij}$ represents the symmetric interaction term between visible unit $i$ and hidden unit $j$; $b_i$ and $a_j$ are bias terms. The joint distribution is:
\begin{displaymath}
\begin{split}
   p_\theta(v,h) &= \frac{1}{Z(\theta)} \exp\left[ -E_\theta(v,h) \right]\\
   Z(\theta) &= - \sum_v \sum_h \exp\left[ -E_\theta(v,h) \right]
\end{split}
\end{displaymath}
where $Z(\theta)$ is the normalizing constant.

\begin{figure}[t]
\centering
\includegraphics[scale=0.5]{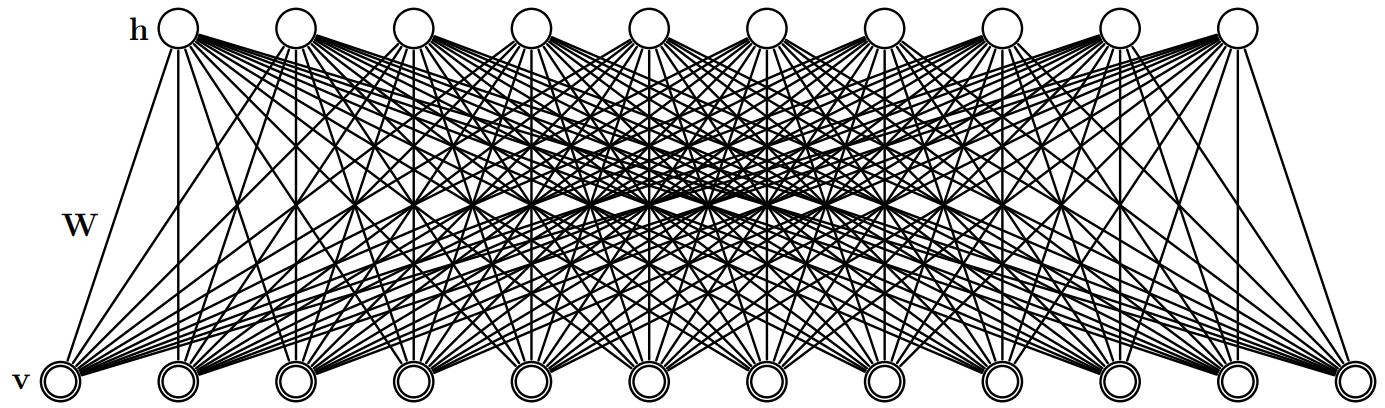}
\caption{Restricted Boltzmann Machine. The top layer represents a vector of stochastic binary hidden variables $h$ and the bottom layer represents a vector of stochastic binary visible variables $v$. \citep{Salakhutdinov2009LearningDG}}
\label{fig:rbm}
\end{figure}

Due to the special bipartite structure of RBMs, given $v$, different hidden units $h_j$'s are separated and thus are conditionally independent, according to the Markov property of UGMs. Therefore, the conditional distribution of $h$ given $v$ is factored:
\begin{displaymath}
\begin{split}
    p_\theta(h|v) &= \prod_j p_\theta(h_j|v)\\
    p_\theta(h_j|v) &\propto \exp \left(\sum_i v_i W_{ij} h_j + a_j h_j \right)
\end{split}
\end{displaymath}
from which we could easily obtain the conditional probability of a single unit $h_j$, expressed by the sigmoid function $\sigma(\cdot)$:
\begin{equation}
\label{eq:rbm_h_given_v}
    p_\theta(h_j = 1|v) = \sigma \left(\sum_i W_{ij} v_i + a_j \right)
\end{equation}

Similarly, the conditional distribution of $v$ given $h$ is also factored and given by:
\begin{align}
    p_\theta(v|h) &= \prod_i p_\theta(v_i|h) \nonumber \\
    p_\theta(v_i|h) &\propto \exp \left(\sum_j v_i W_{ij} h_j + b_i v_i \right) \nonumber \\
    p_\theta(v_i = 1|h) &= \sigma \left(\sum_j W_{ij} h_j + b_i \right) \label{eq:rbm_v_given_h}
\end{align}

Remarkably, it can be seen from \eqref{eq:rbm_h_given_v} that an RBM is related to a stochastic version of a neural network, also known as a \emph{sigmoid belief network} (SBN)\index{Sigmoid belief network (SBN)} \citep{Neal1992ConnectionistLO,Saul1996}.
To see this\footnote{Such relationship could also seen from \eqref{eq:rbm_v_given_h} from an opposite direction.}, imagine that the nodes $v_{1:D}$ and the edges of an RBM as shown in \figref{fig:rbm} are viewed as the input layer and the synaptic connections of a two-layer SBN; the output layer at $h_{1:H}$ fires, $h_j$ taking $0$ or $1$, $j=1,\cdots,H$, stochastically from a sigmoid activation function. Therefore, \emph{the conditional distribution $p_\theta(h|v)$ induced from an RBM can be viewed as implementing a two-layer SBN, and a two-layer SBN is very similar to an ordinary two-layer feedfoward neural network}, except it stochastically fires instead of outputting activations to the next layer (see \figref{fig:rbm_sbn}).
This resemblance between RBMs and NNs is the underlying intuition that a stack of RBMs can be trained as pre-training for a multi-layer neural network \citep{Salakhutdinov2009LearningDG}.


\begin{figure}[h]
\begin{subfigure}{0.5\textwidth}
\centering
\includegraphics[scale=0.6]{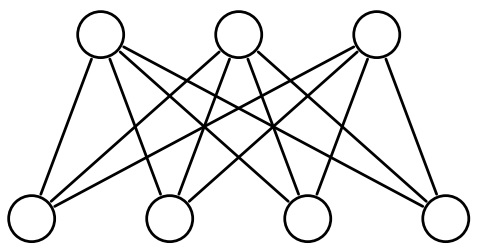}
\caption{}  
\end{subfigure}
\begin{subfigure}{0.5\textwidth}
\includegraphics[scale=0.6]{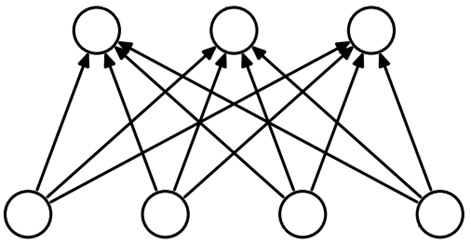}
\centering
\caption{}  
\end{subfigure}
\caption{(a) Restricted Boltzmann machine (RBM), (b) Sigmoid belief network (SBN).}  
\label{fig:rbm_sbn} 
\end{figure}

\subsection{EBMs parameterized by neural networks}
\label{sec:ebm_nn}

Classic EBM models employ simple energy functions, e.g., the energy functions in both the Ising model and the RBM model are bilinear. Recently, beyond the classic EBM models, there have emerged a bundle of \emph{deep EBM models} (deep undirected generative models), which are characterized by using multiple layers of stochastic or deterministic variables.

Those deep EBM models with multiple stochastic hidden layers such as deep belief networks (DBNs)\index{Deep belief network (DBN)} \citep{hinton2006a} and deep Boltzmann machines (DBMs)\index{Deep Boltzmann machine (DBM)} \citep{dbm} involve very difficult inference and learning, which severely limits their applications beyond of the form of pre-training.
Another type of deep EBM models, which appear to be more successfully applied, is to \emph{directly define the energy function through a multi-layer neural network}. In this case, the layers of the network do not represent latent variables but rather are deterministic transformations of input observations.

For simplicity, we will first introduce unconditional models in the following. It is relatively straightforward to extend to models with conditioning variables, which will be detailed in Chapter \ref{ch:conditional}.

\begin{definition}[EBMs parameterized by neural networks]
\label{def:EBM_NN}
\index{EBMs parameterized by neural networks}
    Generally, consider an EBM to define a probability distribution for a collection of random variables $x\in \mathcal{X}$ with parameter $\theta$ in the form:
\begin{equation}
\label{eq:unsup-RF}
p_{\theta}(x)=\frac{1}{Z(\theta)} \exp\left[  U_{\theta}(x) \right] 
\end{equation}
where $\mathcal{X}$ denotes the space of all possible values of $x$, and $Z(\theta)$ denotes the normalizing constant:
\begin{equation}\label{eq:Z}
Z(\theta)=\int\exp\left[  U_{\theta}(x) \right] dx    
\end{equation}
$U_{\theta}(x) : \mathcal{X} \rightarrow \mathbb{R}$ denotes the (log) potential function\footnote{In the literature, 
log potential function is sometimes also referred to as potential function. Whether taking log or not should be clear from the context, although with abuse of nomenclature.
Moreover, reversing the potential function will obtain the energy function, and vise versa. 
\emph{So an EBM could be equivalently defined by an energy function or a potential function}.} which assigns a scalar value to each configuration of $x$ in $\mathcal{X}$ and \emph{can be very flexibly parameterized} through neural networks of different architectures.
For different applications, $\mathcal{X}$ could be discrete or continuous, and $x$ could be fix-dimensional or trans-dimensional (i.e., sequences of varying lengths). For example, images are fix-dimensional continuous data (i.e., $\mathcal{X}=\mathbb{R}^D$), and natural languages are sequences taking discrete tokens (i.e., $\mathcal{X}=\bigcup_{l}\mathbb{V}^l$ where $\mathbb{V}$ is the vocabulary of tokens).
The general idea is to parameterize $U_{\theta}(x)$ by a neural network, taking multi-variate $x$ as input and outputting scalar $U_{\theta}(x)$, so that we can take advantage of the representation power of neural networks, as shown in \figref{fig:ebm-nn}.
\end{definition}

\begin{figure}[t]
\centering
\includegraphics[scale=0.5]{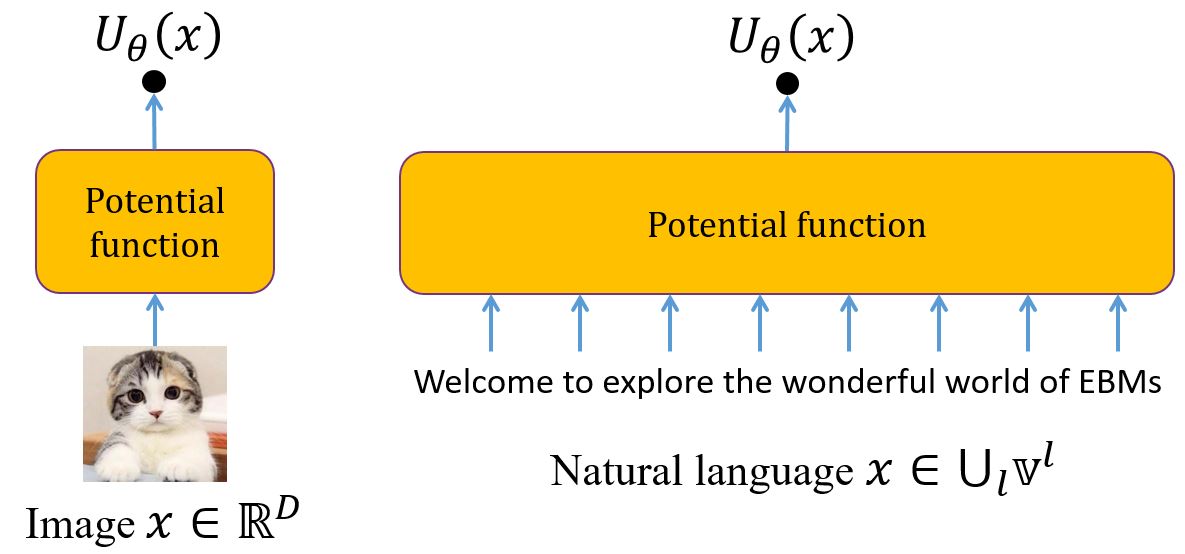}
\caption{Potential functions in EBMs can be flexibly parameterized by neural networks for images, natural languages and so on. }
\label{fig:ebm-nn}
\end{figure}

Historically, this type of EBMs has been studied several times in different contexts. They are once called deep energy models (DEMs) in \citep{ngiam2011learning,Kim2016DeepDG}, generative ConvNet in \citep{xie2016theory}, descriptive models in \citep{descriptor,xie2018cooperative}, neural trans-dimensional random field language models in \citep{wang2017language,BinICASSP2018,wang2018improved,gao2020integrating}, neural random fields (NRFs)\index{Neural random fields (NRFs)} in \citep{song2018learning}.
There are some specific differences between implementation (or say, parameterization) of the potential functions in these deep EBM models.
\begin{itemize}
    \item The potential function in \citep{ngiam2011learning,Kim2016DeepDG} uses the form of a product of experts \citep{hinton2002training} and is composed of linear and squared terms and the aggregated (i.e., the sum of) logistic regression responses of a set of weak classifier (``expert'') for images $x$:
    \begin{equation}\label{eq:dem}
        U_\theta(x) = b^T x -\frac{1}{\sigma^2} x^T x + \sum_i \log(1+e^{w_i^T f_\theta(x)+c_i})
    \end{equation}
The first two terms capture the mean and the global variance, and the last term comes from a set of experts over the feature data space $f_\theta(x)$. 
$f_\theta(x)$ is the output of a feedforward neural network. One can allow activations other than logitic sigmoid as shown in \eqref{eq:dem}.
\item In \citep{wang2017language,BinICASSP2018,wang2018improved,gao2020integrating,deng2020residual}, deep EBM models are defined over sequences (natural language sentences).
\item In \citep{xie2016theory,descriptor,xie2018cooperative,wang2017language,deng2020residual}, the EBM is defined in the form
of exponential tilting of a reference distribution $q(x)$:
    \begin{equation}\label{eq:exp_tilting}
    p_{\theta}(x)=\frac{1}{Z(\theta)} q(x) \exp\left[  U_{\theta}(x) \right]
    \end{equation}
In \citep{xie2016theory,descriptor,xie2018cooperative} for image modeling, a Gaussian white noise distribution is used as $q(x)$. In \citep{wang2017language,deng2020residual} for language modeling, autoregressive language models based on LSTM or Transformer networks are often used as $q(x)$.
\end{itemize}

Despite the different parameterizations of the potential function in various deep EBM models, these differences would not affect much when presenting learning algorithms, as we will introduce immediately in the next section.

\subsubsection{Comparison between classic undirected graphical models and modern EBM models parameterized by neural networks}

As we introduced before, classic undirected graphical models such as the Ising model and the RBM model employ simple potential functions like linear or bilinear, while modern EBM models utilize neural networks to parameterize potential functions. From this apparent difference, there is a remarkable implication, which will be described below.

In graphical modeling terminology, without loss of generality, let each component of $x$ indexed by a node in an undirected graph. The EBM distribution $p_{\theta}(x)$ is defined to be decomposed over cliques, as shown in \eqref{eq:ebm}.
Such decomposition reduces the complexity of model representation but maybe at the sacrifice of model expressive capacity.
In an EBM model parameterized by a neural network as introduced above, 
the model essentially becomes defined over a fully-connected undirected graph and captures interactions in $x$ to the largest order, since the neural potential function $U(x)$ involves all the components in $x$. 
In this manner, hopefully we can take advantage of the representation power of neural networks for modeling.
As shown above, we can define very flexible potential functions $U_{\theta}(x)$, by utilizing neural networks of various architectures to define the densities \eqref{eq:unsup-RF}.
For this reason, it is often assumed in theoretical analysis that $p_{\theta}(x)$ has infinite capacity (sometime called in the non-parametric setting). In practice, the performances of the models largely depend on how they are optimized in model learning.

\section{Learning EBMs by maximum likelihood}
\label{sec:MLE}

The de facto standard for learning probabilistic models from IID (independent and identically distributed) data is \emph{maximum likelihood estimation} (MLE) \index{Maximum likelihood estimation (MLE)}.
Let $p_{\theta}(x)$, as defined in \eqref{eq:unsup-RF}, be an EBM model parameterized by $\theta$, and $p_{\text{emp}}(x) \triangleq \frac{1}{N} \sum_{i=1}^{N} \delta(x-x_i)$ denote the empirical distribution for a training dataset consisting of $N$ IID data points $\left\lbrace x_1, \cdots, x_N \right\rbrace$.
We can fit $p_{\theta}(x)$ to data by maximizing the scaled log-likelihood of the data, defined by
\begin{equation}
\label{eq:ebm_loglik}
   L(\theta) \triangleq \frac{1}{N} \sum_{i=1}^N \log p_{\theta}(x_i)
   = \left[ \frac{1}{N} \sum_{i=1}^N U_{\theta}(x_i) \right] - \log Z_\theta
\end{equation}
as a function of $\theta$.

Maximizing likelihood is equivalent to minimizing the (inclusive) KL divergence between
$p_\text{emp}(x)$ and $p_\theta(x)$, because
\begin{align*}
KL[p_\text{emp}(x) || p_\theta(x)] &= E_{x \sim p_\text{emp}(x)}[\log p_\text{emp}(x)] - E_{x \sim p_\text{emp}(x)}[\log p_\theta(x)]\\
&= \text{constant} - L(\theta),
\end{align*}
where the second equality holds because $E_{x \sim p_\text{emp}(x)}[\log p_\text{emp}(x)]$ does not depend on $\theta$.

Taking the derivative of log-likelihood with respect to (w.r.t.) $\theta$, the first term of the gradient is a sum over data points and can be written as an expectation under the
empirical distribution:
\begin{displaymath}
\frac{1}{n} \sum_{i=1}^N \nabla_\theta U_{\theta}(x_i)
= E_{p_{\text{emp}}(x)}\left[ \nabla_\theta U_{\theta}(x) \right]
\end{displaymath}
The second term involves taking the derivative of the log normalizing constant, which, as shown below, can also be written as an expectation but under model distribution $p_{\theta}(x)$:
\begin{align}
\label{eq:grad_log_Z_equal_expect}
\nabla_\theta \log Z_{\theta} &= \frac{1}{Z_{\theta}} \nabla_\theta Z_{\theta} = \frac{1}{Z_{\theta}} \int \nabla_\theta \exp\left[  U_{\theta}(x) \right] d x \nonumber \\
&= \frac{1}{Z_{\theta}} \int \exp\left[  U_{\theta}(x) \right] \nabla_\theta U_{\theta}(x) d x \nonumber \\
&= \int \frac{\exp\left[  U_{\theta}(x) \right]}{Z_{\theta}} \nabla_\theta U_{\theta}(x) d x \nonumber \\
&= \int p_{\theta}(x) \nabla_\theta U_{\theta}(x) d x \nonumber\\
&= E_{p_{\theta}(x)} \left[ \nabla_\theta U_{\theta}(x) \right]
\end{align}
Combining the two terms, we obtain the core formula in learning EBMs:
\begin{equation} 
\label{eq:ebm_grad}
\nabla_\theta L(\theta) = E_{p_{\text{emp}}(x)}\left[ \nabla_\theta U_{\theta}(x) \right] - 
E_{p_{\theta}(x)} \left[ \nabla_\theta U_{\theta}(x) \right]
\end{equation}

The maximum likelihood estimate of $\theta$ is obtained as a solution to $ \nabla_\theta L(\theta) = 0$. 
Obviously, the challenge is in calculating the expectation under model distribution, which is often intractable to compute exactly and approximated by \emph{Monte Carlo averaging}\index{Monte Carlo averaging}.

Suppose that we can draw random samples from the EBM $p_{\theta}(x)$, denoted as $x^{(1)},\cdots,x^{(M)} \sim p_{\theta}(x)$, then we can obtain an unbiased estimate of the second term in the log-likelihood gradient\footnote{Here we suppose that the samples are direct samples from $p_{\theta}(x)$. As will be described in \secref{sec:SA_ebm}, such assumption can be relaxed in the stochastic approximation methodology, which allows us to use Markov chain samples.}: 
\begin{equation}
\label{eq:ebm_mc_averaging}
E_{p_{\theta}(x)} \left[ \nabla_\theta U_{\theta}(x) \right] \approx \frac{1}{M} \sum_{j=1}^M \nabla_\theta U_{\theta}(x^{(j)})
\end{equation}
Additionally, note that the computational cost of \eqref{eq:ebm_grad} is linear in $N$ (the number of training data points). So when $N$ is large, we often apply \emph{minibatching}\index{Minibatching} as follows. 
Through random drawing a minibatch of $\kappa_1,\cdots,\kappa_B$ from $\{1,\cdots,N\}$, we can obtain an unbiased estimate of the first term in the log-likelihood gradient: 
\begin{equation}
\label{eq:ebm_minibatching}
E_{p_{\text{emp}}(x)}\left[ \nabla_\theta U_{\theta}(x) \right] 
\approx \frac{1}{B} \sum_{j=1}^B \nabla_\theta U_{\theta}(x_{\kappa_j})
\end{equation}
Combining \eqref{eq:ebm_mc_averaging} and \eqref{eq:ebm_minibatching} allows us to optimize the parameters with stochastic gradient ascent. See further introduction in \secref{sec:SA_ebm}.


We have shown above the basic idea of applying Monte Carlo methods in maximum likelihood learning of EBMs. 
As long as we can draw random samples from the model $p_{\theta}(x)$, we have access to an unbiased estimate of the log-likelihood gradient, allowing us to optimize the parameters with stochastic gradient ascent. \textbf{So a critical step in learning EBMs by Monte Carlo methods} is the simulation (sampling) from the EBM distribution $p_{\theta}(x)$, as defined in \eqref{eq:unsup-RF} in general.

In some applications, directly generating independent samples from an distribution is not feasible. There are two broad classes of strategies for sampling from high-dimensional distributions.
The first is the MCMC strategy, which produces statistically dependent samples based on the theory of \emph{Markov chains}. The second is the \emph{importance sampling} (IS) strategy, in which independent samples are generated from a trial distribution (a.k.a. a proposal distribution) and then weighted according to the importance weight. The two methods will be introduced as follows in \secref{sec:MCMC} and \secref{sec:IS} respectively. Further introduction can be found in monographs \citep{neal1993probabilistic, liu2001monte} or general textbooks on machine learning \citep{koller2009probabilistic,murphy2012machine,bishop2006pattern}.

\subsection{Markov Chain Monte Carlo (MCMC)}
\label{sec:MCMC}

Let $p_\theta(x)$ be the target distribution under investigation. The basic idea of \emph{Markov Chain Monte Carlo} (MCMC)\index{Markov Chain Monte Carlo (MCMC)} is to construct a Markov chain in the state space of $x$, denoted by $\mathcal{X}$, so that the limiting (or say, stationary or equilibrium) distribution of this chain is the target distribution $p_\theta$.
Roughly speaking, this means that the fraction of time spent in each state along the chain in a long run is equal to $p_\theta$.

\subsubsection{Metropolis–Hastings algorithm}
\index{Metropolis-Hastings (MH) algorithm}
A classic MCMC method is the \emph{Metropolis-Hastings} (MH) algorithm, which is described in \algref{alg:MH}.
Starting with any configuration $x^{(0)}$, the MH algorithm proceeds by iterating ``propose'' and ``accept/reject'', as shown in Line \ref{alg:line:proposal} and Line \ref{alg:line:acc_rej} respectively. A single run of ``propose'' and ``accept/reject'' is often called a \emph{MH transition}, which defines a particular Markov chain.
\begin{itemize}
    \item First, we propose to move from the previous state $x^{(t-1)}$ to a new state $x'$ with probability $q(x'|x^{(t-1)})$, where $q$ is called the \emph{proposal distribution}\index{Proposal distribution}. 
    \item Having proposed a move to $x'$, we then decide whether to accept this proposal or not according to some formula, which ensures that the limiting distribution of this chain is the target distribution $p_\theta$. If the proposal is accepted, the new state is $x'$, otherwise the new state is the same as the previous state, $x^{(t-1)}$ (i.e., we repeat the sample).
\end{itemize}
At the end of the iterations, we obtain a realization (or say, a single run) of the Markov chain, $x^{(1)}, \cdots, x^{(T)}$.
It can be shown that this particular Markov chain leave $p_\theta$ invariant.
Note that theoretically, only the limiting distribution of the chain follows $p_\theta$. So it is necessary to discard a few initial samples until the Markov chain has \emph{burned in}\index{Burned in}, or entered its stationary distribution. 
The remained samples can then be used for Monte Carlo averaging such as in
\eqref{eq:ebm_mc_averaging} to estimate expectations. 

In practice, the accept/reject step is taken by drawing $U \sim \text{Uni}[0,1]$, calculate the \emph{acceptance probability}
\begin{equation}
\label{eq:MH_r}
    r = \min \left\{ 1, \frac{ p_\theta (x') q(x^{(t-1)}|x') }{ p_\theta(x^{(t-1)}) q(x'|x^{(t-1)}) } \right\}
\end{equation}
and update
\begin{displaymath}
x^{(t)} =
\left\{ \begin{array}{ll}
x', & \text{if}~U \leq r \\
x^{(t-1)}, & \text{otherwise}.
\end{array} \right.
\end{displaymath}

\begin{algorithm}[t]
	\algsetup{linenosize=\small}
	\caption{Metropolis-Hastings Algorithm}
 	\label{alg:MH}
	\begin{algorithmic}[1]
	    \REQUIRE A target distribution $p_\theta(x)$, a proposal distribution $q(x'|x)$
		\STATE Randomly initialize $x^{(0)}$;
		\FOR{$t=1$ to $T$} 
		\STATE Generate $x'$ from the proposal $q(x'|x^{(t-1)})$; \label{alg:line:proposal}
		\STATE Accept $x^{(t)}=x'$ with probability 
  $\min \left\{ 1, \frac{ p_\theta (x') q(x^{(t-1)}|x') }{ p_\theta (x^{(t-1)}) q(x'|x^{(t-1)}) } \right\}$, otherwise set $x^{(t)} = x^{(t-1)}$; \label{alg:line:acc_rej}
		\ENDFOR
		\STATE {\bf Return:} $\{ x^{(1)}, \cdots, x^{(T)} \}$
	\end{algorithmic}
\end{algorithm}

Critically, the MH algorithm only needs to know the target distribution up to a normalization constant. In particular, consider the EBM distribution $p_{\theta}(x)=\frac{1}{Z_\theta} \exp\left[  U_{\theta}(x) \right] $, then \emph{MH ratio}\index{MH ratio}
\begin{displaymath}
\frac{ p_\theta (x') q(x^{(t-1)}|x') }{ p_\theta(x^{(t-1)}) q(x'|x^{(t-1)}) } = \frac{ \frac{1}{Z_\theta} \exp\left[  U_{\theta}(x') \right] q(x^{(t-1)}|x') }{ \frac{1}{Z_\theta} \exp\left[  U_{\theta}(x^{(t-1)}) \right] q(x'|x^{(t-1)}) }
\end{displaymath}
in which the $Z_\theta$'s cancel. Hence, we can sample from $p_{\theta}(x)$ even if the normalizing constant $Z_\theta$ is unknown.

Remarkably, the user is free to use any kind of proposal they want, subject to some theoretical conditions. This makes MH quite a flexible method. We introduce two special algorithms that are instances of the general MH algorithm. 

\paragraph{The Metropolis algorithm.}\index{Metropolis algorithm} 
If the proposal transition function is symmetric, so $q(x'|x) = q(x|x')$, the acceptance probability is given by the following formula:
\begin{displaymath}
r = \min \left\{ 1, \frac{p_\theta (x')}{p_\theta(x^{(t-1)})} \right\}
\end{displaymath}
We see that if $x'$ is more probable than $x$, we definitely move there (since $\frac{p_\theta (x')}{p_\theta(x^{(t-1)})} > 1$), but if
$x'$ is less probable, we may still move there anyway, depending on the relative probabilities. So instead of greedily moving to only more probable states, we occasionally allow ``downhill'' moves to less probable states.

\paragraph{The Metropolis independence sampler (MIS).}\index{Metropolis independence sampler (MIS)} 
Another special choice of the transition function is in the form of $q(x') = q(x'|x)$; that is, the proposed move $x'$ is generated independent of the previous state $x^{(t-1)}$. In MIS, the acceptance probability becomes:
\begin{displaymath}
r = \min \left\{ 1, \frac{w (x')}{w(x^{(t-1)})} \right\}
\end{displaymath}
where $w(x)=\frac{p_\theta(x)}{q(x)}$ is the usual \emph{importance weight}\index{Importance weight}.

\begin{algorithm}[t]
	\algsetup{linenosize=\small}
	\caption{Gibbs sampler}
	\label{alg:Gibbs}
	\begin{algorithmic}[1]
		\REQUIRE A target distribution $p(x)$ for $x=(x_1,\cdots,x_n)$
		\STATE Randomly initialize $x^{(0)}=(x_1^{(0)},\cdots,x_n^{(0)})$;
		\FOR{$t=1$ to $T$} 
		\STATE Pick $x_1^{(t)}$ from the distribution for $x_1$ given $x_2^{(t-1)},x_3^{(t-1)},\cdots,x_n^{(t-1)}$; \label{alg:gibbs:1}
		\STATE Pick $x_2^{(t)}$ from the distribution for $x_2$ given $x_1^{(t)},x_3^{(t-1)},\cdots,x_n^{(t-1)}$;
		\STATE $\vdots$
		\STATE Pick $x_i^{(t)}$ from the distribution for $x_i$ given $x_1^{(t)},\cdots,x_{i-1}^{(t)},x_{i+1}^{(t-1)},\cdots,x_n^{(t-1)}$;
		\STATE $\vdots$
		\STATE Pick $x_n^{(t)}$ from the distribution for $x_n$ given $x_1^{(t)},x_2^{(t)},\cdots,x_{n-1}^{(t)}$; \label{alg:gibbs:n}
		\ENDFOR
		\STATE {\bf Return:} $\{ x^{(1)}, \cdots, x^{(T)} \}$
	\end{algorithmic}
\end{algorithm}

\subsubsection{Gibbs sampling}\index{Gibbs sampling}

The \emph{Gibbs sampler} is conceptually the simplest of the Markov chain sampling method, and as we introduce below, could be viewed as the MCMC analog of coordinate descent.

Suppose we wish to sample from the joint distribution for $x=(x_1,\cdots,x_n)$ given by $p(x_1,\cdots,x_n)$, where the range of the $x_i$ may be either continuous or discrete.
The Gibbs sampler does this by repeatedly replacing each component, say $x_i$, with a value picked from the \emph{full conditional}\index{Full conditional} for variable $x_i$, i.e., the distribution of $x_i$ conditional on the current values of all other components $(x_1,\cdots,x_{i-1},x_{i+1},\cdots,x_n) \triangleq x_{\setminus i} $.
This process can be seen as generating a realization of a Markov chain that is built from a set of base transition probabilities $B_i$, for $i=1,\cdots,n$. $B_i$ leaves all the components except $x_i$ unchanged, and draws a new $x_i$ from its full conditional, which is assumed to be a feasible operation.

The Gibbs sampling algorithm can be described as simulating a homogeneous Markov chain, $x^{(0)}, x^{(1)}, x^{(2)}, \cdots$, with transition matrix $P=B_1 \times B_2 \times \cdots \times B_n$, as shown in \algref{alg:Gibbs}.
Generating $x^{(t)}$ from $x^{(t-1)}$, i.e., from Line \ref{alg:gibbs:1} to Line \ref{alg:gibbs:n}, is called a \emph{sweep}.
Note that the new value for $x_{i-1}$ is used immediately when picking the new value for $x_i$.

Starting from the Gibbs sampler, we provide three useful points.
\begin{itemize}
\item \textbf{Constructing a Markov chain from base transitions.} 
In sampling application, our goal is to find an ergodic Markov chain that converges to the target invariant distribution $p(x)$, at as fast a rate as possible. The Gibbs sampling embodies a useful, general method to construct such a Markov chain, as described below. Consider to construct the transition probabilities for such a chain from a set of base transition probabilities, given by $B_1,\cdots,B_s$\footnote{Generally, the number of base transitions $s$ is not necessarily equal to $n$, the dimensionaliy of $x$.}, each of which leaves the target distribution invariant. 
It can be shown that when the base transitions are applied in sequence, if a distribution is invariant with respect to all the base transitions, then it is also invariant with respect to $P=B_1 \times B_2 \times \cdots \times B_s$ \citep{neal1993probabilistic}.
We show in the next point that each base transition in Gibbs sampler leaves the target distribution invariant, so that we can understand why the Gibbs sampler works.
\item \textbf{Gibbs sampler is a special case of MH}, and thus leaves the target distribution invariant.
Each base transition in Gibbs sampler is equivalent to using MH with a proposal of the form
\begin{displaymath}
q(x'|x) = p(x'_i | x_{\setminus i}) 1(x'_{\setminus i} = x_{\setminus i})
\end{displaymath}
That is, we move to a new state where $x_i$ is sampled from its full conditional, but $x_{\setminus i}$ is left unchanged. It turns out that the acceptance rate of such proposal is 1, because the MH ratio 
\begin{align*}
\frac{p(x')q(x|x')}{p(x)q(x'|x)} &= \frac{p(x'_i|x'_{\setminus i})p(x'_{\setminus i})p(x_i|x'_{\setminus i})}{p(x_i|x_{\setminus i})p(x_{\setminus i})p(x'_i|x_{\setminus i})}\\
&=\frac{p(x'_i|x_{\setminus i})p(x_{\setminus i})p(x_i|x_{\setminus i})}{p(x_i|x_{\setminus i})p(x_{\setminus i})p(x'_i|x_{\setminus i})}=1
\end{align*}
where we exploited that fact that $x'_{\setminus i} = x_{\setminus i}$, and that $q(x'|x)=p(x'_i|x_{\setminus i})$. So every time the Gibbs sampler draws a new value from the full conditional of a component and always accepts it.
\item \textbf{MH within Gibbs sampling.}\index{MH within Gibbs sampling} Gibbs sampling assumes that drawing from the full conditional of each component is tractable. When sampling from the full conditionals of a certain component is intractable, we can replace the exact sampling of this component by a MH sampling step, i.e., a single run of ``propose'' and ``accept/reject''. The resulting algorithm is thus called \emph{MH within Gibbs sampling}\index{MH within Gibbs sampling}.
\end{itemize}

\subsubsection{Gradient guided MCMC}
\label{sec:LD_HMC}

For continuous distribution, MCMC samplers leveraging continuous dynamics (namely continuous-time Markov processes described by stochastic differential equations), such as Langevin dynamics (LD) and Hamiltonian Monte Carlo (HMC)\index{Hamiltonian Monte Carlo (HMC)} \citep{neal2011mcmc}, are known to be efficient in exploring the continuous state space.
Simulating the continuous dynamics leads to the target distribution as the stationary distribution.
In practice, a discretization of the continuous-time system is needed necessitating a Metropolis-Hastings (MH) correction, though still with high acceptance probability.
Recently, stochastic gradient variants of continuous-dynamic samplers have emerged, showing that adding the ``right amount'' of noise to stochastic gradient
ascent iterates leads to samples from the target posterior as the step size is annealed \citep{sgld,sghmc}.
In either manners, the Markov transition kernel defined by the continuous dynamical system usually involves using the gradients of the target distribution w.r.t the data vector $x$.

Remarkably, the gradient of log-density of an EBM model w.r.t. the data vector $x$ is easy to calculate:
\begin{displaymath}
\nabla_x
\log p_\theta(x) =
\nabla_x U_\theta(x)
- \underbrace{\nabla_x \log Z(\theta)}_{=0}
= \nabla_x U_\theta(x)
\end{displaymath}
which does not require the calculation of the normalizing constant.

In the following, we mainly introduce LD and SGLD. For HMC and stochastic gradient Hamiltonian Monte Carlo (SGHMC), readers can refer to \citep{neal2011mcmc,sghmc,ma2015complete}.

\index{Langevin dynamics (LD)}
\paragraph{Langevin dynamics (LD) sampler.}
Given current sample $x_0$, a new observation is proposed as
\begin{align}\label{eq:sgld}
x_{\frac{1}{2}} =   x_0 + \frac{\sigma^2}{2} \nabla_x U_{\theta} ( x_0) + \sigma \varepsilon,
\end{align}
where $\varepsilon \sim \mathcal{N}(0, I)$ is a Gaussian noise, and $\sigma$ is a step size.
The next sample $x_1$ may directly be the proposal $ x_{\frac{1}{2}}$,
in which case the Markov transition from $ x_0$ to $ x_1$ does not strictly leave $p_{\theta}$ invariant,
but the sampling bias may usually be small for $\sigma \approx 0$.
To allow large $\sigma$, a correction can be achieved
by accepting or rejecting the proposal $ x_{\frac{1}{2}}$, i.e., setting $ x_1=  x_{\frac{1}{2}}$ or $ x_0$,
with the Metropolis-Hastings probability.
Langevin sampling with rejection is known as the Metropolis-Adjusted Langevin Algorithm (MALA) \citep{besag1994mala, roberts1996exponential}.


\index{Stochastic gradient Langevin dynamics (SGLD)}
\paragraph{Stochastic gradient Langevin dynamics (SGLD).}
Recently, stochastic gradient samplers have emerged in simulating posterior samples in large-scale Bayesian inference, such as SGLD (stochastic gradient Langevin dynamics) \citep{sgld} and SGHMC (Stochastic Gradient Hamiltonian Monte Carlo) \citep{sghmc}.
To illustrate, consider the posterior $p(\theta|\mathcal{D})$ of model parameters $\theta$ given the observed dataset $\mathcal{D}$, with abuse of notation. 
We have $p(\theta|\mathcal{D}) \propto \exp \left[ \sum_{x \in \mathcal{D}} \log p_\theta(x) + \log p(\theta)  \right] $, which is taken as the target distribution.
Instead of using full-data gradients  $\frac{\partial}{\partial \theta} \log p(\theta|\mathcal{D})$, which needs a sweep over the entire dataset, these stochastic gradient samplers subsample the dataset and use stochastic gradients
$
\frac{\partial}{\partial \theta} \left[ \frac{|\tilde{\mathcal{D}|}}{|\mathcal{D}|} \sum_{x \in \tilde{\mathcal{D}}} \log p_\theta(x) + \log p(\theta)  \right]
$ in the dynamic simulation, where $\tilde{\mathcal{D}} \subset \mathcal{D}$ is a subsampled data subset.
In this manner, the computation cost is significantly reduced in each iteration and such Bayesian inference methods scale to large datasets.

In practice, sampling is based on a discretization of the continuous dynamics. Despite the discretization error and the noise introduced by the stochastic gradients, it can be shown that simulating the discretized dynamics with stochastic gradients also leads to the target distribution as the stationary distribution, when the step sizes are annealed to zero at a certain rate\footnote{A Metropolis-Hastings (MH) correction can be applied, when it is hard to check the annealing condition is satisfied or not.}.
The convergence of SGLD and SGHMC can be obtained from \citep{Sato2014ApproximationAO,sghmc,ma2015complete}. We summarize in Theorem \ref{theorem:SGLD} for SGLD.

\begin{theorem} \label{theorem:SGLD}
	Denote the target density as $p(z;\lambda)$ with given $\lambda$.
	Assume that one can compute a noisy, unbiased estimate $\Delta(z;\lambda)$ (a stochastic gradient) to the gradient $\frac{\partial}{\partial z} \log p(z;\lambda)$.
	For a sequence of asymptotically vanishing time-steps $\left\lbrace \delta_l, l \ge 1 \right\rbrace$ (satisfying $\sum_{l=1}^\infty \delta_l = \infty$ and $\sum_{l=1}^\infty \delta_l^2 < \infty$), the SGLD algorithm iterates as follows, starting from $z^{(0)}$:
	\begin{equation}\label{eq:SGLD}
	\begin{aligned}
	z^{(l)} =& z^{(l-1)}
	+ \delta_l \Delta (z^{(l-1)};\lambda)
	+ \sqrt{2\delta_l} \eta^{(l)},\\
	&\quad\eta^{(l)} \sim \mathcal{N}(0,I), l=1,\cdots\,
	\end{aligned}
	\end{equation}
	The iterations of Eq. (\ref{eq:SGLD}) lead to the target distribution $p(z;\lambda)$ as the stationary distribution.
\end{theorem}

\subsection{Importance sampling}
\label{sec:IS}
\index{Importance sampling (IS)}

One of the principal reasons for wishing to sample from complicated distributions is to be able to estimate expectations of the form \eqref{eq:ebm_mc_averaging}. The technique of \emph{importance sampling} (IS) provides a framework for approximating expectations directly.

Suppose, generally, one is interested in estimating
\begin{equation}
\label{eq:expect}
E_{p_{\theta}(x)} \left[ g(x) \right] = \int p_{\theta}(x) g(x) dx
\end{equation}
Importance sampling is based on the use of a \emph{proposal distribution}\index{Proposal distribution} $q(x)$ from which it is easy to draw samples, say, $x^{(1)},\cdots,x^{(M)} \sim q(x)$. We can then express the expectation by Monte Carlo averaging, i.e., in the form of a finite sum over samples $\{x^{(j)}\}$ drawn from $q(x)$:
\begin{align}
E_{p_{\theta}(x)} \left[ g(x) \right] &= \int q(x)
\frac{p_{\theta}(x)}{q(x)} g(x) dx \nonumber \\
&\approx \frac{1}{M} \sum_{j=1}^M \frac{p_{\theta}(x^{(j)})}{q(x^{(j)})} g(x^{(j)}) \label{eq:IS}
\end{align}
which is an \emph{unbiased} estimate of the expectation in \eqref{eq:expect}.
The quantities $w^{(j)} = \frac{p_{\theta}(x^{(j)})}{q(x^{(j)})}$ are known as \emph{importance weights}\index{Importance weight}, and they correct the bias that $\{x^{(j)}\}$ are drawn from the proposal distribution rather than from the target distribution. 

It will often be the case that the distribution $p_\theta(x)$ can only be evaluated up to a normalization constant (e.g., in EBMs), so that $p_\theta(\bx) = \tilde{p}_\theta(\bx)/Z_p$ where $\tilde{p}_\theta(x)$ can be evaluated easily, whereas $Z_p$ denotes the unknown normalizing constant. Generally, we may wish to use a proposal $q(x) = \tilde{q}(x)/Z_q$, which is also in the un-normalized form. We then have
\begin{align}
E_{p_{\theta}(x)} \left[ g(x) \right] &= \frac{Z_q}{Z_p} \int q(x)
\frac{\tilde{p}_{\theta}(x)}{\tilde{q}(x)} g(x) dx \nonumber \\
&\approx \frac{Z_q}{Z_p} \frac{1}{M} \sum_{j=1}^M \frac{\tilde{p}_{\theta}(x^{(j)})}{\tilde{q}(x^{(j)})} g(x^{(j)}) \label{eq:IS_unbiased}
\end{align}
We can use the same samples to evaluate the ratio $\frac{Z_q}{Z_p}$ with the result:
\begin{align}
   \frac{Z_p}{Z_q} &= \frac{1}{Z_q} \int \tilde{p}_{\theta}(x) dx 
   = \int q(x) \frac{\tilde{p}_{\theta}(x)}{\tilde{q}(x)}  dx \nonumber \\
   &\approx \frac{1}{M} \sum_{j=1}^M \frac{\tilde{p}_{\theta}(x^{(j)})}{\tilde{q}(x^{(j)})}  \label{eq:IS_Z_ratio}
\end{align}
which is an \emph{unbiased} estimate of $\frac{Z_p}{Z_q}$.
When $\tilde{q}(x)$ is self-normalized (i.e., $Z_q=1$), \eqref{eq:IS_Z_ratio} shows a way of using importance sampling to estimate the normalizing constant $Z_p$.

Combining \eqref{eq:IS_unbiased} and \eqref{eq:IS_Z_ratio} and letting $\tilde{w}^{(j)} = \frac{\tilde{p}_{\theta}(x^{(j)})}{\tilde{q}(x^{(j)})}$, we can approximate the expectation by
\begin{equation} \label{eq:IS_biased}
E_{p_{\theta}(x)} \left[ g(x) \right] \approx 
\frac{ \tilde{w}^{(1)} g(x^{(1)})+\cdots+\tilde{w}^{(M)} g(x^{(M)}) }
{ \tilde{w}^{(1)} +\cdots+\tilde{w}^{(M)} }
\end{equation}
which turns to be a \emph{biased} estimate of the expectation in \eqref{eq:expect}.
Further, by defining \emph{normalized importance weights} $\omega^{(j)} = \frac{\tilde{w}^{(j)}}{\sum_{j=1}^M \tilde{w}^{(j)}}$, \eqref{eq:IS_biased} can be re-written in a simpler form:
\begin{equation}\label{eq:IS_biased_2}
E_{p_{\theta}(x)} \left[ g(x) \right] \approx 
\sum_{j=1}^M \omega^{(j)} g(x^{(j)}) 
\end{equation}
which is often referred to as \emph{self-normalized importance sampling} (SNIS)\index{Self-normalized importance sampling (SNIS)}, for example, in \citep{parshakova2019global}.
A major advantage of using the biased estimate \eqref{eq:IS_biased_2} instead of the unbiased estimate \eqref{eq:IS_unbiased} is that in using the former (although biased), we need \emph{only} to know the ratio $\frac{p_\theta(x)}{q(x)}$ up to a multiplicative constant; whereas in the latter, the ratio needs to be known exactly.

Remarkably, the success of the importance sampling approach depends crucially on how well the proposal distribution $q(x)$ matches the target distribution $p_\theta(x)$.

\subsection{Stochastic approximation methods}
\label{sec:SA}
\index{Stochastic approximation (SA)}

In the above, we introduce the basics of some classic Monte Carlo methods and the general idea of applying them in maximum likelihood learning of EBMs. We show in \eqref{eq:ebm_grad} that the log-likelihood gradient for learning EBMs is equal to the difference between empirical expectation and model expectation, and in \eqref{eq:ebm_mc_averaging}, that the model expectation is approximated by Monte Carlo sampling from EBM distribution $p_{\theta}(x)$.
Combining \eqref{eq:ebm_grad}, \eqref{eq:ebm_mc_averaging} and \eqref{eq:ebm_minibatching}, we could obtain a naive algorithm of learning EBMs by Monte Carlo methods, as shown in \algref{alg:naive_learning_ebm}.

\begin{algorithm}[t]
	\algsetup{linenosize=\small}
	\caption{A naive algorithm of learning EBMs by Monte Carlo methods}
 	\label{alg:naive_learning_ebm}
	\begin{algorithmic}
	    \REQUIRE A target EBM distribution $p_\theta(x)$
		\FOR{each minibatch of size $B$} 
		      \STATE Obtain empirical expectations by \eqref{eq:ebm_minibatching};
                \FOR{$j=1$ to $M$} 
                    \STATE Draw $x^{(j)}$ with $p_\theta(x)$ as the target distribution;
                \ENDFOR
                \STATE Obtain model expectations by \eqref{eq:ebm_mc_averaging};
                \STATE Update parameter $\theta$ by gradient \eqref{eq:ebm_grad};
            \ENDFOR
	\end{algorithmic}
\end{algorithm}

Typically we use MCMC to generate the samples $x^{(1)},\cdots,x^{(M)}$, for each minibatch. For EBM distribution $p_\theta(x)$, which can only be evaluated up to a normalization constant, using the unbiased IS estimate \eqref{eq:IS_unbiased} is intractable. 
Using the biased IS estimate \eqref{eq:IS_biased_2} will produce biased gradient estimates, which were used in some prior studies \citep{parshakova2019global}.

In learning EBMs by Monte Carlo methods, at first thought (as shown in \algref{alg:naive_learning_ebm}), there are two loops.
The outer loop iterates over minibatches of training data.
The inner loop iterates to generate samples via MCMC (e.g., MH), but running MCMC sufficiently long (with large $M$) to approaching convergence at the inner loop would be extremely slow.
Fortunately, it was shown by \citep{younes1989parametric} that we can start the MCMC chain at its previous value from the outer loop, and just take a few Markov moves in the inner loop (i.e., using small $M$).
So in this way, the Markov chain evolves persistently across outer loops.
This algorithm, called \emph{stochastic maximum likelihood} (SML) \citep{younes1989parametric}, along with its variants appeared in the literature, turn out to be application of the more general stochastic approximation (SA) methodolgy to learning EBMs.
See further introduction in \secref{sec:SA_ebm}.

\paragraph{Note.}
The above cognition of the EBM learning by Monte Carlo is very important. Many people may think that learning with Monte Carlo methods is very slow. But since we do not need to run MCMC to convergence at the inner loop, but just a few steps. Learning with MCMC is not so expensive as people might think.

\begin{algorithm}[tb]
	\caption{The general stochastic approximation (SA) algorithm}\label{alg:SA}
	\begin{algorithmic}	
		\FOR {$t=1,2,\cdots$}
		\STATE \underline{Monte Carlo sampling:} Draw a sample $z^{(t)}$ with a Markov transition kernel $K_{\lambda^{(t-1)}}(z^{(t-1)},\cdot)$, which starts with $z^{(t-1)}$ and admits $p_{\lambda^{(t-1)}}(\cdot)$ as the invariant distribution.
		\STATE \underline{SA updating:} Set $\lambda^{(t)} = \lambda^{(t-1)} + \gamma_t F_{\lambda^{(t-1)}}(z^{(t)})$, where $\gamma_t$ is the learning rate.
		\ENDFOR
	\end{algorithmic}
\end{algorithm}

\subsubsection{Introduction to stochastic approximation (SA) methodology}
\label{sec:SA_intro}

Stochastic approximation methods are an important family of iterative stochastic optimization algorithms, introduced in \citep{SA51} and extensively studied \citep{benveniste2012adaptive,chen2002stochastic}.
Basically, stochastic approximation provides a mathematical framework for stochastically solving a root finding problem, which has the form of expectations being equal to zeros.
Suppose that the objective is to find the solution $\lambda^*$ of $f(\lambda) = 0$ with
\begin{equation}
\label{eq:SA}
f(\lambda) = E_{z \sim p_\lambda(\cdot) } [ F_\lambda(z) ],
\end{equation}
where $\lambda$ is a $d$-dimensional parameter vector in $\Lambda \subset R^d$, and $z$ is an observation from a probability distribution $p_\lambda(\cdot)$ depending on $\lambda$.
$F_\lambda(z) \in R^d $ is a function of $z$, providing $d$-dimensional stochastic measurements of the so-called mean-field function $f(\lambda)$.
Intuitively, we solve a system of simultaneous equations, $f(\lambda) = 0$, which consists of $d$ constraints, for determining $d$-dimensional $\lambda$.

Given some initialization $\lambda^{(0)}$ and $z^{(0)}$, a general SA algorithm iterates \emph{Monte Carlo sampling} and \emph{parameter updating}, as shown in Algorithm \ref{alg:SA}.
The convergence of SA has been established under conditions \citep{benveniste2012adaptive, andrieu2005stability, song2014weak}, including a few technical requirements for the mean-field function $f(\lambda)$, the Markov transition kernel $K_{\lambda^{(t-1)}}(z^{(t-1)},\cdot)$ and the learning rates.
Particularly, when $f(\lambda)$ corresponds to the gradient of some objective function, then $\lambda^{(t)}$ will converge to local optimum, driven by stochastic gradients  $F_\lambda(z)$.
For completeness, we provide a short summary on the convergence of $\left\lbrace \lambda_t, t \ge 1\right\rbrace $ in Algorithm \ref{alg:SA}, based on Theorem 1 in \citep{song2014weak}.

\begin{theorem}
	\label{th:SA}
	Let $\left\lbrace \gamma_t\right\rbrace $ be a monotone non-increasing sequence of positive numbers such that\footnote{In practice, we can set a large learning rate at the early stage of learning and decrease to $1/t$ for convergence.} $\sum_{t=1}^\infty \gamma_t = \infty$ and $\sum_{t=1}^\infty \gamma_t^2 < \infty$. Assume that $\Lambda$ is compact and the Lyapunov condition on $f(\lambda)$ and the drift condition on the transition kernel $K_{\lambda}(\cdot|\cdot)$ hold. Then we have: $d(\lambda_t, \mathcal{L}) \to 0$ almost surely as $t \to \infty$, where $\mathcal{L}=\left\lbrace \lambda: f(\lambda)=0 \right\rbrace $ and $d(\lambda, \mathcal{L}) = \inf_{\lambda' \in \mathcal{L}} || \lambda - \lambda' ||$.
\end{theorem}

Remarkably, Algorithm \ref{alg:SA} shows stochastic approximation with Markovian perturbations \citep{benveniste2012adaptive}. It is more general than the non-Markovian SA which requires exact sampling $z^{(t)} \sim p_{\lambda^{(t-1)}}(\cdot)$ at each iteration and in some tasks can hardly be realized.
In non-Markovian SA, we check that $F_\lambda(z)$ is unbiased estimates of $f(\lambda)$, while in SA with Markovian perturbations, we check the ergodicity property of the Markov transition kernel.

To speed up convergence, during each SA iteration, it is possible to generate a set of multiple observations $z$ by performing the Markov transition repeatedly 
and then use the average of the corresponding values of $F_\lambda(z)$ for updating $\lambda$, which is known as \emph{SA with multiple moves} \citep{Wang2017LearningTR}, as shown in \algref{alg:SA-multiple-move}.

\begin{algorithm}[tb]
	\caption{SA with multiple moves}\label{alg:SA-multiple-move}
	\begin{algorithmic}	
		\FOR {$t=1,2,\cdots$}
		\STATE
		\begin{enumerate}
			\item \underline{Monte Carlo sampling:} Set $z^{(t,0)}=z^{(t-1,K)}$.
			For $k$ from $1$ to $K$,
			generate $z^{(t,k)} \sim K_{\lambda^{(t-1)}}(z^{(t, k-1)},\cdot)$,
			where $K_{\lambda^{(t-1)}}(z^{(t, k-1)},\cdot)$ is a Markov transition kernel that admits $p_{\lambda^{(t-1)}}(\cdot)$ as the invariant distribution.
			\item \underline{SA updating:} Set $\lambda^{(t)} = \lambda^{(t-1)} + \gamma_t \{ \frac{1}{K} \sum_{z\in B^{(t)}} F_{\lambda^{(t-1)}}(z) \}$,  where $B^{(t)} = \{ z^{(t,k)} | k = 1,\cdots,K \}$.
		\end{enumerate}	
		\ENDFOR
	\end{algorithmic}
\end{algorithm}

\paragraph{Note I.}
Perhaps the most familiar application of SA in machine learning literature is the \emph{stochastic gradient descent} (SGD)\index{Stochastic gradient descent (SGD)} technique, particularly the \emph{minibatching}\index{Minibatching} technique.
When the objective (and therefore its gradient) is a sum of many terms that can be computed independently, SGD samples one term at a time and follows one noisy estimate of the gradient with a decreasing step size.
Furthermore, it can be easily seen that SGD training with minibatches is an application of SA with multiple moves.

\paragraph{Note II.}
Generally, SA represents an iterative methodology to find the root of an expectation. Each iteration consists of a sampling step and a parameter updating step.
Basically, we use MCMC to simulate the noisy measurements to approximate the expectation. 
A keypoint is that we do not need to wait the chain to converge, but use a decaying learning rate to guarantee the convergence.
Intuitively, as the learning rate becomes sufficiently small compared to the mixing rate of the Markov chain, the chain will stay close to the stationary distribution, even if it only runs for one Markov move per parameter update.

\subsubsection{Application of SA to learning EBMs}
\label{sec:SA_ebm}

\begin{algorithm}[t]
	\caption{Stochastic maximum likelihood for fitting an EBM}\label{alg:SA_EBM}
	\begin{algorithmic}	
		\FOR {$t=1,2,\cdots$}
		\STATE \underline{Sampling:} Draw $x_\kappa$ from training data, and simulate a sample $x^{(t)}$ with a Markov transition kernel $K_{\theta^{(t-1)}}(x^{(t-1)},\cdot)$, which starts with $x^{(t-1)}$ and admits $p_{\theta^{(t-1)}}(\cdot)$ as the invariant distribution.
		\STATE \underline{Updating:} Update $\theta$ by gradient ascent as:
  \begin{equation} \label{eq:SA_update}
\theta^{(t)} = \theta^{(t-1)} + \gamma_t \{ \nabla_\theta U_{\theta}(x_\kappa) - \nabla_\theta U_{\theta}(x^{(t)}) \}|_{\theta = \theta^{(t-1)}}      
  \end{equation}
		\ENDFOR
	\end{algorithmic}
\end{algorithm}

It can be easily seen that the EBM gradients $\nabla_\theta L(\theta)$ in \eqref{eq:ebm_grad} exactly follows the form of \eqref{eq:SA}, as summarized in \thref{th:ebm_grad_SA_form}.
So the problem of maximum likelihood estimate of EBM parameters can then be solved by setting the gradients to zeros and applying the SA algorithm to finding the root for the resulting system of simultaneous equations.

\begin{theorem}
\label{th:ebm_grad_SA_form}
Consider an EBM distribution $p_\theta(x)$ parameterized with $\theta$ as shown in \eqref{eq:unsup-RF}, and a training dataset consisting of IID data points $\left\lbrace x_1, \cdots, x_N \right\rbrace$.
Introduce an index variable $\kappa$ which is uniformly distributed over $\{1,\cdots,N\}$.
	The log-likelihood gradients w.r.t. $\theta$ as shown in \eqref{eq:ebm_grad} can be recast in the expectation form of \eqref{eq:SA} (i.e. as expectation of stochastic gradients), by letting $\lambda \triangleq \theta$, $z \triangleq (\kappa, x)^T$, $p_\lambda(z) \triangleq \frac{1}{N} p_\theta(x)$, $f(\lambda) \triangleq \nabla_\theta L(\theta)$, and
	\begin{displaymath}
	F_\lambda(z) \triangleq 
\nabla_\theta U_{\theta}(x_\kappa) - \nabla_\theta U_{\theta}(x)
	\end{displaymath}
\end{theorem}
\begin{proof}
	This can be readily seen by rewriting \eqref{eq:ebm_grad} as:
	\begin{displaymath}
\nabla_\theta L(\theta) = E_{\kappa \sim \text{Uni}[1,N], x \sim p_{\theta}(x)}\left[ \nabla_\theta U_{\theta}(x_\kappa) - \nabla_\theta U_{\theta}(x) \right]
	\end{displaymath}
and applying the independence between $\kappa$ and $x$.
\end{proof}

Combining \thref{th:ebm_grad_SA_form} and general SA (\algref{alg:SA}), the particular resulting pseudocode for learning EBMs is shown in \algref{alg:SA_EBM}, which is often known as \emph{stochastic maximum likelihood} (SML)\index{Stochastic maximum likelihood (SML)} \citep{younes1989parametric}.
\algref{alg:SA_EBM} corresponds to SA with single move. 
Further, by applying SA with multiple moves (\algref{alg:SA-multiple-move}), at each iteration, we can draw a minibatch from training data (minibatching)\index{Minibatching}, say drawing
$\kappa_1,\cdots,\kappa_B$ from $\{1,\cdots,N\}$.
At each iteration, we could directly draw $x^{(t,1)},\cdots,x^{(t,M)} \sim p_{\theta^{(t-1)}}(x)$ when it is tractable, or run multiple steps of a single chain, or multiple parallel chains, or a combination of both, to draw multiple samples, say obtaining $x^{(t,1)},\cdots,x^{(t,M)}$ that admit $p_{\theta^{(t-1)}}(x)$ as the invariant distribution. 
Then, at each iteration, parameter updating in \eqref{eq:SA_update} can be replaced by:
  \begin{displaymath}
\theta^{(t)} = \theta^{(t-1)} + \gamma_t 
\left.\left\{ 
\frac{1}{B} \sum_{j=1}^B \nabla_\theta U_{\theta}(x_{\kappa_j})
- \frac{1}{M} \sum_{j=1}^M \nabla_\theta U_{\theta}(x^{(t,j)})
\right\}\right|_{\theta = \theta^{(t-1)}}
  \end{displaymath}
  
\paragraph{Historical comments.}
The general stochastic approximation methodology was originally proposed in \citep{SA51}. 
The stochastic maximum likelihood method, proposed in \citep{younes1989parametric}, turns out to be application of the more general SA methodology to learning EBMs.
The same idea was applied to training RBMs in \citep{tieleman2008training}, which is called \emph{persistent contrastive divergence} (PCD)\index{Persistent contrastive divergence (PCD)} to emphasize that the Markov chain is not reset between parameter updates.
The regular \emph{contrastive divergence} (CD)\index{Contrastive divergence (CD)} method, proposed in \citep{hinton2002training}, restarts the Markov chain at the training data rather than at the previous state. This will not converge to MLE.
As commented in \citep{Salakhutdinov2009LearningDG}, ``Clearly, the widely used practice of CD1 learning is a rather poor “substitute” for maximum likelihood learning.''

\subsection{Variational methods with auxiliary models}

As introduced before, Monte Carlo sampler is a crucial component which affects maximum likelihood learning of EBMs.
A recent progress as studied in \citep{Kim2016DeepDG,wang2017language,Kuleshov2017NeuralVI,xie2018cooperative} is to pair the target EBM $p_\theta(x)$ with an auxiliary directed generative model (often called \emph{generator}) $q_\phi(x)$ parameterized by $\phi$, which approximates sampling from the target EBM.
Learning is performed by maximizing the log-likelihood of training data under $p_\theta$ or some bound of the log-likelihood, and simultaneously minimizing some divergence between the target EBM $p_\theta$ and the auxiliary generator $q_\phi$\footnote{Such optimization using two objectives has also been employed in training other types of models apart from learning EBMs, such as learning GAN with log$D$ trick \citep{goodfellow2014generative}, the wake-sleep algorithm \citep{Hinton1995the} for learning Helmholtz Machines.}:
\begin{displaymath}
\left\{
\begin{split}
& \text{Maximize over~} \theta \text{~the log-likelihood itself or some bound} \\
& \text{Minimize over~} \phi \text{~some divergence between~} p_\theta \text{~and~} q_\phi.\\
\end{split}
\right.
\end{displaymath}
Different learning methods mainly differ in the objective functions used in the joint training of $p_\theta$ and $q_\phi$, and thus have different computational and statistical properties.
There are also other factors that distinguish different studies in learning EBMs with auxiliary models, e.g. modeling discrete or continuous data, different model choices of the target EBM and the auxiliary generator.

Many methods in learning EBMs with auxiliary models are related to variational methods.
\textbf{Variational methods}\index{Variational methods} provide an optimization-based principle to inference and learning \citep{jordan1999introduction,frey2005comparison}.
A classic application of variational methods is in Bayesian inference, called \emph{variational inference} (VI)\index{Variational inference (VI)}. VI posits a family of approximating distributions $q$ and then finds the member of that family that is closest to the target posterior distribution $p$, mostly by minimizing the exclusive-divergence $KL(q||p)$. 
Variational methods have also been widely used in the context of maximum likelihood parameter estimation, which is often called \emph{variational learning}\index{Variational learning}.
In particular, \citep{neal1998view} shows a link between variational bound and maximum likelihood parameter estimation via the Expectation-Maximization (EM) algorithm\index{Expectation-Maximization (EM) algorithm}. Variational learning in early days is called variational EM \citep{frey2005comparison}.

Remarkably, classic variational methods mostly optimize the exclusive divergence, and hence could be classified as the \emph{exclusive-variational}\index{Exclusive-variational approach} approach.
Recently, there have emerged some variational methods that optimize the inclusive divergence $KL(p||q)$, which has good statistical properties that makes it more appropriate for certain inference and learning problems. 
These studies include joint stochastic approximation (JSA)\index{Joint stochastic approximation (JSA)} \citep{xu2016joint,ou2020joint}, Markovian score climbing (MSC)\index{Markovian score climbing (MSC)} \citep{naesseth2020markovian}, parallel Markov chain score ascent (pMCSA) \citep{kim2022markov}, and transport score climbing (TSC) \citep{zhang2022transport} for learning latent-variable models (belonging to directed graphical models), AugSA plus JSA \citep{wang2017language} and inclusive-NRF \citep{song2018learning} for learning EBMs (belonging to undirected graphical models).
We could refer these studies collectively as the \emph{inclusive-variational}\index{Inclusive-variational approach} approach. See \citep{ou2018review} for more introduction on variational methods and the two approaches.

In practice, the performance of learning EBMs with auxiliary models often performs better than without auxiliary models.
In the following, we first follow \citep{song2018learning} to give a short literature review of this class of studies, and then mainly detail the inclusive-variational approach for learning EBMs.

\subsubsection{Related work in MLE of EBMs with auxiliary models}

Let $p_\theta(x)$ denote the target EBM as defined in \eqref{eq:unsup-RF}, $q_\phi(x)$ the auxiliary generator which allows efficient sampling and approximates sampling from the target EBM, and $p_\text{emp}(x)$ the empirical distribution for training data.

\begin{itemize}
\item It is shown in \citep{song2018learning} that we have the following \emph{evidence upper bound} (EUBO)\index{Evidence upper bound (EUBO)} w.r.t. $\theta$ for EBMs:
\begin{align*}
\text{EUBO}(x;\theta,\phi) &= \log p_\theta(x) + KL(q_\phi(x) || p_\theta(x))\\
&= U_\theta(x) -\log Z(\theta) - (E_{q_\phi(x)}[p_\theta(x)]+H[q_\phi(x)])\\
& = U_\theta(x) -( E_{q_\phi(x)}[U_\theta(x)] + H[q_\phi(x)] )\\
& \ge \log p_\theta(x)
\end{align*}
It is further shown in \citep{song2018learning}
that learning in \citep{Kim2016DeepDG} amounts to maximizing the EUBO bound w.r.t. $\theta$, while simultaneously minimizing the gap, i.e., the exclusive-divergence $KL[q_\phi||p_\theta]$ w.r.t. $\phi$:
\begin{equation}
\label{eq:exclusive-variational}
\left\{
\begin{split}
& \max_{\theta} E_{x \sim p_\text{emp}(x)} \text{EUBO}(x;\theta,\phi) \\
& \min_{\phi} KL\left[  q_\phi(x) || p_\theta(x) \right] \\
\end{split}
\right.
\end{equation}
Remarkably, the EUBO bound involves the intractable entropy term $H \left[ q_\phi \right]$ and tends to enforce the generator to \emph{seek modes}, yielding missing modes. 
In \eqref{eq:exclusive-variational}, we optimize the exclusive-divergence w.r.t. an auxiliary distribution to approximate a target distribution $p_\theta(x)$. 
Hence we classify \eqref{eq:exclusive-variational} as the \emph{exclusive-variational} approach, which is called exclusive-NRF\index{Exclusive-NRF algorithm} in \citep{song2018learning}.
\item
Learning in \citep{wang2017language,song2018learning} minimizes the inclusive-divergence $KL[p_\theta||q_\phi]$ w.r.t. $\phi$, which could be classified as the \emph{inclusive-variational} approach for learning EBMs. 
The main idea is to perform maximum likelihood learning of $p_\theta$ and simultaneously minimize the inclusive-divergence between the target EBM $p_\theta$ and the auxiliary generator $q_\phi$ by
\begin{equation}
\label{eq:inclusive-variational}
\left\{
\begin{split}
& \min_{\theta} KL\left[  p_\text{emp}({x}) || p_\theta({x}) \right] \\
& \min_{\phi} KL\left[  p_\theta(x) || q_\phi(x) \right] \\
\end{split}
\right.
\end{equation}
The first line of \eqref{eq:inclusive-variational} is equivalent to maximum likelihood fitting of the target EBM $p_\theta$ under the empirical distribution $p_\text{emp}$, which requires sampling from $p_\theta$.
Simultaneously, the second line optimizes the generator $q_\phi$ to be close to $p_\theta$ so that $q_\phi$ becomes a good proposal for sampling from $p_\theta$.

Compared to the exclusive-variational approach, the inclusive-variational approach shown in Eq. (\ref{eq:inclusive-variational}) has several advantages.
First, minimizing inclusive-divergence avoids the annoying entropy term, which is suffered by minimizing the exclusive-divergence.
Second, inclusive-divergence minimization tends to drive the auxiliary generator, acting like an adaptive proposal in adaptive MCMC \citep{andrieu2008tutorial,roberts2009examples}, to \emph{cover modes} of the target density $p_\theta$. 
Mode-covering is a desirable property for proposal design in MCMC. In contrast, minimizing exclusive-divergence leads to variational approximations that seek modes and underestimate uncertainty.
The auxiliary model $q_\phi(x)$ and the sampler for $p_\theta(x)$ can be very flexibly designed, depending on the nature of data $x$, discrete or continuous.
\begin{itemize}
\item \citep{wang2017language} mainly studies neural random field language models, using LSTM generators (autoregressive with no latent variables) and employing Metropolis independence sampler (MIS) - applicable for \emph{discrete data}  (natural sentences). The learning algorithm proposed in \citep{wang2017language}, called \emph{AugSA plus JSA}\index{AugSA plus JSA algorithm}, is an instance of the inclusive-variational approach for learning EBMs over discrete data.
\item \citep{song2018learning} mainly designs neural random field models (NRFs)\index{Neural random fields (NRFs)} for \emph{continuous data} (e.g., images), choosing latent-variable generators and developing SGLD (stochastic gradient Langevin dynamics)/SGHMC (stochastic gradient Hamiltonian Monte Carlo) samplers to exploit noisy gradients in the continuous space. The learning algorithm proposed in \citep{song2018learning}, called \emph{inclusive-NRF}\index{Inclusive-NRF algorithm}, is an instance of the inclusive-variational approach for learning EBMs over continuous data.
\end{itemize}
\item
In \citep{xie2018cooperative} (CoopNet)\index{CoopNet}, motivated by interweaving maximum likelihood training of the EBM $p_\theta(x)$ and the latent-variable generator $q_\phi(h,x)$, a joint training method is introduced to train EBMs.
There are clear differences that distinguish the inclusive-variational approach.
First, CoopNet uses LD (Langevin dynamics) sampling to generate samples, but two LD sampling steps are intuitively interleaved according to $ \frac{\partial}{\partial x} \log  p_\theta(x)$ (with $L_x$ steps) and $\frac{\partial}{\partial h} \log q_\phi(h,x)$ (with $L_h$ steps) separately, not aiming to draw samples from $p_\theta(x) q_\phi(h|x)$.
This is different from the stochastic gradient sampler in the augmented space in inclusive-NRF, which moves $(x,h)$ jointly.
Second, according to theoretical understanding in \citep{xie2018cooperative}, Coopnet considers the following joint optimization problem:
\begin{displaymath}
\label{eq:coopnet_obj}
\left\{
\begin{split}
& \min_{\theta} \left\lbrace KL\left[  {p}_\text{emp}({x}) || p_\theta({x}) \right] - KL\left[ r(h,x)  || p_\theta(x) \right]\right\rbrace \\
& \min_{\phi} KL\left[  r(h,x) || q_\phi(h,x) \right] \\
\end{split}
\right.
\end{displaymath}
where $r(h,x)$ denotes the distribution of $(x^{(L_x)}, h^{(L_h)})$, resulting from the CoopNet sampler.
This objective is also clearly different from inclusive-NRF, which aims to minimize the inclusive-divergence $KL[p_\theta||q_\phi]$ w.r.t. $\phi$.
It is found in \citep{song2018learning} that inclusive-NRF with SGLD outperforms CoopNet in image generation.
\item
Learning in \citep{Kuleshov2017NeuralVI} minimizes the $\chi^2$-divergence $\chi^2[q_\phi||p_\theta] \triangleq \int \frac{(p_\theta-q_\phi)^2}{q_\phi} $ w.r.t. $\phi$, which also tends to drive the generator to cover modes. But this approach is severely limited by the high variance of the gradient estimator w.r.t. $\phi$, and is only tested on the simpler MNIST and Omniglot.
\item
Learning in \citep{han2019divergence} further extends CoopNet and introduces an inference model, apart from the target EBM and the latent-variable generator, and jointly optimizes the three models under a divergence triangle.
\end{itemize}

\subsubsection{The inclusive-variational approach for learning EBMs}

\begin{figure*}[t]
	\center
	\includegraphics[width=0.9\textwidth]{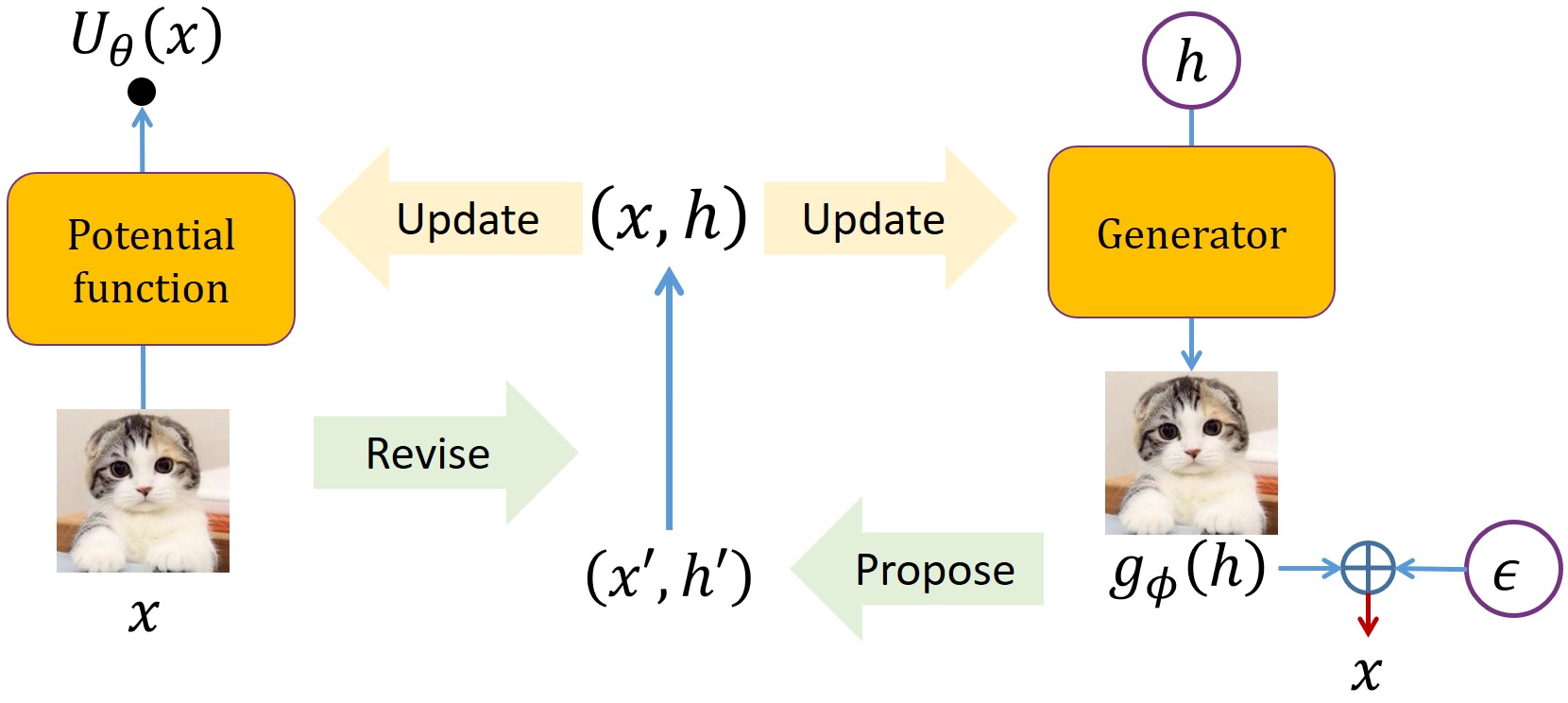}
	\caption{Overview of the inclusive-variational approach for learning EBMs for continuous data. Two neural networks are used to define the EBM's potential function $U_{\theta}(x)$ and the auxiliary generator $g_\phi(h)$ respectively. The parameters of both networks, $\theta$ and $\phi$, are updated by using the revised samples $(x,h)$ in the augmented space, which are obtained by revising the samples $(x',h')$ proposed by the auxiliary generator, according to the stochastic gradients defined by both the target EBM and the auxiliary generator. \citep{song2018learning}} 
	\label{fig:inclusive-NRF}
\end{figure*}

The basic idea of using the inclusive-variational approach for learning EBMs is described in \eqref{eq:inclusive-variational}.
The auxiliary model $q_\phi(x)$ and the sampler for $p_\theta(x)$ can be very flexibly designed, depending on the nature of data $x$, discrete or continuous.
In the following, we mainly introduce the inclusive-variational approach for learning EBMs for \emph{continuous data}, which is called inclusive-NRF in \citep{song2018learning} and illustrated in Figure \ref{fig:inclusive-NRF}. 
The important features of inclusive-NRF is that the auxiliary model $q_\phi(x)$ is a latent-variable model, and stochastic gradient guided samplers (SGLG/SGHMC) are developed, which is particularly useful for sampling from a continuous distribution $p_\theta(x)$. For the inclusive-variational approach for learning EBMs for \emph{discrete data}, readers could see \citep{wang2017language}.

\paragraph{The NRF model.}
EBMs parameterized by neural network, as defined in \eqref{eq:unsup-RF}, are called Neural random fields (NRFs) in \citep{song2018learning}, which are usually denoted by $p_\theta(x)$. The potential $U_{\theta}(x) : \mathbb{R}^{d_x} \rightarrow \mathbb{R}$ is realized by a neural network, which takes the multi-dimensional $x \in \mathbb{R}^{d_x}$ as input and outputting the scalar $u_{\theta}(x) \in \mathbb{R}$.

\paragraph{An inclusive-divergence minimized auxiliary generator} $q_\phi(x)$ is introduced to approximate sampling from the target EBM, particularly for fixed-dimensional continuous observations $x \in \mathbb{R}^{d_x}$ (e.g. images).
We use a directed generative model, $q_\phi(x,h) \triangleq q(h)q_\phi(x|h)$, for the auxiliary generator, which is defined as follows\footnote{Note that during training, $\sigma^2$ is absorbed into the learning rates and does not need to be estimated.}:
\begin{equation}\label{eq:generator}
\begin{aligned}
h &\sim \mathcal{N}(0,I_h),\\
x &= g_\phi(h)+\epsilon, \epsilon \sim \mathcal{N}(0,\sigma^2 I_{\epsilon}).
\end{aligned}
\end{equation}
Here $g_\phi(h):\mathbb{R}^{d_h} \rightarrow \mathbb{R}^{d_x}$ is implemented as a neural network with parameter $\phi$, which maps the latent code $h$ to the observation space.
$I_h$ and $I_{\epsilon}$ denote the identity matrices, with dimensionality implied by $h$ and $\epsilon$ respectively.
Drawing samples from the generator $q_\phi(x,h)$ is simple as it is just ancestral sampling\index{Ancestral sampling} \citep{murphy2012machine} from a 2-variable directed graphical model.

By using Fisher equality (\appref{sec:fisher_eq}), we have the following gradients for $\theta$ and $\phi$ respectively, where in the first equation, we use $\tilde{x}$ and $x$ to differentiate samples from the empirical distribution ${p}_\text{emp}({x})$ and those from the model distribution $p_\theta(x)$.
\begin{prop} \label{prop:nrf_gradient_proof}
	The gradients for optimizing the two objectives in Eq. (\ref{eq:inclusive-variational}) can be derived as follows:
	\begin{equation}
	\label{eq:jrf_unsup_gradient}
	\left\{
	\begin{split}
	&-\frac{\partial}{\partial \theta} KL\left[  {p}_\text{emp}({x}) || p_\theta({x}) \right]
	=E_{{p}_\text{emp}(\tilde{x})}\left[\nabla_\theta u_\theta(\tilde{x})\right]-E_{p_\theta(x)}\left[\nabla_\theta u_\theta(x)\right]\\
	&-\frac{\partial}{\partial \phi} KL\left[  p_\theta(x) || q_\phi(x) \right]
	=E_{p_\theta(x) q_\phi(h|x)}\left[ \nabla_\phi logq_\phi(x,h)\right]
	\end{split}
	\right.
	\end{equation}
\end{prop}

\begin{algorithm*}[t]
	\caption{The inclusive-NRF algorithm for learning EBMs for continuous data with latent-variable auxiliary models}
	\label{alg:learning-NRF-IAG}
	\begin{algorithmic}
		\REPEAT
		\STATE \underline{Sampling:}
		Draw a minibatch $\mathcal{M}=\left\lbrace (\tilde{x}^i,x^i,h^i), i=1,\cdots\,|\mathcal{M}|\right\rbrace $ from ${p}_\text{emp}(\tilde{x}) p_\theta(x) q_\phi(h|x)$ (see Algorithm \ref{alg:model-sampling});
		
		\STATE \underline{Updating:}
		
		Update $\theta$ by ascending:		
		$\frac{1}{|\mathcal{M}|} \sum_{(\tilde{x},x,h) \sim \mathcal{M}}
		\left[\nabla_\theta u_\theta(\tilde{x}) - \nabla_\theta u_\theta(x) \right] $;
		
		Update $\phi$ by ascending:
		$
		\frac{1}{|\mathcal{M}|} \sum_{(\tilde{x},x,h) \sim \mathcal{M}}
		\nabla_\phi \log q_\phi(x,h) $;
		
		\UNTIL{convergence}
	\end{algorithmic}
\end{algorithm*}

By Proposition \ref{prop:nrf_gradient_proof}, we can obtain the gradients w.r.t. $\theta$ and $\phi$ (to be ascended).
In practice, we apply minibatch based stochastic gradient descent (SGD) to solve the optimization problem \eqref{eq:inclusive-variational}, as shown in Algorithm \ref{alg:learning-NRF-IAG}.

Ideally, the learning of $\theta$ could be conducted without $\phi$, by using an MCMC sampler (e.g. LD) to draw samples from $p_\theta(x)$.
But the chain often mixes between modes so inefficiently that severely slow down the learning of $\theta$ especially when the target density $p_\theta(x)$ is multimodal. 
This is the main difficulty that hinders the effective training of NRFs.
Introducing auxiliary generator $q_\phi$ to approximate the target NRF $p_\theta$ is inspired by and related to two advanced MCMC ideas - auxiliary variable MCMC \citep{neal2011mcmc} and adaptive MCMC \citep{andrieu2008tutorial,roberts2009examples}.
\begin{itemize}
\item The classic example of adaptive MCMC is adaptive scaling of the variance of the step-size in random-walk Metropolis \citep{roberts2009examples}.
In inclusive-NRF, the auxiliary generator acts like an adaptive proposal, updated by using samples from the target density\footnote{Minimizing the inclusive-divergence tends to drive the generator (the proposal) to have higher entropy than the target density, which is a desirable property for proposal design in MCMC.}.
\item 
Further to be detailed next, the target density is extended to be $p_\theta(x) q_\phi(h|x)$, which leaves the original target as the marginal, but sampling in the augmented space $(x,h)$ can be easier (more efficiently), with the help of the adaptive proposal $q_\phi(x,h)$.
This follows the basic idea of auxiliary variable MCMC \citep{neal2011mcmc} - sampling in an augmented space could be more efficient.
\end{itemize}

\paragraph{Developing stochastic gradient samplers for EBMs for continuous data.}
In Algorithm \ref{alg:learning-NRF-IAG}, we need to draw samples $(x,h) \in \mathbb{R}^{d_x + d_h}$ in the augmented space defined by the target joint distribution $p_\theta(x) q_\phi(h|x)$ given current $\theta$ and $\phi$.
Gradient guided samplers (\secref{sec:LD_HMC}), such as Langevin dynamics (LD) and Hamiltonian Monte Carlo (HMC) \citep{neal2011mcmc}, are known to be efficient in exploring the continuous state space.
The gradients of the target distribution can be derived as follows:
\begin{equation} \label{eq:grad-x-h}
\left\{
\begin{split}
&\frac{\partial}{\partial x} \log \left[ p_\theta(x) q_\phi(h|x) \right] 
= \frac{\partial}{\partial x} \left[ \log  p_\theta(x) +  \log q_\phi(h,x) -  \log q_\phi(x) \right]\\
&\frac{\partial}{\partial h} \log \left[ p_\theta(x) q_\phi(h|x) \right] = \frac{\partial}{\partial h} \log q_\phi(h,x)
\end{split}
\right.
\end{equation}
It can be seen that it is straightforward to obtain the gradient w.r.t. $h$ and the first two terms in the gradient w.r.t. $x$.
However, calculating the third term $\frac{\partial}{\partial x} \log q_\phi(x)$ in the gradient w.r.t. $x$ is intractable. Therefore we are interested in developing stochastic gradient variants of those samplers, which rely on using noisy estimate of $\frac{\partial}{\partial x} \log q_\phi(x)$.

By considering $z \triangleq (x,h)$, $p(z; \lambda) \triangleq p_\theta(x) q_\phi(h|x)$, $\lambda \triangleq (\theta, \phi)^T$, and Eq. (\ref{eq:grad-x-h}),
we can use Theorem \ref{theorem:SGLD} to develop the sampling step for Algorithm \ref{alg:learning-NRF-IAG}, as presented in Algorithm \ref{alg:model-sampling}.
For the gradient w.r.t. $x$, the intractable term $\frac{\partial}{\partial x} \log q_\phi(x)$ is estimated by a stochastic gradient.

\begin{prop} \label{prop:unbiased-est}
	Given $q_\phi(h,x)$, we have
	\begin{equation} \label{eq:unbiased-est}
	\frac{\partial}{\partial x} \log q_\phi(x)
	= E_{h^* \sim q_\phi(h^*|x)} \left[ \frac{\partial}{\partial x} \log q_\phi(h^*,x) \right].
	\end{equation}
\end{prop}
\begin{proof}
By using Fisher equality (\appref{sec:fisher_eq}).
\end{proof}

Motivated by Proposition \ref{prop:unbiased-est},
ideally we draw $h^* \sim q_\phi(h^*|x)$ and then use $\frac{\partial}{\partial x} \log q_\phi(h^*,x)$ as an unbiased estimator of $\frac{\partial}{\partial x} \log q_\phi(x)$.
In practice, at step $l$, given $x^{(l-1)}$ and starting from $h^{(l-1)}$, we run one step of LD sampling over $h$ targeting $q_\phi(h|x^{(l-1)})$, to obtain $h^{(l-1)*}$ and calculate $\frac{\partial}{\partial x^{(l-1)}} \log q_\phi(h^{(l-1)*},x^{(l-1)})$. This gives a biased but tractable estimator to  $\frac{\partial}{\partial x} \log q_\phi(x)$. It is empirically found in experiments in \citep{song2018learning} that more steps of this inner LD sampling do not significantly improve the performance for NRF learning.

So instead of using the exact gradient $\frac{\partial}{\partial z} \log p(z;\lambda)$ as shown in Eq. (\ref{eq:grad-x-h}), \cite{song2018learning} developed a tractable biased stochastic gradient $\Delta(z;\lambda)$ as follows:
\begin{equation}  \label{eq:stochastic-grad}
\Delta(z;\lambda) \triangleq
\left( \begin{array}{c}
\frac{\partial}{\partial x} \left[ \log  p_\theta(x) +  \log q_\phi(h,x) -  \log q_\phi(h^*,x) \right] \\
\frac{\partial}{\partial h} \log q_\phi(h,x)
\end{array} \right),
\end{equation}
where $h^*$ is an approximate sample from $q_\phi(h^*|x)$ obtained by running one step of LD from $(h,x)$.
Remarkably, as shown in Algorithm \ref{alg:model-sampling}, the starting point $(h^{(0)}, x^{(0)})$ for the SGLD/SGHMC recursions is obtained from an ancestral sampling\index{Ancestral sampling} from $ q_\phi(h,x)$. Thus at step $l=1$, $h^{(0)}$ is already a sample from $q_\phi(h|x^{(0)})$ given $x^{(0)}$, and we can directly use $h^{(0)}$ as $h^{(0)*}$ without running the inner LD sampling.
Afterwards, for $l > 1$, the conditional distribution of $h^{(l-1)}$ given $x^{(l-1)}$ is close to $q_\phi(h|x^{(l-1)})$, though strictly not.
One or more steps of LD could be run to obtain $h^{(l-1)*}$ to reduce the bias in the stochastic gradient estimator.

With the above stochastic gradients in Eq. (\ref{eq:stochastic-grad}), the sampling step in Algorithm \ref{alg:learning-NRF-IAG} can be performed by running $|\mathcal{M}|$ parallel chains, each chain being executed by running finite steps of SGLD/SGHMC with tractable gradients w.r.t. both $x$ and $h$, as shown in Algorithm \ref{alg:model-sampling}. 
Intuitively, the auxiliary generator first gives a proposal $(x',h')$, and then the system follows the gradients of $p_\theta(x)$ and $q_\phi(h,x)$ (w.r.t. $x$ and $h$ respectively) to revise $(x',h')$ to $(x,h)$.
The gradient terms pull samples moving to low energy region of the random field and adjust the latent code of the generator, while the noise term brings randomness.
In this manner, we obtain Markov chain samples in the augmented space defined by $p_\theta(x) q_\phi(h|x)$.


\begin{algorithm}[!ht]
	\caption{Sampling in the augmented space defined by $p_\theta(x) q_\phi(h|x)$}
	\label{alg:model-sampling}
	\begin{algorithmic}
		\STATE 1. Conduct ancestral sampling from the auxiliary generator $q_\phi(x,h)$, i.e. first draw $h' \sim q(h')$, and then draw $x' \sim q_\phi(x'|h')$;
		\STATE 2. Starting from $(x',h') = z^{(0)}$, run finite steps of SGLD $(l=1,\cdots,L)$ to obtain $(x,h)=z^{(L)}$, which we call \emph{sample revision}, according to Eq. (\ref{eq:SGLD}). 
		\STATE In particular, the SGLD recursions are conducted as follows:
		\begin{equation} \label{eq:SGLD-specific}
		\left\{
		\begin{split}
		x^{(l)} = x^{(l-1)}
		+ \delta_l \frac{\partial}{\partial x^{(l-1)}} &\left[ \log  p_\theta(x^{(l-1)}) +  \log q_\phi(h^{(l-1)},x^{(l-1)}) \right. \\
		&\left. -\log q_\phi(h^{(l-1)*},x^{(l-1)}) \right]
		+ \sqrt{2\delta_l} \eta_x^{(l)},\\
		h^{(l)} = h^{(l-1)}
		+ \delta_l \frac{\partial}{\partial h^{(l-1)}} &\log q_\phi(h^{(l-1)},x^{(l-1)})
		+ \sqrt{2\delta_l} \eta_h^{(l)},\\
		&\eta^{(l)}\triangleq (\eta_x^{(l)}, \eta_h^{(l)})^T \sim \mathcal{N}(0,I)
		\end{split}
		\right.
		\end{equation}
		where, for $l > 1$, $h^{(l-1)*}$, which is an approximate sample from $q_\phi(h|x^{(l-1)})$ given $x^{(l-1)}$, is obtained from running one step of LD as follows, starting from $h^{(l-1)}$:
		\begin{equation} \label{eq:h-LD}
		\begin{split}
		h^{(l-1)*}=h^{(l-1)}+ \delta_{l}^{*} \frac{\partial}{\partial h^{(l-1)}} &\log q_\phi(h^{(l-1)},x^{(l-1)})
		+ \sqrt{2\delta_{l}^{*}} \eta_{h}^{(l)*},\\ &\eta_h^{(l)*} \sim \mathcal{N}(0,I);
		\end{split}
		\end{equation}
		for $l=1$, we directly use $h^{(0)}$ as $h^{(0)*}$, since, by initialization, $h^{(0)}$ is an exact sample from $q_\phi(h|x^{(0)})$ given $x^{(0)}$.
		\STATE \textbf{Return} $(x,h)$, i.e. $z^{(L)}$.
	\end{algorithmic}
\end{algorithm}

\subsection{Non-MLE methods for learning EBMs}

Maximum likelihood estimation (MLE) has been the most widely used objective for learning probabilistic models.
When training an EBM with MLE, we need to sample from the EBM per training iteration. 
The training efficiency of the MLE method highly depends on the mixing efficiency of the Markov chain. In high-dimensional problems, it is very challenging to design Markov chains with fast mixing rates. 
Thus, MLE training of EBMs may converge slowly.
As alternatives to MLE training, non-MLE methods for learning EBMs (as a kind of unnormalized models) have been explored, such as \emph{noise-contrastive estimation} \citep{nce,gutmann2012NCE} and \emph{score matching} \citep{hyvarinen2005SM}.

Another motivation to pursue non-MLE methods is that the optimization criterion used has profound effect on the behavior of the optimized model \citep{theis2016a}. Maximizing likelihood is equivalent to minimizing the KL divergence between $p_\text{ora}(x)$ and $p_\theta(x)$, because
\begin{align}
KL[ p_{\text{ora}}(x) || p_\theta(x)] &=
-E_{x \sim p_{\text{ora}}(x)} [ \log p_\theta(x) ]+E_{x \sim p_{\text{ora}}(x)} [ \log p_{\text{ora}}(x) ] \nonumber \\
&= -E_{x \sim p_{\text{ora}}(x)} [ \log p_\theta(x) ] + \text{constant} \nonumber \\
&\approx -E_{x \sim p_{\text{emp}}(x)} [ \log p_\theta(x) ] + \text{constant} \label{eq:KL_ML}\\
&= - L(\theta) + \text{constant} \nonumber
\end{align}
where $p_\text{ora}(x)$ denote the the underlying (oracle) data distribution, and $L(\theta)$ is the log-likelihood in \eqref{eq:ebm_loglik}. \eqref{eq:KL_ML} is an unbiased Monte Carlo integration since the expectation under $p_\text{ora}(x)$ is approximated by empirical samples $\left\lbrace x_i\right\rbrace_{i=1}^N$.

It is known that the KL divergence is asymmetric, and  optimizing which direction of the KL divergence leads to different trade-offs. The KL approximation covers the data distribution while reverse-KL has more of a mode-seeking behavior \citep{Minka2005DivergenceMA}. 
The basic \emph{score matching} (SM)\index{Score matching (SM)} objective minimizes a discrepancy between two distributions called the Fisher divergence:
\begin{displaymath}
D_F( p_{\text{ora}}(x) || p_\theta(x) ) =  E_{x \sim p_{\text{ora}}(x)} \left[ \frac{1}{2}  || \nabla_x \log p_{\text{ora}}(x) - \nabla_x \log p_\theta(x) ||^2 \right]   
\end{displaymath}

Learning under most criteria is provably consistent given infinite model capacity and data.
Most of the methods are firstly examined over continuous data, with few on discrete data.
The score matching method is based on minimizing the expected squared distance of the \emph{score function}\footnote{The gradient of log-density with respect to the data vector $x$ is called the score function.}\index{Score function} of the data distribution and the score function given by the model, and thus is not applicable to training EBMs over discrete data such as natural languages. In contrast, the noise-contrastive estimation method has no such limitation, which will be detailed below.

\section{Learning EBMs by noise-contrastive estimation (NCE)}
\label{sec:NCE}
\index{Noise-contrastive estimation (NCE)}

Noise-contrastive estimation (NCE) is proposed in \citep{nce,gutmann2012NCE}, as a typical non-MLE method, for learning unnormalized statistical models.
Its basic idea is ``learning by comparison'', i.e. to perform nonlinear logistic regression to discriminate between \emph{data samples} drawn from the data distribution $p_{\text{ora}}(\bx)$ and \emph{noise samples} drawn from a known noise distribution $q(\bx)$.
An advantage of NCE is that the normalizing constant can be treated as a normal parameter and updated together with the model parameters.

Denote the target unnormalized model by:
\begin{equation} \label{eq:NCE_p}
 p_\theta(\bx) = \frac{1}{Z_\theta} \tilde{p}_\theta(\bx)   
\end{equation}
where we highlight that the model, parameterized by $\theta$, is unnormalized, and the normalizing constant $Z_\theta = \int \tilde{p}_\theta(\bx) d \bx$.
To apply NCE, we treat $\log Z_\theta$ as an additional parameter $\zeta$ and rewrite \eqref{eq:NCE_p} in the following form:
\begin{equation}\label{eq:trf-nce}
p_{\theta,\zeta}(\bx) = \exp\left[  -\zeta + \log \tilde{p}_\theta(\bx)  \right] 
\end{equation}
As shown below, model parameters $(\theta, \zeta)$ will be jointly estimated in NCE.

Introduce a fixed, known noise distribution denoted by $q(x)$, and consider a binary classification.
There are two classes of samples, with prior probabilities $P(C=+)$ and $P(C=-)$. For class $+$, the sample is drawn from the data distribution
$p_\text{ora}$; for class $-$, the sample is drawn from the noise distribution $q$. This defines a generation process of samples in NCE.

Given a sample $\bx$ from such a generation process, the class posterior probabilities are defined as follows:
\begin{align*}
 p(C=+|\bx) &= \frac{p(C=+)p(\bx|C=+)}{p(C=+)p(\bx|C=+)+p(C=-)p(\bx|C=-)}\\
 p(C=-|\bx) &= 1-p(C=+|\bx)
\end{align*}
By our design of the binary classification experiment, $p(\bx|C=-)$ is the noise distribution $q(\bx)$. The (unknown) class-conditional density for class $+$ is assumed to be modeled by the target model $p_{\theta,\zeta}(\bx)$.
Let the ratio between the prior probabilities $\frac{p(C=-)}{p(C=+)}$ be $\nu$ (i.e., the ratio of noise sample size to real sample size).
Then the posterior probabilities can be parameterized as follows:
\begin{equation}\label{eq:NCE_post}
\begin{split}
p(C=+|\bx; \theta,\zeta) &= \frac{p_{\theta,\zeta}(\bx)}{p_{\theta,\zeta}(\bx) + \nu q(x)} \\
p(C=-|x; \theta,\zeta) &= 1 - p(C=+|x; \theta,\zeta)    
\end{split}    
\end{equation}

NCE estimates the model parameters $(\theta, \zeta)$ by optimizing the two-class classification through maximizing the following conditional log-likelihood:
\begin{equation} \label{eq:NCE-obj}
J_\text{NCE}(\theta,\zeta) = E_{\bx \sim p_\text{ora}(\bx)} \left[ \log p(C=+|\bx;\theta,\zeta)\right]  + 
\nu E_{\bx \sim q(\bx)} \left[  \log p(C=-|\bx;\theta,\zeta) \right] 
\end{equation}
The objective function $J(\theta,\zeta)$ is a sum of two expectations.
The first is the expectation w.r.t. the data distribution $p_\text{ora}(\bx)$, which can be approximated by randomly drawing samples from training data, namely approximating $p_\text{ora}(\bx)$ by $p_\text{emp}(\bx)$.
The second is the expectation w.r.t. the noise distribution $q(\bx)$, which can be approximated by drawing samples from the noise distribution.

\begin{algorithm}[t!]
	\caption{NCE for fitting an unnormalized model}
	\label{alg:NCE}
	\begin{algorithmic}
		\REPEAT
		\STATE \underline{Sampling:}
		Draw an empirical minibatch $\mathcal{D} \sim p_\text{ora}(x)$ and a noise minibatch $\mathcal{B} \sim q(x)$, satisfying $\nu = |\mathcal{B}|/|\mathcal{D}|$;
		\STATE \underline{Updating:}
		Update $(\theta,\zeta)$ by ascending:
  \begin{displaymath}
\frac{1}{|\mathcal{D}|} \sum_{\bx \sim \mathcal{D}}
		\nabla_{\theta, \zeta}\log p(C=+|\bx;\theta, \zeta)
+ \frac{\nu}{|\mathcal{B}|} \sum_{\bx \sim \mathcal{B}}
		\nabla_{\theta, \zeta}\log p(C=-|\bx;\theta, \zeta)      
  \end{displaymath}
		\UNTIL{convergence}
	\end{algorithmic}
\end{algorithm}

Setting to zeros the gradients of $J(\theta,\zeta)$ w.r.t. $(\theta,\zeta)$, we can apply the SA algorithm to find its root and thus solve the optimization problem in \eqref{eq:NCE-obj}. The pseudocode of NCE is shown in \algref{alg:NCE}.
The relevant gradients can be further simplified as follows:
\begin{equation} \label{eq:NCE-grad}
\begin{split}
\nabla_{\theta, \zeta}\log p(C=+|\bx;\theta, \zeta) = 
p(C=-|\bx;\theta, \zeta) \nabla_{\theta, \zeta} \log p_{\theta,\zeta}(\bx) \\
\nabla_{\theta, \zeta}\log p(C=-|\bx;\theta, \zeta) = 
- p(C=+|\bx;\theta, \zeta) \nabla_{\theta, \zeta} \log p_{\theta,\zeta}(\bx)
\end{split}
\end{equation}

It is shown in \citep{gutmann2012NCE} that under the ideal situation of infinite amount of data and infinite model capacity, we have the following theorem (nonparametric estimation).
It is further shown \citep{gutmann2012NCE} that the NCE estimator is consistent.

\begin{theorem}[Nonparametric estimation]
	\label{theorem:NCE}
$J_\text{NCE}(\theta,\zeta)$ attains its maximum at $p_{\theta,\zeta}(\bx) = p_\text{ora}(\bx)$.
There are no other extrema if the noise density $q(x)$ is chosen such that it is nonzero whenever $p_\text{ora}(x)$ is nonzero.
\end{theorem}	

The noise distribution $q$ and the ratio $\nu$ have an influence on the
accuracy of the NCE estimate of model parameters $(\theta, \zeta)$. A natural question to ask in applying NCE is what, from a statistical standpoint, the best choice of $q$ and $\nu$ is, to get estimates with a small estimation error.
This question is discussed in the original paper of NCE \citep{gutmann2012NCE}, which give the following suggestions:
\begin{enumerate}
    \item Choose noise for which an analytical expression for $\log q$ is available.
    \item  Choose noise that can be sampled easily.
    \item Choose noise that is in some aspect, for example with respect to its covariance structure, similar to the data.
    \item Make the noise sample size as large as computationally possible.    
\end{enumerate}

\subsection{Dynamic noise-contrastive estimation (DNCE)}
\label{sce:DNCE}
\index{Dynamic noise-contrastive estimation (DNCE)}

Remarkably, there exist two problems in applying NCE learning.
First, reliable NCE needs a large $\nu$, especially when the noise distribution is not close to the data distribution.
And the time and memory cost for gradient calculation are almost linearly increased with $\nu$.
Second, the expectation w.r.t. the data distribution $p_\text{ora}$ in \eqref{eq:NCE-obj} is approximated by the expectation w.r.t. the empirical distribution $p_\text{emp}$ (namely the training data), which is rather sparse for high-dimensionality data modeling.
The model estimated by NCE is thus easily overfitted to the empirical distribution.
\emph{Dynamic noise-contrastive estimation} (DNCE) was proposed in \citep{wang2018improved} to address the above problems, with two modifications.

First, instead of using a fixed noise distribution, a dynamic noise distribution $q_\phi(\bx)$ with parameter $\phi$ is introduced in DNCE.
In addition to maximizing w.r.t. $(\theta, \zeta)$ the NCE objective function $J_\text{NCE}(\theta,\zeta)$, DNCE simultaneously performs maximum likelihood optimization of $\phi$ over training data:
\begin{displaymath}
\max_{\phi} E_{\bx \sim p_\text{ora}(\bx)} \left[ \log q_\phi(\bx) \right]
\end{displaymath}
The motivation is to push the noise distribution to be close to the data distribution, so that we can achieve reliable model estimation even using a small $\nu$.
Theoretically, one can optimize the noise distribution beforehand and then use a fixed noise density in NCE. 
It is found that dynamic noise distribution helps optimization, by gradually increasing the difficulty of the two-class discrimination task \citep{wang2018improved}.
If the noise distribution $q_\phi$ is too different from the data distribution $p_\text{ora}$, the two-class discrimination problem might be too easy and would not require the system to learn much about the structure of the data. But if $q_\phi$ is too close to $p_\text{ori}$ from the beginning, the discrimination problem might be too difficult to proceed.
 
Second, instead of using the standard NCE objective function $J_\text{NCE}(\theta,\zeta)$ in \eqref{eq:NCE-obj}, a modified objective function is proposed as follows\footnote{In $J_\text{DNCE}(\theta,\zeta)$ and $p_\text{int}(\bx)$, we suppress their dependency on $\phi$, since, as will be shown later, the optimization of $J_\text{DNCE}(\theta,\zeta)$ is taken  only over $(\theta, \zeta)$ while fixing $\phi$.}:
\begin{equation} \label{eq:DNCE-obj}
J_\text{DNCE}(\theta,\zeta) = E_{\bx \sim p_\text{int}(\bx) } \left[ \log p(C=+|\bx;\theta,\zeta)\right]
+ \nu E_{\bx \sim q_\phi(\bx)} \left[  \log p(C=-|\bx;\theta,\zeta) \right]
\end{equation}
where
\[p_\text{int}(\bx) = \alpha p_\text{ora}(\bx) + (1-\alpha) q_\phi(\bx)\] 
denotes an interpolation of the data distribution and the noise distribution, and $0 < \alpha < 1$ is the interpolating factor.
$p(C=+|\bx;\theta,\zeta)$ and $p(C=-|\bx;\theta,\zeta)$ are defined the same as in \eqref{eq:NCE_post}, except that the fixed noise distribution $q(\bx)$ is replaced by the dynamic noise distribution $q_\phi(\bx)$.
Intuitively, as the noise distribution $q_\phi$ converges to the data distribution, using the interpolated distribution $p_\text{int}$ will increase the number of data-like samples by adding samples drawn from the noise distribution. This could avoid the model to be overfitted to the sparse training set.

Putting the two modifications together, DNCE conducts the following joint optimization,
\begin{displaymath}
\left\{
\begin{split}
& \max_{\theta,\zeta} J_\text{DNCE}(\theta,\zeta) \\
& \max_{\phi} E_{\bx \sim p_\text{ora}(\bx)} \left[ \log q_\phi(\bx) \right] \\
\end{split}
\right.
\end{displaymath}
which can be solved by applying minibatch-based stochastic gradient ascent.

At each iteration, a set of data samples, denoted by $\mathcal D$, is sampled from $p_\text{ora}$, with the number of samples in $\mathcal{D}$ denoted as $|\mathcal{D}|$.
Additionally, two sets of noise samples are drawn from the noise distribution $q_\phi$, denoted by $\mathcal{B}_1$ and $\mathcal{B}_2$,
whose sizes satisfy $|\mathcal{B}_1| = \frac{1-\alpha}{\alpha} |\mathcal{D}|$ and $|\mathcal{B}_2| = \frac{\nu}{\alpha} |\mathcal{D}|$ respectively.
As a result, the union of $\mathcal{D}$ and $\mathcal{B}_1$ can be viewed as samples drawn from the interpolated distribution $p_\text{int}$, with $|\mathcal{D} \cup \mathcal{B}_{1}| = \frac{|\mathcal{D}|}{\alpha}$.
Model parameters $(\theta, \zeta)$ are updated by ascending the following stochastic gradients:
\begin{displaymath}
\begin{split}
 &\frac{\alpha}{|\mathcal{D}|}\sum_{\bx \in \mathcal{D} \cup \mathcal{B}_{1}}
p(C=-|\bx;\theta, \zeta) \nabla_{\theta, \zeta} \log p_{\theta,\zeta}(\bx) \\
&-\frac{\alpha}{|\mathcal{D}|}\sum_{\bx \in \mathcal{B}_{2}}
p(C=+|\bx;\theta, \zeta) \nabla_{\theta, \zeta} \log p_{\theta,\zeta}(\bx)
\end{split}
\end{displaymath}
Noise model parameter $\phi$ are updated by ascending the following stochastic gradient:
\begin{displaymath}
\frac{1}{|\mathcal{D}|}\sum_{\bx\in \mathcal{D}}
\nabla_{\phi} \log q_{\phi}(\bx)
\end{displaymath}

A final remark is that the theoretical consistency of DNCE learning in the nonparametric limit can be shown by the following theorem.

\begin{theorem}
Suppose that an arbitrarily large number of data samples can be drawn from $p_\text{ora}$, and the model distribution $p_{\theta}(\bx)$ and the noise distribution $q_\phi(\bx)$ have infinite capacity.
Then we have 
\begin{enumerate}[(i)]
    \item The KL divergence $KL(p_\text{ora}||q_\phi)$ can be minimized to attain zero;
    \item If $KL(p_\text{ori}||q_\phi)$ attains zero at $\phi^*$, and the conditional log-likelihood \eqref{eq:DNCE-obj} attains a maximum at $(\theta^*,\zeta^*)$, then we have 
    \[p_{\theta^*}(\bx) = q_{\phi^*}(\bx) = p_\text{ora}(\bx)\]
\end{enumerate}
\end{theorem}

\begin{proof}
\begin{enumerate}[(i)]
\item This conclusion can be easily seen by consistencey of MLE, since minimizing $KL(p_\text{ora}||q_\phi)$ is equivalent to MLE of $q_\phi$.
\item 
From $KL(p_\text{ori}||q_{\phi^*})=0$, we have 
$q_{\phi^*} = p_\text{ora}$.

By \thref{theorem:NCE}, with fixed $\phi^*$, \eqref{eq:DNCE-obj} has the only extremum at
$p_{\theta^*}(\bx) = p_\text{int}(\bx)|_{\phi=\phi^*}
= \alpha p_\text{ora}(\bx) + (1-\alpha) q_{\phi^*}(\bx)$.
\end{enumerate}
The conclusion is clear from combining the above two equations.
\end{proof}

\section{Generation from EBMs}
\label{sec:gen}

Given an EBM, an important inference task is sampling from the model, i.e., drawing or generating samples from the model. Sampling is not only a critical step in maximum likelihood learning of EBMs (as we introduce in \secref{sec:MLE}), but also itself forms as an important class of applications. 
Generating text, images, speech, or other media has received increasing interests, and recently has been collectively referred to as generative AI\footnote{\url{https://en.wikipedia.org/wiki/Generative_artificial_intelligence}}\index{Generative AI}.

Transformer-based \citep{vaswani2017attention} autoregressive language models (ALMs)\index{Autoregressive language model (ALM)}, generating text sequentially from left to right, have been the dominant approach for text generation \citep{radford2018improving}.
Key to their success is local normalization,
i.e. they are defined in terms of a product of conditional distributions, one for each token in the sequence. 
These models can be trained efficiently via maximum likelihood teacher-forcing, and sampling from ALMs is straightforward by ancestral sampling \index{Ancestral sampling}.
Unfortunately, local normalization also brings some drawbacks for these \emph{locally-normalized sequence models}\index{Locally-normalized sequence model}, when compared to \emph{globally-normalized sequence models}\index{Locally-normalized sequence model} (namely EBM based sequence models)\footnote{
There are some studies, which do not involve EBM modeling, but use a masked language modeling objective to train the model.
This approach, referred to as \emph{non-autoregressive generation}\index{Non-autoregressive generation}, performs iterative decoding, i.e., generating non-autoregressively, then masking out and regenerating, and so cycles for a number of iterations \citep{ghazvininejad2019mask}.
}.
As will be detailed in \secref{sec:bias} and \secref{sec:residual_ebm}, the drawbacks include:
\begin{itemize}
\item Discrepancy between training and inference
(related to \emph{exposure bias});
\item Limitation in long-range coherency due to only left-to-right modeling and decoding (related to \emph{label bias});
\item Inflexibility in controlled generation (e.g., satisfying hard lexical constraints and/or soft topical constraints in text generation).
\end{itemize}
There are similar concerns for generating other sequence media (e.g., speech).
There have been studies for speech synthesis by EBMs \citep{sun2023energy} and also by some non-autoregressive models, such as FastSpeech 2 \citep{ren2020fastspeech} and diffusion models \citep{popov2021grad}.

Remarkably, given a learned EBM, some applications need likelihood evaluation (e.g., in language modeling for speech recognition), while other tasks require generation from the learned generative model (e.g., in many NLP tasks such as summarization, dialog, and machine translation).
Generation from generative models basically is sampling from them. 
In practice, generating from locally-normalized sequence models, sometimes also referred to as decoding, can be readily realized by greedy decoding (beam search) or nucleus sampling \citep{holtzman2019curious} as engineering variants of ancestral sampling.
There is no easy methods for sampling from EBMs.
Basically we have to resort to Monte Carlo methods, such as MCMC and importance sampling (IS), which is usually not as computational efficient as ancestral sampling.
Perhaps this is the most challenging practical limitation of EBMs for their applications in generation, and the dominant approach to text generation is still based on large neural auto-regressive models.

\begin{table}[t]
	\caption{A survey of different sampling methods used in generating text from EBMs. 
	The target model is the EBM, while a proposal is required for both MCMC and IS.	
		Different proposals are used in different applications.
Shorthands: ALMs (autoregressive language model), MLM (masked language model), SNIS (self-normalized importance sampling), ASR (automatic speech recognition), CTG (controlled text generation), CTGAP (conditional text generation after prefix).}
	\label{tab:gen_survey}
	\centering
	\begin{xltabular}{\textwidth}{l|c|X}
	\hline
	Sampling method & Proposal &Application \\
	\hline\hline
	\multirow{2}{7em}{MH within Gibbs sampling} 
	&\multirow{2}{6em}{Conditional of word class} &ASR \citep{wang2015trans,Wang2017LearningTR}\\
	\cline{2-3}
	&ALM &ASR \citep{wang2017language}\\
	\cline{2-3}
	&MLM &CTG \citep{miao2019cgmh,goyal2021exposing,mixandmatch} (see \secref{sec:sampling_mix_and_match})\\
	\hline
	SNIS &ALM &CTG \citep{parshakova2019global,khalifa2020distributional}; CTGAP \citep{deng2020residual} (see \secref{sec:generation_residual})\\
	\hline		
	Langevin dynamics &- &CTG \citep{qin2022cold}\\
	\hline	
\end{xltabular}
\end{table}

Theoretically, sampling from EBMs can be performed by the MCMC and importance sampling methods, which are introduced in \secref{sec:MCMC} and \secref{sec:IS} respectively.
Gradient-based MCMC methods (\secref{sec:LD_HMC}), such as Langevin dynamics, are good choices for sampling for continuous data (e.g., images), by using the gradient of the potential $\nabla_x U_\theta(x)$. 
However, since text is discrete, the gradient is not well-defined, making it non-trivial to apply gradient-based MCMC methods for sampling text from EBMs.

Various MCMC and IS methods have been explored in applications of EBMs for generating text.
In \tbref{tab:gen_survey}, we survey some recent studies, which provide concrete examples.
Remarkably, both MCMC and IS methods need proposal distributions, and the design of proposal distributions heavily depends on specific applications.
We comment on the particular proposal distributions used in different applications in \tbref{tab:gen_survey}. some further discussions are as follows.

\paragraph{MH within Gibbs sampling.}
Suppose we use Gibbs sampling to generate a sequence of $n$ tokens, $x=(x_1,\cdots,x_n)$, from an EBM distribution $p(x_1,\cdots,x_n)$.
Exact Gibbs sampling needs to calculate the conditional distribution $p(x_i|x_{\setminus i})$ of token $x_i$ for each position $i$, given all the other tokens $x_{\setminus i} \triangleq (x_1,\cdots,x_{i-1},x_{i+1},\cdots,x_n)$.
For modern EBMs developed for text, such as in \citep{wang2015trans,Wang2017LearningTR,wang2017language} and so on, this is computational expensive, because calculating $p(x_i|x_{\setminus i})$ needs to enumerate all possible values of $x_i$ from $\mathcal{V}$ and to compute the joint probability $p(x_i, x_{\setminus i})$ for each possible value, where $\mathcal{V}$ denotes the vocabulary.
Metropolis-Hastings (MH) within Gibbs sampling has been explored, with a proposal, to draw MCMC samples from $p(x_i|x_{\setminus i})$.

In \citep{wang2015trans,Wang2017LearningTR}, word classing is introduced to accelerate sampling, which means that each word is assigned to a single class.
Through applying MH within Gibbs sampling, we first sample the class by using a reduced model as the proposal, which includes only the features that depend on $x_i$ through its class\footnote{This is possible, since features used in \citep{wang2015trans,Wang2017LearningTR} are discrete features ($n$-gram features). (See introduction in \secref{sec:trf_discrete_feature})}, and then sample the word. This reduces the computational cost from $|\mathcal{V}|$ to $|\mathcal{C}|+|\mathcal{V}|/|\mathcal{C}|$ on average, where $|\mathcal{C}|$ denotes the number of classes.

The computation reduction in using word classing in EBMs with neural features (i.e., parameterized by neural networks) is not as significant as in EBMs with discrete features, because EBMs parameterized by neural networks involve a much larger context, which makes the sampling computation with the reduced model still expensive.
In later work in \citep{wang2017language}, an ALM (autoregressive language model) is introduced to propose for $p(x_i|x_{\setminus i})$, and in \citep{miao2019cgmh,goyal2021exposing,mixandmatch}, a MLM (masked language models) is used as the proposal.
The proposal model can be jointly trained with the EBM, as in \citep{wang2017language}, or pre-trained language models\index{Pre-trained language model (PLM)} can be directly used for the proposal \citep{miao2019cgmh,goyal2021exposing,mixandmatch}.

\paragraph{Self-normalized importance sampling (SNIS).}
The basics are introduced in \secref{sec:IS}. See \secref{sec:generation_residual} for details in applications.

\paragraph{Langevin dynamics.} 
The basics are introduced in \secref{sec:LD_HMC}. See \secref{sec:COLD} for details in applications.

\chapter{EBMs for sequential data with applications in language modeling}
\label{ch:lm}



In this chapter, we are mainly concerned with learning the (\emph{marginal}) distribution of observation $x$ itself by EBMs.
Considering the sequential nature of speech and language, we show how to develop EBMs for sequential data, or more generally, for \emph{trans-dimensional} data.

EBMs are mostly developed in \emph{fixed-dimensional} settings, for example, in the modeling of fixed-size images.
Trans-dimensional setting means that the observations can be of different dimensions.
A familiar case is temporal modeling of sequential data, where each observation is a sequence of a random length.
\emph{Language modeling} falls exactly in this trans-dimensional setting, where an observation sequence $x$ is a natural language sentence (i.e., a token sequence).

\section{Autoregressive language model (ALM)}
\label{sec:ALM}

\emph{Language modeling} involves determining the joint probability $p(x)$ of a sentence $x$, which can be denoted as a pair
$x =(l,x^l)$, where $l$ is the length and $x^l =(x_1,\dots,x_l)$ is a sequence of $l$ tokens.
Currently, the dominant approach to language modeling is the locally-normalized or conditional modeling, which decomposes the joint probability of $x^l$ into a product of conditional probabilities by using the chain rule,
\begin{equation}\label{eq:chain}
p(x_1,\dots,x_l) = \prod_{i=1}^l p(x_i|x_1, \dots, x_{i-1}).
\end{equation}
Language models (LMs)\index{Language model (LM)} in the form of \eqref{eq:chain} is known as \emph{autoregressive language models} (ALMs)\index{Autoregressive language model (ALM)}.
Remarkably, for an ALM to make the sum of the probabilities of all sequences equal to 1, it is necessary to place a special token $\langle \texttt{EOS} \rangle$ at the end of sentences and to include this in the product of \eqref{eq:chain} \citep{chen1999empirical}. Otherwise, the sum of the probabilities of all sequences of a given length is 1, and the sum of the probabilities of all sequences is then infinite. 

In early days before the deep learning era, the history of $x_i$, denoted as $(x_1,\cdots,x_{i-1})$, is often reduced to equivalence classes through a mapping $\phi(x_1,\cdots,x_{i-1})$ with the assumption
\begin{displaymath}\label{eq:CLM}
p(x_i|x_1,\cdots,x_{i-1}) \approx p(x_i|\phi(x_1,\cdots,x_{i-1})).
\end{displaymath}
A classic example is the traditional $n$-gram language models (LMs) with
\begin{displaymath}
\phi(x_1,\cdots,x_{i-1})=(x_{i-n+1},\dots,x_{i-1}),
\end{displaymath}
assuming that current token $x_i$ depends on history only through the previous $n-1$ tokens, i.e., the $(n-1)$-order Markov assumption.
Various smoothing techniques have been used for parameter estimation, and particularly, the modified Kneser-Ney (KN) smoothed $n$-gram LMs are still widely used because of its simplicity and good performance \citep{chen1999empirical}.

Recently, neural network LMs have begun to surpass the traditional $n$-gram LMs, and also follow the locally-normalized approach. The mapping $\phi(x_1,\cdots,x_{i-1})$ condenses the history into a hidden vector $h_i \in \mathbb{R}^D$ through a neural network (NN), which can be a feedforward NN \citep{schwenk2007continuous}, a recurrent NN \citep{mikolov2011,lstm2012}, or more recently, a Transformer NN \citep{vaswani2017attention}.
Specifically, in either manners, the neural network calculates the conditional probability at each position as follows:
\begin{equation} \label{eq:ALM}
p(x_i=k|x_1,\cdots,x_{i-1}) = \frac{\exp(z_k)}{\sum_{j=1}^{|\mathcal{V}|} \exp(z_j)}   
\end{equation}
which is in the form of a multi-class logistic regression, as introduced in \eqref{eq:softmax}.
The logits $z_k = w_k^T h_i + b_k, k=1,\cdots,K$ are calculated from a linear layer on top of hidden vector $h_i$.
$w_k \in \mathbb{R}^D, b_k \in \mathbb{R}$ denote the weight vector and bias of the linear layer, respectively. $\mathcal{V}$ denotes the vocabulary of all possible tokens.

\paragraph{Drawbacks of ALMs.}
Both the classic $n$-gram LMs and the recent neural network LMs are autoregressive language models, which are locally normalized.
Unfortunately, \emph{local normalization in ALMs brings some drawbacks}. As will be detailed in \secref{sec:bias}, ALMs are prone to exposure bias \citep{wiseman2016sequence,ranzato2016sequence} and label bias \citep{lafferty2001conditional,andor2016globally}.


\section{Energy-based language model (ELM)}
\label{sec:ELM}

\subsection{Globally-normalized ELM (GN-ELM)}
\label{sec:GN-ELM}

Energy-based language models (ELMs)\index{Energy-based language model (ELM)} parameterize an unnormalized distribution for natural sentences and are radically different from autoregressive language models.
Let $x$ be a natural sentence (i.e., a token sequence). An \emph{energy-based language model} (ELM) is defined as follows
\begin{equation}
\label{eq:elm}
    p_\theta(x)=\frac{\exp(U_\theta(x))}{Z(\theta)}
\end{equation}
where $U_\theta(x)$ denotes the potential function with parameter $\theta$, $Z(\theta)=\sum_x \exp(U_\theta(x))$ is the normalizing constant, and $p_\theta(x)$ is the probability of sentence $x$.
For reasons to be clear below (mainly to be differentiated from TRF-LM), the model in \eqref{eq:elm} is called \emph{globally-normalized ELM} (GN-ELM).\index{Globally-normalized ELM (GN-ELM)}

ELMs potentially address the drawbacks of ALMs introduced above, as they do not require any local normalization.
Early attempts on building energy-based language models are GN-ELMs and date back to \citep{rosenfeld2001whole}, which proposes \emph{whole-sentence maximum entropy} (WSME)\index{Whole-sentence maximum entropy (WSME)} language models\footnote{Due to the connection between log-linear model and maxent model as we introduced before in \secref{sec:ugm_loglinear}, this model is called WSME.}.
Specifically, a WSME model has the log-linear form
\begin{equation}\label{eq:WSME}
p(x;\lambda) = \frac{1}{Z(\lambda)} e^{\lambda^T f(x)}.
\end{equation}
Here $f(x)$ is a vector of features, which are computable functions of $x$ such as $n$-grams conventionally used,
$\lambda$ is the corresponding parameter vector, and $Z(\lambda)= \sum_x e^{\lambda^T f(x)}$ is the global normalizing constant.

There has been little work on WSME-LMs, mainly in \citep{rosenfeld2001whole,amaya2001improvement,ruokolainen2010using}.
Although the whole-sentence approach has the potential advantage of being able to flexibly integrate a richer set of features,
the empirical results of previous WSME-LMs are not satisfactory, almost the same as traditional $n$-gram LMs.
After incorporating lexical and syntactic information, a mere relative improvement of 1\% and 0.4\% respectively in perplexity and in WER (word error rate)\index{Word error rate (WER)} was reported for the resulting WSME-LM \citep{rosenfeld2001whole}. Subsequent studies of using WSME LMs with grammatical features, as in \citep{amaya2001improvement} and \citep{ruokolainen2010using}, reported perplexity improvement above 10\% but no WER improvement when using WSME LMs alone.

In recent years, there are encouraging progresses in both theories and applications of ELMs. A new class of ELMs, called trans-dimensional random fields (TRFs), have been developed, which are different from GN-ELMs and \emph{present the first strong empirical evidence supporting the power of using energy-based approach to language modeling} \citep{wang2015trans,Wang2017LearningTR}.
Applications of ELMs have covered computation of sentence likelihoods (up to a constant) for speech recognition \citep{wang2015trans,Wang2017LearningTR,wang2017language,BinICASSP2018,wang2018improved,gao2020integrating} (to be introduced in the next section), text generation \citep{deng2020residual} (to be covered in \secref{sec:residual_ebm}), language model pre-training \citep{clark2020pre} (to be covered in \secref{sec:electric}), calibrated natural language understanding \citep{he2021joint} (to be covered in \secref{sec:JRF_calibrate}), and so on.

\subsection{Trans-dimensional random field (TRF) LMs}
\label{sec:TRF_modeldef}
\index{Trans-dimensional random field (TRF)}

To describe trans-dimensional observations in general, a new energy-based probabilistic model, called the \emph{trans-dimensional random field} (TRF) model, has been proposed in \citep{wang2015trans,Wang2017LearningTR}, which explicitly mixes a collection of random fields in sample spaces of different dimensions.
\emph{A GN-ELM is globally-normalized over all sequences of all lengths}. In contrast, \emph{a TRF-LM is a collection of random fields, each normalized on subspaces of different lengths separately and weighted by empirical length probabilities}.


Suppose that it is of interest to build random field models for multiple sets of observations of different dimensions, such as images of different sizes or sentences of different lengths.
Denote by $x^j$ an observation in a sample space $\mathcal{X}^j$ of dimension $j$, ranging from 1 to $m$.
The space of all observations is then the union $\mathcal{X} =\cup_{j=1}^m \mathcal{X}^j$. To emphasize the dimensionality, each observation in
$\mathcal{X}$ can be represented as a pair $x=(j,x^j)$, even though $j$ is identified as the dimension of $x^j$.
By abuse of notation, write $f(x) = f(x^j)$ for features of $x$.

For $j=1,\ldots,m$, assume that the observations $x^j$ are distributed from a random field in the form
\begin{displaymath} \label{eq:sub-model}
p_j(x^j; \lambda) = \frac{1}{Z_j(\lambda)} e^{U_\theta (x^j)},
\end{displaymath}
$U_{\theta}(x^j) : \mathcal{X}^j \rightarrow \mathbb{R}$ denotes the potential function which assigns a scalar value to each configuration of $x$ in $\mathcal{X}$ and \emph{can be very flexibly parameterized} through linear functions or nonlinear functions with neural networks of different architectures, to be explained in \secref{sec:TRF-energy}.
$\theta$ denotes the set of parameters, and
$Z_j(\theta)$ is the normalizing constant:
\begin{displaymath}\label{eq:zeta}
Z_j(\theta) = \sum_{x^j} e^{U_\theta (x^j)}, \quad \quad j=1,\dots,m.
\end{displaymath}
Moreover, assume that dimension $j$ is associated with a probability $\pi_j$ for $j=1,\ldots,m$ with $\sum_{j=1}^m \pi_j=1$. Therefore, the pair $(j,x^j)$ is jointly distributed as
\begin{equation}\label{eq:trf-model}
p(j,x^j;\pi, \theta) = \pi_j\, p_j(x^j;\theta) = \frac{\pi_j}{Z_j(\theta)} e^{U_\theta (x^j)},
\end{equation}
where $\pi=(\pi_1,\ldots,\pi_m)^T$.

Here we actually define a mixture of random fields for joint modeling sentences of different dimensions (namely lengths).
There is a random field for each length.
By maximum likelihood, the mixture weights can be estimated to be the empirical length probabilities.

\begin{table}[t!]
	\caption{The development of TRF-LMs.}
	\label{tab:TRF_history}
	\centering
	\begin{xltabular}{\textwidth}{l|X}
		\hline
		Work & Contribution \\
		\hline\hline
		\multirow{2}{8em}{\cite{wang2015trans,Wang2017LearningTR}} 
		& $\bullet~$Discrete features\\
		& $\bullet~$Augmented stochastic approximation (AugSA) for model training\\
		\hline
		\multirow{2}{8em}{\cite{wang2017language}} 
		& $\bullet~$Potential function as a deep CNN\\
		& $\bullet~$Model training by AugSA plus JSA (joint stochastic approximation)\\
		\hline
		\multirow{3}{8em}{\cite{BinICASSP2018}} 
		& $\bullet~$Potential function in the form of exponential tilting, revisited in residual EBMs \citep{deng2020residual}\\
		& $\bullet~$Use LSTM on top of CNN\\	
		& $\bullet~$NCE is introduced to train TRF-LMs\\
		\hline
		\multirow{2}{8em}{\cite{wang2018improved}} 
		& $\bullet~$Simplify the potential definition by using only bidirectional LSTM\\
		& $\bullet~$Propose Dynamic NCE for improved model training\\
		\hline	
		\cite{gao2020integrating} & $\bullet~$Mixed-feature TRFs, by integrating discrete and neural features\\
		\hline
		\cite{ExploringELM} & $\bullet~$Pre-trained language models (PLMs) are used as the backbones of energy functions, noise distributions in NCE and proposal distributions in Monte Carlo\\
		\hline
	\end{xltabular}
\end{table}

As outlined in Table \ref{tab:TRF_history}, there are a series of works in the development of TRF-LMs, each with its own contribution in different parameterizations of potential functions and model training methods. Before expanding introductions around Table \ref{tab:TRF_history}, let us first recognize the differences between GN-ELM and TRF-LM.

\subsection{Comparison between GN-ELM and TRF-LM}

The following comment on the connection and difference between GN-ELM and TRF-LM models is adapted from the comparison between WSME and TRF models in \citep{Wang2017LearningTR}.
Suppose that we add the dimension features
in the GN-ELM model \eqref{eq:elm} and obtain
\begin{equation} \label{eq:model-g-len}
p(j,x^j;\lambda,\nu) = \frac{1}{Z(\theta,\nu)} e^{{\nu}^T \delta(j) + U_\theta (x^j)},
\end{equation}
where $\delta(j) = (\delta_1(j), \cdots, \delta_m(j))^T$ denotes the dimension features such that $\delta_l(j) = 1(j=l)$.
$\nu=(\nu_1,\ldots,\nu_m)^T$ denotes the corresponding parameter vector, and $Z(\theta,\nu)$ is the global normalizing constant
\begin{equation} \label{eq:global-Z}
Z(\theta,\nu)= \sum_{j=1}^{m} \sum_{x^j \in \mathcal{X}^j} e^{{\nu}^T \delta(j) + U_\theta (x^j)} \\
= \sum_{j=1}^m e^{\nu_j} Z_j (\theta).
\end{equation}

Similar to Proposition 1 in \citep{Wang2017LearningTR}, it can be seen that when both fitted by maximum likelihood estimation, model \eqref{eq:model-g-len} is equivalent to model \eqref{eq:trf-model} but with different parameterization.
The parameters in model \eqref{eq:model-g-len} are $(\theta,\nu)$, whereas the parameters in model \eqref{eq:trf-model} are $(\pi, \theta)$.
Therefore, an important distinction of TRF-LM from GN-ELM lies in the use of dimension features, which has significance consequences in both model definition and model learning.

First, it is clear that model \eqref{eq:trf-model} is a mixture of random fields on subspaces of different dimensions,
with mixture weights explicitly as free parameters. Hence model \eqref{eq:trf-model} will be called a \emph{trans-dimensional random field} (TRF).
Moreover, by maximum likelihood, the mixture weights can be estimated to be the empirical dimension probabilities. 

Second, it is instructive to point out that model \eqref{eq:elm} is essentially also a mixture of random fields, but the mixture weights implied are fixed to be proportional to the normalizing constants $Z_j(\theta)$:
\begin{equation} \label{eq:WSME2}
p(j, x^j; \theta) = \frac{Z_j(\theta)}{Z(\lambda)} \cdot \frac{1}{Z_j(\theta)} e^{ U_\theta (x^j) },
\end{equation}
where $Z(\theta)= \sum_x e^{ U_\theta (x) }=\sum_{j=1}^{m} Z_j(\theta)$.
Typically the unknown mixture weights in \eqref{eq:WSME2} may differ from
the empirical length probabilities and also from each other by orders of magnitudes, e.g. $10^{40}$ or more in the experiments of \citep{Wang2017LearningTR}.
As a result, it is very difficult to efficiently sample from model \eqref{eq:elm}, in addition to the fact that the length probabilities are poorly fitted for model \eqref{eq:elm}.
Setting mixture weights to the known, empirical length probabilities enables us to develop an effective learning algorithm, as introduced in \citep{Wang2017LearningTR}.
Basically, the empirical weights serve as a control device to improve sampling from multiple distributions \citep{liang2007stochastic,Tan2015}.

\section{ELMs for speech recognition}

As an important application, ELMs have been successfully used as a means for calculating sentence scores in automatic speech recognition (ASR).
The design of energy function and the optimization of parameters are central research questions in applying ELMs, which are introduced in the following two subsections respectively.


\subsection{Architectures of energy functions}
\label{sec:TRF-energy}

One is generally free to choose the energy function in ELMs, as long as it assigns a scalar energy to every sentence, no matter for TRF-LMs or GN-ELMs.
A subtle difference between defining the energy function $-U_\theta(x)$ in \eqref{eq:elm} for GN-ELMs and defining $-U_\theta(x^j)$ in \eqref{eq:trf-model} for TRF-LMs is whether to include the special token $\langle \texttt{EOS} \rangle$ or not.
We do not need to include $\langle \texttt{EOS} \rangle$ at the end of $x^j$ in calculating $U_\theta(x^j)$ for TRF-LMs. For GN-ELMs, we often include $\langle \texttt{EOS} \rangle$ at the end of $x$ in calculating $U_\theta(x)$, hoping to help modeling of sentence lengths.
Except this difference, the architectures of energy functions for TRF-LMs can be readily used for GN-ELMs, and vise versa.

Broadly speaking, there are three types of energy functions.
\begin{itemize}
\item Early ELMs are log-linear models using \emph{discrete features}, including \cite{rosenfeld2001whole} (WSME-LM) and \cite{wang2015trans,Wang2017LearningTR} (TRF-LM). These ELMs could thus be referred to as discrete ELMs.
\item Later, based on CNN and LSTM networks, ELMs using neural network based energy functions (neural ELMs) have been developed \citep{wang2017language,BinICASSP2018,wang2018improved}, outperforming neural ALMs with similar model sizes.
	ELMs using neural network based energy functions could be regarded as using \emph{neural features}, by following the discussion in \secref{sec:ugm_loglinear}.
\item By integrating discrete and neural features, \emph{mixed-feature} TRFs have also been proposed \citep{gao2020integrating}, demonstrating the advantage of energy-based models in integrating discrete and neural features.	
\end{itemize}

Recently, based on Transformer networks \citep{vaswani2017attention} and large pre-trained lanugage models (PLMs) such as BERT \citep{bert} and GPT2 \citep{gpt2}, neural ELMs have been further advanced \citep{ExploringELM}.
Extensive experiments are conducted on two datasets, AISHELL-1 \citep{aishell1} and WenetSpeech \citep{wenet}. The results show that the best ELM achieves competitive results with the finetuned GPT2 and performs significantly better than the finetuned BERT. Further analysis show that the ELM obtains better confidence estimate performance than the finetuned GPT2.

In the following, we summarize the architectures used to define $U_\theta(x)$ in the literature, in roughly chronological order.
We first introduce the classic log-linear energy function using discrete features. Then, a suite of nonlinear energy functions are shown, which are called \texttt{Hidden2Scalar}, \texttt{SumInnerProduct}, \texttt{SumTargetLogit}, \texttt{SumMaskedLogit} and \texttt{SumTokenLogit}, respectively.

\emph{Notably, although one energy architecture is proposed in one context of either TRF-LMs or GN-ELMs, it can be applied in both TRF-LMs and GN-ELMs.}
Let $x=\{x_i\}_{i=1...|x|}$, where $x_i \in \{1,\cdots,V\}$ is the $i$-th token in $x$.
$|x|$ denotes the length of sentence $x$ in tokens.
$V$ denotes the size of token vocabulary.

\begin{figure}[t]
	\centering
	\includegraphics[scale=0.7]{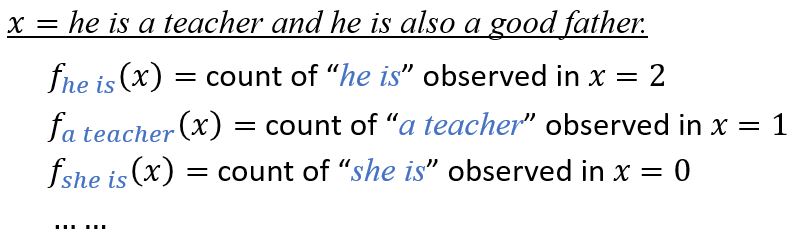}
	\caption{Example of discrete features.}
	\label{fig:discrete_feature_example}
\end{figure}

\subsubsection{Linear energy using discrete features}
\label{sec:trf_discrete_feature}
Linear energy functions using discrete features\index{Discrete feature} basically corresponds to log-linear models, as illustrated in Example \ref{eg:word_morph}, i.e., defining
\begin{equation} \label{eq:discrete_feature}
U_\lambda(x) = \lambda^T f(x)
\end{equation}
Here we follow the notations in \citep{Wang2017LearningTR}, 
where $f(x)=(f_1(x),f_2(x), \dots, f_d(x))^T$ is a \emph{feature vector}, $\lambda=(\lambda_1,\lambda_2,\dots ,\lambda_d)^T$ is the corresponding \emph{parameter vector} (instead of using $\theta$).

A feature $f_i(x)$, $i=1,\ldots,d$, can be any computable function of the input $x$.
For various applications such as language modeling, the parameters of local potential functions across locations are often \emph{tied} together \citep{inducingfeatures}.
Each feature $f_i(x)$ is often defined in the form
$f_i(x) = \sum_{k} f_{i}(x,k)$, where $f_{i}(x,k)$ is a binary function of $x$ evaluated at position $k$. In a \emph{trigram}\index{Trigram} example, $f_{i}(x,k)$ equals to 1 if three specific words appear at positions $k$ to $k+2$
and $k\le j-2$.
The binary features $f_{i}(x,k)$
share the same parameter $\lambda_i$ for different positions $k$ and dimensions $j$, so called  position-independent and dimension-independent.
Hence, the feature $f_i(x)$ indicates the count of nonzero $f_{i}(x,k)$ over $k$ in the observation $x^j$  and takes values as non-negative integers. Intuitively, $f_i(x)$ returns the count of a specific phrase (often called a $n$-gram feature) observed in the input sentence $x$, as shown in \figref{fig:discrete_feature_example}.

\newcommand{\wg}{w}
\newcommand{\cg}{c}
\newcommand{\wj}{ws}
\newcommand{\cj}{cs}
\newcommand{\ws}{wsh}
\newcommand{\cs}{csh}
\newcommand{\cpw}{cpw}
\newcommand{\cpwj}{cpw}
\newcommand{\tied}{tied}

\begin{table}
	\caption{Feature definition in TRF LMs \citep{Wang2017LearningTR}} 
	\label{tab:feature}
	\centering
	\begin{tabular}{c|l}
		\hline
		Type    &   Features \\ \hline \hline
		\wg       &   $(w_{-3}w_{-2}w_{-1}w_0)$$(w_{-2}w_{-1}w_0)$$(w_{-1}w_0)$$(w_0)$ \\ \hline
		\cg       &   $(c_{-3}c_{-2}c_{-1}c_0)$$(c_{-2}c_{-1}c_0)$$(c_{-1}c_0)$$(c_0)$ \\ \hline
		\wj      &   $(w_{-3}w_0)$$(w_{-3}w_{-2}w_0)$$(w_{-3}w_{-1}w_0)$$(w_{-2}w_0)$ \\ \hline
		\cj      &   $(c_{-3}c_0)$$(c_{-3}c_{-2}c_0)$$(c_{-3}c_{-1}c_0)$$(c_{-2}c_0)$ \\ \hline
		\ws      &   $(w_{-4}w_0)$ $(w_{-5}w_0)$  \\  \hline
		\cs      &   $(c_{-4}c_0)$ $(c_{-5}c_0)$  \\  \hline
		\cpw     &   $(c_{-3}c_{-2}c_{-1}w_0)$ $(c_{-2}c_{-1}w_0)$$(c_{-1}w_0)$ \\ \hline
		\tied    &   $(c_{-9:-6},c_0)$ $(w_{-9:-6},w_0)$ \\
		\hline
	\end{tabular}
\end{table}

The energy-based language modeling approach allows a very flexible use of features, not limited to ordinary $n$-gram features. In \citep{Wang2017LearningTR}, a variety of features as shown in \tbref{tab:feature} are used, mainly based on word and class information.
Each word is deterministically assigned to a single class, by running the automatic clustering algorithm proposed in \citep{clustering} on the training dataset.
In \tbref{tab:feature},
$w_i, c_i, i=0,-1,\dots,-5$, denote the word and its class at different position offset $i$,
e.g., $w_0, c_0$ denotes the current word and its class.
The word/class $n$-gram features ``\wg''/``\cg'', skipping $n$-gram features ``\wj''/``\cj'' \citep{goodman2001abit}, higher-order skipping features ``\ws''/``\cs'', and the crossing features ``\cpw'' (meaning class-predict-word) are introduced in \citep{wang2015trans}. 
In \citep{Wang2017LearningTR}, the tied long-skip-bigram features ``\tied'' \citep{GoogleSkip15} are further introduced,
in which the skip-bigrams with skipping distances from 6 to 9 share the same parameter.
In this way we can leverage long distance contexts without increasing the model size.
All the features $f(x)$ in model \eqref{eq:discrete_feature} are constructed from \tbref{tab:feature} in a position-independent and dimension-independent manner,
and only the features observed in the training data are used.
It is shown in \citep{Wang2017LearningTR} that the TRF-LM using features ``\wg+\cg+\wj+\cj+\ws+\cs+\tied'' outperforms the KN 5-gram LM significantly with 10\% relative error reduction.

\begin{figure}
	\centering
	\includegraphics[scale=0.6]{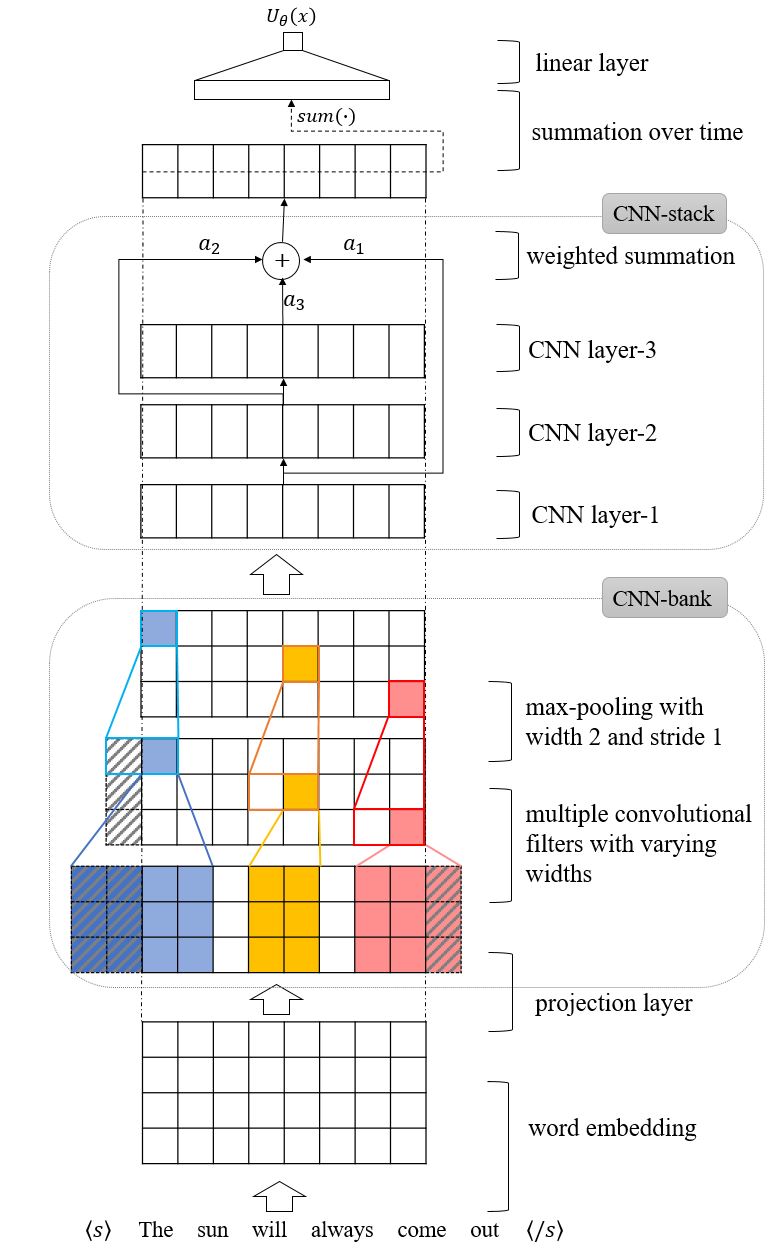}
	\caption{Hidden2Scalar: a deep CNN architecture used to define the potential function $U_\theta(x)$. Shadow areas denote the padded zeros. \citep{wang2017language}
	}
	\label{fig:ASRU17}
\end{figure}

\begin{figure}
	\centering
	\includegraphics[width=0.7\linewidth]{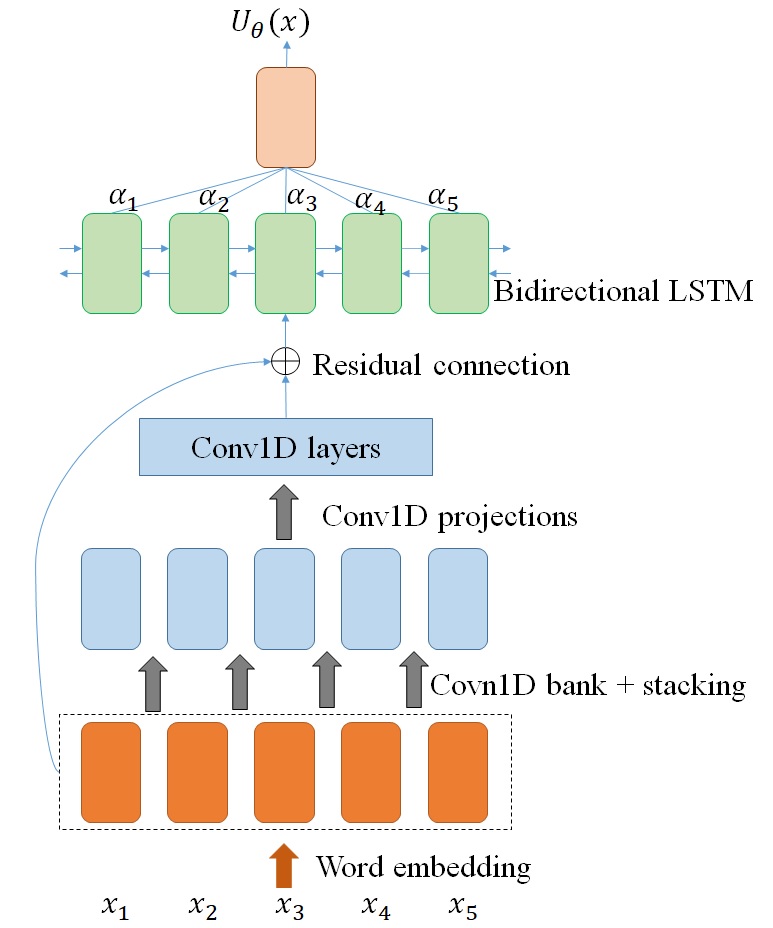}
	\caption{Hidden2Scalar: a bidirectional LSTM on top of CNN used to define the potential function $U_\theta(x)$. \citep{BinICASSP2018}}
	\label{fig:ICASSP2018}
\end{figure}

\begin{figure}
	\centering
	\includegraphics[width=0.7\linewidth]{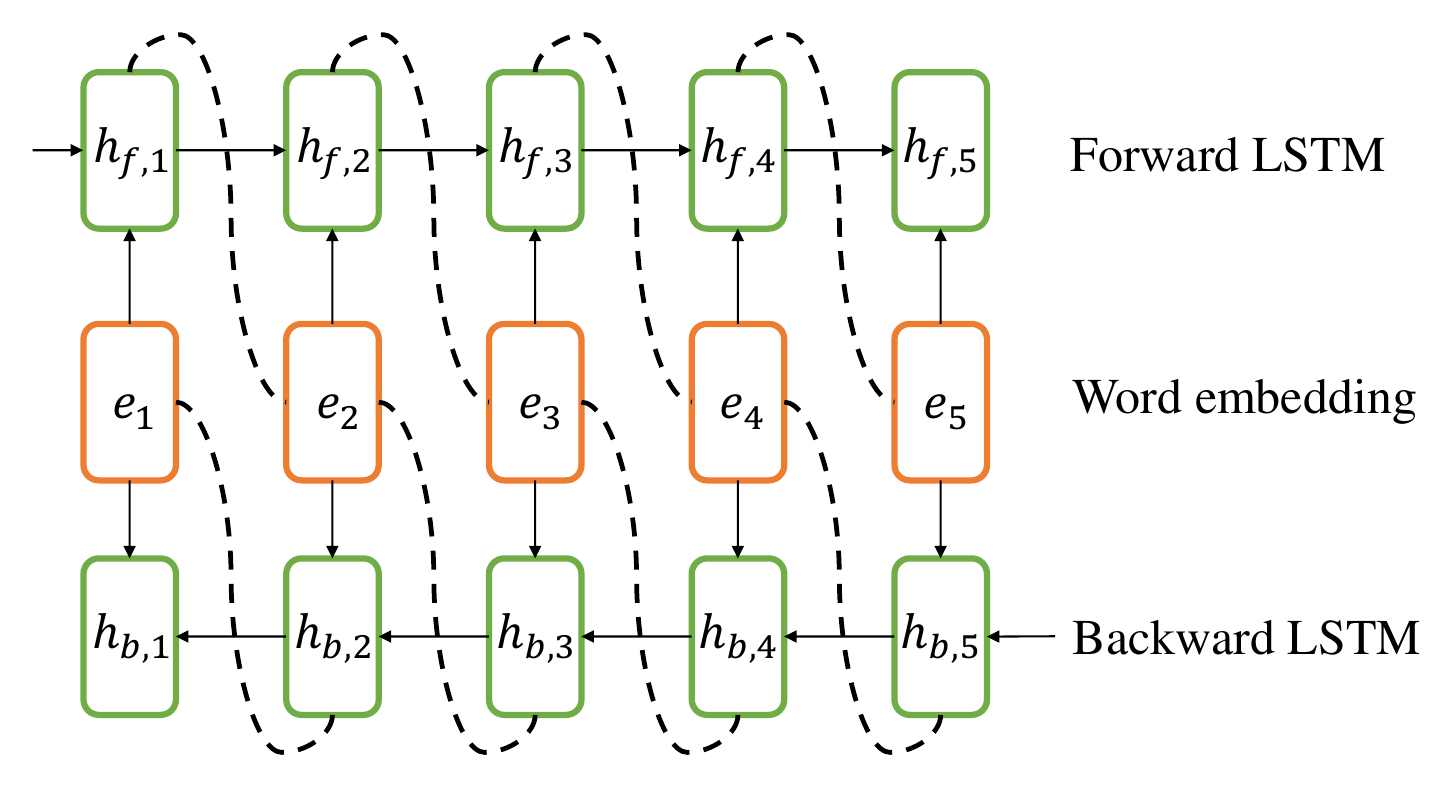}
	\caption{SumInnerProduct: a bidirectional LSTM used to define the potential function  $U_\theta(x)$. \citep{wang2018improved}}
	\label{fig:SLT2018.png}
\end{figure}

\subsubsection{Non-linear energy: Hidden2Scalar}
Generally speaking, like in \citep{wang2017language, BinICASSP2018, deng2020residual, he2021joint}, we can use a text encoder to encode $x$ and  denote the encoder output (hidden vectors) by $\text{enc}_{\theta}(x)$. At position $i$, we have $\text{enc}_{\theta}(x)[i]$.
Then, the potential function can be defined as
\begin{equation}
\label{eq:hidden2scalar}
U_\theta(x) = \text{Linear} \left( \sum_{i=1}^{|x|} \text{enc}_\theta(x)[i] \right)
\end{equation}
where $\text{Linear}(\cdot)$ denotes a trainable linear layer whose output is a scalar.
This energy function is obtained by transforming neural hidden vectors into a scalar, hence it is named by \texttt{Hidden2Scalar}.

The text encoder can be based on a fully CNN architecture \citep{wang2017language} as shown in \figref{fig:ASRU17}, or a bidirectional LSTM (BLSTM) stacked on top of CNN as shown in \figref{fig:ICASSP2018}. Recently, BERT based text encoders have been used in \citep{deng2020residual, he2021joint, ExploringELM}.
In \citep{deng2020residual}, in the final layer of RoBERTa \citep{Liu2019RoBERTaAR}, the mean-pooled hidden states are projected to a scalar energy value.
In \citep{he2021joint}, three variants of energy functions (\texttt{scalar}, \texttt{hidden}, and \texttt{sharp-hidden}) are defined on top of a RoBERTa based text encoder, which will be detailed in \secref{sec:JRF_calibrate}.

\subsubsection{Non-linear energy: SumInnerProduct}

In \citep{wang2018improved}, a bidirectional LSTM based potential function is defined as follows, illustrated in \figref{fig:SLT2018}.
First, each word $x_i$ ($i=1,\ldots,|x|$) in a sentence is mapped to an embedded vector $e_i \in R^d$.
Then the word embedding vectors are fed into a bidirectional LSTM to extract the long-range sequential features from the forward and backward contexts.
Denote by $h_{f,i}, h_{b,i} \in R^d$ the hidden vectors of the forward and backward LSTMs respectively at position $i$.
Finally, we calculate the inner product of the hidden vector of the forward LSTM at current position and the embedding vector at the next position,
and calculate the inner product of the hidden vector of the backward LSTM at current position and the embedding vector at the pervious position (dash line in \figref{fig:SLT2018}).
The potential function $\phi(x^l;\theta)$ is computed by summing all the inner products, hence named by \texttt{SumInnerProduct},
\begin{equation}\label{eq:phi}
U_\theta(x) = \sum_{i=1}^{|x|-1} h_{f,i}^T e_{i+1} + \sum_{i=2}^{|x|} h_{b,i}^T e_{i-1}
\end{equation}
where $\theta$ denotes all the parameters in the neural network. 
The \texttt{SumInnerProduct} energy provides a theoretical-solid framework to incorporate the bidirectional LSTM features.

\subsubsection{Non-linear energy: SumTargetLogit}

The \texttt{SumInnerProduct} energy uses a bidirectional network.
In order to exploit pre-trained language models\index{Pre-trained language model (PLM)}, which mostly are ALMs and unidirectional, we could consider energy definition tailored to ALMs \citep{ExploringELM}.
Given history $x_{1:i-1}$, let the output logits to predict the next token be denoted by $z_{\theta}(x_{1:i-1})$, whose dimension is equal to $V$. The $k$-th logit is denoted by $z_{\theta}(x_{1:i-1})[k]$.
Then, the potential is defined as
\begin{equation}
\label{eq:sumtargetlogit}
U_\theta(x)=\sum_{i=1}^{|x|} z_{\theta}(x_{1:i-1})[x_i]
\end{equation}
This potential function sums the logits corresponding to the target token (next token) at each position, hence it is named by \texttt{SumTargetLogit}.
In contrast, the ALM applies local normalization (softmax) to the logits $z_{\theta}(x_{1:i-1})$ to obtain the conditional probability of $x_i$ given history $x_{1:i-1}$.

\subsubsection{Non-linear energy: SumMaskedLogit}
For \emph{masked language model} (MLM)\index{Masked language model (MLM)}, e.g., BERT, pseudo-log-likelihood (PLL)\index{Pseudo-log-likelihood (PLL)} is introduced for scoring sentences \citep{pll}. Inspired by this, we can define the potential function as follows \citep{ExploringELM}:
\begin{equation}
\label{eq:SumMaskedLogit}
U_\theta(x) = \sum_{i=1}^{|x|} g_\theta(\text{mask}(x,i))[i][x_i]
\end{equation}
where $g_\theta$ denotes the MLM, whose output, at each position, is the logits before softmax. 
$g_\theta(\text{mask}(x,i))$ means masking the $i$-th token in $x$ and sending the masked sequence into the MLM for a forward pass. At position $i$, the logit corresponding to the masked token $x_i$ is denoted as $g_\theta(\text{mask}(x,i))[i][x_i]$.
In \eqref{eq:SumMaskedLogit}, the potential is defined by summing the logits corresponding to masked tokens, hence it is named by \texttt{SumMaskedLogit}.
Notably, this architecture is much time-consuming than others, since it requires $|x|$ forward passes to calculate the energy of one sentence, therefore this architecture is primarily for stimulating ideas rather than conducting experiments.

\subsubsection{Non-linear energy: SumTokenLogit}
To overcome the deficiency of \texttt{SumMaskedLogit}, a simplication is proposed \citep{ExploringELM}, i.e.,  omitting the masking step and feeding $x$ directly to the MLM, so that the logits at all positions can be calculated in parallel. The potential is defined as:
\begin{equation}
\label{eq:sumtokenlogit}
U_\theta(x) =  \sum_{i=1}^{|x|}  g_\theta(x)[i][x_i]
\end{equation}

\subsubsection{Comparison of non-linear energy architectures}
For comparing the different non-linear energy architectures, experiments in \citep{ExploringELM} find that the bi-directional architectures (\texttt{Hidden2Scalar} and \texttt{SumTokenLogit}) based on BERT are generally better than the unidirectional architecture (\texttt{SumTargetLogit}).

\subsection{Integrating discrete and neural features}

There has been a long recognition that discrete features ($n$-gram features) and neural network based features have complementary strengths for language models (LMs).
Generally, LMs with neural features (e.g. neural ALMs, neural TRF-LMs) outperform LMs with discrete features (e.g. KN $n$-gram LMs, discrete TRF-LMs), but interpolation between them usually gives further improvement.
This suggests that discrete and neural features have complementary strengths.
Presumably, the $n$-gram features mainly capture local \emph{lower-order} interactions between words, while the neural features particularly defined by LSTMs can learn \emph{higher-order} interactions.
Additionally, by embedding words into continuous vector spaces, neural LMs are good at learning \emph{smoothed regularities}, while discrete LMs may be better suited to handling \emph{symbolic knowledges} or idiosyncrasies in human language, as noted in \citep{ostendorf2016continuous}.
Currently, model interpolation, either linear or log-linear \citep{chen2019exploiting,wang2016model}, is often a second step, after the discrete and neural models are separately trained beforehand.
The interpolation weights are ad-hoc fixed or estimated over held-out data (different from the training data in the first step).
This two-step integration is sub-optimal.

The ELM approach can provide a unified and simplified approach in integrating both discrete and neural features, based on its capability in flexibly integrating a rich set of features.
\emph{Basically, with the ELM approach, one is free to define the potential function in any sensible way with much flexibility.}
It is straightforward to define a mixed-feature TRF-LM \citep{gao2020integrating}, in which the potential function is a sum of a linear potential using discrete features and a nonlinear potential using neural features. Mixed-feature TRF-LMs can be trained by applying the dynamic noise-contrastive estimation (DNCE) method \citep{wang2018improved}, as introduced in \secref{sce:DNCE}. 

\begin{figure}
	\centering
	\includegraphics[width=0.7\linewidth]{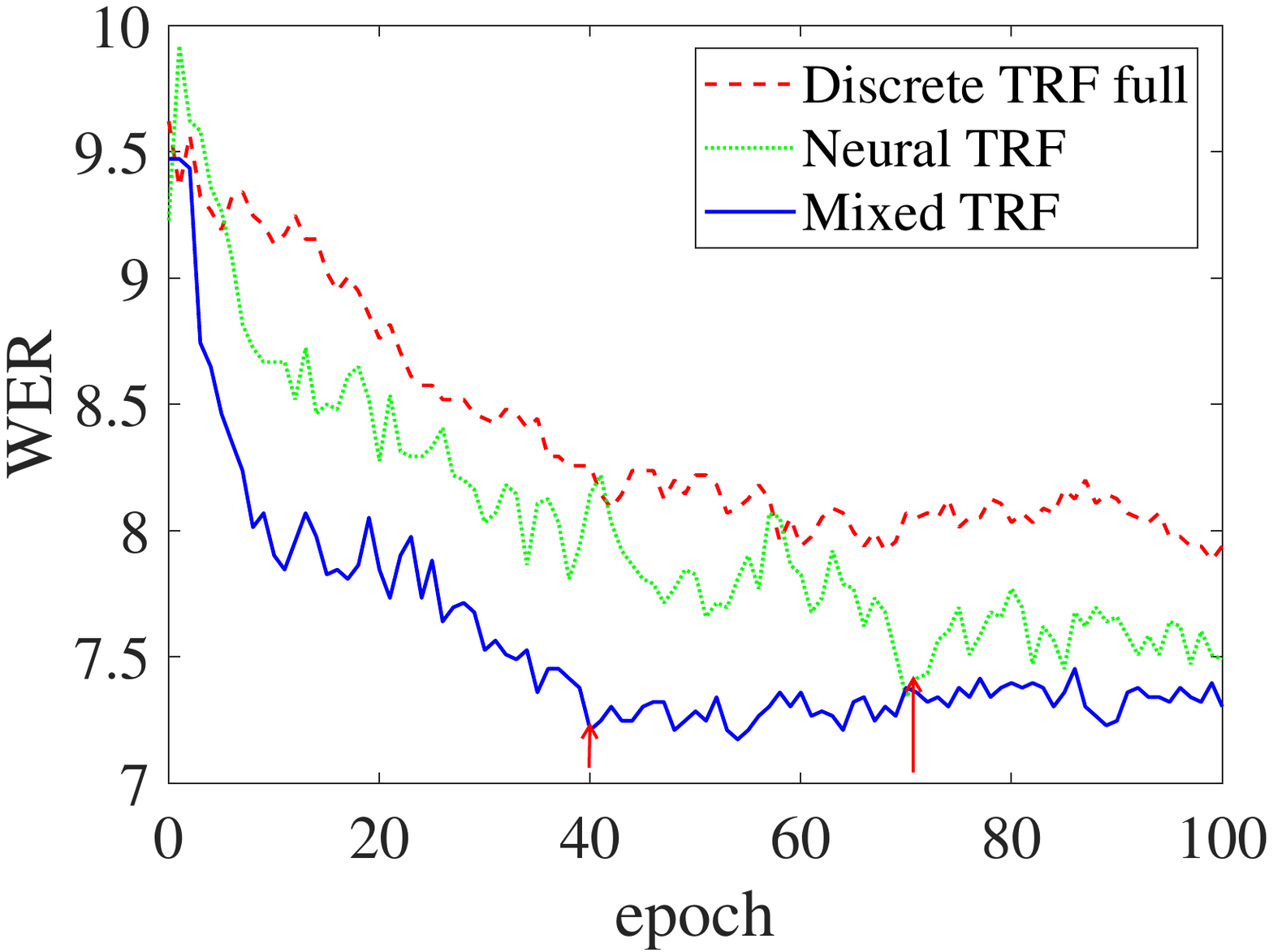}
	\caption{The WER curves of the three TRF-LMs during the first 100 training epochs are plotted. \citep{gao2020integrating}}
	\label{fig:mix_TRF_convergence}
\end{figure}

Mixed-feature TRF-LMs represent the first single LM model that incorporates both discrete and neural features without relying on a second-step interpolation.
Apart from naturally integrating discrete and neural features, another bonus from using mixed-feature TRF-LMs is that faster training convergence and shorter training time can be achieved, using only 58\% training epochs when compared to training neural TRF-LMs alone (see \figref{fig:mix_TRF_convergence}).
Notably, the log-likelihood of the training data with respect to (w.r.t.) the parameters of discrete features is concave. This helps to reduce the non-convexity of the optimization problem for maximum likelihood training.
Also, after incorporating the linear potential, the nonlinear potential only needs to capture the residual interactions between tokens. This may also explain the faster training convergence of mixed-feature TRF models.

In \citep{gao2020integrating}, various LMs are trained over the Wall Street Journal (WSJ) portion of Penn Treebank (PTB) \citep{marcus1993building} and Google one-billion-word corpus\footnote{\url{https://github.com/ciprian-chelba/1-billion-word-language-modeling-benchmark}}, and evaluated in N-best list rescoring experiments for speech recognition.
Among all single LMs (i.e. without model interpolation), the mixed-feature TRF-LMs perform the best, improving over both discrete TRF-LMs and neural TRF-LMs alone, and also being significantly better than LSTM ALMs. 
Compared to interpolating two separately trained models with discrete and neural features respectively, 
the performance of mixed-feature TRF-LMs matches the best interpolated model, and with simplified one-step training process and reduced training time.

\subsection{Residual ELMs}
\index{Residual ELM}\label{sec:residual_ELM}

As mentioned previously in this section, there are three types of energy functions based on discrete, neural and mixed features, respectively. Orthogonal to this\footnote{It means that the energy architectures described previously can be applied to $U_\theta(x)$ in residual ELMs.}, there is another architecture for EBMs, i.e., in the form of \emph{exponential tilting}\index{Exponential tilting} of a reference distribution. In the fields of lanugage modeling, this model formulation dates back to \cite{rosenfeld2001whole}, \cite{BinICASSP2018}, and recently \cite{deng2020residual}. Specifically, the probability of a sentence $x$ is defined as follows:
\begin{equation} \label{eq:residual_ELM}
p_\theta(x) = \frac{1}{Z(\theta)} q(x) e^{U_\theta(x)}
\end{equation}
where a reference distribution $q(x)$ is introduced as the baseline distribution; $U_\theta(x)$ denotes the residual potential function with parameter $\theta$; $Z(\theta)=\sum_x q(x) \exp(U_\theta(x))$ is the normalizing constant. Hence, \eqref{eq:residual_ELM} is called a \emph{residual ELM}\index{Residual ELM}, after \cite{deng2020residual}.

The role of residual energy $-U_\theta(x)$ is to fit the difference between the data distribution and the reference distribution.
If $q(x)$ is a good approximation of the data distribution, such as the LSTM based ALMs used in \citep{BinICASSP2018}, the Transformer based ALMs in \citep{deng2020residual},
fitting the difference between the data distribution and the reference distribution $q(x)$ shall be much simpler than fitting the data distribution directly.
Moreover, when the reference distribution $q(x)$ is used as the proposal distribution in Monte Carlo methods for learning a residual EBM, we have the importance weight in the following simple form:
\begin{displaymath}
\frac{p_\theta(x)}{q(x)} \propto e^{U_\theta(x)}
\end{displaymath}
Similarly, when the reference distribution $q(x)$ is used as the noise distribution in NCE methods for learning a residual EBM, we will have a simple form of \eqref{eq:NCE_post}.

\subsection{Training methods} 
\label{sec:TRF-train-methods}

Different types of ELMs, including GN-ELMs \eqref{eq:elm}, TRF-LMs \eqref{eq:trf-model} and residual ELMs \eqref{eq:residual_ELM}, all obeys the general from of general EBMs \eqref{eq:unsup-RF}.
Basically, we can use the methods introduced in \secref{sec:MLE} (MLE) and \secref{sec:NCE} (NCE) for learning general EBMs to train ELMs. 
A GN-ELM \eqref{eq:elm} has the same form as a EBM \eqref{eq:unsup-RF}, so those methods can be seamlessly applied. In TRF-LMs \eqref{eq:trf-model}, length probabilities are introduced and normalization is conducted for each length separately, some special care is needed.


For more details about MLE for learning TRF-LMs, refer to \citep{Wang2017LearningTR,wang2017language}.
For more details about DNCE for learning TRF-LMs, refer to \citep{wang2018improved,gao2020integrating}.

\section{Energy-based cloze models for representation learning over text} 
\label{sec:electric}

\subsubsection{Motivation}
In this section, we show the capability of EBMs for representation learning over text \citep{clark2020pre}.
The cloze task of predicting the identity of a token given its surrounding context has proven highly effective for representation learning over text. 
BERT \citep{bert}, as a representative \emph{masked language model} (MLM)\index{Masked language model (MLM)}, implements the cloze task by replacing input tokens with a special placeholder token \texttt{[MASK]}, but there are some drawbacks with the BERT approach:

\begin{figure}
	\centering
	\includegraphics[width=1\linewidth]{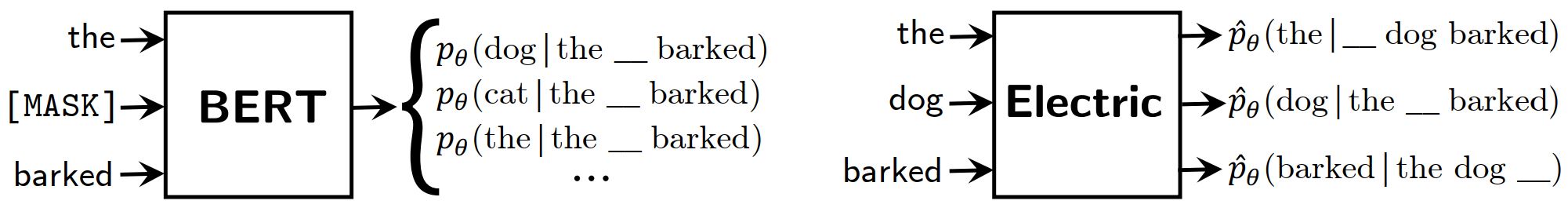}
	\caption{Comparison of BERT and Electric. Both model the conditional probability of a token given its surrounding context. BERT produces normalized conditional distribution for masked positions, while Electric calculates unnormalized conditional probabilities for all input tokens. \citep{clark2020pre}}
	\label{fig:Electric}
\end{figure}

\begin{itemize}
\item It suffers from the drawback in efficiency (only 15\% of tokens are masked out at a time);
\item It introduces a pre-train/fine-tune mismatch where BERT sees \texttt{[MASK]} tokens in training but not in fine-tuning;
\item In principle, it does not produce log-likelihoods for sentences (even up to an additive constant) and, in this sense, does not define a language model. The pseudo-log-likelihood (PLL) has been introduced for scoring sentences by BERT \citep{pll}, but calculating the PLL for a sentence $x$ requires $|x|$ passes of the transformer (once with each token masked out), and is thus computationally expensive.
\end{itemize}


Some details are as follows. BERT and related masked language models \citep{bert,Liu2019RoBERTaAR} train a large neural network to perform the cloze task.
In contrast to the standard language modeling task to learn the \emph{joint} probability $p_{\text{ora}}(x)$, 
these models try to learn the \emph{conditional} probabilities of masked tokens occurring in their surrounding context.
Specifically, multiple positions (e.g., 15\%) in the input sentence $x=(x_1,\cdots,x_l)$ are randomly drawn, denoted by $R=\{t_1, \cdots, t_k\}$, $1\le t_j \le l$, $j=1,\cdots,k$, and those positions are replaced by \texttt{[MASK]} (i.e., masked).
The masked sentence, denoted by $\text{mask}(x,R)$, is encoded into vector representations by a Transformer network \citep{vaswani2017attention}. Then the vector representation at position $t_j$ is passed into a softmax layer to calculate a distribution over the vocabulary for position $t_j, j=1,\cdots,k$. The conditional probabilities are maximized for learning the model parameters $\theta$:
\begin{displaymath}
\sum_{j=1}^k \log p_\theta(x_{t_j}| \text{mask}(x,R) )
\end{displaymath}

\subsubsection{The Electric model}

The new model proposed in \citep{clark2020pre}, called \emph{Electric}\index{Electric model}, is closely related to the Electra pre-training method \citep{clark2020electra}, and implements the cloze task based on using conditional EBMs.

Specifically, Electric does not use masking or a softmax layer.
Electric first maps the unmasked input $x=(x_1,\cdots,x_l)$ into contextualized vector representations $h(x)=(h_1,\cdots,h_l)$ using a Transformer network.
Then, the conditional probability of a token $x_t$ occurring in the surrounding context $x_{\setminus t} = (x_1,\cdots,x_{t-1},x_{t+1},\cdots,x_l)$ is modeled by a \emph{conditional EBM}:
\begin{equation}
\label{eq:electric}
p_{\theta}(x_t|x_{\setminus t}) = \frac{1}{Z_\theta(x_{\setminus t})} \exp(- w^T h_t), 1 \le t \le l
\end{equation}
where $w$ is learnable weight vector, and $h_t$ is the vector representation at position $t$.
A comparison of BERT and Electric is illustrated in \figref{fig:Electric}.

\subsubsection{Training of the Electric model}
\label{sec:electric_training}

The conditional EBMs defined in \eqref{eq:electric} can be trained using NCE (\secref{sec:NCE}), and more specifically its conditional version (\secref{sec:cond_nce}). First, we define the un-normalized output
\begin{displaymath}
\hat{p}_{\theta}(x_t|x_{\setminus t}) = \exp(- w^T h_t)
\end{displaymath}
Denote the training dataset by $\mathcal{D}$.
In NCE, a binary classifier is trained to distinguish positive token $x_t$ vs negative token $\hat{x}_t$, with $k$ negatives and $l$ positives (e.g., $k = \lceil 0.15 l \rceil $) for a sentence $x$ of length $l$. 
Denote the random positions by $R=\{t_1, \cdots, t_k\}$, $1\le t_j \le l, j=1,\cdots,k$. 
Formally, referring to \eqref{eq:NCE-obj}, the NCE loss $\mathcal{L}(\theta)$ is as follows:
\begin{equation}\label{eq:electric_loss}
\begin{split}
&l \cdot E_{x \sim \mathcal{D},t \sim \text{Uni}(1,l)} \left[ -\log\frac{l \cdot \hat{p}_{\theta}(x_t|x_{\setminus t})}{l \cdot \hat{p}_{\theta}(x_t|x_{\setminus t}) + k \cdot q(x_t|x_{\setminus t}) } \right]\\
+ &k \cdot E_{\begin{subarray}{l} x \sim \mathcal{D},t \sim R,\\ \hat{x}_t \sim q(\hat{x}_t|x_{\setminus t}) \end{subarray}} \left[ -\log\frac{k \cdot q(\hat{x}_t|x_{\setminus t})}{l \cdot \hat{p}_{\theta}(\hat{x}_t|x_{\setminus t}) + k \cdot q(\hat{x}_t|x_{\setminus t}) } \right]
\end{split}
\end{equation}
where $q(x_t|x_{\setminus t})$ denotes the noise distribution, which was realized by a two-tower cloze model \citep{clark2020pre}. 
Specifically, the noise model runs two transformers $T_{\text{LTR}}$ and $T_{\text{RTL}}$ over the input sequence. These transformers apply causal masking so one processes the sequence left-to-right (LTR) and the
other operates right-to-left (RTL). The model's predictions come from a softmax layer applied to the concatenated states of the two transformers:
\begin{align*}
&\overrightarrow{h} = T_{\text{LTR}}(x),  \overleftarrow{h} = T_{\text{RTL}}(x)\\
&q(x_t|x_{\setminus t}) = \text{softmax}(W [\overrightarrow{h}_{t-1}, \overleftarrow{h}_{t+1} ])_{x_t}
\end{align*}
The noise distribution is trained simultaneously with Electric using standard maximum likelihood estimation over the data. Thus, the training method used by Electric is in fact DNCE (\secref{sce:DNCE}).

Note that the NCE loss \eqref{eq:electric_loss} does not introduce additional parameters for normalizing constants $Z_\theta(x_{\setminus t})$. 
In the setting of nonparametric estimation (\thref{theorem:NCE}), the NCE loss is minimized when $\hat{p}_{\theta}(x_t|x_{\setminus t})$ matches the oracle density.
Thus, we assume that the model is of sufficient capacity and the model can learn to be self-normalized such that $Z_\theta(x_{\setminus t})=1$.

A further examination of \eqref{eq:electric_loss} reveals that its optimization is computationally expensive to run. It requires $k+1$ forward passes through the Transformer to compute the $\hat{p}_{\theta}$'s, once for the positive samples $x_t|x_{\setminus t}$ and once for every negative sample $\hat{x}_t|x_{\setminus t}$.
So a modified algorithm, allowing more efficient calculation, is to 
replace the $k$ random positions simultaneously by negative tokens $\hat{x}_t \sim q(x_t|x_{\setminus t})$. The resulting noised sequence is denoted by
\begin{displaymath}
x^{\text{noised}} = \text{replace}(x, R, (\hat{x}_{t_1},\cdots,\hat{x}_{t_k}))
\end{displaymath}
To apply this efficiency trick, Electric assumes $\hat{p}_{\theta}(\cdot|x_{\setminus t}) = \hat{p}_{\theta}(\cdot|x^{\text{noised}}_{\setminus t})$, for ``$\cdot$'' taking $x_t$ or $\hat{x}_t$, i.e., assuming that extra noise replacing does not change the conditional distribution much, for both positive and negative tokens.
The modified loss for a sentence $x$ of length $l$ becomes as follows\footnote{To compare with \eqref{eq:electric_loss}, this modified loss is written for $l-k$ positives and $k$ negatives, while \eqref{eq:electric_loss} is for $l$ positives and $k$ negatives.}:
\begin{equation}\label{eq:electric_loss_mod}
\begin{split}
&l \cdot E_{x \sim \mathcal{D},t \sim \text{Uni}(1,l)} \left[ -\log\frac{(l-k) \cdot \hat{p}_{\theta}(x_t|x^{\text{noised}}_{\setminus t})}{(l-k) \cdot \hat{p}_{\theta}(x_t|x^{\text{noised}}_{\setminus t}) + k \cdot q(x_t|x_{\setminus t}) } \right]\\
+ &k \cdot E_{\begin{subarray}{l} x \sim \mathcal{D},t \sim R,\\ \hat{x}_t \sim q(\hat{x}_t|x_{\setminus t}) \end{subarray}} \left[ -\log\frac{k \cdot q(\hat{x}_t|x_{\setminus t})}{(l-k) \cdot \hat{p}_{\theta}(\hat{x}_t|x^{\text{noised}}_{\setminus t}) + k \cdot q(\hat{x}_t|x_{\setminus t}) } \right]
\end{split}
\end{equation}
Calculating \eqref{eq:electric_loss_mod} requires just one pass through the Transformer for $k$ noise sample and $l-k$ data samples. This brings significant computation reduction.

\subsubsection{Performance of the Electric model}

Experiments in \citep{clark2020pre} on GLUE natural language understanding benchmark \citep{wang2018glue} and SQuAD question
answering dataset \citep{rajpurkar2016squad} show that Electric substantially outperforms BERT but slightly under-performs ELECTRA. 
However, Electric is particularly useful in its ability to efficiently produce pseudo-log-likelihood (PLL) scores for text, as defined below.
\begin{equation}\label{eq:PLL}
\text{PLL}(x) 
= \sum_{t=1}^l \log \hat{p}_{\theta}(x_t|x_{\setminus t}) 
= \sum_{t=1}^l - w^T h_t
\end{equation}

Experiments on the 960-hour LibriSpeech corpus \citep{panayotov2015librispeech} show Electric is better at re-ranking the outputs of a speech recognition system than GPT2 \citep{gpt2}  and is much faster at re-ranking than BERT because it scores all input tokens simultaneously rather than having to be run multiple times with different tokens masked out. 
It appears that EBMs are a promising alternative to the standard ALMs and MLMs currently used for language representation learning.

\chapter{Conditional EBMs with applications}
\label{ch:conditional}

In Chapter \ref{ch:lm}, we mainly introduce EBMs for modeling the \emph{marginal} distribution of natural language sentences,  with applications to language modeling for speech recognition and language representation learning.

In this chapter, we introduce EBMs for modeling \emph{conditional} distributions, i.e., using EBMs for conditional models (\secref{sec:cond}).
In \secref{sec:cond_basics}, we present basics for conditional EBMs, or equivalently, conditional random fields (CRFs).
In the following sections, we show how they can be applied in not only discriminative tasks such as speech recognition (\secref{sec:crf_asr}) and natural language labeling (\secref{sec:crf_labeling}) but also conditional text generation tasks (\secref{sec:cond_gen}).

\section{CRFs as conditional EBMs}
\label{sec:cond_basics}
\index{Contional EBM}
\index{Conditional random field (CRF)}

As introduced in (\secref{sec:cond}), many real-world applications are solved by conditional models.
\emph{Conditional random fields} (CRFs) \citep{lafferty2001conditional,sutton2012introduction} have been known to be one of the most successful conditional models, especially for sequence labeling.
A CRF is basically a conditional distribution $p(y|x)$ defined as a random field, or equivalently, an undirected graphical model. A formal definition is as follows. Similar to the equivalent meaning of random fields and EBMs in unconditional setting, conditional random fields and conditional EBMs are exchangeable terms.

\begin{figure}
	\centering
	\includegraphics[width=0.5\linewidth]{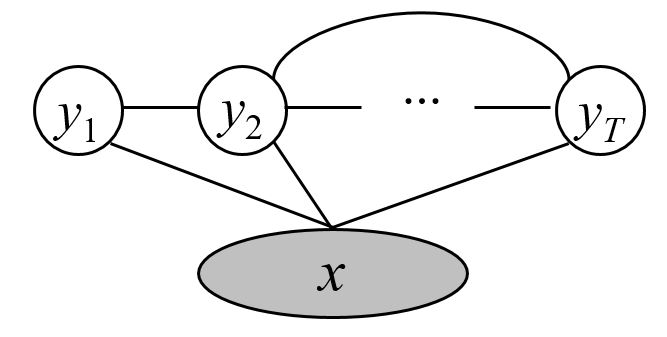}
	\caption{Graphical model representation of a conditional random field (CRF).}
	\label{fig:crf}
\end{figure}

\begin{definition}[Conditional random field] \label{def:crf}
Consider probability distributions $p(y|x)$, where $x$ represents an observation corresponding to the \emph{input} variable, and $y$ is the \emph{output} variable that we wish to predict.
A variable can either be scalar- or vector-valued.
Suppose $y=y_1,\cdots,y_T$.
In an undirected graph consisting of $x$ and $(y_1,\cdots,y_T)$ (e.g., as shown in \figref{fig:crf}), let $\mathcal{C}$ denote the set of cliques in the subgraph induced by $y$.
Associated with each clique $C \in \mathcal{C}$, let $\phi_C(y_C,x)$ denote a (log) potential function.
An conditional random field (CRF) in terms of this undirected graph  consists of a family of distributions that factorize as:
	\begin{equation} \label{eq:crf_basic}
	p(y|x) = \frac{1}{Z(x)} \prod_{C \in \mathcal{C}} \phi_C(y_C,x)
	\end{equation}
	where $Z(x)$ is the normalizing constant given by
	\begin{equation} \label{eq:crf_partition_func}
	Z(x) = \sum_{y} \prod_{C \in \mathcal{C}} \phi_C(y_C,x)
	\end{equation}
\end{definition}

\eqref{eq:crf_basic} actually means that for every assignment $x$, $p(y|x)$ factorizes according to the subgraph induced by $y$. Since $x$ is given, $x$ is treated as a constant, and included in the definition of potentials.

\subsection{Linear-chain CRFs}
\label{sec:linear_chain_crf}
\index{Linear-chain CRF}

Linear-chain CRFs are a class of CRFs, particularly useful for \emph{sequence labeling}\index{Sequence labeling} in the area of natural language processing (NLP), such as part-of-speech (POS) tagging \citep{fromscratch}, named entity recognition (NER) \citep{lample2016neural,ma2016end}, chunking \citep{sogaard2016deep} and syntactic parsing \citep{durrett2015neural}, and also in other areas such as bioinformatics \citep{Sato2005RNASS,peng2009conditional}.
Formally, in sequence labeling, the input $x$ is a sequence with the same length of the output $y$.
Thus, given a sequence of observations $x = (x_1,\cdots\,x_T) = x_{1:T}$, the task of sequence labeling is to predict a sequence of labels $y = (y_1,\cdots\,y_T) = y_{1:T}$, with one label for one observation in each position.
$y_i \in \left\lbrace 1,\cdots,K \right\rbrace$ denotes the label at position $i$.

A \emph{linear-chain CRF} defines a conditional distribution for label sequence $y$ given observation sequence $x$ in the following form:
\begin{equation}
\label{eq:linear_chain_crf}
p(y|x) \propto \exp \left\lbrace   \sum_{t=1}^T\phi_t(y_t,x)+
\sum_{t=1}^T \psi_t(y_{t-1},y_t,x) \right\rbrace.
\end{equation}
Here the labels $y_t$'s are structured to form a chain, giving the term linear-chain. Over the linear-chain, we define node potentials and edge potentials.
\begin{itemize}
\item $\phi_t(y_t,x)$ is often called the \emph{node potential}\index{Node potential} at position $t$.
Traditionally, people use \emph{discrete features}\index{Discrete feature}, as introduced in Example \ref{eg:word_morph} and \secref{sec:trf_discrete_feature}, to define node potentials.
The feature functions (or simply, the features) are usually hand-crafted indicator functions, which we think are useful for labeling, as shown below. 
\begin{align*}
f_1(y_t,x) = &\delta(y_t = \text{prep}, x_t = \text{on})\\
f_2(y_t,x) = &\delta(y_t = \text{adv}, x_t~\text{ends in ly})\\
&\cdots
\end{align*}
Then, the potential function at position $t$ can be defined as follows:
\begin{equation}
\phi_t(y_t,x) = \sum_i \lambda_i f_i(y_t, x)
\end{equation}
where $i$ indexes the different features. Notably, the potential functions $\phi_t(\cdot,x)$ defined above at different positions take identical forms, i.e.,  position-independent.
\item 
A recent progress is the development of \emph{Neural CRFs} (NCRFs)\index{Neural CRF (NCRF)}, which combines the sequence-level discriminative ability of CRFs and the representation ability of neural networks (NNs). 
In different studies, these models are called conditional neural field \citep{peng2009conditional}, neural CRF \citep{artieres2010neural}, recurrent CRF \citep{Mesnil2015UsingRN}, and LSTM-CRF \citep{lample2016neural,ma2016end}.
Though there are detailed differences between these existing models, generally they are all defined by using NNs (of different network architectures) to implement the non-linear node potentials in CRFs, while still keeping the linear-chain hidden structure, i.e., using a bigram table as the edge potential (to be detailed below). 
Suppose that the input sequence $x_{1:T}$ is transformed into vector representations $h_{1:T} = (h_1,\cdots\,h_T) \in \mathbb{R}^{T \times D}$ via an appropriate neural network. 
The vector representations $h_{1:T}$ can be viewed as features, and the neural network is referred as a \emph{feature extractor}\index{Feature extractor}, as discussed previously in \secref{sec:logistic_regression}.
Then, the node potential can be calculated via a linear layer on top of the vector representations:
\begin{equation}
\label{eq:neural_node_pot}
\phi_t(y_t = k,x) = w_k^T h_t + b_k \triangleq \phi_t^k, k=1,\cdots,K
\end{equation}
where $w_k \in \mathbb{R}^D, b_k \in \mathbb{R}$ denote the weight vector and bias of the linear layer, independent of $t$.
$\phi_t^k$ denotes the potential value at positon $t$ for label $k$.
\item 
$\psi_t(y_{t-1},y_t,x)$ is the \emph{edge potential}\index{Edge potential} defined on the edge connecting $y_{t-1}$ and $y_t$. It is mostly implemented as a transition matrix $A$:
\begin{equation}
\label{eq:edge_pot}
\psi_t(y_{t-1}=j,y_t=k,x) = A_{j,k}
\end{equation}
\end{itemize}

\subsection{Label bias and exposure bias}
\label{sec:bias}
\index{Label bias}
\index{Exposure bias}

It is known that \emph{locally-normalized sequence models}\index{Locally-normalized sequence model} are prone to label bias \citep{lafferty2001conditional,andor2016globally} and exposure bias \citep{wiseman2016sequence,ranzato2016sequence}.
Depending on locally-normalized or globally-normalized, unconditional or conditional, there are four classes of models, as overviewed in \tbref{tab:seq_model}.
\begin{itemize}
\item The label bias problem was raised initially in the conditional setting \citep{lafferty2001conditional}. 
Later in this section, we will show that the label bias problem also exist in the unconditional setting.
\item The exposure bias is related to the manner of model training, rather than caused by the model itself; so it exist in both unconditional and conditional settings.
\end{itemize}
In summary, the problems of label bias and exposure bias exist in both unconditional and conditional settings. 
An advantage of \emph{globally-normalized sequence models}\index{Globally-normalized sequence model} is that they avoid the problems of label bias and exposure bias, as explained below.

\subsubsection{The label bias problem}

\begin{table}
	\caption{A general classification of sequence models, with some common examples.} \label{tab:seq_model}
	\centering
	\begin{tabular}{l|l|l}
		\hline
		&   Unconditional & Conditional \\ 
		\hline \hline
		Locally-normalized & ALM & seq2seq  \\ 
		\hline
		Globally-normalized & ELM & CRF/Conditional EBM \\ 
		\hline
	\end{tabular}
\end{table}

Conditional, locally-normalized sequence models such as \emph{maximum entropy Markov models} (MEMMs)\index{Maximum entropy Markov model (MEMM)} \citep{mccallum2000maximum} and \emph{sequence-to-sequence models} (seq2seq)\index{Sequence-to-sequence model (seq2seq)} \citep{sutskever2014sequence} are potential victims of the label
bias problem.
In general, a conditional, locally-normalized sequence model is described by the conditional probability $p(y|x)$ with the following decomposition, where $x = (x_1,\cdots\,x_T) = x_{1:T}$ is an input sequence and $y = (y_1,\cdots\,y_L) = y_{1:L}$ is its corresponding output sequence whose length $L$ may differ from $T$.
\begin{equation} \label{eq:seq2seq}
p(y_{1:L}|x_{1:T}) = \prod_{i=1}^L p(y_i | x_{1:T},y_1,\cdots,y_{i-1}).
\end{equation}
The ouput probabilities at each time-step, $p(y_i | x_{1:T},y_1,\cdots,y_{i-1})$, are locally normalized, \emph{so successors of incorrect histories receive the same mass as do the successors of the true history} \citep{wiseman2016sequence}.
So locally normalized sequence models often have a very weak ability to revise earlier decisions \citep{andor2016globally}. This problem has been called the label bias problem.

The label bias problem is more severe when the output probability of $y_i$ can only depend on a partial input sequence $x_{1:t(i)}$, where $t(i)$ denotes the available length of the partial input when producing $y_i$.
For example, in \emph{streaming speech recognition}\index{Streaming speech recognition}, each token $y_i$ must be recognized shortly after it was spoken.
It is found in \citep{variani2022global} that by switching from a locally normalized model to a globally normalized model, the streaming speech recognition performance can be significantly improved.

In the case of using $p(y_i | x_{1:t(i)},y_1,\cdots,y_{i-1})$, the model becomes:
\begin{equation} \label{eq:seq2seq_2}
p(y_{1:L}|x_{1:T}) = \prod_{i=1}^L p(y_i | x_{1:t(i)},y_1,\cdots,y_{i-1}).
\end{equation}
The label bias problem can be understood by considering the independence assumption made by the model. In this case, we have $y_{1:i} \perp x_{t(i)+1:T} | x_{1:t(i)}$, i.e., 
\begin{displaymath}
p(y_{1:i}|x_{1:T}) = p(y_{1:i}|x_{1:t(i)})
\end{displaymath}
Thus, observations from later in the input sequence $x_{t(i)+1:T}$ has no effect on the posterior probability of history input $y_{1:i}$.
Intuitively, we would like the model to be able to revise an earlier decision made during search, when later evidence becomes available that rules out the earlier decision as incorrect \citep{andor2016globally}.

\begin{figure}
	\centering
	\includegraphics[width=0.9\linewidth]{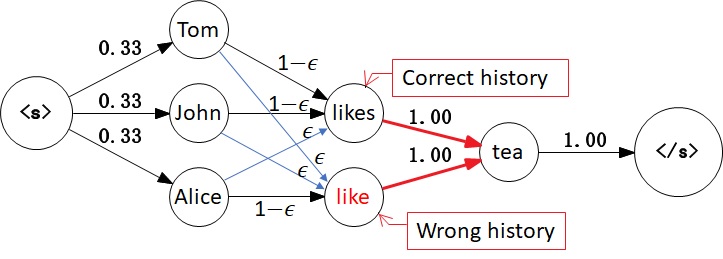}
	\caption{State transitions resulting from estimating an autoregressive language model from training data - ``Tom likes tea'', ``John likes tea'', and ``Alice like tea''.
	For some transitions not appeared in the training data, the transition probabilities are smoothed to take small values $\epsilon$.
	We pad the beginning and the end of a sentence with special tokens, $\langle \texttt{s} \rangle$ and $\langle /\texttt{s} \rangle$, respectively \citep{chen1999empirical}.}
	\label{fig:label_bias_directed}
\end{figure}

\begin{figure}
	\centering
	\includegraphics[width=0.9\linewidth]{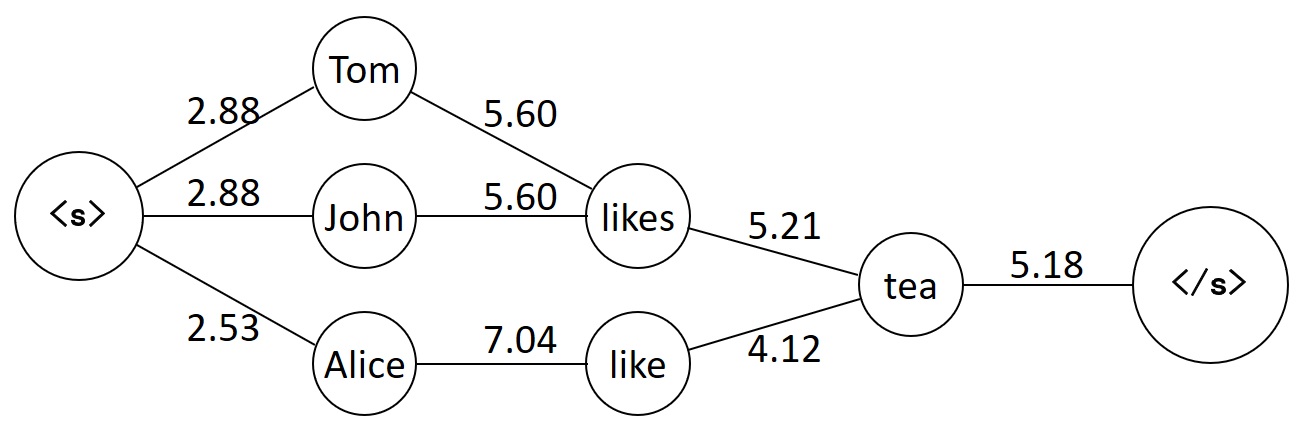}
	\caption{Estimating a globally-normalized energy-based language model (ELM) from training data - ``Tom likes tea'', ``John likes tea'', and ``Alice like tea''. The bi-gram features used by the ELM are similar to those used in the bigram ALM, and so can also be illustrated by a graph. The estimated parameters are shown over the  edges, which represent the corresponding bi-gram features.}
	\label{fig:label_bias_undirected}
\end{figure}

The above discussion of the label bias problem is for conditional, locally-normalized sequence models. 
In the following, we will show, through an empirical example\footnote{This English example is adapted from a Chinese example in Section 3.4.1 of \citep{BinWang2018thesis}'s.}, that the label bias problem causes trouble not only in \emph{conditional}, locally-normalized sequence models, but also in \emph{unconditional}, locally-normalized sequence models; and can be naturally overcome in globally-normalized sequence models.
Suppose that we would like to build a model for sentences with the following training data, which consist of three sentences:
\begin{align*}
&\text{Tom likes tea}\\
&\text{John likes tea}\\
&\text{Alice like tea}
\end{align*}
Since the training data are often noisy, there exist some samples with some incorrect, infrequent patterns such as in ``Alice like tea''.
We estimate different models over these training data, apply them to score test data
\begin{align*}
&\text{Alice likes tea}\\
&\text{Tom like tea}
\end{align*}
and examine the performance of different models.

Let us first consider an autoregressive language model (ALM), which is a typical unconditional, locally-normalized sequence model, trained over these training data.
We estimate a smoothed bi-gram ALM, i.e., we smooth the transition probabilities for unseen bi-grams. 
The resulting ALM with estimated transition probabilities are shown in \figref{fig:label_bias_directed}.
We can see that transitions from the wrong history (``like'') and the correct history (``likes'') to ``tea'' get the same score, due to local normalization of conditional probabilities.
Consequently, when scoring test data - ``Alice likes tea'' and ``Tom like tea'', the bi-gram ALM cannot score them correctly - the two test samples are scored as being of the same probability:
\begin{align*}
P(\text{Alice likes tea}) &= P(\text{Alice}|\langle \texttt{s} \rangle) P(\text{likes}|\text{Alice}) P(\text{tea}|\text{likes}) P(\langle \texttt{s} \rangle|\text{tea})\\
&= 0.33 \times \epsilon  \times 1.0 \times 1.0 = 0.33 \epsilon\\
P(\text{Tom like tea}) &= P(\text{Tom}|\langle \texttt{s} \rangle) P(\text{like}|\text{Tom}) P(\text{tea}|\text{like}) P(\langle \texttt{s} \rangle|\text{tea})\\
&= 0.33 \times \epsilon  \times 1.0 \times 1.0 = 0.33 \epsilon
\end{align*}
In fact, ``Alice likes tea'' should be more likely than ``Tom like tea'', considering that the correct form of verb appears more often than the incorrect form in the training data. We can see that the bi-gram ALM seems not to be a good model choice for this example, due to local normalization. To be more precise, this label bias problem is caused by \emph{the improper local normalization following incorrect histories}.
Incorrect histories occur in training data, as part of noise. Successors after incorrect histories are in fact biased labels, even they appear in training data. This issue should be somehow fixed in order to better model the regularities in the data\footnote{We usually do not assume that the data generating distribution is exactly the same as the empirical distribution, but rather the modeling task is to seek a smoothed distribution based on the empirical distribution (See discussion of the concept of learning in \secref{sec:probabilistic_approach}).}. However, it is worthwhile to remark that how adverse this label bias issue will cause depends on the particular data and model under investigation. Further theoretical analysis will be interesting future work.

Next, let us consider a globally-normalized energy-based language model (GN-ELM) (\secref{sec:GN-ELM}), trained over the same training data, but with a different model assumption from ALMs. The globally-normalized model, which is defined as follows, also uses the bi-gram features, similar to those used in the bi-gram ALM:
\begin{equation}\label{eq:label_bias_ELM}
P(x_{1}, x_2, x_{3}) = \frac{1}{Z} \exp \left(  \sum_{k=1}^9 \lambda_k f_k(x_{1}, x_2, x_{3}) \right) 
\end{equation}
where the bi-gram features are indicator functions as follows:
\begin{align*}
&f_1 = \delta(x_1=\text{Tom}),\\
&f_2 = \delta(x_1=\text{John}),\\
&f_3 = \delta(x_1=\text{Alice}),\\
&f_4 = \delta(x_1=\text{Tom},x_2=\text{likes})), \\
&f_5 = \delta(x_1=\text{John},x_2=\text{likes}), \\
&f_6 = \delta(x_1=\text{Alice},x_2=\text{like}), \\
&f_7 = \delta(x_2=\text{likes}, x_3=\text{tea}), \\
&f_8 = \delta(x_2=\text{like}, x_3=\text{tea}), \\
&f_9 = \delta(x_3=\text{tea}).
\end{align*}
Maximum likelihood estimate (MLE) of  \eqref{eq:label_bias_ELM} can be performed by the stochastic approximation (SA) method or the improved iterative scaling (IIS) method, as introduced in \citep{Wang2017LearningTR}.
The estimated parameters, $\{\lambda_k,k=1,\cdots,9\}$, are shown in \figref{fig:label_bias_undirected}. 
We show the parameters over the edges, which represent the corresponding bi-gram features.
The estimated ELM model can be used to score test data - ``Alice likes tea'' and ``Tom like tea'':
\begin{align*}
\log P(\text{Alice likes tea}) &=2.53+5.21+5.18-\log Z=-7.06\\
\log P(\text{Tom like tea}) &=2.88+4.12+5.18-\log Z=-7.79
\end{align*}
where the log normalizing constant, in this simple example, can be directly calculated as follows:
\begin{displaymath}
\log Z = \log \sum_{
	\substack{
	x_1 \in \{\text{Tom,John,Alice}\},\\x_2 \in \{\text{likes, like} \}, x3=\text{tea}   }} \exp \left(  \sum_{k=1}^9 \lambda_k f_k(x_{1}, x_2, x_{3}) \right) =19.98
\end{displaymath}
Interestingly, according to the ELM modeling, ``Alice likes tea'' is naturally scored as being more probable than ``Tom like tea'', without relying on any other ad-hoc tricks. It seems that the ELM model is smoothed more correctly (not disturbed by the noise in the training data) - this is a desirable result.
Compare to the ALM modeling, the ELM modeling abstains from using local normalization. There is \emph{no improper local normalization following incorrect histories}. In this sense, we could say that the label bias issue is avoided.


\subsubsection{The exposure bias problem}

\begin{figure}
	\centering
	\includegraphics[width=0.9\linewidth]{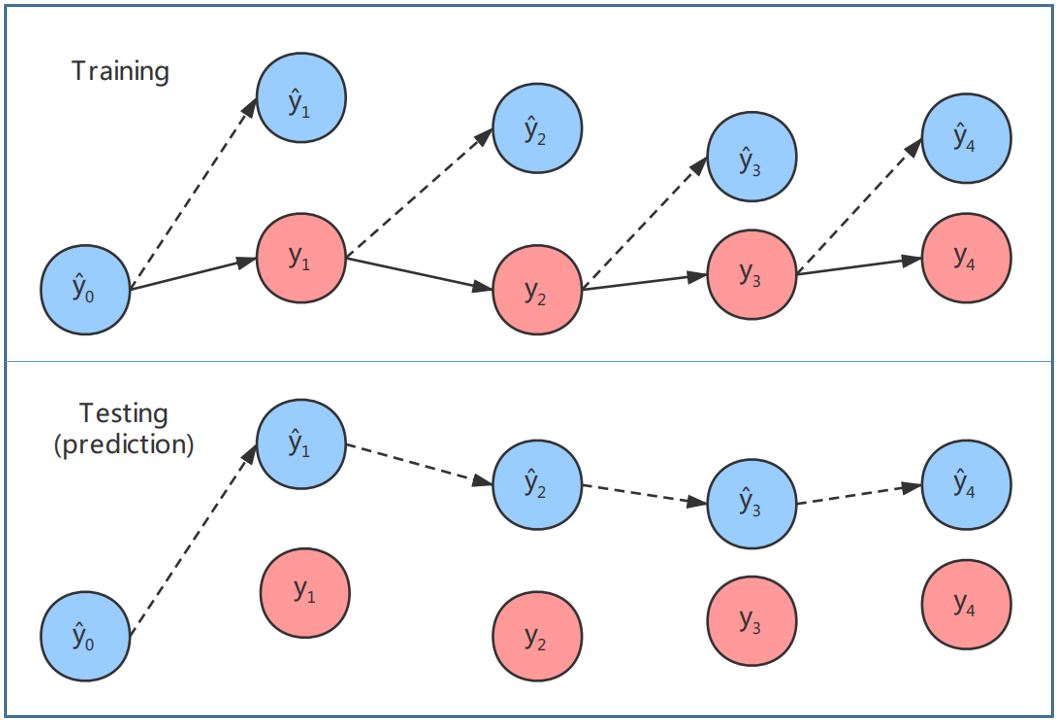}
	\caption{Illustration of exposure bias. $y$: real, $\hat{y}$: predicted.}
	\label{fig:Exposure bias}
\end{figure}

Locally-normalized sequence models, whether conditional or unconditional, suffers from the exposure bias problem. Locally-normalized sequence models are usually trained in a manner called \emph{teacher forcing}\index{Teacher forcing} \citep{williams1989learning}.
Take the training of a ALM as an example. 
\begin{itemize}
	\item In training, the model maximizes the likelihood of each successive target word, conditioned on the gold history of the target word.
	As shown in \figref{fig:Exposure bias}, in training, the model is only exposed to real data, which predicts $\hat{y}_i$ given the real data of the history, $y_1,\cdots,y_{i-1}$.
	\item In testing, the model predicts the next step $\hat{y}_i$, using its own predicted words in testing, i.e., $\hat{y}_1,\cdots,\hat{y}_{i-1}$.
\end{itemize}
\emph{Such mismatch between training (teacher forcing) and testing (prediction) of locally-normalized sequence models} has been called the exposure bias problem.
The model is never exposed to its own errors during training, and so the inferred histories at test-time do not resemble the gold training histories \citep{wiseman2016sequence}.

Remarkably, exposure bias results from training in a certain way, which maybe alleviated by some ad-hoc methods such as \emph{scheduled sampling}\index{Scheduled sampling} \citep{bengio2015scheduled}. In contrast, label bias results from properties of the model itself.
Thus, it may be more difficult to overcome label bias than to avoid exposure bias.

\subsection{Training of CRFs}
\label{sec:crf_training}

To introduce the training methods for CRFs/conditional EBMs, we write the model in Definition \ref{def:crf} in a simpler form, with input $x$ and output $y$:
\begin{equation}\label{eq:crf_def2}
p_{\theta}(y|x)=\frac{1}{Z_\theta(x)} \exp\left[  U_{\theta}(x,y) \right] 
\end{equation}
$U_{\theta}(x,y) : \mathcal{X}\times \mathcal{Y} \rightarrow \mathbb{R}$ denotes the (log) potential function, which assigns a scalar value to each combined configuration of $x$ in $\mathcal{X}$ and $y$ in $\mathcal{Y}$, and can be very flexibly parameterized through neural networks of different architectures.
$\mathcal{X}$ and $\mathcal{Y}$ denotes the space of all possible values of $x$ and $y$, respectively.
Normalizing is taken only over $\mathcal{Y}$, and $Z_\theta(x)$ denotes the normalizing constant:
\begin{equation}\label{eq:Z_crf}
Z_\theta(x)=\sum_{y \in \mathcal{Y}}\exp\left[  U_{\theta}(x,y) \right]
\end{equation}

\subsubsection{Learning CRFs by conditional maximum likelihood (CML)}
\index{Conditional maximum likelihood (CML)}

Suppose we have a training dataset consisting of $N$ independent and identically distributed (IID) data points $\mathcal{D} = \left\lbrace (x_i,y_i), i=1, \cdots, N \right\rbrace$.
We can fit $p_{\theta}(x)$ to data by maximizing the log conditional likelihood of training data, defined by
\begin{equation}
\label{eq:conditional_ebm_loglik}
L(\theta) \triangleq \frac{1}{N} \sum_{i=1}^N \log p_{\theta}(y_i|x_i)
= \frac{1}{N} \sum_{i=1}^N \left\lbrace  U_{\theta}(x_i,y_i) - \log Z_\theta(x_i) \right\rbrace 
\end{equation}
as a function of $\theta$.

Taking the derivative of the log conditional likelihood with respect to $\theta$ and making use of \eqref{eq:grad_log_Z_equal_expect} about the derivative of the log normalizing constant, we obtain the core formula in learning conditional EBMs:
\begin{equation} \label{eq:conditional_ebm_grad}
\nabla_\theta L(\theta) 
= \frac{1}{N} \sum_{i=1}^N \left\lbrace 
\nabla_\theta U_{\theta}(x_i,y_i) - 
E_{p_{\theta}(y|x_i)} \left[ \nabla_\theta U_{\theta}(x_i,y) \right] 
\right\rbrace 
\end{equation}

The maximum likelihood estimate of $\theta$ is obtained as a solution to $ \nabla_\theta L(\theta) = 0$. 
It can be easily seen that the gradients $\nabla_\theta L(\theta)$ in \eqref{eq:conditional_ebm_grad} in learning conditional EBMs exactly follows the form of \eqref{eq:SA}, as summarized in \thref{th:ebm_grad_SA_form}.
So the problem of learning conditional EBMs by \emph{conditional maximum likelihood} (CML) can then be solved by setting the gradients to zeros and applying the SA algorithm to finding the root for the resulting system of simultaneous equations. 
\emph{Techniques in learning (unconditional) EBMs are applicable in learning conditional EBMs here.} See \secref{sec:SA_ebm} for details. For example:
\begin{itemize}
\item Minibatching from training data when $N$ is large;
\item When calculating the model expectation $E_{p_{\theta}(y|x_i)} \left[ \nabla_\theta U_{\theta}(x_i,y) \right] $ is intractable, we can resort to Monte Carlo methods and conduct Monte Carlo averaging.
\item The model expectation can be exactly calculated under some limited circumstances, mostly in low tree-width\footnote{
	The tree-width of a graph is defined as the minimum width over all possible tree decompositions of the graph, which measures roughly how close the graph is to a tree. The \emph{tree-width}\index{Tree-width} of a tree is 1.
}
random fields (e.g., chain-structured) with moderately sized \emph{state spaces}\index{State space}\footnote{The state space of a multivariate model is the set of all possible values for each coordinate of a multivariate observation.}.
\end{itemize}

\subsubsection{CML training of neural linear-chain CRFs}
Here we provide more details about CML training of a neural linear-chain CRF, for which the model expectation can be exactly calculated.
Continue with the notations in \secref{sec:linear_chain_crf}, and consider a neural linear-chain CRF as follows, with input $x = (x_1,\cdots\,x_T) = x_{1:T}$, output $y = (y_1,\cdots\,y_T) = y_{1:T}$. Combining \eqref{eq:linear_chain_crf}, \eqref{eq:neural_node_pot} and \eqref{eq:edge_pot}, we obtain
\begin{align}
\label{eq:neural_linear_chain}
p_{\theta}(y|x)&=\frac{1}{Z_\theta(x)} \exp\left[  U_{\theta}(x,y) \right] \nonumber \\
U_{\theta}(x,y) &= \sum_{t=1}^T \left[ \phi_t^{y_t} + A_{y_{t-1},y_t} \right] 
\end{align}
where the parameters $\theta$ consists of the network parameters and the transition matrix $A$.

Suppose the state space of $y_{1:T}$ is $\left\lbrace 1,\cdots,K \right\rbrace$, i.e., $y_i \in \left\lbrace 1,\cdots,K \right\rbrace$.
Thus, at each position $t=1,\cdots,T$, there are $K$ node potential values, $\phi_t^{k}$, $k=1,\cdots,K$, which are calculated by a neural network, as defined in \eqref{eq:neural_node_pot}. 
The network is trained to optimize the log conditional likelihood \eqref{eq:conditional_ebm_loglik}.
The gradient of the log conditional likelihood w.r.t. $\theta$ for a single data point $(x_i, y_i)$ is:
\begin{equation} \label{eq:conditional_ebm_grad_single}
\frac{\partial \log p_\theta(y_i|x_i)}{\partial \theta}
= \frac{\partial U_{\theta}(x_i,y_i)}{\partial \theta}
 - E_{p_{\theta}(y|x_i)} \left[ \frac{\partial U_{\theta}(x_i,y)}{\partial \theta} \right] 
\end{equation}

Note that $\log p_\theta(y_i|x_i)$ depends on the network parameters through $\phi_t^{y_t}$. Thus, an important quantity in calculating \eqref{eq:conditional_ebm_grad_single} for the network parameters is the gradient w.r.t. to the potential values, which can be derived from \eqref{eq:neural_linear_chain} as follows:
\begin{align*}
\frac{\partial \log p_\theta(y_i|x_i)}{\partial \phi_t^k} &= \delta(y_t=k) - E_{p_{\theta}(y|x_i)} \left[ \delta(y_t=k) \right]\\
&= \delta(y_t=k) - p(y_t=k)
\end{align*}
This difference between the empirical count and the expected count is the error signal received by the NN based feature extractor during training.
The expected count $p(y_t=k)$ is often known as the posterior state occupation probability, which can be calculated using the alpha-beta variables from the forward-backward algorithm \citep{rabiner1989tutorial}. 
The gradients for the network parameters can be calculated from the gradient w.r.t. to the potential values based on the back-propagation procedure.

\subsubsection{Learning CRFs by conditional NCE}
\label{sec:cond_nce}
\index{Conditional NCE}

A theoretical analysis of NCE in the estimation of conditional unnormalized models in the form of \eqref{eq:crf_def2} has been given in \citep{ma2018noise}.
A subtle but important question when generalizing NCE to the conditional case is raised, i.e., the unconditional EBM in \eqref{eq:unsup-RF} has a single partition function, which is estimated as a parameter of the model, whereas the conditional model in \eqref{eq:crf_def2} has a separate partition function $Z_\theta(x)$ for each value of $x$.

Introduce the unnormalized density $\tilde{p}_{\theta}(y|x) = \exp\left[  U_{\theta}(x,y) \right]$ and rewrite \eqref{eq:crf_def2} as follows:
\begin{equation}\label{eq:crf_def3}
p_{\theta}(y|x)=\frac{1}{Z_\theta(x)} \tilde{p}_{\theta}(y|x)
\end{equation}
where, for each value of $x$, the partition function is defined by
\begin{displaymath}\label{eq:Z_crf2}
Z_\theta(x) = \sum_{y \in \mathcal{Y}} \tilde{p}_{\theta}(y|x)
\end{displaymath}

Applying NCE in estimating conditional unnormalized models yields the \emph{conditional NCE} method. Suppose we have a training dataset consisting of $N$ independent and identically distributed (IID) data points $\mathcal{D} = \left\lbrace (x_i,y_i), i=1, \cdots, N \right\rbrace$.
The basic idea of NCE is to perform nonlinear logistic regression to discriminate between \emph{data samples} drawn from the data and \emph{noise samples} drawn from a known noise distribution $q(x|y)$.
Formally, referring to \eqref{eq:NCE-obj}, conditional NCE estimates the model parameters $\theta$ by maximizing the following objective function:
\begin{align*}
J(\theta) = E_{\kappa \sim Uni(1,N)}& \left\lbrace  \log \frac{\tilde{p}_{\theta}(y_\kappa|x_\kappa)}{\tilde{p}_{\theta}(y_\kappa|x_\kappa) + \nu q(y_\kappa|x_\kappa) } \right.  \\
 + & \left. \nu E_{y \sim q(y|x_\kappa)}  \left(   \log \frac{\nu q(y|x_\kappa)}{\tilde{p}_{\theta}(y|x_\kappa) + \nu q(y|x_\kappa)} \right) \right\rbrace 
\end{align*}
where $\nu$ represents the ratio of noise sample size to real sample size.

It is shown in \citep{ma2018noise} that the conditional NCE method gives consistent parameter estimates
under the assumption that $Z_\theta(x)$ is constant with respect to $x$. 
Equivalently, the conditional NCE method is consistent under the assumption that the function $\tilde{p}_{\theta}(y|x)$ is powerful enough to incorporate $Z_\theta(x)$.
Thus, under the assumption that the model is of sufficient capacity, 
the model can learn to be self-normalized such that $Z_\theta(x)=1$.
Remarkably, the discussion in \citep{ma2018noise} is in line with \thref{theorem:NCE} about NCE (nonparametric estimation), since in the setting of nonparametric estimation, the NCE loss is minimized when $\tilde{p}_{\theta}(y|x)$ matches the oracle density and thus is self-normalized.

As a final note, it can be easily seen that training of the Electric model in \secref{sec:electric_training} is exactly an instance of the conditional NCE method.

\section{CRFs for speech recognition}
\label{sec:crf_asr}

In this section, we show the development of CRFs for \emph{automatic speech recognition}\index{Automatic speech recognition (ASR)} (ASR). First, we introduce some background knowledge in ASR, particularly the \emph{connectionist temporal classification} (CTC) method for recent end-to-end ASR.
Then, we elaborate on the recently developed approach of CRF-based single-stage acoustic modeling with CTC topology \citep{CTCCRF_IC19,CAT}.

ASR basically is a sequence discriminative problem.
Given acoustic observations $x = (x_1,\cdots\,x_T) = x_{1:T}$, the task of ASR is to find the most likely labels $y = (y_1,\cdots\,y_L) = y_{1:L}$.
Acoustic observations are usually spectral feature vectors.
Different units can be used for labels, such as	phone (namely monophone), triphone, character, wordpiece and word, as shown in \figref{fig:asr_label_unit}.

\begin{figure}[t]
	\centering
	\includegraphics[width=1.0\linewidth]{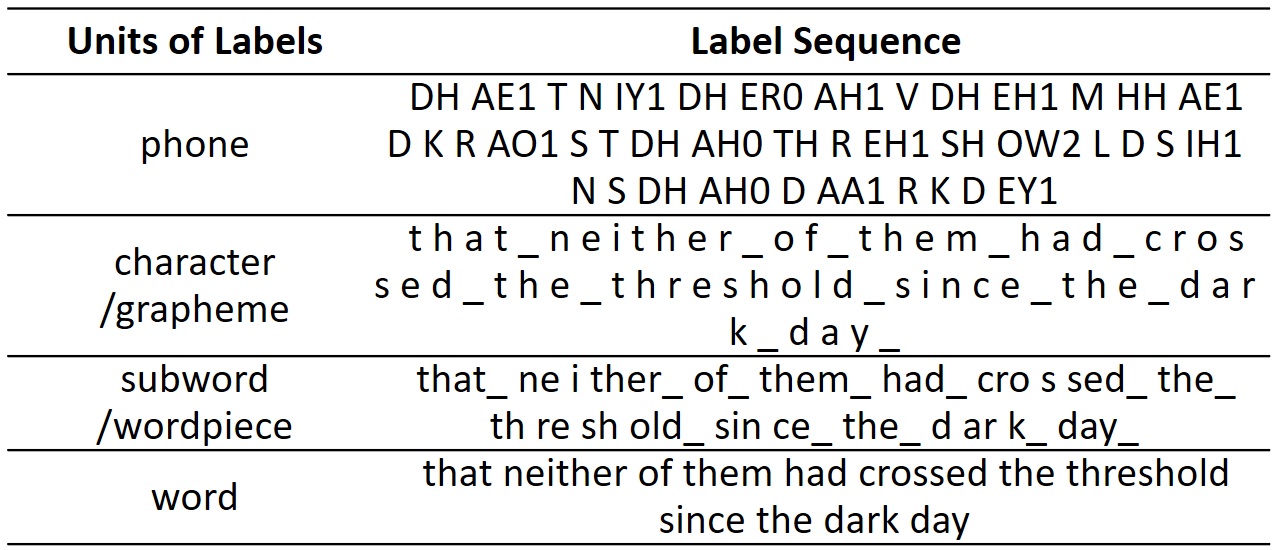}
	\caption{Different units of labels can be used in speech recognition.}
	\label{fig:asr_label_unit}
\end{figure}
	
\subsection{Connectionist Temporal Classification (CTC)}
\label{sec:ctc}

\subsubsection{The DNN-HMM hybrid approach}
\index{DNN-HMM hybrid}

The history of speech recognition dates back to 1970s \citep{jelinek1976continuous} and even earlier. The classic model for speech recognition is HMMs, as we review in \secref{sec:hmm}.
In building an HMM for speech recognition, consisting of acoustic sequence $x = (x_1,\cdots\,x_T) = x_{1:T}$ and state sequence $\pi = (\pi_1,\cdots\,\pi_T) = \pi_{1:T}$, we need to specify two distributions, the state-transition distribution and the state-output distribution.

The state sequence $\pi$ forms a Markov chain. In terminology of Markov chains, state transitions in a Markov chain are determined by its \emph{state topology}\index{State topology of a Markov chain}, or say, its \emph{state transition graph}\index{State transition graph of a Markov chain}.
The state transition graph is determined by combining some knowledge sources, as illustrated in \figref{fig:ASR_HMM_state_transitions}.
For example, we need a language model to model how words are connected to form a sentence. We need a lexicon to model how phones are connected to form a word.
Depending on left and right phonetic contexts, we can define different context-dependent phones, e.g., triphones.
Each context-dependent phone can be decomposed into several phonetic states, which intuitively correspond to the beginning, steady, closing phases of a phone.
The resulting state transition graph, often represented by a \emph{weighted finite-state transducer}\index{Weighted finite-state transducer (WFST)} (WFST), encodes common constraints in human speech\footnote{For subtle differences between the two graphs (WFSTs and state transition graphs), readers can refer to \citep{ou2010study}.}. 
In practice, each knowledge source can be represented by a component WFST, and the combination of several knowledge sources can be easily realized by applying the composition and some compression operations to the component WFSTs and producing a single integrated WFST \citep{mohri2008speech}.

\begin{figure}
	\centering
	\includegraphics[width=1.0\linewidth]{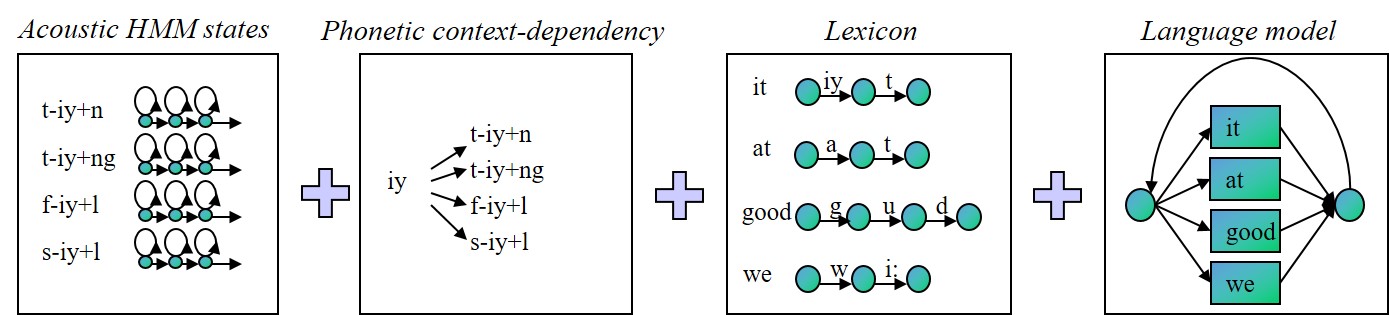}
	\caption{State transitions in HMMs for speech recognition are constrained by a number of knowledge sources.}
	\label{fig:ASR_HMM_state_transitions}
\end{figure}

Formally, a WFST represents a weighted relation between sequences of input symbols and sequences of output symbols \citep{mohri2008speech}.
A state sequence $\pi$ for an utterance of $T$ frames can be seen as a path traversing in the integrated WFST for $T$ steps, with $\pi = (\pi_1,\cdots\,\pi_T)$ matching the input symbols and the emitted output symbols corresponding to $y = (y_1,\cdots\,y_L)$.
It can be seen that in this way, a path $\pi$ uniquely determines a label sequence $y$, but not vice versa.
Thus, based on the integrated WFST, a many-to-one mapping, $\mathcal{B}_{\text{HMM}}: \pi \rightarrow y$, is defined.
The state topology basically determines the mapping. We also say that the mapping represents the state topology.

For state-output distribution $p(x_{t} | \pi_{t} )$, Gaussian mixture models (GMMs) were commonly used in the so-called GMM-HMM approach, before the deep learning era.
Recently, deep neural networks (DNNs)\index{Deep neural network (DNN)} of various architectures have become dominantly used for modeling the state-output distributions.
Based on the following Bayesian formula, the state likelihood $p(x_{t} | \pi_{t} )$ can be calculated from the state posteriori $p(\pi_{t} | x_{t} )$, divided by the state prior $p(\pi_{t})$, while the marginal likelihood $p(x_{t})$, being a constant, can be ignored.
\begin{displaymath}
p(x_{t} | \pi_{t} ) = \frac{ p(\pi_{t} | x_{t} ) p(x_{t})}{p(\pi_{t}) }
\end{displaymath}
State prior probabilities are estimated from the training data.
State posterior probabilities are calculated from the DNN, but we need frame-level alignments in training of the DNN.

In summary, the approach described above is often referred to as the \emph{DNN-HMM hybrid}\index{DNN-HMM hybrid} approach for ASR \citep{dahl2012context}.
It is featured by using the frame-level loss (cross-entropy) to train the DNN to estimate the posterior probabilities of HMM states.

\subsubsection{End-to-end ASR}


Notably, building DNN-HMM systems for ASR runs in multi-stages.
A GMM-HMM training is firstly needed to obtain frame-level alignments and then the DNN-HMM is trained.
The hybrid approach usually consists of an DNN-HMM based acoustic model (AM)\index{Acoustic model (AM)}, a state-tying decision tree for context-dependent phone modeling, a pronunciation lexicon and a language model (LM), which can be compactly combined into a weighted finite-state transducer (WFST).
The WFST not only operationally represents the state topology, but also is useful as a compact data structure for efficient decoding.

A recent trend in ASR is to develop end-to-end ASR models \citep{graves2006connectionist,Miao2015EESEN,graves2012sequence,chorowski2014end}. 
The end-to-end approach is characterized by eliminating the construction of GMM-HMMs and phonetic decision-trees, training the DNN from scratch (in single-stage) and, even ambitiously, removing the need for a pronunciation lexicon and training the acoustic and language models jointly rather than separately.
For developing end-to-end ASR models, there are two main issues. 
\begin{itemize}
\item The first is how to handle alignment between $x$ and $y$, since the two sequences generally differ in length ($T \neq L$).
\item The second is how to obtain $p(y|x)$, the posteriori probability of the label sequence $y$ given the acoustic sequence $x$. To be end-to-end, we need a differentiable sequence-level loss of mapping the acoustic sequence $x$ to the label sequence $y$.
\end{itemize}

Three widely-used end-to-end models are based on connectionist temporal classification (CTC) \citep{graves2006connectionist}, RNN-transducer (RNN-T)  \citep{graves2012sequence}, and attention based encoder-decoder (AED) \citep{chorowski2014end} respectively. The three models are featured by different losses. In CTC and RNN-T, the alignment is handled explicitly by introducing hidden state sequence $\pi$, while AED implements implicit, soft alignment via the attention mechanism.

Remarkably, all the three models are (conditional) locally normalized sequence models, according to the classification of sequence models as shown in \tbref{tab:seq_model}. In \secref{sec:ctc-crf}, we will introduce a recently developed (conditional) globally normalized sequence model, called CTC-CRF, for end-to-end ASR.

\subsubsection{The CTC method}

\begin{figure}
	\centering
	\includegraphics[width=1.0\linewidth]{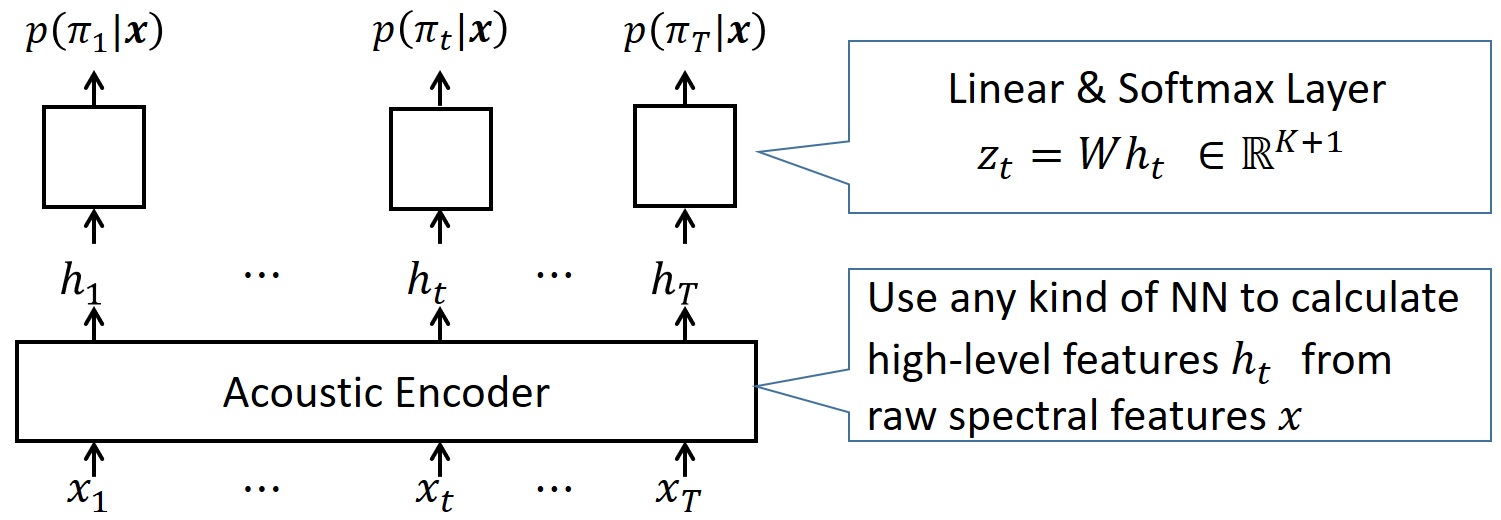}
	\caption{Overview of CTC architecture.}
	\label{fig:CTC_overview}
\end{figure}

The motivation of CTC \citep{graves2006connectionist} is to enable the training of $p(y|x)$ without requiring the frame-level alignments between the acoustics $x$ and the transcripts $y$. 
To this end, CTC introduces a \emph{blank} symbol <b> in addition to the ordinary labels, and further introduces a state sequence $\pi = (\pi_1,\cdots\,\pi_T) = \pi_{1:T}$, which aids the aligning between $x_{1:T}$ and $y_{1:L}$.
See \figref{fig:CTC_overview} for an overview of CTC architecture.

\paragraph{Handle alignment.}
Given acoustic sequence $x_{1:T}$, at each frame $t$, the possible values that $\pi_t$ can freely take is $\mathcal{Y} \cup \text{<b>}$, where $\mathcal{Y}$ denotes the the alphabet of ordinary labels.
Suppose the alphabet size is $K$, then number of possible values that $\pi_t$ can freely take is $K+1$.
When given $y_{1:L}$, the state transitions followed by $\pi$ is constrained by the state topology, so that the output sequence is $y_{1:L}$.
CTC defines a special state topology, which is enforced by a special many-to-one mapping $\mathcal{B}_{\text{CTC}}: \pi_{1:T} \rightarrow y_{1:L}$. The mapping  $\mathcal{B}_{\text{CTC}}$ is defined by reducing repetitive symbols in $\pi$ to a single symbol, and removing all blank symbols, e.g., 
\begin{displaymath}
\mathcal{B}_{\text{CTC}}(-\text{CC}--\text{AA}-\text{T}-) = \text{CAT}
\end{displaymath}
It can be seen that for given acoustic sequence $x_{1:T}$ and label sequence $y_{1:L}$, all possible alignments between them can be organized in a lattice, as shown in \figref{fig:CTC}.
A \emph{path}\index{Path in CTC} from the top left to the bottom right in the lattice represents an alignment between $x$ and $y$, which, in terminology of CTC \citep{graves2006connectionist}, is denoted by $\pi_{1:T} \in \{\mathcal{Y} \cup \text{<b>}\}^T $.

\begin{figure}
	\centering
	\includegraphics[width=0.8\linewidth]{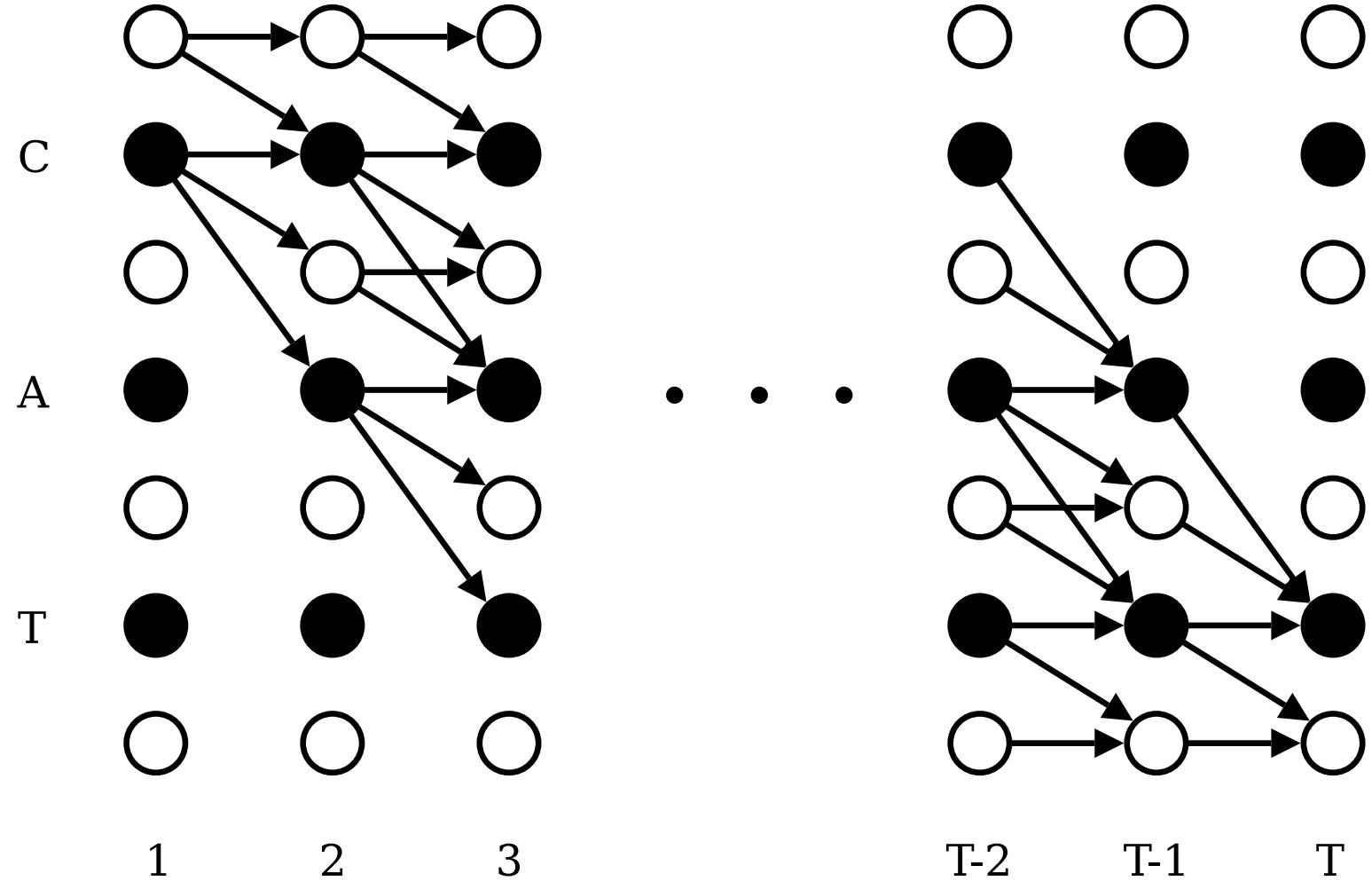}
	\caption{Illustration of the lattice, which contains all the possible alignments between the acoustic sequence and the label sequence `CAT'. Also illustration of the forward-backward algorithm. Black circles
		represent ordinary labels, and white circles represent blanks. Arrows signify allowed transitions. \citep{graves2006connectionist}}
	\label{fig:CTC}
\end{figure}

\paragraph{Obtain $p(y|x)$.}
We can use any kind of neural network (NN) to calculate the high-level feature vectors $h_{1:T} = (h_1,\cdots\,h_T) \in \mathbb{R}^{D \times T}$ from the raw spectral features $x_{1:T}$, or simply say, we encode $x$ into $h$. For each frame $t$, we obtain $h_t$.
The vector representations $h_{1:T}$ can be viewed as features, and the neural network is referred as a feature extractor as discussed previously in \secref{sec:logistic_regression}, or an acoustic encoder in the context of ASR.

Then, we can apply a linear layer followed by a softmax layer to calculate the posteriori distribution of $\pi_t$, as follows:
\begin{align*}
z_t &= W^T h_t + b \in \mathbb{R}^{K+1}\\
p(\pi_t= k |x ) &=  \frac{\exp(z_t^k)}{ \sum_{j=1}^K \exp(z_t^j) } \triangleq p_t^k, k=1,\cdots,K+1
\end{align*}
where $W \in \mathbb{R}^{(K+1)\times D}, b \in \mathbb{R}^{K+1}$ denote the weight matrix and bias vector of the linear layer, respectively.
$p_t^k$ represents the the probability of observing label $k$ at time $t$.
The un-normalized outputs $z_t$ are often called logits, and $z^k_t$ denotes the $k$-th logit corresponding to label $k$.

CTC assumes the conditional independence between states in a path $\pi_{1:T}$, and defines the path posteriori as follows:
\begin{equation}
\label{eq:ctc_pi_post}
p(\pi |x ) = \prod_{t=1}^T p(\pi_t |x )
\end{equation}
Finally, we use the CTC topology to calculate the posteriori probability of the label sequence $y$, by summing over all possible paths, which map to $y$:
\begin{equation}
\label{eq:ctc_y_post}
p(y |x ) = \sum_{\pi: \mathcal{B}_{\text{CTC}}(\pi)=y} p(\pi |x )
\end{equation}

\paragraph{CTC model training.}
In the previous introduction of the CTC model, we suppress the model parameters in formula. In the following, we make explicit of the model parameters $\theta$, which parameterizes the acoustic encoder network.

The network is trained to optimize the log conditional likelihood \eqref{eq:ctc_y_post}.
According to Fisher equality \eqref{eq:fisher_eq_2}, the gradient of the log conditional likelihood w.r.t. $\theta$ for a single data point $(x, y)$ is:
\begin{equation} \label{eq:ctc_grad}
\begin{split}
\frac{\partial \log p_\theta(y |x )}{\partial \theta}
&= E_{p_{\theta}(\pi|x,y)} \left[ \frac{\partial \log p_{\theta}(\pi,y|x) }{\partial \theta} \right] \\
&= E_{p_{\theta}(\pi|x,y)} \left[ \frac{\partial \log p_{\theta}(\pi|x) }{\partial \theta} \right]
\end{split}
\end{equation}
where the second line holds because $\pi$ deterministically determines $y$ with $\mathcal{B}_{\text{CTC}}$.

Note that $\log p_\theta(y|x)$ depends on the network parameters through logits $z_t^{k}$'s. Thus, an important quantity in calculating \eqref{eq:ctc_grad} for the network parameters is the gradient w.r.t. to the logits, which can be derived as follows:
\begin{align}
\frac{\partial \log p_\theta(y|x)}{\partial z_t^k} &= E_{p_{\theta}(\pi|x,y)} \left[ \frac{ \log p_{\theta}(\pi|x) }{\partial z_t^k} \right] (\because \text{Fisher equality}~\eqref{eq:fisher_eq_2})\nonumber\\
 &= E_{p_{\theta}(\pi|x,y)} \left[ \frac{ \log p_t^{\pi_t} }{\partial z_t^k} \right] (\because \text{\eqref{eq:ctc_pi_post}})\nonumber\\
&= E_{p_{\theta}(\pi|x,y)} \left[ \delta(\pi_t=k) - p_t^k \right]\nonumber\\
&= p_\theta(\pi_t=k|x,y) - p_t^k \label{eq:ctc_grad2}
\end{align}
This difference between the posteriori probability and the prior probability $p_t^k$ (without observing $y$) is the error signal received by the NN based feature extractor during training.
$p_\theta(\pi_t=k|x,y)$ is often known as the posterior state occupation probability, which can be calculated using the alpha-beta variables from the forward-backward algorithm \citep{rabiner1989tutorial}. 
The gradients for the network parameters can be calculated from the gradient w.r.t. to the logits based on the back-propagation procedure.

\subsection{CRF-based acoustic modeling with CTC topology}
\label{sec:ctc-crf}

\subsubsection{Motivation}

From \eqref{eq:ctc_pi_post}, we see that the CTC model assumes the conditional independence between states in a path. To overcome this drawback, the RNN-T model (RNN-transducer) and the CTC-CRF model (CRF with CTC topology) have been developed in \citep{graves2012sequence} and \citep{CTCCRF_IC19}, respectively.

A second motivation is that we are interested in bridging the hybrid and the end-to-end approaches for ASR, trying to inherit the data-efficiency of the hybrid approach and the simplicity of the end-to-end approach. 
Remarkably, when comparing the hybrid and end-to-end approaches (modularity versus a single neural network, separate optimization versus joint optimization), it is worthwhile to note the pros and cons of each approach.
\begin{itemize}
\item The end-to-end approach aims to subsume the acoustic, pronunciation, and language models into a single neural network and perform joint optimization.
This appealing feature comes at a cost, i.e. the end-to-end ASR systems are \emph{data hungry}\index{Data hungry}, which require above thousands of hours of labeled speech to be competitive with the hybrid systems \citep{vs,chiu2018state,tuske2019advancing}.
\item In contrast, the modularity of the hybrid approach permits training the AM and LM independently and on different data sets. A decent acoustic model can be trained with around 100 hours of labeled speech, whereas the LM can be trained on text-only data, which is available in vast amounts for many languages. In this sense, modularity promotes \emph{data efficiency}\index{Data efficiency}.
Due to the lack of modularity, it is difficult for an end-to-end model to exploit the text-only data, though there are recent efforts to alleviate this drawback \citep{toshniwal2018comparison,zheng2022empirical}.
\end{itemize}

\subsubsection{The CTC-CRF model}
\label{sec:ctccrf}

The CTC-CRF approach consists of separable AM and LM, which meets the principle to be data efficient by keeping necessary modularity.
In the following, we mainly describe the CTC-CRF based AM.
Different types of LMs, whether autoregressive LMs or energy-based LMs, can be used  with the CTC-CRF based AM.
Different schemes are available for decoding, such as WFST based decoding with $n$-gram LMs, one-pass decoding with shallow fusion of AM and LM scores, or two-pass decoding with LM rescoring.
More details can be found in the toolkit \citep{CAT}.

Continue with the notations in CTC, and note that the core formula in establishing the CTC model is \eqref{eq:ctc_pi_post} and \eqref{eq:ctc_y_post}, while \eqref{eq:ctc_pi_post} makes the conditional independence assumption.
\emph{The main idea of CTC-CRF} is that we can still use the CTC topology to define the posteriori probability of the label sequence $y$, $p_\theta(y |x )$, from the path posteriori $p_\theta(\pi|x)$, as defined in \eqref{eq:ctc_y_post}, but we define the path posteriori as a globally normalized sequence model, or say, a condition random field (CRF), as follows:
\begin{equation}
\label{eq:ctccrf_pi_post}
p_\theta(\pi|x) = \frac{ \exp(\phi_\theta(\pi,x))}
{ \sum_{\pi'}{\exp(\phi_\theta({\pi', x}))} }
\end{equation}
Here $\phi_{\theta}(\pi, x)$ denotes the potential function of the CRF, defined as:
\begin{equation}
\label{eq:ctccrf_pot}
\phi_{\theta}(\pi, x) = \log p(\mathcal{B}(\pi))+ \sum_{t=1}^{T} \log p_{\theta}(\pi_t|x)
\end{equation}
$\sum_{t=1}^{T} \log p_{\theta}(\pi_t|x)$ defines the \emph{node potential}, which is calculated from the acoustic encoder network with parameters $\theta$.
$\log p(\mathcal{B}(\pi))$ defines the \emph{edge potential}, realized by an $n$-gram LM of labels.

Remarkably, regular CTC suffers from the conditional independence between the states in $\pi$. In contrast, by incorporating $\log p(\mathcal{B}(\pi))$ into the potential function in CTC-CRF, this drawback is naturally avoided.
The difference between the CTC model and the CTC-CRF model can be clearly seen from their graphical model representations, as shown in \figref{fig:ctc_vs_ctccrf}.
Note that the $n$-gram LM of labels means the transition structure between labels is of ($n$-1)-th order. 
The transition structure between $\pi_t$'s, when represented by a WFST, is determined by the composition of two component WFSTs, i.e., the WFST representation of the CTC topology and the WFST representation of the $n$-gram LM of labels. 
For reasons to be clear in the following, the resulting WFST is referred to as the \emph{denominator WFST}.
Due to the composition operation, the order of the transition structure between states ($\pi_t$'s) is larger than ($n$-1)-th order. Thus, as a reminder for reading \figref{fig:ctc_vs_ctccrf}(b), the edge potential does not involve exactly $n$ consecutive nodes for a $n$-gram LM of labels.
The graphical model representation of CTC-CRF in \figref{fig:ctc_vs_ctccrf}(b) is mainly for concept illustration.
In practice, in training of a CTC-CRF model, the forward-backward algorithm involving $\log p(\mathcal{B}(\pi))$ can be conducted in the denominator WFST (to be detailed in the following).

Finally, note that we may use a $n$-gram LM of words in decoding. It is reminded not to confuse the $n$-gram LM of labels used in defining the potential in CTC-CRF and the $n$-gram LM of words used in decoding.

\begin{figure}
	\centering
	\includegraphics[width=1.0\linewidth]{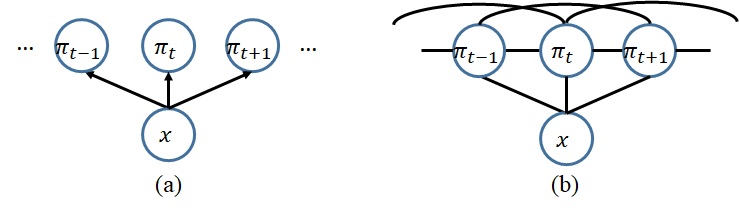}
	\caption{Graphical model representation of the CTC model (a) and the CTC-CTF model (b). Note that the edge potential does not involve exactly $n$ consecutive nodes for a $n$-gram LM of labels, as detailed in the text of \secref{sec:ctccrf}.}
	\label{fig:ctc_vs_ctccrf}
\end{figure}

\subsubsection{Training of the CTC-CRF model}

Note that only the parameters from the acoustic encoder network, denoted by $\theta$, need to be trained in the CTC-CRF model, since the $n$-gram LM of labels is estimated from the training transcripts and fixed in estimating $\theta$.

The network is trained to optimize the log conditional likelihood \eqref{eq:ctc_y_post}.
Making use of the Fisher equality \eqref{eq:fisher_eq_2} and the gradient of CRF's log conditional likelihood \eqref{eq:conditional_ebm_grad}, the gradient of CTC-CRF's log conditional likelihood w.r.t. $\theta$ for a single data point $(x, y)$ is:
\begin{equation} \label{eq:ctccrf_grad}
\begin{split}
\frac{\partial \log p_\theta(y |x )}{\partial \theta}
&= E_{p_{\theta}(\pi|x,y)} \left[ \frac{ \log p_{\theta}(\pi|x) }{\partial \theta} \right] (\because \eqref{eq:fisher_eq_2},~\text{similar to CTC})\\
&= E_{p_{\theta}(\pi|x,y)} \left[ \frac{\partial \phi_{\theta}(\pi, x)}{\partial \theta} - E_{p_{\theta}(\pi'|x)} \left[ \frac{\partial \phi_{\theta}(\pi', x)}{\partial \theta} \right] \right]\\
&= E_{p_{\theta}(\pi|x,y)} \left[ \frac{\partial \phi_{\theta}(\pi, x)}{\partial \theta} \right] - E_{p_{\theta}(\pi'|x)} \left[ \frac{\partial \phi_{\theta}(\pi', x)}{\partial \theta} \right]
\end{split}
\end{equation}
where the second term in the last line does not depend on $\pi$ and thus can be moved out of the expectation w.r.t. $\pi$.
Notably, $p_{\theta}(\pi'|x)$ denotes the model distribution, as defined in \eqref{eq:ctccrf_pi_post} but with $\pi'$ to represent the path.
As commonly found in estimating CRFs, the above gradient is the difference between empirical expectation and model expectation. 
The two expectations are similar to the calculations using the numerator graph and denominator graph in LF-MMI respectively \citep{povey2016purely}.

Remarkably, \eqref{eq:ctccrf_grad} can be derived in another way, which reveals the concept of numerator and denominator. Combining \eqref{eq:ctc_y_post} and \eqref{eq:ctccrf_pi_post} yields CTC-CRF's log conditional likelihood:
\begin{equation} \label{eq:crf-obj2}
 \log p_\theta(y |x ) = \log \frac{  \sum_{\pi: \mathcal{B}_{\text{CTC}}(\pi)=y} \exp(\phi_{\theta}(\pi, x))}{\sum_{\pi'}{\exp(\phi_{\theta}({\pi', x}))}}
\end{equation}
It can be easily seen that the gradient of the above objective function involves two gradients calculated from the numerator and denominator respectively, which can be shown to be equal to the two terms in \eqref{eq:ctccrf_grad}.

Denote $\log p(\pi_t= k |x )=\phi_t^k$, and note that $\log p_\theta(y|x)$ depends on the network parameters through potential values $\phi_t^{k}, 1 \le t \le T, 1 \le k \le K+1$. Thus, an important quantity in calculating \eqref{eq:ctccrf_grad} for the network parameters $\theta$ is the gradient w.r.t. to the potential values, which can be derived as follows:
\begin{align}
\frac{\partial \log p_\theta(y|x)}{\partial \phi_t^k} &= E_{p_{\theta}(\pi|x,y)} \left[ \frac{\partial \phi_{\theta}(\pi, x)}{\partial \phi_t^k} \right] - E_{p_{\theta}(\pi'|x)} \left[ \frac{\partial \phi_{\theta}(\pi', x)}{\partial \phi_t^k} \right] \nonumber\\
&= E_{p_{\theta}(\pi|x,y)} \left[ \delta(\pi_t=k) \right]-E_{p_{\theta}(\pi'|x)} \left[ \delta(\pi'_t=k) \right] \label{eq:ctccrf_grad2}
\end{align}
This difference is the error signal received by the NN based feature extractor during training.
The gradients for the network parameters $\theta$ can be calculated from the gradient w.r.t. to the potential values based on the back-propagation procedure.

Both terms in \eqref{eq:ctccrf_grad2} can be obtained via the forward-backward (FB) algorithm.
\begin{itemize}
	\item Calculating the first term, often referred to the numerator calculation, amounts to running the FB algorithm over the WFST determined by $y$, which is similar to conducting the FB algorithm in calculating the first term of \eqref{eq:ctc_grad2} in CTC.
	\item Calculating the second term, often referred to the denominator calculation, involves running the forward-backward algorithm over the denominator WFST $\mathbf{T}_\text{den}$.
$\mathbf{T}_\text{den}$ is an composition of the CTC topology WFST and the WFST representation of the $n$-gram LM of labels.
The $n$-gram LM of labels is thus often called the denominator $n$-gram LM, to be differentiated from the word-level LM in decoding.
\end{itemize}

\begin{table}
	\caption{Comparison of different models for ASR. 
	HMM topology denotes that labels (including silence) are modeled by multiple states with left-to-right transitions, possible self-loops and skips. CTC topology denotes the special state transitions used in CTC (including blank).
	Locally/globally normalized denotes the formulation of the model distribution.
	In defining the joint distribution of a model,
	locally normalized models use conditional probability functions, while globally normalized models use un-normalized potential functions.
	SS-LF-MMI is classified as globally normalized, though it is cast as MMI-based discriminative training of a pseudo HMM and the HMM model is locally normalized.
	AED does not use states to align label sequence $y$ and observation sequence $x$.
	}
	\label{tab:model_comparison}
	\centering
	\begin{tabular}{l|c|c|c}
		\hline
		\multirow{2}{4em}{Model} & State  & Training &Locally/globally  \\ 
		& topology & objective & normalized \\
		\hline \hline
		HMM & HMM & $p(x|y)$ & local\\
		CTC & CTC & $p(y|x)$ & local \\
		SS-LF-MMI & HMM & $p(y|x)$ & global \\
		CTC-CRF & CTC & $p(y|x)$ & global \\
		RNN-T & RNN-T & $p(y|x)$ & local \\
		AED & - & $p(y|x)$ & local \\
		\hline 
	\end{tabular}
\end{table}

\subsubsection{Related work}

In \tbref{tab:model_comparison}, we give a brief review of existing models in ASR, depending on state topologies, training objectives and whether the model distribution is locally or globally normalized.
We differentiate HMM topology and CTC topology, though the later may be interpreted as a special HMM topology \citep{zeyer2017ctc}.
The two differ not only in the state transition structure but also in the label inventory used (which affects not only the definition of the whole state space but also the estimation of the denominator LM).

Further, graphical model representations of existing ASR models are plotted in \figref{fig:ASR_models}, which clearly show the differences between those models.
An ASR model involves an acoustic observations $x = (x_1,\cdots\,x_T) = x_{1:T}$ and a label sequence $y = (y_1,\cdots\,y_L) = y_{1:L}$.
HMM, CTC and CTC-CRF are defined in \eqref{eq:hmm}, \eqref{eq:ctc_pi_post}, and \eqref{eq:ctccrf_pi_post}, respectively.
\emph{RNN-transducer} (RNN-T)\index{Recurrent neural network transducer (RNN-T)} \citep{graves2012sequence} is defined by
\begin{equation}
\label{eq:rnnt}
p(\pi_{1:T+L}|x_{1:T}) = \prod_{j=1}^{T+L} p(\pi_j|\pi_{1:j-1})
\end{equation}
Here $\pi_{1:T+L} = (\pi_1,\cdots\,\pi_{T+L})$ denote the state sequences, or say, the path with $T$ blanks and $L$ labels in RNN-T, such that removing the blanks in $\pi_{1:T+L}$ yields $y_{1:L}$ (see \secref{sec:rnnt} for details).
\emph{Attention based encoder-decoder} (AED)\index{Attention based encoder-decoder (AED)} \citep{chorowski2014end} is defined by
\begin{equation}
\label{eq:AED}
p(y_{1:L}|x_{1:T}) = \prod_{i=1}^{L} p(y_i|x,y_1,\cdots,y_{i-1})
\end{equation}
In summary, from \tbref{tab:model_comparison} and \figref{fig:ASR_models}, we can cleary see that CTC-CRF is fundamentally different from those prior models.
For comparison between CTC-CRF and single-stage (SS) lattice-free maximum-mutual-information (LF-MMI) \citep{hadian2018flat}, readers can refer to \citep{CTCCRF_IC19}.



\begin{figure}
	\centering
	\includegraphics[width=1.0\linewidth]{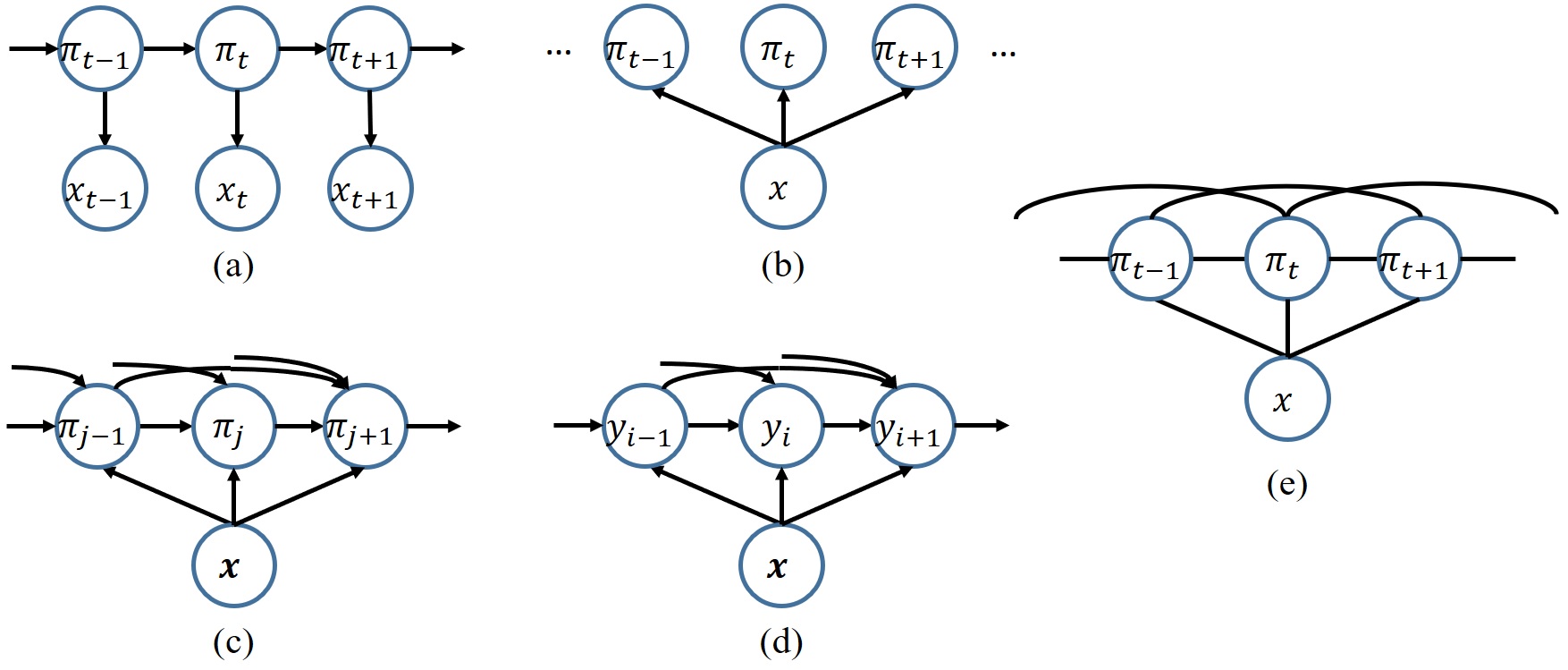}
	\caption{Graphical model representations of different ASR models: (a) HMM, defined in \eqref{eq:hmm}, (b) CTC, defined in \eqref{eq:ctc_pi_post}, (c) RNN-T, (d) AED, (e) CTC-CRF, defined in \eqref{eq:ctccrf_pi_post}.}
	\label{fig:ASR_models}
\end{figure}

\paragraph{Relation to CRF-based acoustic models.}
ASR is a sequence transduction problem in that the input and output sequences differ in lengths, and both
lengths are variable.
An idea in applying CRFs to ASR is to introduce a (hidden) state sequence $\pi$ to align the label sequence $y$ and observation sequence $x$, and define a CRF $p(\pi | x)$ over the (hidden) state sequence $\pi$.
As shown in \eqref{eq:ctc_y_post}, deriving $p(y | x)$ based on $p(\pi | x)$ depends on the mapping between $\pi$ and $y$, which is determined by the state topology that allows for different choices, e.g., CTC topology or HMM topology.
This kind of hidden CRFs was explored in \citep{gunawardana2005hidden} for phone classification, using zero, first and second order features.
Generally speaking, (hidden) CRFs using neural features for ASR are underappreciated. 
The CTC-CRF model proposed in \citep{CTCCRF_IC19} represents the first exploration of CRFs with CTC topology and advances the CRF-based approach with strong empirical results.
Segmental CRFs \citep{lu2016segmental} provide another solution to the alignment problem.

\subsubsection{Performance of the CTC-CRF model}

The CTC-CRF model inherits the data-efficiency of the hybrid approach and the simplicity of the end-to-end approach. 
CTC-CRF eliminates the conditional independence assumption in CTC and performs significantly better than CTC on a wide range of benchmarks, including WSJ (80-h), AISHELL (170-h Chinese), Switchboard (260-h), Librispeech (1000-h), and Fisher-Switchboard (2300-h) (the numbers in the parentheses are the size of training data in hours) \citep{CTCCRF_IC19,CAT}.
It has been also shown \citep{CTCCRF_IC19,CAT} that CTC-CRF is on par with other state-of-the-art end-to-end models, like RNN-T and AED.

The CTC-CRF models have also been used in a variety of tasks in ASR, which show their great potential.
\begin{itemize}
	\item Streaming ASR, particularly the Chunking, Simulating Future Context and Decoding (CUSIDE) approach \citep{an2022cuside};
	\item Neural architecture search \citep{zheng2021efficient};
	\item Children Speech Recognition \citep{yu2021slt};
	\item Modeling with wordpieces and Conformer architectures \citep{zheng2021advancing};
	\item Multilingual and crosslingual speech recognition, particularly the JoinAP (Joining of Acoustics and
	Phonology) approach \citep{zhu2021multilingual}.
\end{itemize}

\section{CRFs for sequence labeling in NLP}
\label{sec:crf_labeling}

Conditional random fields (CRFs) have been shown to be one of the most successful approaches to \emph{sequence labeling}\index{Sequence labeling}.
Various linear-chain \emph{neural CRFs}\index{Neural CRF (NCRF)} (NCRFs) have been developed, as introduced in \secref{sec:linear_chain_crf}.
The node potential modeling is improved by using NNs, but the linear-chain structure is still kept, i.e., using a bigram table as the edge potential.
NCRFs represent an extension from conventional CRFs, where both node potentials and edge potentials are implemented as linear functions using discrete indicator features.
However, linear-chain NCRFs capture only first-order\footnote{Fixed $n$-th order can be cast as first-order.} interactions and neglect higher-order dependencies between labels, which can be potentially useful in real-world sequence labeling applications, e.g., as shown in \citep{zhang2018does} for chunking and NER.
How can we improve CRFs to capture long-range dependencies in the label sequence (preferably non-Markovian)?

\paragraph{Related work.}
Extending CRFs to model higher-order interactions than pairwise relationships between labels is an important research problem for sequence labeling.
There are some prior studies, e.g. higher-order CRFs \citep{Chatzis2013TheIC}, semi-Markov CRFs \citep{Sarawagi2004SemiMarkovCR} and latent-dynamic CRFs \citep{morency2007latent}, but not using NNs.
Using NNs to enhance the modeling of long-range dependencies in CRFs is under-appreciated in the literature. 
A related work is structured prediction energy networks (SPENs) \citep{belanger2016structured}, which use
neural networks to define energy functions that potentially can capture long-range dependencies between structured outputs/labels.
SPENs depend on relaxing labels from discrete to continuous and use gradient descent for test-time inference, which is time-consuming. Training and inference with SPENs are still challenging, though with progress \citep{tu2018learning}.

Outside of the globally normalized sequence models, where CRFs represent a typical class, attention-based encoder-decoder (AED) and RNN-T exploit non-Markovian dependences between labels, but both are locally normalized sequence models and thus suffer from the label bias and exposure bias problems, as described in \secref{sec:bias}.
The work in \citep{wiseman2016sequence} extends AED, by removing the final softmax in the RNN decoder to learn global sequence scores, but cast as a non-probabilistic variant of the sequence-to-sequence model.
A recent work in \citep{cui2021reducing} aims to reducing exposure bias in training RNN-T.

In this section, we mainly introduce a progress made by \emph{CRF transducers}\index{CRF transducer} \citep{hu2019neural}\footnote{Reproducible code is at \url{https://github.com/thu-spmi/SPMISeq}}, which introduce a LSTM-RNN to implement a new edge potential so that long-range dependencies in the label sequence are captured and modeled in CRFs.
So there are two LSTM-RNNs in a CRF transducer, one extracting features from observations to define the node potential and the other capturing (theoretically infinite) long-range dependencies between labels to define the edge potential.
In this view, a CRF transducer is similar to a RNN-transducer (RNN-T) \citep{graves2012sequence}, which also uses two LSTM-RNNs.

In the following, we firstly briefly introduce RNN-T, and then describe CRF transducer in details.
We continue with the notations in \secref{sec:linear_chain_crf} for sequence labeling. 
Given a sequence of observations $x = (x_1,\cdots\,x_T) = x_{1:T}$, the task of sequence labeling is to predict a sequence of labels $y = (y_1,\cdots\,y_T) = y_{1:T}$, with one label for one observation in each position.
$y_i \in \left\lbrace 1,\cdots,K \right\rbrace$ denotes the label at position $i$.

\subsection{RNN-Transducer (RNN-T)}
\label{sec:rnnt}

RNN Transducers (RNN-T) are originally developed for general sequence-to-sequence learning \citep{graves2012sequence}, which do not assume that the input and output sequences are of equal lengths and aligned, e.g., in speech recognition. 
In the following, we introduce RNN transducers in a simple form for applications in sequence labeling, i.e., for the aligned setting - one label for one observation at each position.
To this end, we define
\begin{equation}
\label{eq:rnnt_simp}
 p(y|x)=\prod_{i=1}^T p(y_i|y_{0:i-1},x)
\end{equation}
and implement $p(y_i|y_{0:i-1},x)$ through two networks - transcription network $F$ and prediction network $G$ as follows:
\begin{equation} \label{eq:RNNT-conditional}
p(y_i=k|y_{0:i-1},x) = \frac{\exp({f_i^k + g_i^k})}{\sum_{k'=1}^{K}\exp({f_i^{k'} + g_i^{k'}})}
\end{equation}
Here $F$ scans the observation sequence $x$ and outputs the transcription vector sequence $f = (f_1,\cdots\,f_T)=f_{1:T}$. 
$G$ scans the label sequence $y_{0:T-1}$ and outputs the prediction vector sequence $g = (g_1,\cdots\,g_T)=g_{1:T}$. 
$y_0$ denotes the beginning symbol (<bos>) of the label sequence. 
For a sequence labeling task with $K$ possible labels, $f_i$ and $g_i$ are $K$ dimensional vectors.
Superscript $k$ is used to denote the $k$-th element of the vectors.
Remarkably, the prediction network $G$ can be viewed as a language model of labels, capable of modeling long-range dependencies in $y$, which is exactly the motivation to introducing $G$ in RNN transducers.

To ease comparison, we will also refer to the network below the CRF layer in linear-chain NCRFs as a transcription network, which also implement $\phi_t(y_t=k,x)$ as $f_t^k$.

\subsection{From RNN-T to CRF transducer}
\label{sec:CRF-T}

\begin{table}[t]
\begin{center}
\caption{Model comparison and connection.}
\label{tab:model-table}
\begin{tabular}{l|c|c}
\hline
\multirow{2}{4em}{Model} & Globally  & Long-range dependencies\\ 
& normalized & between labels\\
\hline \hline
	Linear-chain CRF  &$\surd$ &$\times$ \\
	RNN Transducer  &$\times$ &$\surd$ \\
	CRF transducer &$\surd$ &$\surd$ \\
	\hline
\end{tabular}
\end{center}	
\end{table}

\begin{figure}
	\centering
	\includegraphics[width=1.0\linewidth]{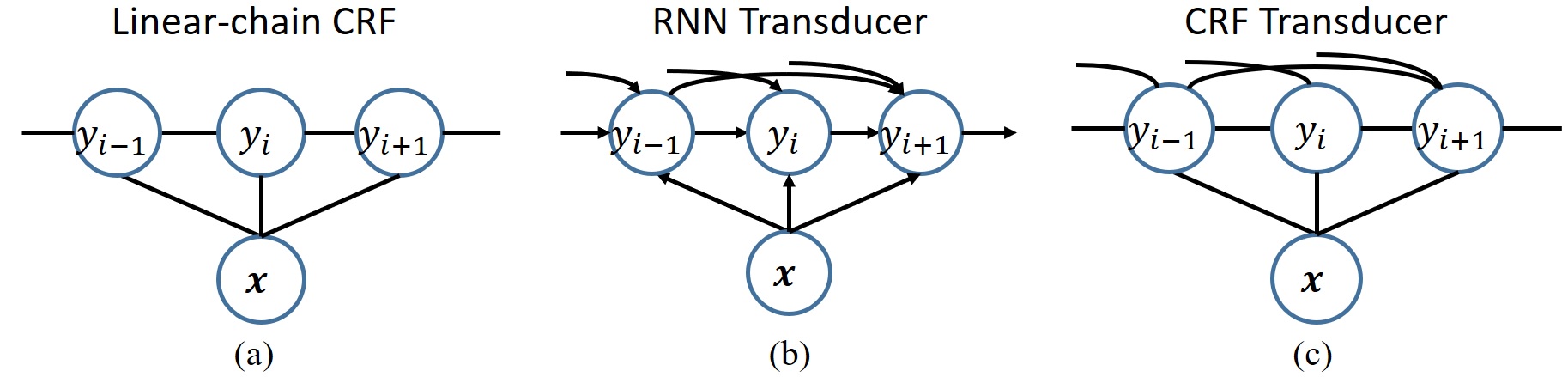}
	\caption{Graphical model representations of (a) a linear-chain CRF, (b) a RNN-T for the aligned setting, and (c) a CRF transducer.
Notably, the graphical representation of the RNN-T for the aligned setting, as defined in \eqref{eq:rnnt_simp}, is different from that of the usual RNN-T as shown in \figref{fig:ASR_models}(c).}
	\label{fig:CRFT_comparisons}
\end{figure}

In the following, we introduce CRF transducers, which combine the advantages of linear-chain NCRFs (globally normalized, using LSTM-RNNs to implement node potentials) and of RNN transducers (capable of capturing long-range dependencies in labels), and meanwhile overcome their drawbacks, as illustrated in \tbref{tab:model-table}.
Further, graphical model representations of different models are plotted in \figref{fig:CRFT_comparisons}, which clearly show the differences between those models.

\paragraph{Model definition.}
A CRF transducer defines a globally normalized, conditional distribution $p(y|x;\theta)$ as follows:
\begin{displaymath}
\ p(y|x;\theta) = \frac{\exp \left\lbrace  u(y,x;\theta)\right\rbrace }{Z(x;\theta)}.
\end{displaymath}
where $Z(x;\theta) = \sum_{y'\in\mathcal{D}_T}\exp \left\lbrace u(y',x;\theta) \right\rbrace $ is the global normalizing term and $\mathcal{D}_T$ is the set of allowed label sequences of length $T$. The total potential $u(y,x;\theta)$ is decomposed as follows:
\begin{displaymath}
u(y,x;\theta) = \sum_{i=1}^{T} \left\lbrace \phi_i(y_i,x;\theta) + \psi_i(y_{0:i-1},y_i;\theta) \right\rbrace.
\end{displaymath}
where $\phi_i(y_i,x;\theta)$ is the node potential at position $i$, $\psi_i(y_{0:i-1},y_i;\theta)$ is the clique potential involving labels from the beginning up to position $i$. Thus the underlying undirected graph for the label sequence $y$ is fully-connected, which potentially can capture long-range dependencies from the beginning up to each current position.

\begin{figure}[t]
	\centering
	\centerline{\includegraphics[width=8.5cm]{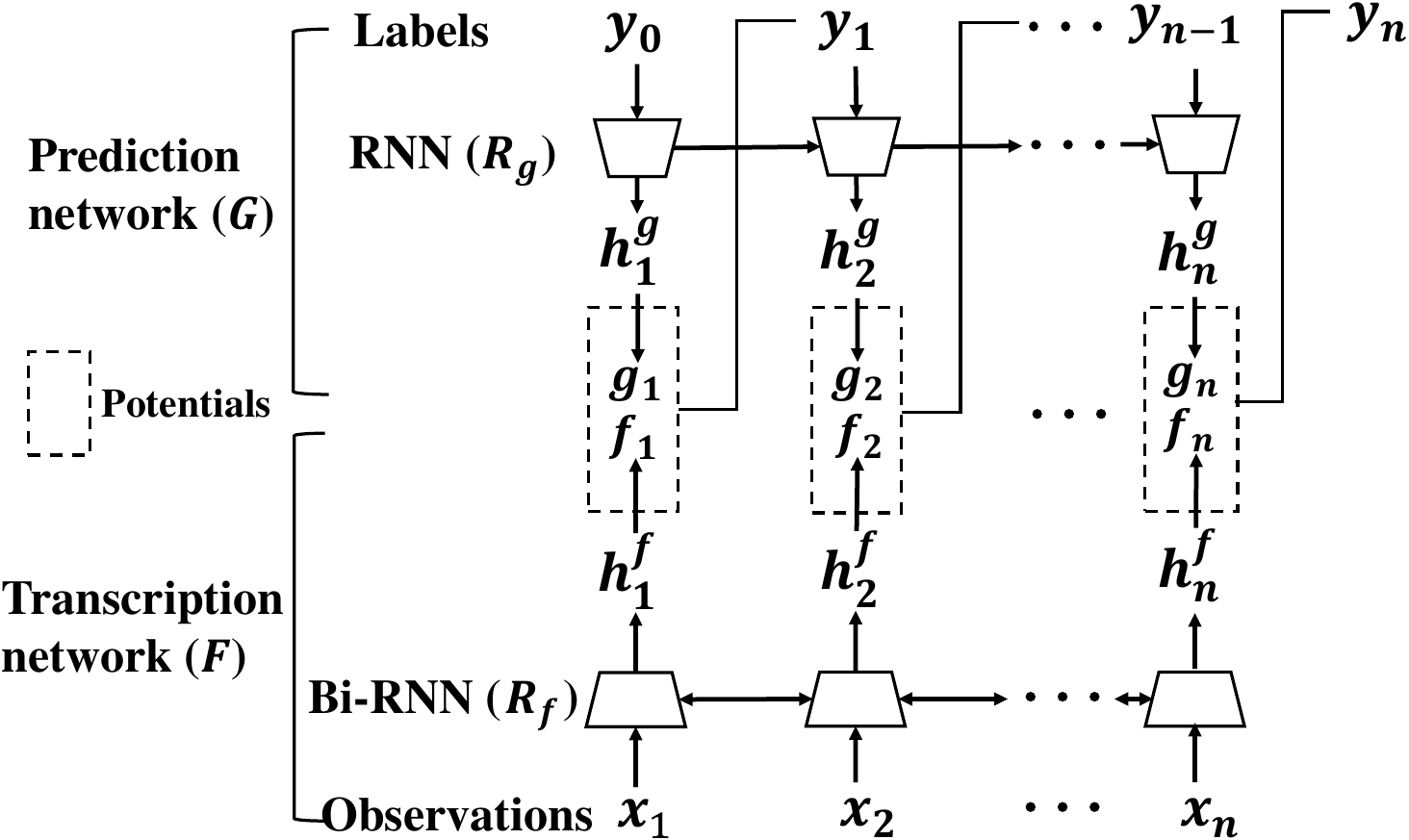}}
	\caption{The architecture of a CRF transducer.}
	\label{fig:crft.png}
\end{figure}

\paragraph{Neural network architectures.}
Like in RNN transducers, we introduce two networks in CRF transducers, as shown in \figref{fig:crft}. The transcription network $F$ implements the node potential $\phi_i(y_i,x;\theta)$, which represents the score for $y_i$ based on observations $x$. In the experiments on NLP sequence labeling, each word $x_i$ is represented by a concatenation of a pre-trained word embedding vector and another embedding vector obtained from a character-level CNN \citep{hu2019neural}.
The transcription network $F$ is a bidirectional
RNN ($R_f$) that scans the sequence of the concatenated vectors for words to generate hidden vectors $h_i^f=[\overrightarrow{h_i^f};\overleftarrow{h_i^f}]$, which are then fed to a linear layer with output size of $K$ to generate $f_i \in \mathbb{R}^K$.

The prediction network $G$ implements the clique potential $\psi_i(y_{0:i-1},y_i;\theta)$, which represents the score for $y_i$ by taking account of dependencies between $y_i$ and previous labels $y_{0:i-1}$.
In the experiments, each label $y_i$ is represented by a label embedding vector, initialized randomly. $G$ is a unidirectional RNN ($R_g$) that accepts the label sequence $y$ and generates hidden vectors $h_i^g=\overrightarrow{h_i^g}$, which are then fed to a linear layer with output size of $K$ to generate $g_i \in \mathbb{R}^K$.

It can be seen from above that a CRF transducer is similar to a RNN transducer. The difference is that a RNN transducer is locally normalized through softmax calculations as shown in Eq. (\ref{eq:RNNT-conditional}), while a CRF transducer is globally normalized, locally producing (un-normalized) potential scores.

\paragraph{Potential design.}
Based on $f_i$ and $g_i$, there are two possible designs to implement the potentials $\phi_i$ and $\psi_i$, which can be chosen empirically in the experiments. The first design is:
\begin{equation} \label{eq:design-a}
\begin{aligned}
\phi_i(y_i=k,x;\theta)= f_i^k \\
\psi_i(y_{0:i-1},y_i=k;\theta)= g_i^k
\end{aligned}
\end{equation}
The second design is:
\begin{equation} \label{eq:design-b}
\begin{aligned}
\phi_i(y_i=k,x;\theta) = \log \frac{\exp({f_i^k})}{\sum_{k'=1}^{K}\exp(f_i^{k'})} \\
\psi_i(y_{0:i-1},y_i=k;\theta)= \log \frac{\exp({g_i^k})}{\sum_{k'=1}^{K}\exp(g_i^{k'})}
\end{aligned}
\end{equation}

\paragraph{Decoding and training.}
CRF transducers break the first-order Markov assumption in the label sequence as in linear-chain NCRFs and thus do not admit dynamic programming for decoding. Instead, beam search can be used to approximately find the most probable label sequence:
\begin{displaymath}
\hat{y}=\argmax_{y'\in\mathcal{D}_T} p(y'|x;\theta) = \argmax_{y'\in\mathcal{D}_T} u(y',x;\theta).
\end{displaymath}

Training data consists of inputs $x$ paired with oracle
label sequences $y^*$. Stochastic gradient
descent (SGD) on the negative log-likelihood of the
training data is conducted:
\begin{displaymath}
L(y^*;\theta) = -u(y^*,x;\theta) + \log Z(x;\theta).
\end{displaymath}
It is easy to calculate the gradient of the first term. However, the gradient of the log normalizing term involves model expectation:
\begin{displaymath}
\nabla_\theta \log Z(x;\theta) = E_{p(y'|x;\theta)} \left[ \nabla_\theta u(y',x;\theta) \right] 
\end{displaymath}
The calculations of the normalizing term and the model expectation can be exactly performed for linear-chain NCRFs (via the forward and backward algorithm), but are intractable for CRF transducers.
It is empirically found in the experiments \citep{hu2019neural} that the method of beam search with early updates \citep{collins2004incremental} marginally outperforms Monte Carlo based methods for training CRF transducers.

The basic idea is that we run beam search and approximate the normalizing term by summing over the paths in the beam.
Early updates refer to that as the training sequence is being decoded, we keep track of the location of the oracle path in the beam; If the oracle path falls out of the beam at step $j$, a stochastic gradient step is taken on the following objective:
\begin{displaymath}
L(y^*_{1:j};\theta) = -u(y^*_{1:j};\theta) + log \sum_{y'\in\mathcal{B}_j}\exp \left\lbrace u(y'_{1:j};\theta) \right\rbrace
\end{displaymath}
where $u(y_{1:j};\theta) = \sum_{i=1}^{j} \left\lbrace \phi_i(y_i,x;\theta) + \psi_i(y_{0:i-1},y_i;\theta) \right\rbrace $ denotes the partial potential (with abuse of the notation of $u$).
The set $\mathcal{B}_j$ contains all paths in the beam
at step $j$, together with the oracle path prefix $y^*_{1:j}$.

\paragraph{Performance of CRF transducers.}
Different sequence labeling methods are evaluated over POS tagging, chunking and NER (English, Dutch) in \citep{hu2019neural}.
Experiment results show that CRF transducers achieve consistent improvements over linear-chain NCRFs and RNN transducers across all the four tasks, and can improve state-of-the-art results.

\newpage
\section{EBMs for conditional text generation}
\label{sec:cond_gen}

\subsection{Residual energy-based models} 
\label{sec:residual_ebm}
\index{Residual ELM}

\subsubsection{Motivation}
Text generation is ubiquitous in many NLP tasks, from summarization, to dialogue and machine translation.
In this section, we will further introduce residual energy-based language models, as briefly covered in \secref{sec:residual_ELM}, and show their application in (conditional) text generation.
The dominant approach to text generation is based on autoregressive language models (ALMs), especially recent large ALMs parameterized by large neural networks \citep{radford2018improving}, which are locally-normalized.
Unfortunately, local normalization also brings some drawbacks, as described in \citep{deng2020residual} and listed below.
First, the designer of the model needs to specify the order in which tokens are generated. 
Second, at training time the model is conditioned on ground truth context while at test time it is conditioned on its own generations, a
discrepancy referred to as \emph{exposure bias} (\secref{sec:bias}). 
Finally, while heuristics like beam search somewhat help rescore at the sequence level, generation generally lacks long-range coherency
because it is produced by the greedy selection of one token at a time 
without lookahead\footnote{This drawback of generation without lookahead is related to the \emph{label bias} problem of locally-normalized sequence models, namely being weak in revising earlier decisions (\secref{sec:bias}).}.

In principle, EBMs potentially address all these issues, as they do not require any local normalization. EBMs are not prone to exposure bias and label bias, and they can score the whole input at once. EBMs may enable generation of large chunks of text, which should help improve coherency.
However, a difficulty of applying EBMs in text generation is that sampling from EBMs is intractable, and so is maximum likelihood training.
A recent work in \citep{deng2020residual} develops residual EBMs for text generation and shows impressive results, which will be detailed below.



\subsubsection{The residual EBM model}

The formulation of residual EBMs, as introduced in \secref{sec:residual_ELM}, has multi-fold benefits for text generation \citep{deng2020residual}.
First, by incorporating an autoregressive language
model in the residual formulation, we can leverage recent advancements in autoregressive language modeling. Second,
the autoregressive language model provides a natural proposal distribution for training of residual EBMs, and
the training can be made efficient by using the conditional noise contrastive estimation objective (\secref{sec:cond_nce}). Lastly, this formulation enables efficient evaluation
and generation via importance sampling, as we shall detail in the following.
In some sense, this last point is perhaps the central contribution of \citep{deng2020residual}, as it allows estimating
perplexity of the residual EBM, and thus allows these EBMs to be compared in a standard way to
other models.

\cite{deng2020residual} investigates an EBM trained on the residual of a pre-trained autoregressive LM, particularly for conditional generation of discrete sequences.
Given a conditioning prefix $x_1,\cdots,x_c$ with $x_j \in  \mathbb{V}$ where $\mathbb{V}$ is the vocabulary, the probability of generating a sequence of total length $T > c$ is defined as follows:
\begin{equation}
\label{eq:ELM_deng}
\begin{split}
&p_\theta(x_{c+1},\cdots,x_T|x_1,\cdots,x_c)\\
=&\frac{p_{\text{LM}}(x_{c+1},\cdots,x_T|x_1,\cdots,x_c) \exp(-E_\theta(x_1,\cdots,x_c,x_{c+1},\cdots,x_T))}{Z_\theta(x_1,\cdots,x_c)}\\
=&\frac{\tilde{p}_\theta(x_{c+1},\cdots,x_T|x_1,\cdots,x_c)}{Z_\theta(x_1,\cdots,x_c)}
\end{split}
\end{equation}
where $p_{\text{LM}}$ denotes the pre-trained autoregressive LM and is fixed throughout training, $E_\theta(\cdot)$ is the residual energy function parameterized by $\theta$. 
In the experiment of \citep{deng2020residual}, $E_\theta$ is initialized with BERT/RoBERTa; in the final layer the mean-pooled hidden states are projected to a scalar energy value.
$p_\theta$ is called \emph{the joint model}.
$Z_\theta(x_1,\cdots,x_c)$ is the normalizing factor known as partition function.
$\tilde{p}_\theta$ denotes the un-normalized probability.


\subsubsection{Training of the residual EBM}

The conditional EBMs defined in \eqref{eq:ELM_deng} can be trained using NCE (\secref{sec:NCE}), and more specifically its conditional version (\secref{sec:cond_nce}). 
For the sake of reducing clutter in the notation, we will drop the conditioning variables $x_1,\cdots,x_c$ and use $x$ to denote a target token sequence (namely $x_{c+1},\cdots,x_T$) in the following discussion.
Denote the training dataset by $\mathcal{D}$.

NCE requires two distributions: the model distribution and a noise distribution. 
Here the model distribution is the joint model of \eqref{eq:ELM_deng}, $p_\theta$, while the noise distribution is the pre-trained LM, $p_{\text{LM}}$. 
NCE then trains a binary classifier on the difference of log-probability scores of these two models. 
Since the joint model is the product of the energy function (whose parameters we want to learn) with the pre-trained LM, the difference reduces to: 
$\log \tilde{p}_\theta - \log p_{\text{LM}} = - E_\theta $. 
Therefore, under these modeling assumptions of residual learning and noise model, the NCE objective function becomes:


\begin{equation}
\label{eq:residual_nce}
    \begin{aligned}
        {J}_\theta&=\mathop{\mathbb{E}}\limits_{x_+ \sim \mathcal{D}} \log \frac{\tilde{p}_\theta(x_+)}{\tilde{p}_\theta(x_+)+ p_{\text{LM}}(x_+)} +  \mathop{\mathbb{E}}\limits_{x_- \sim p_{\text{LM}}} \log \frac{ p_{\text{LM}}(x_-)}{\tilde{p}_\theta(x_-)+ p_{\text{LM}}(x_-)}\\
        &=\mathop{\mathbb{E}}\limits_{x_+ \sim \mathcal{D}} \log\frac{1}{1+ \exp(E_\theta(x_+))} +  \mathop{\mathbb{E}}\limits_{x_- \sim p_{\text{LM}}} \log\frac{1}{1+\exp(-E_\theta(x_-))}
    \end{aligned}
\end{equation}
where $x_+$ denotes the positive sequence taken from the human generated training set, and $x_-$ the negative sequence drawn from the pre-trained LM (for a given ground truth prefix). 
Again, we see that NCE training of the energy function reduces to training a binary classifier to discriminate between real text and text generated by a pre-trained autoregressive language model.
The experiment of \citep{deng2020residual} uses a prefix of size 120 tokens and generates the following 40 tokens; with the notation of \eqref{eq:ELM_deng}, $c = 120$ and $T = 160$. 
For NCE training, for efficiency $16$ samples per prefix for CC-News \citep{bakhtin2019real} were generated offline, and sampled uniformly from those samples at training time.

\begin{algorithm}[t!]
	\caption{Top-$k$ sampling for the residual EBM}
	\label{alg:residualEBM_sampling}
	\begin{algorithmic}
		\REQUIRE Number of samples $n$ drawn from $p_{\text{LM}}$, value of $k$ in top-k
		\STATE \texttt{// Get a set of samples from} $p_{\text{LM}}$
		\STATE Sample $n$ samples $\{x^1,...,x^n\}$ from $p_{\text{LM}}$ with top-k sampling
		\STATE Calculate energies $s^i=E_\theta(x^i)$ for each $x^i \in \{x^1,...,x^n\}$
		\STATE \texttt{// Resample from the set of LM samples}
		\STATE Sample $x=x^i$ with probability $\frac{\exp(-s^i)}{\sum_{j=1}^n \exp(-s^j)}$
		\STATE {\bf Return:} $x$
	\end{algorithmic}
\end{algorithm}

\subsubsection{Generation from the residual EBM}
\label{sec:generation_residual}
In order to generate from the residual EBM model \eqref{eq:ELM_deng} efficiently, \cite{deng2020residual} uses self-normalized importance sampling (SNIS) (\secref{sec:IS})\index{Self-normalized importance sampling (SNIS)}. 
Under the assumptions that the model from which we wish to draw samples is the joint model, which is the product of the autoregressive model and the energy function, and that the proposal distribution is the autoregressive model itself, sampling proceeds simply by: 
a) sampling from the autoregressive language model, followed by b) resampling according to the energy
function.
The algorithm is shown in \algref{alg:residualEBM_sampling}, where a top-$k$ constraint on the pre-trained language model\index{Pre-trained language model (PLM)} is introduced to improve the quality of samples in the set.
Without the top-k constraint, as the number of samples goes to infinity, we would recover exact samples from the joint model distribution.

\subsubsection{Evaluation of the residual EBM}

A commonly used protocol for evaluating generative sequence models, especially language models, is perplexity (PPL)\index{Perplexity (PPL)}, which is equal to 
\begin{displaymath}
\text{PPL} = 2^{ - \frac{1}{T-c} \sum_{i=c+1}^T \log_2 p_\theta(x_i | x_1,\cdots,x_{i-1})  }
\end{displaymath}
PPL can be interpreted as the average number of tokens the model is uncertain of at every time step. 
For the residual EBM model, the log-likelihood required by PPL relies on estimating the partition function 
\begin{displaymath}
Z_\theta = \sum_x p_{\text{LM}}(x) \exp(- E_\theta(x))
 = \mathbb{E}_{x \sim p_{\text{LM}}} \exp(- E_\theta(x)).
\end{displaymath}
Based on \citep{nowozin2018debiasing}, two estimators are derived in \citep{deng2020residual} for the lower and upper bounds of the partition function respectively, as shown in the following theorem.

\begin{theorem}
	Denote $T_n$ as the empirical estimate of $\log \mathbb{E}_{x \sim p_{\text{LM}}} \exp(-E_\theta(x))$ with $n$ samples $x_i \sim P_{\text{LM}}$, $i=1,\cdots,n$: $T_n = \log\frac{1}{n}\sum_{i=1}^n \exp(-E(x_i))$, then ${\forall}\epsilon>0, {\exists} N>0$ such that ${\forall}n>N$ we have
	\begin{equation}
	\label{eq:residual_bounds}
	Z_\theta-\epsilon<\mathbb{E}[T_n]<Z_\theta<\mathbb{E}[(2n-1)T_n-2(n-1)T_{n-1}]<Z_\theta+\epsilon.
	\end{equation}
\end{theorem}

Similarly to locally normalized models, we can also factorize the probabilities of an entire sequence step by step,
$p_\theta(x) = \prod_{t=1}^T p_\theta(x_t|x_{<t})$.
By marginalizing over the future, the following step-wise probabilities are derived in \citep{deng2020residual}:
\begin{equation}
\label{eq:residual_stepwise_prob}
p_\theta(x_t|x_{<t})
=p_{\text{LM}}(x_t|x_{<t})\frac{\mathbb{E}_{x'_{t+1:T}\sim p_{\text{LM}}(\cdot|x_{\le t})}[\exp(-E_\theta(x_{\le t}, x'_{t+1:T}))]}{\mathbb{E}_{x'_{t:T}\sim p_{\text{LM}}(\cdot|x_{\le t-1})}[\exp(-E_\theta(x_{\le t-1}, x'_{t:T}))]}
\end{equation}
Note that both the numerator and the denominator in \eqref{eq:residual_stepwise_prob} take the same form as the partition function, we can also use \eqref{eq:residual_bounds} to estimate the upper and lower bounds. For example, the lower bound of $\log p_\theta(x_t|x_{<t})$ can be obtained by using the lower bound of the numerator and the upper bound of the denominator.
Remarkably, for $t=T$, we can calculate the log probability by exhaustive enumeration. This gives us an idea of the true performance of the model at the last step, and it also provides a sanity-check of the tightness of the estimators.


\subsubsection{Performance of the residual EBM}
In \citep{deng2020residual}, experiments on two datasets, CC-News \citep{bakhtin2019real} and the Toronto Book Corpus \citep{zhu2015aligning} show that residual EBMs demonstrated improved generation ability against strong autoregressive baselines, both in terms of estimated perplexity and through human evaluation.

\subsection{Controlled text generation from pre-trained language models} 
\label{sec:controlled_text_gen}

\subsubsection{Motivation}
While large transformer-based autoregressive language models trained on massive amounts of data exhibit exceptional capabilities to generate natural language text, effective methods for generating text that satisfy global constraints
and possess holistic desired attributes remains an active area of research \citep{khalifa2020distributional,mixandmatch}.
For instance, we may want to avoid toxic content; or steer generations towards a certain topic or style.
Much of the prior work has approached controlled generation via either training domain-conditioned neural language models or finetuning/modifying an underlying large pre-trained base model for attribute sensitive generation (see references in \cite{mixandmatch}). 
Not only do these approaches involve computational overhead and estimation errors associated with the training of language models, but they are also dependent on access to a large amount of attribute-specific language data which can be impractical in many scenarios and exacerbate privacy concerns.

In order to address these limitations, the study in \citep{mixandmatch} thus aims to \emph{eschew training} and focuses on \emph{generation-time control} from pre-trained modules.
\emph{The mix-and-match method} proposed in \citep{mixandmatch}, as shown in \figref{fig:MixMatch}, draws samples from a test-time combination of pre-trained black-box experts that each scores a desired property of output text - for example, fluency, attribute sensitivity, or faithfulness to the context.
Specifically, the product of these black-box experts is viewed as a EBM, and sampling is performed (without further training or fine-tuning), using a specialized Gibbs sampler with a Metropolis-Hastings (MH) correction step \citep{goyal2021exposing}, which is basically MH within Gibbs sampling, as introduced in \secref{sec:MCMC}.

A related work in \citep{khalifa2020distributional} proposes \emph{a distributional approach} for addressing controlled text generation from pre-trained language models (PLMs)\index{Pre-trained language model (PLM)}.
Both ``pointwise'' constraints (hard requirements on each individual) and ``distributional'' constraints (collective statistical requirements such as moment constraints) are permitted to added in the target LM, while minimizing KL divergence from the initial pre-trained LM distribution. The optimal target distribution is also uniquely determined as a residual EBM model and can be trained through a variant of policy gradient based on importance sampling.

Also for controlled text generation, the work in \citep{qin2022cold} is motivated by the need to enrich decoding algorithms that can work directly with pre-trained language models (PLMs) without task-specific fine-tuning, and support complex combinations of hard and soft constraints to control the generated text on the fly.
Previous studies use MH within Gibbs sampling or SNIS, as surveyed in \tbref{tab:gen_survey}.
A new decoding method, called \emph{Constrained decoding with Langevin dynamics} (COLD)\index{Constrained decoding with Langevin dynamics (COLD)}, is developed in \citep{qin2022cold}, which introduces Langevin dynamics to text-based EBMs for efficient gradient-based sampling. 
Specifically, the COLD based text generation performs sampling by iteratively updating a continuous relaxation of text using gradients of the energy function. The resulting continuous text samples are then mapped back to the discrete space with a simple guided discretization approach, yielding text sequences that are fluent and adhere to the constraints.
Experiments are conducted in three text generation applications - lexically-constrained generation, abductive reasoning, and counterfactual reasoning and show the effectiveness of the COLD approach, both in terms
of automatic and human evaluation.

\begin{figure}[t]
	\centering
	\centerline{\includegraphics[width=1.0\linewidth]{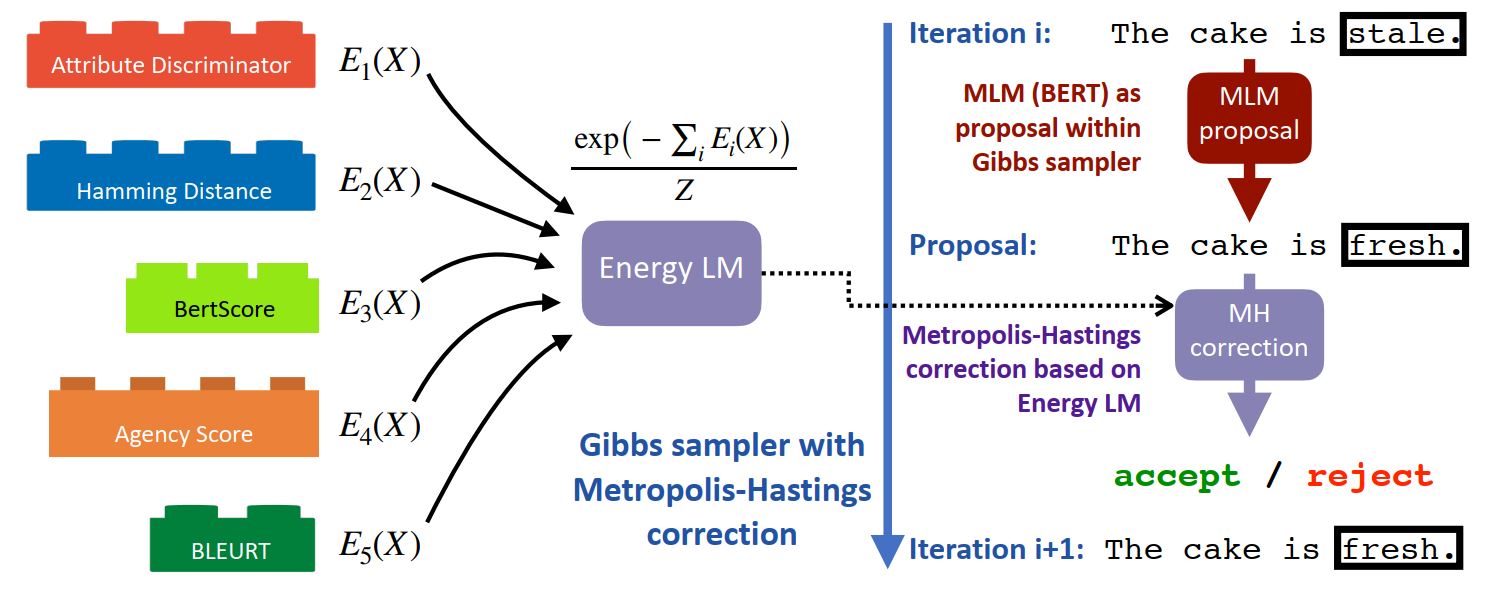}}
	\caption{Overview of mix-and-match LM. The Lego pieces show different experts that can be used to form the energy LM and help control different features in the generated text. The right side shows the $i$-th step in the the Gibbs sampling chain, where a proposal is made by the MLM, and then it is accepted/rejected based on the energy score. \citep{mixandmatch}}
	\label{fig:MixMatch}
\end{figure}

\subsubsection{The mix-and-match language model}
\index{Mix-and-match language model}

In \citep{mixandmatch}, the problem of performing controlled generation is framed as a problem of sampling from a specialized energy based (or globally normalized) sequence model that defines a probability distribution that satisfies the desired constraints we wish to impose in the controlled generation setting. 

As described below, this energy based model is composed of pre-trained components and does not require any further optimization. Specifically, an energy-based sequence model defines the probability distribution over the space of possible sequences $\mathcal{X}$ as:
\begin{displaymath}
p_\theta(X) = \frac{e^{- E_\theta(X)}}{\sum_{X' \in \mathcal{X}} e^{- E_\theta(X')}}
\end{displaymath}
where $E_\theta(X)$ refers to the scalar energy of a sequence $X$ that is parametrized by $\theta$. Lower energy corresponds to the higher likelihood of $X$. In contrast to the
common autoregressive sequence models, exact likelihood computation and efficient sampling
from EBM models is challenging. Despite these challenges, the EBM paradigm offers increased flexibility via sequence-level features and constraints. As we discuss next, this capability lets us easily define expressive functions for controlled generation of sequences which is not readily offered by the autoregressive modeling paradigm. 

Note that the task of controlled generation requires concentrating probability mass over a small subspace of sequences in $\mathcal{X}$ that satisfies various constraints pertaining to fluency, target attributes, and other control variables. 
The EBM based mix-and-match language model, as defined below, involves a linear combination of various black-box experts in order to obtain a distribution whose samples satisfy the requirements of a desired controlled generation task: 
\begin{equation}
    E_{\text{M\&M}}(X)=\sum_{i=1}^k \alpha_i E_i(X)
\end{equation}
where the mix-and-match (M\&M) energy is composed of $k$ expert energy components, which are weighted by scalar hyperparameters $\alpha$.
The following black-box experts are used in \citep{mixandmatch}:
\begin{itemize}
    \item $E_\text{mlm}(X)$: Masked language models (MLMs) like BERT \citep{bert} are used as a black-box to model the form and fluency of sentences.
    Particularly, $E_\text{mlm}(X)$ is defined as the negative of the sum of unnormalized logits iteratively computed at each position obtained via the forward pass of the MLM after masking the corresponding position \citep{goyal2021exposing}.
    \item $E_\text{disc}(X)$: This particular expert refers to the energy obtained via the discriminator for the attributes of interest. 
    What this module returns is the raw logits of the discriminator, for the target attribute. 
    For instance, if we have a sentiment classifier, and want to produce positive sentiment, 
    then $E_\text{disc}(X) = -\log p(+ |X)$.
    \item $E_\text{hamm}(X,X')$: For a given sequence $X'$, this quantity refers to the hamming distance between the sequence $X$ and $X'$. This penalizes token level deviation from $X'$ which is useful if we are interested in only making minor edits to $X'$.
    \item $E_\text{fuzzy}(X,X')$: Similar to the hamming distance, this quantity refers to the BertScore \citep{Zhang2020BERTScore} computed between $X$ and $X'$ which can be viewed as a fuzzy hamming distance that takes semantic similarity into account.
\end{itemize}

\subsubsection{Sampling from the mix-and-match model}
\label{sec:sampling_mix_and_match}
In \citep{mixandmatch}, Metropolis-Hastings (MH) based MCMC scheme (\secref{sec:MCMC}) is used to sample from the M\&M model. The proposal distribution is implemented by a masked language model (MLM) like BERT. At each MH step, we mask out a token at a random position $i$ in current sequence $X$, propose a new token 
by sampling from the MLM conditional softmax at the masked position\footnote{Note that the proposed move $\overline{X}_i$ is generated independent of the previous state $X_i$, thus the sampling algorithm used here is in fact MIS with Gibbs sampling.}, and obtain the new sequence $\overline{X}$. The new sequence is accepted with the probability 
\begin{displaymath}
\min\left\lbrace 1, \frac{\exp(-E_\text{M\&M}(\overline{X}))p_\text{mlm}(X_i|X_{\backslash i})}{\exp(-E_\text{M\&M}(X))p_\text{mlm}(\overline{X}_i|X_{\backslash i})}\right\rbrace 
\end{displaymath}
In experiments in \citep{mixandmatch}, a MCMC chain is run for sentence generation for more than 8 epochs, where an epoch refers to one masking cycle over all positions of the sequence.

\subsubsection{The performance of the mix-and-match model}
Two kinds of controlled generation tasks are performed in \citep{mixandmatch}. 
\paragraph{Prompted generation.} This task focuses on generating well-formed sentences that start with a specified prompt and also satisfy a target attribute for which we have access to a discriminator.
An example task would be to generate positive sentiment sequences starting with \texttt{This movie}.
The energy function takes the form:
\begin{displaymath}
E(X) = E_\text{mlm}(X)+\alpha E_\text{disc}(X)
\end{displaymath}
where $\alpha$ is a hyperparameter that controls the tradeoff between the MLM score and the discriminator's influence.
\paragraph{Controlled text revision.} This task aims to edit a source sequence $X'$ to generate sequence $X$ to satisfy the desired target attributes. The energy function for this task is :
\begin{displaymath}
E(X)=E_\text{mlm}(X)+\alpha E_\text{disc}(X)+\beta E_\text{hamm}(X,X')+\gamma E_\text{fuzzy}(X,X')
\end{displaymath}
This energy function, in addition to valuing well-formedness and satisfying target attribute requirements, also focuses on maintaining faithfulness to the source sequence $X'$.


In \citep{mixandmatch}, the effectiveness of the mix-and-match approach has been shown on controlled generation tasks (with sentiment or topic) and style-based text revision tasks (controllable debiasing, style transfer), by outperforming recently proposed methods that involve extra training, fine-tuning, or restrictive assumptions over the form of models.


\subsubsection{Constrained decoding with Langevin dynamics (COLD)}
\index{Constrained decoding with Langevin dynamics (COLD)}
\label{sec:COLD}

Constrained text generation aims to produce text samples $y$ that satisfy a set of constraints (usually
conditioned on an input $x$ omitted for brevity).
Let $y = (y_1,\cdots, y_T)$ denote a discrete sequence where each $y_t$ is a token from a vocabulary $\mathcal{V}$. 
Assume each constraint can be captured by a
constraint function $f_i(y) \in \mathbb{R}$, where higher values of $f_i$ mean that the text y better satisfies the
constraint. 
Generating text under the constraints can be
seen as sampling from a energy-based distribution $y \sim p(y)$:
\begin{displaymath}
p(y) \propto \exp \left\lbrace U(y) \right\rbrace, U(y) = \sum_i \lambda_i f_i(y)
\end{displaymath}
EBMs are flexible - any differentiable function that outputs a goodness score of text can be used as a constraint function, as long as it reflects the requirements of the target task. Three example constraints are shown in \citep{qin2022cold}: soft fluency constraint, future-token prediction constraint, and N-gram similarity constraint.

For efficient sampling from $p(y)$, we want to use Langevin dynamics, which makes use of the gradient $\nabla_y \log p(y) = \nabla_y U(y)$. 
However, in text generation, $y$ is a discrete sequence and the gradient $\nabla_y \log p(y)$ is not well-defined.
To address this problem, \citep{qin2022cold} proposed a new sampling/decoding method, called \emph{Constrained decoding with Langevin dynamics} (COLD), which consists of three steps - continuous relaxation of text, performing Langevin dynamics with an energy defined on a sequence of continuous ``soft'' token vectors, and finally, guided discretization.

\paragraph{Continuous relaxation of text.}

Instead of defining the energy function on discrete tokens, the energy function is defined on a sequence of continuous vectors $\tilde{y} = (\tilde{y}_1,\cdots,\tilde{y}_T)$, which is called a soft sequence. 
Each position in the soft sequence is a vector of logits $\tilde{y}_t \in \mathbb{R}^{|\mathcal{V}|}$, each element corresponding to the logit of a word in the vocabulary.
Taking the softmax of $\tilde{y}_t$ yields a distribution over the vocabulary for position $t$. An EBM model can be defined on the soft sequence $\tilde{y}$ as follows:
\begin{equation}
\label{eq:COLD_energy}
p(\tilde{y}) \propto \exp \left\lbrace U(\tilde{y}) \right\rbrace
\end{equation}

\paragraph{Langevin dynamics over soft tokens.}
In the continuous space, we can perform gradient guided MCMC (\secref{sec:LD_HMC}) such as Langevin dynamics to generate samples $\tilde{y}^{(1)}, \tilde{y}^{(2)}, \cdots$, which ends
up generating samples from the distribution induced by the energy function \eqref{eq:COLD_energy}.

\paragraph{From soft tokens to discrete text.}
After receiving a soft sequence sample $\tilde{y}$ from running Langevin dynamics, we map the soft sequence to a discrete text sequence, which we consider as the output of COLD decoding. 
A simple method would be selecting the most-likely token at each position $t$, 
\begin{displaymath}
y_t = \argmax_{v \in \mathcal{V}} \tilde{y}_t(v)
\end{displaymath}
However, the resulting text suffers from fluency issues even if the soft fluency constraint is used, due to competing constraints that sacrifice fluency. 
To overcome this, COLD uses the underlying pre-trained language model (PLM) as a ``guardian'' for obtaining the discrete sequence. 
Specifically, at each position $t$, we first use the PLM to produce the top-$k$ most-likely candidate tokens based on its generation distribution conditioning on preceding tokens, which we denote as $\mathcal{V}_t^k$. 
We then select from the top-$k$ candidates the most likely token based on the soft sample $\tilde{y}$: 
\begin{displaymath}
y_t = \argmax_{v \in \mathcal{V}_t^k} \tilde{y}_t(v)
\end{displaymath}
We refer to this method as ``top-$k$ filtering''. 
The resulting text tends to be fluent because each token is among the top-$k$ most probable tokens from the PLM.

\chapter{Joint EBMs with applications}
\label{ch:jem}

In this chapter, we introduce EBMs for modeling \emph{joint} distributions for both fixed-dimensional and sequential data, with the applications for semi-supervised learning, training more calibrated models, and improved language modeling with additional relevant linguistic tags (e.g., part-of-speech tags).
First, we present basics for semi-supervised learning. 
Then, we introduce the fixed-dimensional case of joint EBMs (\secref{sec:JEM}), then move on to the sequential case (\secref{sec:JRF}). 
Finally, we present the application of EBM-based joint modeling for semi-supervised learning and calibrated natural language understanding in \secref{sec:JRF_semi} and \secref{sec:JRF_calibrate}, respectively.

\section{Basics for semi-supervised learning}
\label{sec:SSL}
\index{Semi-supervised learning (SSL)}

As we have witnessed, supervised learning from large amounts of labeled data, particularly with deep neural networks (DNNs)\index{Deep neural network (DNN)}, has achieved tremendous success in various intelligence tasks, spanning speech processing, computer vision, and natural language processing (NLP)\index{Natural language processing (NLP)}.
However, collecting labeled data is difficult and expensive, but there are often easily-available unlabeled data.
This has motivated the community to develop \emph{semi-supervised learning} (SSL).
SSL aims to leverage both labeled and unlabeled data to train a conditional model for inference (for either discriminative or generation tasks), which, basically, is to find the posteriori of label $y$ given observation $x$.
A plethora of semi-supervised learning methods have emerged to train deep neural networks (DNNs) \citep{miyato2018virtual,laine2016temporal,tarvainen2017mean,sohn2020fixmatch,chen2020simple,cvt}, spanning over various domains such as image classification, natural language labeling and so on.
People may wonder how unlabeled observations $x$'s may help finding the posterior.
A common belief is that the the information contained in the unlabeled observations can provide some kind of priors, or alternatively say, some regularizations or inductive biases, for finding the posterior $p(y|x)$.


Roughly speaking, recent SSL methods with DNNs 
can be distinguished by the priors they adopt, and, can be divided into two classes\footnote{We mainly discuss the SSL methods for using DNNs.
	General discussion of SSL can be referred to \citep{zhu2006semi}.} - based on generative models or discriminative models, which are referred to as \emph{generative SSL}\index{Generative SSL} and \emph{discriminative SSL}\index{Discriminative SSL},  respectively.
An intuitive way to differentiate between generative and discriminative SSL is that generative SSL typically requires modeling the marginal distribution of the data, while discriminative SSL only involves modeling the conditional distribution.

\subsection{Discriminative SSL}
\index{Discriminative SSL}

Discriminative SSL methods often assume that the outputs from the discriminative classifier are smooth with respect to local and random perturbations of the inputs.
 Examples include virtual adversarial training (VAT) \citep{miyato2018virtual} and a number of recently developed consistency-regularization based methods
 such as temporal ensembling \citep{laine2016temporal}, mean teachers \citep{tarvainen2017mean}, FixMatch \citep{sohn2020fixmatch} and contrastive learning based methods such as SimCLR \citep{chen2020simple}.
 
Discriminative SSL methods thus heavily rely on domain-specific data augmentations, e.g., RandAugment \citep{cubuk2020randaugment}, which are tuned intensively for images and lead to impressive performance in some image domains. But discriminative SSL is often less successful for other domains, where these augmentations are less effective (e.g., medical images and text). For instance, random input perturbations are more difficult to apply to discrete data like text \citep{cvt}.
Although there are some efforts to use data-independent model noises, e.g., by dropout \citep{srivastava2014dropout}, domain-specific data augmentation is indispensable.

\subsection{Generative SSL}
\label{sec:generativeSSL}
\index{Generative SSL}

Generative SSL methods exploit unsupervised learning of generative models over unlabeled data, which inherently does not require data augmentations and generally can be applied to a wider range of domains.
Generative SSL usually involves blending unsupervised learning and supervised learning. These methods make fewer domain-specific assumptions and tend to be domain-agnostic.
The performance comparisons between generative and discriminative SSL methods are mixed.
It is found that consistency based discriminative SSL methods often outperform generative SSL methods in image domain.
However, in text domain, the generative SSL methods such as those based on pre-trained word vectors \citep{mikolov2013distributed,pennington2014glove} and pre-trained language models (PLMs)\index{Pre-trained language model (PLM)} \citep{gpt2,raffel2020t5} are more successful and widely used.

Considering observation $x$ and label $y$, there exist two different methods for the generative SSL approach - \emph{joint-training}\index{Joint-training for generative SSL} \citep{larochelle2012learning,Kingma2014SemiSupervisedLW} and \emph{pre-training}\index{Pre-training for generative SSL} \citep{hinton2006a}.
\begin{itemize}
\item In joint-training, a joint model of $p(x,y)$ is defined.
When we have label $y$, we maximize $p(y|x)$ (the supervised objective), and when the label is unobserved, we marginalize it out and maximize $p(x)$ (the unsupervised objective). Semi-supervised learning over a mix of labeled and unlabeled data is formulated as maximizing the (weighted) sum of $\log p(y|x)$ and $\log p(x)$.
Given infinite model capacity and data, the joint-learning based SSL is consistent and can find the oracle $p(y|x)$\footnote{This is because that the maximum likelihood estimator is consistent.}.
\item  In pre-training, we perform unsupervised representation learning on unlabeled data, which is followed by supervised training (called fine-tuning) on labeled data.
Thus, pre-training is usually based on a model that only defines the marginal distribution $p(x)$ without the need to involve $y$.
Its empirical effectiveness is mostly understood from the perspective of representation learning. There is no theoretical guarantee that this kind of pre-training will lead to finding the oracle $p(y|x)$.
This method of pre-training followed by fine-tuning has been widely used in natural language processing \citep{radford2018improving}\footnote{But, more recently in NLP, an approach of jointly modeling the $input$, the $output$, and the $task$ description, has emerged to gain more attention and achieved superior performances \citep{gpt2}.
	The $input$, the $output$, and the $task$ description can all be specified as a sequence of tokens, and a language model is trained for estimating natural language sequences so that $p(task, input, output)$ for various tasks, inputs and outputs are implicitly trained.}. 
\end{itemize}

\begin{figure}[t]
	\centering
	\centerline{\includegraphics[width=1.0\linewidth]{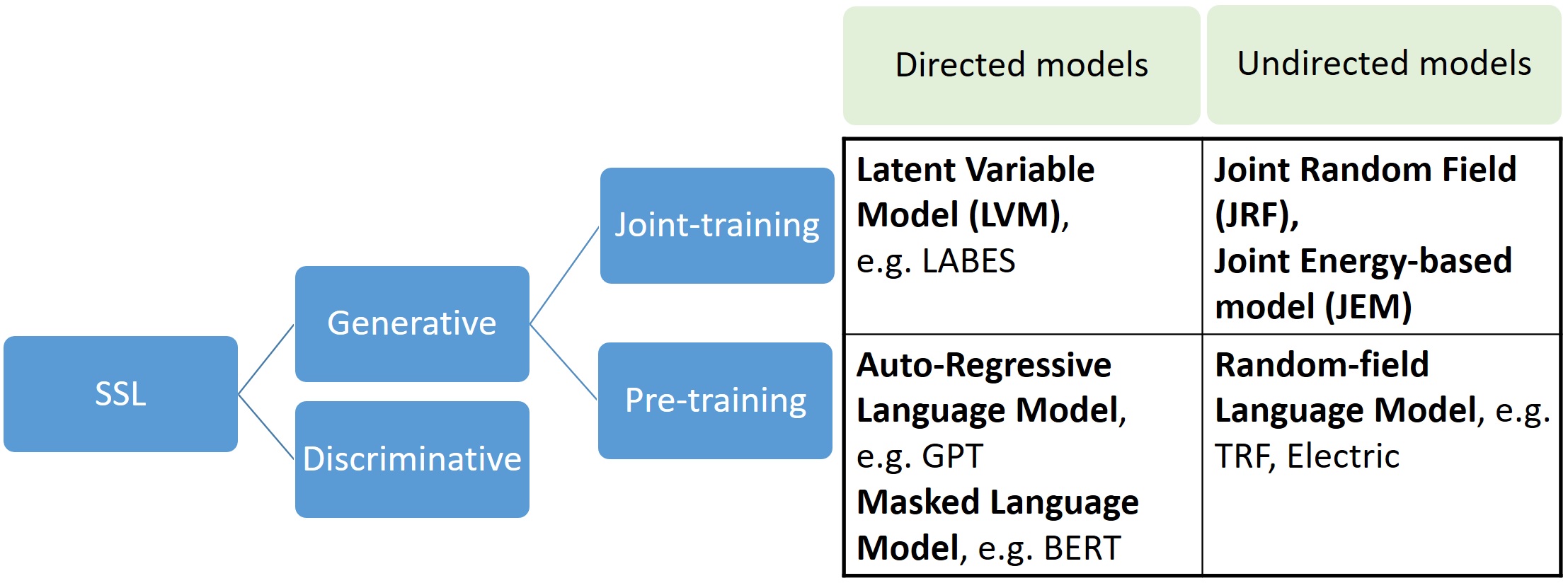}}
	\caption{An overview of SSL and a general categorization of generative SSL methods. Examples are mainly chosen from NLP.}
	\label{fig:GenerativeSSL}
\end{figure}

\paragraph{A categorization of generative SSL methods} is shown in \figref{fig:GenerativeSSL}.
For either of joint-training and pre-training, we can use directed models (locally-normalized) or undirected models (globally-normalized).
So there are four main classes for generative SSL.
The models used in joint-training could be latent-variable model (LVM)\index{Latent-variable model (LVM)} such as in LABES \citep{zhang-etal-2020-probabilistic}, or JRF \citep{jrf} or say JEM \citep{zhaojoint}.
Pre-training could be based on masked language models \citep{bert}, autoregressive language models \citep{radford2018improving}, or random-field language models \citep{Wang2017LearningTR}. 
Further, two approaches of pre-training and joint-training could be combined or compared to each other.
So there are many open questions in designing semi-supervised methods for particular tasks.

\label{sec:EBM_SSL}
\paragraph{EBM based SSL.}
Among existing generative SSL methods, energy-based models (EBMs), as an important class of generative models, have been shown with promising results for semi-supervised learning across various domains.
Early studies date back to the pre-training of restricted Boltzmann machines (RBMs) \citep{hinton2006a} (which are a simple type of EBMs) and the joint-training with classification RBMs \citep{larochelle2012learning}.

Recently, it is shown in \citep{song2018learning} that joint-training via EBMs produces state-of-the-art SSL results on images (MNIST, SVHN and CIFAR-10), compared to previous  generative SSL methods based on Variational AutoEncoders (VAEs) and Generative Adversarial Networks (GANs).
It is also shown in \citep{zhaojoint} that joint-training via EBMs outperforms VAT (virtual adversarial training) \citep{miyato2018virtual} on tabular data from the UCI dataset repository.
Further, joint-training via EBMs has been extended in \citep{jrf} to modeling sequences and consistently outperforms conditional random fields (CRFs) (the supervised baseline) and self-training (the classic semi-supervised baseline) on natural language labeling tasks such as POS (part-of-speech) tagging, chunking and NER (named entity recognition).


Note that although both joint-training and pre-training of EBMs have been used for SSL in the literature, very few studies evaluated and compared the two methods. The results from previous individual works are often not directly comparable to each other, since they are not evaluated in a common experimental setup.
In \citep{song2021empirical}, a suite of experiments are conducted to systematically compare joint-training and pre-training for EBM-based SSL.
Both the amounts of labeled and unlabeled data are varied to give a realistic whole picture of the performances of the two methods for SSL \citep{oliver2018realistic}.
It is found that joint-training EBMs outperform pre-training EBMs marginally but nearly consistently, presumably because the optimization of joint-training is directly related to the targeted task, while pre-training does not.

In the remainder of this chapter, we will detail EBM based SSL, including both pre-training and joint-training.
Note that pre-training only involves the \emph{marginal} distribution $p(x)$, so the basics for EBM based pre-training can be referred to \secref{ch:lm} for learning with discrete $x$ such as natural languages, and for learning with continuous $x$ such as images, be referred to \citep{song2018learning}.
We will mainly introduce the basics for EBM based joint-training, i.e., establishing EBMs for modeling \emph{joint} distributions, first in the fixed-dimensional case (\secref{sec:JEM}) and then in the sequential case (\secref{sec:JRF}). 
Although speech and language processing is primarily concerned with the sequential case, the introduction of the fixed-dimensional case can lay the foundation for understanding the more complicated, sequential case.

\section{Upgrading EBMs to Joint EBMs (JEMs) for fixed-dimensional data}
\label{sec:JEM}

\paragraph{Motivation.} 
Originally, EBMs are established for modeling the marginal distribution $p(x)$ of observations $x$, as introduced in \secref{ch:lm}.
Recently, a kind of EBM, called \emph{semi-supervised EBM}, for modeling the joint distribution of observation $x$ and label $y$ has been developed and used for semi-supervised classification \citep{song2018learning}.
In \citep{grathwohl2019your}, a similar kind of EBM has been studied, which is called \emph{joint EBM} (JEM)\index{Joint EBM (JEM)}.
It is established by reinterpreting a standard discriminative classifier of $p(y|x)$ as an energy based model for the joint distribution $p(x, y)$. 
It is demonstrated that energy based training of the joint distribution improves calibration, robustness, and out-of-distribution detection while also enabling sample generation which rivals the quality of recent GAN approaches.
Hereafter, these EBMs are collectively referred to as JEMs.

\paragraph{Model definition.} 
Note that different models are needed in unsupervised and semi-supervised learning, because SSL needs to additionally consider labels apart from observations.

In semi-supervised tasks, we consider the following EBM for joint modeling of observation $x \in \mathbb{R}^{d_x}$ and class label $y \in \left\lbrace 1,\cdots,K \right\rbrace$:
\begin{equation}
\label{eq:semi-nrf-joint}
p_{\theta}(x,y)=\frac{1}{Z(\theta)} \exp\left[  U_{\theta}(x,y) \right].
\end{equation}
This is different from Eq. (\ref{eq:unsup-RF}) for unsupervised learning which only models $x$ without labels.
To implement the potential function $U_\theta(x,y)$, we consider a neural network $\Psi_\theta(x):\mathbb{R}^{d_x} \rightarrow \mathbb{R}^K$, with $x$ as the input and the output size being equal to the number of class labels, $K$. 
Then we define 
\begin{displaymath}
U_\theta(x,y) = {onehot}(y)^T \Psi_\theta(x)
\end{displaymath}
where ${onehot}(y)$ represents the one-hot encoding vector for the label $y$.
In this manner, the conditional density $p_{\theta}(y|x)$ is the classifier, defined as follows:
\begin{equation}\label{eq:rf-classifier}
p_{\theta}(y|x) = \frac{p_{\theta}(x,y)}{p_{\theta}(x)}
= \frac{\exp\left[  U_{\theta}(x,y) \right]}{\sum_{y'} \exp\left[  U_{\theta}(x,y') \right]}
\end{equation}
which acts like multi-class logistic regression using $K$ logits calculated from $x$ by the neural network $\Psi_\theta(x)$. And we do not need to calculate $Z(\theta)$ for classification.

With the definition the joint density in \eqref{eq:semi-nrf-joint}, it can be shown that, with abuse of notation, the marginal density is
\begin{equation}
\label{eq:jem-marginal}
p_\theta(x) = \frac{1}{Z(\theta)} \exp\left[  U_{\theta}(x) \right]
\end{equation}
where $U_\theta(x) \triangleq \log \sum_{y'} \exp\left[  U_{\theta}(x,y') \right]$.

\paragraph{Model learning.}
Suppose that among the data $\mathcal{D} = \left\lbrace {x}_1, \cdots, {x}_n \right\rbrace $, only a small subset of the observations, for example the first $m$ observations, have class labels, $m \ll n$.
Denote these labeled data as $\mathcal{L} = \left\lbrace({x}_1,{y}_1), \cdots, ({x}_m,{y}_m) \right\rbrace $.
Let $p_{\text{emp}}$ denote the empirical distribution over $\mathcal{D}$.
Then we can formulate the semi-supervised learning as optimizing
\begin{equation}
\label{eq:jem_semi_obj}
\min_{\theta} \left\lbrace KL\left[  p_{\text{emp}}({x}) || p_\theta({x}) \right]
- \alpha_d \sum_{({x},{y}) \sim \mathcal{L}} \log p_\theta({y}|{x}) \right\rbrace 
\end{equation}
which are defined by hybrids of generative and discriminative criteria, similar to \citep{zhu2006semi,larochelle2012learning,Kingma2014SemiSupervisedLW}.
The hyper-parameter $\alpha_d$ controls the relative weight between generative and discriminative criteria.

Once the JEM model and the training objective are established, 
the inclusive-NRF training algorithm (\algref{alg:learning-NRF-IAG})\index{Inclusive-NRF algorithm} developed in \citep{song2018learning} can be applied. Other algorithms for training EBMs, which can achieve proper density estimation such as those recently developed in \citep{zhang2023persistently}, can also be employed.

\section{Upgrading CRFs to Joint random fields (JRFs) for sequential data}
\label{sec:JRF}

\paragraph{Motivation.}
Probabilistic generative modeling is a principled methodology that promisingly can learn from data of various forms (whether labeled or unlabeled) to benefit downstream tasks, which, however, is particularly challenging for sequence data.
Two important sequence tasks are sequence modeling and sequence labeling.

A basic problem of \emph{sequence modeling}\index{Sequence modeling} is to determine the probabilities of sequences. An example is language modeling \citep{chen1999empirical}, which, as described in \secref{sec:ELM}, is a crucial component in many speech and language processing tasks.
For sequences of length $l$, $x^l \triangleq x_1,x_2,...,x_l$, this amounts to calculate $p(l,x^l)$, where we make explicit the role of the length $l$. Ideally, this density modeling can be improved with additional relevant labels. e.g. incorporating part-of-speech (POS)\index{Part-of-speech (POS)} tags for language modeling. 
There are some previous studies in \citep{rosenfeld2001whole,ModelM}.
The difficulty is that the labels (e.g. POS) usually are not available in testing, so a standalone POS tagger is needed to provide hypothesized labels in testing.

The task of \emph{sequence labeling}\index{Sequence labeling}, as described in \secref{sec:crf_labeling},
is, given observation sequence $x^l$, to predict the label sequence $y^l\triangleq y_1,y_2,...,y_l$, with one label for one observation at each position.
Sequence labeling has been widely applied in various tasks, e.g., POS labeling \citep{fromscratch,ling2015finding}, named entity recognition (NER) \citep{huang2015bidirectional,lample2016neural,ma2016end}, and chunking \citep{huang2015bidirectional,sogaard2016deep}.
As introduced in \secref{sec:generativeSSL}, it is desirable for the labeling model to leverage both labeled data (namely pairs of $x^l$ and $y^l$) and unlabeled data (namely $x^l$ without labels), i.e., to conduct semi-supervised learning (SSL).
Pre-training has proved to be effective \citep{fromscratch,bert}, which, however, is task-independent followed by task-dependent fine-tuning. 
Besides pre-training, it is worthwhile to explore task-dependent semi-supervised learning (SSL) in the manner of joint-training, which learns for a task on a mix of labeled and unlabeled data. Self-training is such a method with limited success \citep{scudder1965probability}. 

As introduced in \secref{sec:cond_basics}, conditional random fields (CRFs) \citep{lafferty2001conditional} have been shown to be one of the most successful approaches to sequence labeling.
A CRF is a discriminative model, which directly defines a conditional distribution 
$p(y^l|x^l)$, and thus mainly depends on supervised learning with abundant labeled data.
It is proposed in \citep{jrf} to upgrade CRFs and obtain a joint generative model of $x^l$ and $y^l$, $p(l,x^l,y^l)$, called \emph{joint random fields} (JRFs)\index{Joint random field (JRF)}.
Specifically, the potential function $U(x^l,y^l)$ in the original CRF $p(y^l|x^l)$ is borrowed as the potential function that defines a joint distribution $p(x^l,y^l)$.
This upgrading, operated in the sequential setting, is similar to upgrade EBMs to JEMs in the fixed-dimensional setting \citep{song2018learning,grathwohl2019your}.
Similar to the fixed-dimensional setting, the conditional distribution of $y^l$ given $x^l$ induced from this joint distribution is exactly the original conditional distribution - the CRF $p(y^l|x^l)$\footnote{So writing the JRF as $p(l,x^l,y^l)$ and the CRF as $p(y^l|x^l)$ is correct, not an abuse of notation.}; and the marginal distribution of $p(l,x^l)$ induced from the joint distribution is a 
trans-dimensional random field (TRF) language model \citep{Wang2017LearningTR,wang2018improved}, as described in \secref{sec:TRF_modeldef}.

This development from CRFs to JRFs benefits both modeling and labeling of sequence data.
For sequence modeling, the marginal likelihood $p(l,x^l)$ can be efficiently calculated by JRFs, without the step of producing hypothesized labels by a standalone POS tagger.
For sequence labeling, JRFs admit not only supervised learning from labeled data by maximizing the conditional likelihood $p(y^l|x^l)$ (which is like the training of a CRF), but also unsupervised learning from unlabeled data by maximizing the marginal likelihood $p(l,x^l)$ (which is like the training of a TRF LM), thereby achieving task-dependent semi-supervised learning.

\paragraph{Model definition.}
We will first briefly review linear-chain CRFs, as described in \secref{sec:linear_chain_crf}, but with different notations, which facilitate the introduction of JRFs.
A linear-chain CRF defines a conditional distribution with parameters $\theta$ for label sequence $y^l$ given observation sequence $x^l$ of length $l$:
\begin{equation}
\label{eq:crf}
p_\theta(y^l|x^l)=\frac{1}{Z_\theta(x^l)}\exp(U_\theta(x^l,y^l))    
\end{equation}
Here the potential function 
\begin{equation}
\label{eq:u-theta}
U_\theta(x^l,y^l)=\sum_{i=1}^l \phi_i(y_i,x^l)+\sum_{i=1}^l\psi_i(y_{i-1},y_i,x^l), 
\end{equation}
is defined as a sum of node potentials and edge potentials, and $Z_\theta(x^l)=\sum_{y^l}\exp(U_\theta(x^l,y^l))$ is the normalizing constant.
$\phi_i(y_i,x^l)$ is the node potential defined at position $i$, which, in recently developed neural CRFs \citep{fromscratch,huang2015bidirectional,lample2016neural,ma2016end} is implemented by using features generated from a neural network (NN) of different network architectures.
$\psi_i(y_{i-1},y_i,x^l)$ is the edge potential defined on the edge connecting $y_{i-1}$ and $y_i$, often implemented as a matrix $A$, $\psi_i(y_{i-1}=j,y_i=k,x^l)=A_{j,k}$, thereby defines a linear-chain CRF.
In linear-chain CRFs, there are efficient algorithms for training (the forward-backward algorithm) and decoding (the Viterbi algorithm).


Inspired from the idea of jointly modeling fixed-dimensional observations (e.g. images) and labels via JEMs in \citep{song2018learning,grathwohl2019your}, CRFs can be similarly upgraded and a joint distribution over sequential observations and labels can be established, called JRFs \citep{jrf}.
The keypoint is that we can use $U_\theta(x^l,y^l)$ in Eq. (\ref{eq:u-theta}) from the original CRF to define a joint distribution $p_\theta(x^l,y^l)$:
\begin{equation}
\label{eq:jrf}
p_\theta(l,x^l,y^l)=\pi_l p_\theta(x^l,y^l;l)=\frac{\pi_l}{Z_{\theta}(l)}\exp(U_\theta(x^l,y^l))
\end{equation}
where
$\pi_{l}$ is the empirical prior probability for length $l$.
Notably, a CRF is a conditional distribution, normalizing over label sequences. In contrast, a JRF is a joint distribution, normalizing over both observation and label sequences, with the normalizing constant for length $l$ defined as
\begin{equation}
\label{eq:JRF_Z_l}
Z_{\theta}(l)=\sum_{x^l,y^l}\exp(U_\theta(x^l,y^l)).
\end{equation}

Interestingly, it can be easily seen that the conditional distribution of $y^l$ given $x^l$ induced from the JRF's joint distribution Eq. (\ref{eq:jrf}) is exactly the original CRF Eq. (\ref{eq:crf}).
Further, by marginalizing out $y^l$, the marginal distribution of $p(l,x^l)$ induced from the joint distribution is:
\begin{displaymath}
\label{eq:trf in jrf}
p_\theta(l,x^l)=\frac{\pi_l}{Z_{\theta}(l)}\sum_{y^l}\exp(U_\theta(x^l,y^l))=\frac{\pi_l}{Z_{\theta}(l)}\exp(U_\theta(x^l))
\end{displaymath}
which acts like a 
trans-dimensional random field (TRF) language model \citep{Wang2017LearningTR,wang2018improved}, with the potential defined by
\begin{displaymath}
U_\theta(x^l)=\log\sum_{y^l}\exp(U_\theta(x^l,y^l)).
\end{displaymath}
Notably this marginal potential $U_\theta(x^l)$ is exactly the normalizing constant $\log Z_{\theta}(x^l)$ \eqref{eq:JRF_Z_l} from the CRF.
It can be calculated via the forward algorithm from the linear-chain potential $U_\theta(x^l,y^l)$.

\begin{figure}
	\centering
	\includegraphics[width=1.0\textwidth]{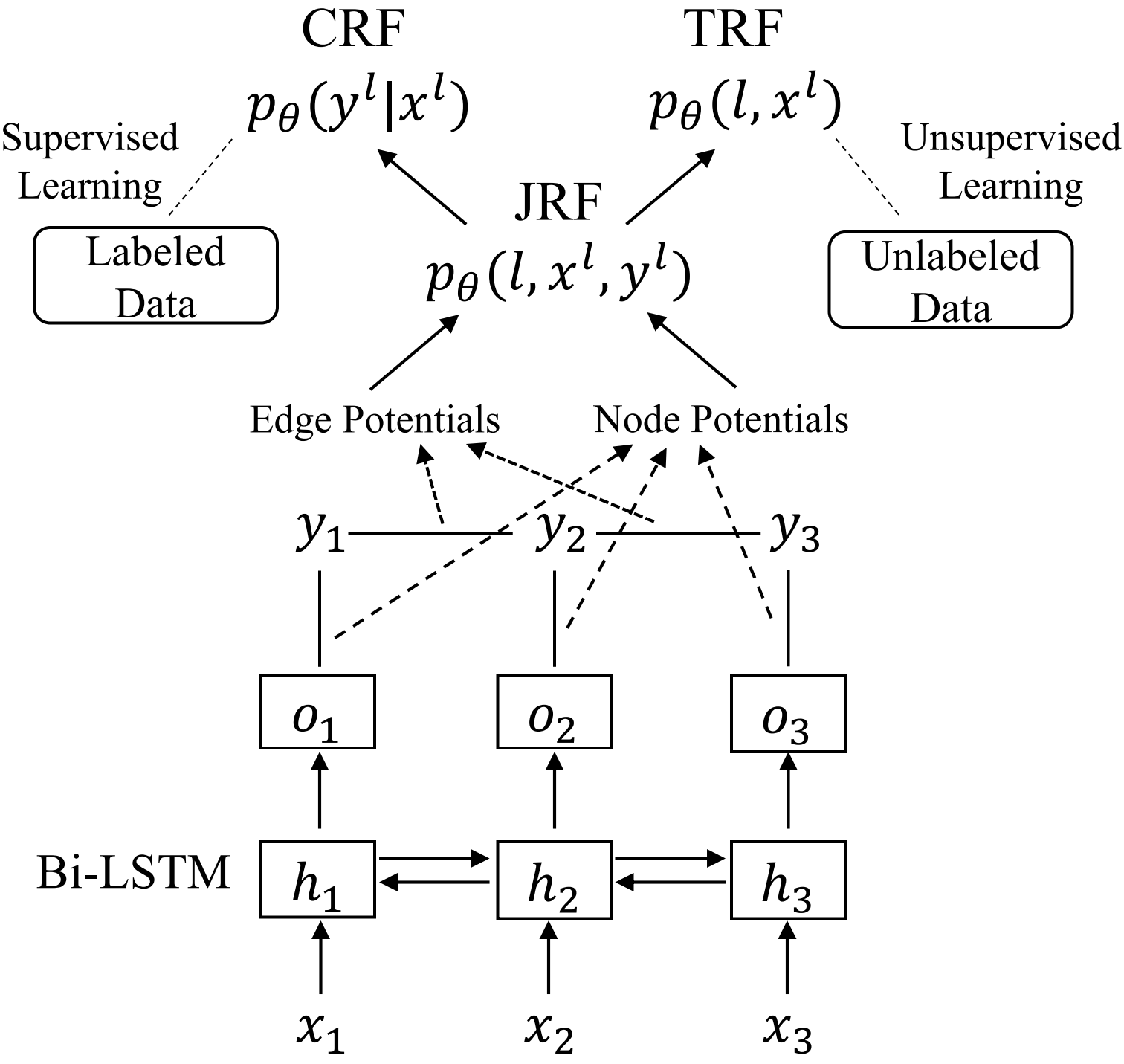}
	\caption{
		Overview of the JRF model.
		The node and edge potentials define a JRF (a joint distribution over $x^l$ and $y^l$).
		Inducing the conditional and the marginal from the joint yields a CRF and a TRF respectively.
		A JRF can be trained from labeled data (acting like a CRF) and also from unlabeled data (acting like a TRF).
		In practice, the node potentials are calculated from the logits $o_i, i=1,\cdots,l$, from the NN, and the edge potential follows a linear-chain definition.
	}
	\label{fig:JRF_overview}
\end{figure}

\paragraph{Model learning.}
Given different data resources (labeled or unlabeled), JRF can be trained under different settings (supervised, unsupervised, or semi-supervised) and applied in different downstream tasks (sequence modeling or labeling), as illustrated in Fig. \ref{fig:JRF_overview}.
Note that the embedded CRF and TRF inside a JRF share all parameters $\theta$, which is different from multi-task learning where only bottom-level parameters are shared \citep{argyriou2007multi}.

\emph{Supervised learning of JRFs}
amounts to the training of the embedded CRF with the following supervised objective, given labeled data in the form of empirical distribution $p_\text{lab}(x^l,y^l)$,
\begin{equation}
\label{eq:jrf_sup} 
\mathop {\max }\limits_\theta  J_\text{sup}(\theta)=E_{(x^l,y^l)\sim p_\text{lab}(x^l,y^l)}[\log p_\theta(y^l|x^l)]
\end{equation}
which can be solved by applying minibatch-based stochastic gradient descent (SGD).
At each iteration, a minibatch of sentences and labels is sampled from $p_\text{lab}(x^l,y^l)$, denoted by $\mathcal{D}_\text{lab}$, and the stochastic gradients are:
\begin{displaymath}
\reallywidehat{\frac{\partial J_\text{sup}(\theta)}{\partial \theta}}= \frac{1}{|\mathcal{D}_\text{lab}|}\sum_{(x^l,y^l)\in \mathcal{D}_\text{lab}} \left[\frac{\partial U_\theta(x^l,y^l)}{\partial\theta} - \frac{\partial U_\theta(x^l)}{\partial\theta}\right]
\end{displaymath}

\emph{Unsupervised learning of JRFs} amounts to the training the embedded TRF, by applying the dynamic noise-contrastive estimation (DNCE) algorithm developed in \citep{wang2018improved}.
Given unlabeled data (e.g. sentences)  in the form of empirical distribution $p_\text{unl}(l,x^l)$, DNCE jointly optimizes over a JRF and a noise distribution $p_\phi(l,x^l)$ (generally a LSTM language model) parameterized by $\phi$:
\begin{equation}
\label{eq:jrf-uns}
\left\{ 
\begin{split}
\mathop {\max }\limits_\theta  &{E_{(l,{x^l})\sim \frac{p_\text{unl}(l,x^l)+p_\phi(l,x^l)}{2} }} \left[  \log\frac{{{p_\theta}(l,{x^l})}}{{{p_\theta }(l,{x^l}) + {p_\phi }(l,{x^l})}} \right]  + \\&{E_{(l,{x^l})\sim{p_\phi }(l,{x^l})}}  \left[ \log\frac{{{p_\phi }(l,{x^l})}}{{{p_\theta }(l,{x^l}) + {p_\phi }(l,{x^l})}} \right]  \triangleq J_\text{uns}(\theta)\\
\mathop {\min }\limits_\phi  &KL\left[ {p_\text{unl}}(l,{x^l})||{p_\phi }(l,{x^l})\right] 
\end{split}
\right. 
\end{equation}
Thanks to optimization of DNCE, the annoying normalizing constants $Z_{\theta}(l)$ in JRFs can be jointly estimated along with the parameter estimation.
Specifically, we introduce parameters $\zeta_l$ for $\log Z_\theta(l)$ and $\zeta=(\zeta_1,\zeta_2,...,\zeta_L)$, where $L$ is a pre-defined maximum length.
Hereafter, we denote by $\xi=(\theta,\zeta)$ all the parameters in the JRF and rewrite $p_\theta(\cdot)$ as $p_\xi(\cdot)$.

At each iteration, a minibatch of sentences are sampled from $p_\text{unl}(l,x^l)$, denoted by $\mathcal{D}_\text{unl}$, two minibatches of sentences sampled from $p_\phi(l,x^l)$, denoted by $\mathcal{B}_1, \mathcal{B}_2$ ($|\mathcal{B}_2|=2|\mathcal{B}_1|=2|\mathcal{D}_\text{unl}|$), 
and the stochastic gradients are:
\begin{displaymath}
\left\{ 
\begin{split}
&\reallywidehat{\frac{\partial J_\text{unl}(\xi)}{\partial \xi}} =  \frac{1}{|\mathcal{B}_2|}\sum_{(l,x^l)\in \mathcal{B}_2} \frac{{{p_\xi}(l,{x^l})}}{{{p_\xi }(l,{x^l}) + {p_\phi }(l,{x^l})}} g(l,x^l;\xi)\\
&~~~~~~~+ \frac{1}{|\mathcal{D}_\text{unl}|+|\mathcal{B}_1|}\sum_{(l,x^l)\in \mathcal{D}_\text{unl} \cup \mathcal{B}_1} \frac{{{p_\phi }(l,{x^l})}}{{{p_\xi }(l,{x^l}) + {p_\phi }(l,{x^l})}} g(l,x^l;\xi)\\
&\reallywidehat{ \frac{\partial KL(p_\text{unl}||p_\phi)}{\partial \phi}} =  - \frac{1}{|\mathcal{D}_\text{unl}|}\sum_{(l,x^l)\in \mathcal{D}_\text{unl}} \frac{\partial \log p_\phi(l,x^l)}{\partial \phi}
\end{split}
\right.
\end{displaymath}
where $g(l,x^l;\xi)$ denotes the gradient of $\log p_\xi(l,x^l)$ w.r.t. $\xi=(\theta,\zeta)$, and the two gradient components w.r.t. $\theta$ and $\zeta$ are $\partial U_\theta(x^l)/ \partial \theta$ and 
$- (\delta(l=1),...,\delta(l=L))$ respectively.

\emph{Semi-supervised learning of JRFs} over a mix of labeled and unlabeled data amounts to combining the above supervised and unsupervised training with the following semi-supervised objective:
\begin{equation}
\label{eq:jrf_semi}
\left\{ 
\begin{split}
\mathop {\max }\limits_\xi  &J(\xi)=J_\text{sup}(\xi)+\alpha J_\text{uns}(\xi)\\
\mathop {\min }\limits_\phi  &KL \left[  {p_\text{unl}}(l,{x^l})||{p_\phi }(l,{x^l}) \right] 
\end{split}
\right. 
\end{equation}
where $\alpha$ is the trade-off weight between supervised and unsupervised learning, and $\xi=(\theta,\zeta)$ is defined before.

\paragraph{The performance of the JRF model.}
The benefits of JRFs to sequence modeling and labeling are demonstrated through two sets of experiments in \citep{jrf}.
First, various traditional language models (LMs) such as Kneser-Ney (KN) smoothed n-gram LM \citep{chen1999empirical}, LSTM LM \citep{zaremba2014recurrent} and TRF LM \citep{wang2018improved} are trained on Wall Street Journal (WSJ) portion of Penn Treebank (PTB) English dataset (without using POS tags). JRF LMs are trained by using POS tags.
These models are then used to rescore the 1000-best list generated from the WSJ'92 test set, with similar experimental setup as in \citep{wang2018improved}.
The JRF model is effective in incorporating POS tags and performs the best with the lowest rescoring word error rate (WER).
Second, experiments on three sequence labeling tasks - POS tagging, chunking and NER, with Google one-billion-word dataset \citep{google1b} as the unlabeled data resource, are conducted.
It is found that the JRF based SSL consistently outperforms the CRF baseline and self-training.

\section{JEMs and JRFs for semi-supervised learning}
\label{sec:JRF_semi}

Now we have introduced the basics for establishing JEMs and JRFs for fixed-dimensional and sequential data in \secref{sec:JEM} and \secref{sec:JRF}, respectively.
As introduced in \secref{sec:EBM_SSL}, there exist two different methods for EBM based SSL - pre-training and joint-training. 
Pre-training of RBMs once received attention in the early stage of training DNNs \citep{hinton2006a}.
Encouraging SSL results have been shown recently for EBM based joint-training.
In \citep{song2018learning,zhaojoint,jrf}, state-of-the-art SSL results are reported based on EBMs and across different data modalities (images, natural languages, an protein structure prediction and year prediction from the UCI dataset repository) and in different data settings (fix-dimensional and sequence data).

So EBM-based SSL can be applied across different data modalities (fix-dimensional and sequence data), by either of pre-training and joint-training.
Thus, there are four cases.
In this section, we systematically introduce these four cases, as summarized in \tbref{tb:EBM_SSL_application}, where we take image classification and natural language labeling as representative tasks.
We continue with the notations in Definition \ref{def:EBM_NN} of EBMs parameterized by neural networks.

\begin{table}[t]
	\begin{center}
		\caption{Applications of EBMs across different domains: comparison and connection (See text for details).}
		\label{tb:EBM_SSL_application}
		\begin{tabular}{c|c|c}
			\hline
			& Image classification         & Natural language labeling     \\
			\hline \hline
			\multirow{2}{*}{Observation}  & $x\in\mathbb{R}^D$           & $x\in\bigcup_{l}\mathbb{V}^l$ \\
			& continuous, fixed-dimensional & discrete, sequence           \\
			\hline
			Label                         & $y\in\{1,2,\cdots,K\}$       & $y\in \bigcup_{l}\{1,2,\cdots,K\}^l$      \\
			\hline
			Pre-training & $U_\theta(x) = w^T h$ & $U_\theta (x)$ in \eqref{eq:u-pretraining-nlp}                \\
			\hline
			Joint-training                & $U_\theta(x,y) = \Psi_\theta(x)[y]$             & $U_\theta(x,y)$ in \eqref{eq:u-joint-nlp}              \\
			\hline
		\end{tabular}
	\end{center}	
\end{table}

\subsection{Pre-training via EBMs for SSL}
\label{sec:pre-training}

Pre-training via EBMs for SSL consists of two stages. The first stage is pre-training an EBM on unlabeled data.
It is followed by a fine-tuning stage, where we can easily use the pre-trained EBM to initialize a discriminative model and further train over labeled data.

\begin{figure}[t]
	\centering
	\centerline{\includegraphics[width=0.9\linewidth]{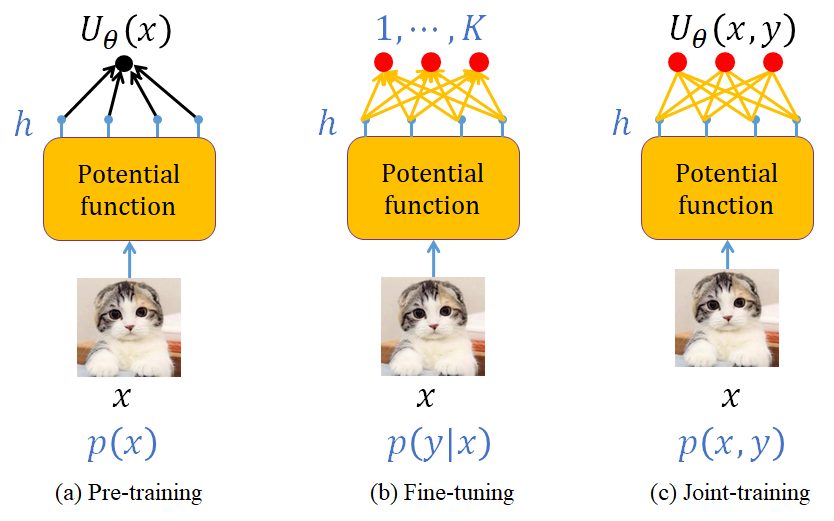}}
	\caption{Illustration of EBM based semi-supervised image classification. (a) Pre-training, (b) Fine-tuning, (c) Joint-training.}
	\label{fig:SSL_image}
\end{figure}

\paragraph{Pre-training of an EBM for semi-supervised image classification} essentially involves estimating $p_\theta(x)$ as defined in Eq. (\ref{eq:unsup-RF}) from unlabeled images.
For the potential function $U_\theta(x)$, we can use a multi-layer feed-forward neural network $\Phi_\theta(x):\mathbb{R}^{D} \rightarrow \mathbb{R}$, which, in the final layer, calculates a scalar via a linear layer, $U_\theta(x) = w^T h$, as shown in \figref{fig:SSL_image}(a).
Here $h \in \mathbb{R}^H$ denotes the activation from the last hidden layer and $w \in \mathbb{R}^H$ the weight vector in the final linear layer. For simplicity, we omit the bias in describing linear layers throughout \secref{sec:JRF_semi}.

In fine-tuning, as shown in \figref{fig:SSL_image}(b), we throw away $w$ and fed $h$ into a new linear output layer, followed by $\text{softmax}(W  h)$, to predict $y$, where $W \in \mathbb{R}^{K \times H}$ denotes the new trainable weight parameters and $y \in \left\lbrace 1,\cdots,K \right\rbrace$ the class label.
{It can be seen that pre-training aims to learn representations that may be useful for multiple downstream tasks, and information about the labels is not utilized until the fine-tuning stage.}

\paragraph{Pre-training of an EBM for semi-supervised natural language labeling (e.g., POS tagging)} is similar to the above procedure of pre-training of an EBM for semi-supervised image classification.
In pre-training, we estimate an EBM-based language model $p_\theta(x)$ from unlabeled text corpus.
Neural networks with different architectures can be used to implement the potential function $\Phi_\theta(x):\mathbb{V}^{l} \rightarrow \mathbb{R}$ given length $l$.
With abuse of notation, here $x=(x_1, \ldots, x_l)$ denotes a token sequence of length $l$, and $x_i \in \mathbb{V}, i=1,\cdots,l$.  
For example, as shown in \figref{fig:SSL_seq_labeling}(a), we can use the bidirectional LSTM based potential function in \citep{wang2018improved} as follows:
\begin{equation}\label{eq:u-pretraining-nlp}
U_\theta(x) = \sum_{i=1}^{l-1} h_{f,i}^T e_{i+1} + \sum_{i=2}^{l} h_{b,i}^T e_{i-1}
\end{equation}
where $e_i, h_{f,i}$ and $h_{b,i}$ are of the same dimensions, denoting the output embedding vector, the last hidden vectors of the forward and backward LSTMs respectively at position $i$.

In fine-tuning, as shown in \figref{fig:SSL_seq_labeling}(b), we add a CRF, as the discriminative model, on top of the extracted representations, $(h_{f,i}, h_{b,i}), i=1,\cdots,l$, to do sequence labeling, i.e., to predict a sequence of labels $y=(y_1, \ldots, y_l)$ with one label for one token at each position, where $y_i \in \left\lbrace 1,\cdots,K \right\rbrace$ denotes the label at position $i$.
Specifically, we concatenate $h_{f,i}$ and $h_{b,i}$, add a linear output layer to define the node potential, and add a matrix $A \in \mathbb{R}^{K \times K}$ to define the edge potential, as in recent neural CRFs \citep{lample2016neural,ma2016end}. The parameters to be fine-tuned are the weights in the linear output layer and the edge potential matrix $A$. 

\begin{figure}[t]
	\centering
	\centerline{\includegraphics[width=1.1\linewidth]{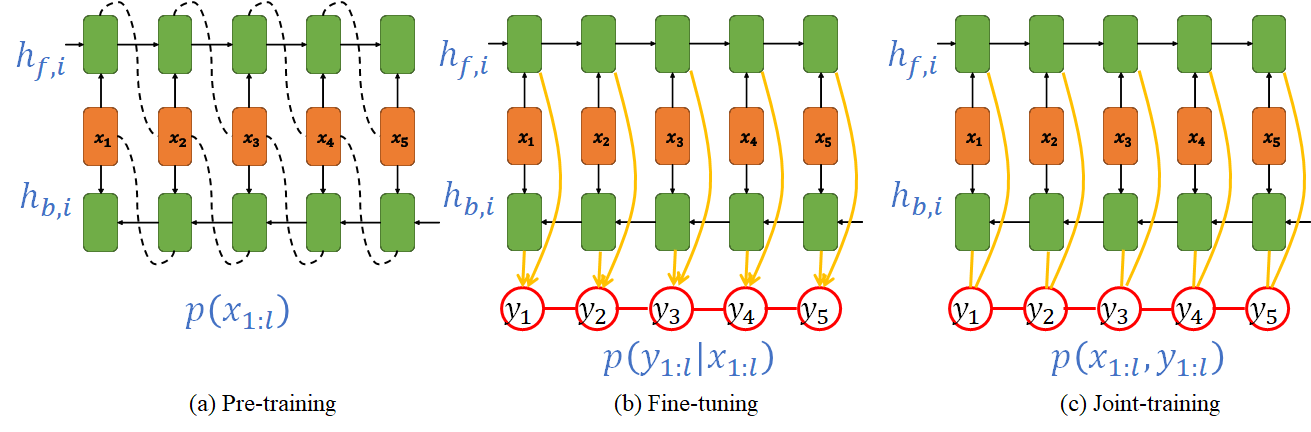}}
	\caption{Illustration of EBM based sequence labeling. (a) Pre-training, (b) Fine-tuning, (c) Joint-training.}
	\label{fig:SSL_seq_labeling}
\end{figure}

\subsection{Joint-training via EBMs for SSL}
\label{sec:joint-training}

The above pre-training via EBMs for SSL considers the modeling of only observations $x$ without labels $y$.
The joint-training refers to the joint modeling of $x$ and $y$:
\begin{equation}\label{eq:joint-RF}
p_{\theta}(x,y)=\frac{1}{Z(\theta)} \exp\left[  U_{\theta}(x,y) \right] 
\end{equation}
Then,  it can be easily seen, as detailed in \secref{sec:JEM}, that the conditional density $p_{\theta}(y|x)$ implied by the joint density Eq. (\ref{eq:joint-RF}) is:
\begin{equation}\label{eq:rf-classifier}
p_{\theta}(y|x) = \frac{p_{\theta}(x,y)}{p_{\theta}(x)}
= \frac{\exp(  U_{\theta}(x,y) )}{\sum_{y'} \exp(  U_{\theta}(x,y') )}
\end{equation}
And the implied marginal density is
\begin{equation}
p_\theta(x) = \frac{1}{Z(\theta)} \exp( U_{\theta}(x) )
\end{equation}
where, with abuse of notation, $U_\theta(x) \triangleq log \sum_y \exp\left[  U_{\theta}(x,y) \right]$.
{Different from pre-training, the unsupervised objective $p_\theta(x)$ depends on the targeted task.}
The key for EBM based joint-training for SSL is to choose suitable $U_\theta(x,y)$ such that both $p_\theta(y|x)$ and $p_\theta(x)$ can be tractably optimized.

\paragraph{Joint-training of an EBM for semi-supervised image classification} considers a neural network $\Psi_\theta(x):\mathbb{R}^{D} \rightarrow \mathbb{R}^K$, which, as shown in \figref{fig:SSL_image}(c), accepts the image $x$ and outputs an vector, whose size is equal to the number of class labels, $K$. 
Then we define $U_\theta(x,y) = \Psi_\theta(x)[y]$, where $[y]$ denotes the $y$-th element of a vector.
With the above potential definition, it can be easily seen that the implied conditional density $p_{\theta}(y|x)$ is exactly a standard $K$-class softmax based classifier, using the $K$ logits calculated by the neural network $\Psi_\theta(x)$ from the input $x$. And we do not need to calculate $Z(\theta)$ for classification. Therefore, we can conduct SSL over a mix of labeled and unlabeled data by maximizing the (weighted) sum of $\log p_\theta(y|x)$ and $\log p_\theta(x)$, where both optimizations are tractable as detailed in \citep{song2018learning}.

\paragraph{Joint-training of an EBM for semi-supervised natural language labeling} is similar to the above procedure of joint-training of an EBM for semi-supervised image classification, 
with $x=(x_1, \ldots, x_l)$ and $y=(y_1, \ldots, y_l)$, $x_i \in \mathbb{V}, y_i \in \left\lbrace 1,\cdots,K \right\rbrace, i=1,\cdots,l$. 
As shown in \figref{fig:SSL_seq_labeling}(c), we consider a neural network $\Psi_\theta(x): \mathbb{V}^{l} \rightarrow \mathbb{R}^{l \times K}$ and define  
\begin{equation}\label{eq:u-joint-nlp}
U_\theta(x,y) = \sum_{i=1}^l \Psi_\theta(x)[i,y_i] + \sum_{i=1}^l A[y_{i-1},y_i]
\end{equation}
where $[\cdot,\cdot]$ denotes the element of a matrix and  $A \in \mathbb{R}^{K \times K}$ models the edge potential for adjacent labels. With the above potential definition, it can be easily seen, as detailed in \secref{sec:JRF}, that the conditional density $p_{\theta}(y|x)$ implied by the joint density Eq.(\ref{eq:joint-RF}) is exactly a CRF with node potentials $\Psi_\theta(x)[i,y_i]$ and edge potentials $A[y_{i-1},y_i]$, and the implied marginal density $p_{\theta}(x)$ is exactly a trans-dimensional random field (TRF) language model \citep{Wang2017LearningTR,wang2017language,BinICASSP2018}. Training of both models are tractable as detailed in \citep{jrf,wang2018improved}.

\subsection{Comparison of joint-training and pre-training}

In \citep{song2021empirical}, a suite of SSL experiments are conducted on standard benchmark datasets in different domains, including the CIFAR-10 and SVHN datasets \citep{song2018learning} for image classification and the POS, chunking and NER datasets \citep{hu2019neural,jrf} for natural language labeling.
It is revealed that joint-training EBMs outperform pre-training EBMs marginally but nearly consistently. Presumably, this is because that the optimization of joint-training is directly related to the targeted task, but pre-training is not aware of the labels for the targeted task.

Full detailed experimental results are referred to \citep{song2021empirical}. Here we present some results for illustration.
In \citep{song2021empirical}, the standard data split for training and testing is used.
When changing the amount of labeled and unlabeled data for training, varying proportions (e.g., 10\%, 100\%) of labels are selected from the original full set of labeled data. 
Hereafter, the amount of labels is thus described in terms of proportions. ``100\% labeled'' means 50K images for CIFAR-10, and 56K, 7.4K, 14K sentences for POS, chunking and NER, respectively.

\begin{table}[t]
	\centering
	\caption{SSL for image classification over CIFAR-10 with 4,000 labels for a full training set of 50K images.
		The upper/lower blocks show the generative/discriminative SSL methods respectively.
		The means and standard deviations are calculated over ten independent runs with randomly sampled labels.
	}
	\label{tab:image-main-result}
	\begin{tabular}{l|c}
\hline
		Methods                                      & error (\%)                               \\
\hline \hline
		CatGAN \citep{springenberg2016unsupervised}   & 19.58$\pm$0.46                           \\
		Ladder network \citep{rasmus2015semi}         & 20.40$\pm$0.47                           \\
		Improved-GAN \citep{imporveGAN}               & 18.63$\pm$2.32                           \\
		BadGAN \citep{badGAN}                         & 14.41$\pm$0.30                           \\
		Sobolev-GAN \citep{sob-gan}                   & 15.77$\pm$0.19                           \\
		{Supervised baseline}           & 25.72$\pm$0.44                           \\
		{Pre-training+fine-tuning EBM}        & 21.40$\pm$0.38                                    \\
		{Joint-training EBM}                  & 15.12$\pm$0.36                           \\
\hline
		\multicolumn{2}{c}{Results below this line cannot be directly compared to those above.} \\
\hline
		VAT small \citep{miyato2018virtual}           & 14.87                                    \\
		Temporal Ensembling \citep{laine2016temporal} & 12.16$\pm$0.31                           \\
		Mean Teacher \citep{tarvainen2017mean}        & 12.31$\pm$0.28                           \\
\hline
	\end{tabular}
\end{table}

\paragraph{SSL for Image Classification.}
In \citep{song2021empirical}, different generative SSL methods are compared over CIFAR-10. 
As in previous works, 4,000 labeled images are randomly sampled for training. The remaining images are treated as unlabeled. 
It can be seen from \tbref{tab:image-main-result} that semi-supervised EBMs, especially the joint-training EBMs, produce strong results on par with state-of-art generative SSL methods\footnote{As discussed in \citep{song2018learning}, Bad-GANs could hardly be classified as a generative SSL method.}. Furthermore, joint-training EBMs outperform pre-training+fine-tuning EBMs by a large margin in this task.
Note that some discriminative SSL methods, as listed in the lower block in \tbref{tab:image-main-result}, also produce superior results but heavily utilize domain-specific data augmentations, and thus are not directly compared to the generative SSL methods.

\paragraph{SSL for Natural Language Labeling.}
In this experiment, different SSL methods are evaluated for natural language labeling over three tasks - POS tagging, chunking and NER.
The following benchmark datasets are used - PTB POS tagging, CoNLL-2000 chunking and CoNLL-2003 English NER, as in \citep{ma2016end,cvt,hu2019neural,jrf}.
Varying proportions of labels are sampled as labeled training data and use the Google one-billion-word dataset \citep{google1b} as the large pool of unlabeled sentences.
A large scale of experiments are conducted, covering the labeling proportions of 2\%, 10\% and 100\% with ``U/L'' (the ratio between the amount of unlabeled and labeled) of 50, 250 and 500 for three tasks, which consist of a total of 27 settings.
The network architectures in \citep{jrf} is used. After some empirical search, hyper-parameters (tuned separately for different methods) are fixed, which are used for all the 27 settings. 

\tbref{tab:nlp-compare} only show the relative numerics, absolute numerics are referred to \citep{song2021empirical}. The main observations are as follows. 
\begin{itemize}
\item The joint-training EBMs outperform the supervised baseline in 25 out of the 27 settings. Since we perform one run for each setting, this indicates 2 outliers.
\item The effects of increasing the labeling sizes on the improvements of the joint-training EBMs over the supervised baseline with a fixed ``U/L'' are mixed. 
For POS/chunking/NER, the largest improvements are achieved under 2\%/10\%/100\% labeled, respectively.
It seems that the working point where an SSL method brings the largest improvement over the supervised baseline is task dependent.
Suppose that the working point is indicated by the performance of the supervised baseline, then the SSL method brings the largest effect when the performance of the supervised baseline is moderate, i.e., neither too low nor too high.
\item Joint-training EBMs outperform pre-training EBMs in 23 out of the 27 settings marginally but nearly consistently.
A possible explanation is that pre-training is not aware of the labels for the targeted task and is thus weakened for representation learning. In contrast, the marginal likelihood optimized in joint-training is directly related to the targeted task.
\item It seems that the degrees of improvements of the joint-training EBMs over the pre-training EBMs are not much affected when varying the labeling sizes and the ``U/L'' ratios.
\end{itemize}

\begin{table}[t]
	\vspace{-2mm}
	\centering
	\caption{
		Relative improvements by joint-training EBMs compared to the supervised baseline (abbreviated as sup.) and the pre-training+fine-tuning EBMs (abbreviated as pre.) respectively.
		The evaluation metric is accuracy for POS and $F_1$ for chunking and NER.
		``Labeled''	denotes the amount of labels in terms of the proportions w.r.t. the full set of labels. ``U/L'' denotes the ratio between the amount of unlabeled and labeled data.}
	\label{tab:nlp-compare}
	\centering
		\begin{tabular}{c|l|ccc|ccc}
\hline
			\multicolumn{2}{c|}{}    & \multicolumn{3}{c|}{joint over sup.} & \multicolumn{3}{c}{joint over pre.}                                           \\
\hline
			Labeled                  & U/L                                  & POS                                 & Chunking & NER  & POS & Chunking & NER  \\
\hline
			\multirow{3}{*}{$2\%$}   & 50                                   & 7.9                                 & 16.5     & -2.7 & 4.7 & 3.4      & 3.7  \\
			& 250                                  & 12.6                                & 16.6     & 1.5  & 4.2 & 0.9      & 0.1  \\
			& 500                                  & 15.1                                & 20.3     & 4.5  & 4.1 & -0.3     & -1.5 \\
\hline
			\multirow{3}{*}{$10\%$}  & 50                                   & 5.6                                 & 18.0     & 0.9  & 3.8 & 3.0      & 5.0  \\
			& 250                                  & 6.0                                 & 18.3     & -1.2 & 3.8 & 9.4      & -0.7 \\
			& 500                                  & 8.5                                 & 21.8     & 1.0  & 5.2 & 3.7      & -4.1 \\
\hline
			\multirow{3}{*}{$100\%$} & 50                                   & 3.1                                 & 10.3     & 6.5  & 3.5 & 5.3      & 1.1  \\
			& 250                                  & 5.0                                 & 13.6     & 8.3  & 3.5 & 7.4      & 3.6  \\
			& 500                                  & 6.2                                 & 14.0     & 8.4  & 4.3 & 6.4      & 2.5  \\
\hline\hline
	\end{tabular}
\end{table}

\newpage
\section{JRFs for calibrated natural language understanding}
\label{sec:JRF_calibrate}
\subsubsection{Motivation}
In this section, we describe how the joint EBMs (JEMs), as introduced in \secref{sec:JEM}, can be used for calibrated natural language understanding.
\emph{Calibration}\index{Calibration} refers to how well a classification model's confidence (reflected by its output posterior probability) aligns with its actual accuracy.
As deep learning models achieve amazing accuracy in computer vision and natural
language processing (NLP), more research attention has been drawn to the calibration aspect of these models. 
As shown by \cite{pmlr-v70-guo17a}, the high accuracy from deep models does not always lead to better calibration. This motivates an important line of works attempting to achieve a better trade-off between accuracy and calibration.

Most existing calibration methods (see references in \cite{he2021joint}) generally rescale the posterior distribution predicted from the classifier after training. Such post-processing methods require a held-out development set with a decent number of samples to be available. 
In another line of work,  \cite{grathwohl2019your}
shows that one can train a joint EBM together with the standard training of a neural classifier. Although calibration is not explicitly addressed during EBM training, the calibration of the resulting classifier is shown to be greatly improved. 
However, the training framework proposed by \cite{grathwohl2019your} is designed for image classifiers, and it can not be readily applied to discrete text data.
\cite{he2021joint} investigates JEM training during the finetuning of pre-trained text encoders (e.g., BERT or RoBERTa) for natural language understanding (NLU) tasks.


\subsubsection{The JEM model for NLU}

\cite{he2021joint} considers the finetuning of pre-trained text encoder on NLU tasks. Assume samples from the data distribution $p_\text{data}$ are in the form of $(x,y)$ pairs, where $x$ usually refers to a single or a pair of sentences, and $y$ refers to the corresponding label. The number of classes are denoted by $|\mathcal{Y}|$.

Given input $x$, a text encoder model (e.g., BERT or RoBERTa) is firstly used to encode $x$, and the resulting embedding is denoted by $\text{enc}(x)$.
For the target classification task, a classifier $f_\text{CLS}$, which could be a simple linear transform or a multi-layer perception (MLP), will be applied to $\text{enc}(x)$. 
The output logits are denoted as $f_\text{CLS}(\text{enc}(x))$, whose dimension is equal to the number of possible classes $|\mathcal{Y}|$.
The $y$-th logit is denoted by $f_\text{CLS}(\text{enc}(x))[y]$. 
The posterior distribution $p_\theta(y|x)$ is obtained by applying a softmax operation to the logits, where $\theta$ refers to the parameters in the model.

In standard finetuning, the cross-entropy (CE)
loss and gradient based optimizers are used to train
the classifier:
\begin{equation}
\mathcal{L}_{CE} = \mathbb{E}_{(x,y)\sim p_\text{data}}(-\log p_\theta(y|x)).
\end{equation}
In the following, we will introduce how a JEM model can be defined on top of the text encoder for calibrated NLU.

\begin{figure}[t]
	\centering
	\centerline{\includegraphics[width=0.8\linewidth]{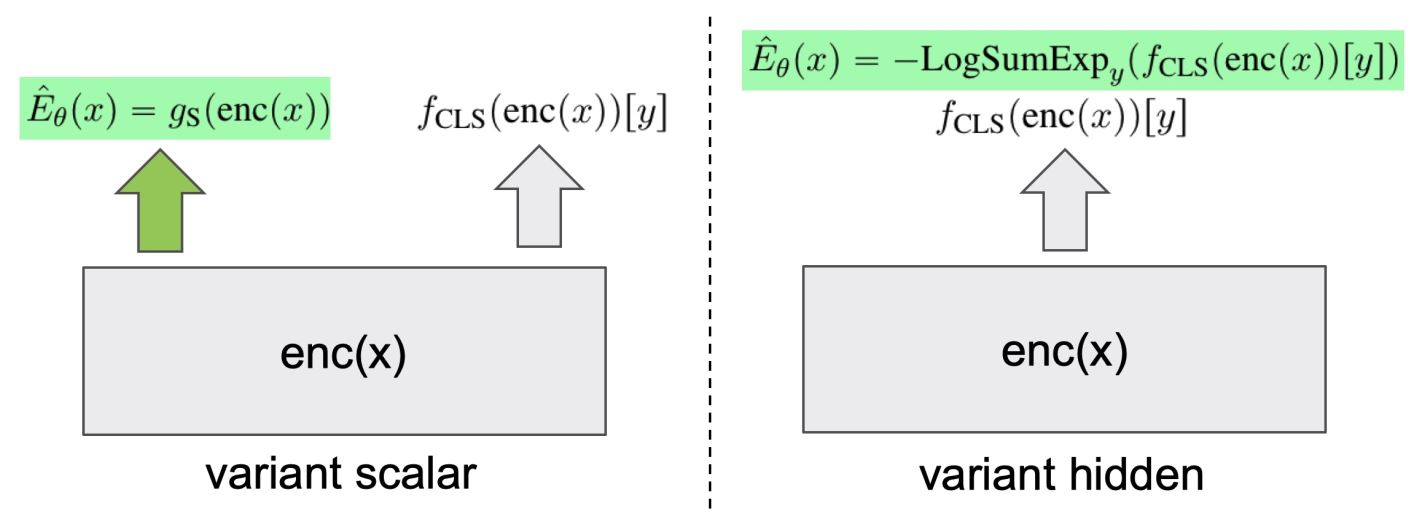}}
	\caption{Comparison of the \emph{scalar} and the \emph{hidden} variants of energy functions. The modules introduced for EBM are shaded in green. \citep{he2021joint}}
	\label{fig:calibrated_JEM}
\end{figure}

\cite{he2021joint} considers a JEM model in a residual form, similar to \secref{sec:residual_ELM} and \secref{sec:residual_ebm}.
The energy function for modeling the marginal distribution of $x$ is defined as follows:
\begin{equation}
E_\theta(x) = -\log p_\text{N}(x) + \hat{E}_\theta(x)
\end{equation}
where $p_\text{N}(x)$ is the base distribution, which will also be used as the noise distribution during NCE training.
\cite{he2021joint} examines three variants of the residual energy function $\hat{E}_\theta(x)$.
\begin{itemize}
	\item \textbf{Variant \emph{scalar}.}~Another linear layer $g_\text{S}$ is introduced to define the energy function, whose output is a scalar:
	\begin{displaymath}
		\hat{E}_\theta(x) = g_\text{S}(\text{enc}(x)).
	\end{displaymath}
	\item \textbf{Variant \emph{hidden}.}~Similar to \citep{grathwohl2019your,song2018learning}, 
	starting from a softmax based classifier, an EBM can be defined with the logits as follows:
	\begin{displaymath}
\hat{E}_\theta(x) = -\text{LogSumExp}_{y=1}^{|\mathcal{Y}|}(f_\text{CLS}(\text{enc}(x))[y]).
\end{displaymath}
Different from the \emph{scalar} variant, here the energy
function directly uses the logits for prediction (visualized in \figref{fig:calibrated_JEM}). Hence the impact on the
model's classification behavior could be larger.
	\item \textbf{Variant \emph{sharp-hidden}.}~The \emph{hidden} variant has a potential weakness that, the correlation between input $x$ and the prediction $y$ is not addressed because the energy is distributed among all the logits.
Motivated by this potential issue, the
following \emph{sharp} variant is proposed, which can be viewed as an approximation to the \emph{hidden} variant, and is found to work well in practice \citep{he2021joint}.
	\begin{displaymath}
\hat{E}_\theta(x) = - \max_y f_\text{CLS}(\text{enc}(x))[y].
\end{displaymath}
\end{itemize}

\subsubsection{NCE training of the JEM model}
Similar to the discussion in \secref{sec:residual_ebm}, NCE is applied to train the residual JEM model.
NCE trains the model to discriminate between real samples from $p_\text{data}$ and
noise samples from a given noise distribution $p_\text{N}$.
The NCE loss is the same as Eq.~(\ref{eq:residual_nce}):
\begin{displaymath}
        \mathcal{L}_\text{NCE}=-\mathop{\mathbb{E}}\limits_{x_+ \sim p_\text{data}} \log\frac{1}{1+\nu \exp(E_\theta(x_+))} - \nu \mathop{\mathbb{E}}\limits_{x_- \sim p_\text{N}} \log\frac{1}{1+\exp(-E_\theta(x_-))/\nu}
\end{displaymath}
where $\nu$ is the ratio of noise samples to real samples.

In \citep{he2021joint}, $\mathcal{L}_{CE}$ and $\mathcal{L}_{NCE}$ are jointly optimized during training, that is:
\begin{displaymath}
    \mathcal{L}_\text{joint}=\mathcal{L}_\text{CE}+\mathcal{L}_\text{NCE}
\end{displaymath}

For constructing the noise distribution $p_\text{N}(x)$, \citep{he2021joint} finetunes the GPT-2 language model \citep{gpt2} with samples from the target training set. 
However during NCE training, the energy model is found to easily discriminate between data samples and noise samples, which makes training ineffective \citep{he2021joint}. 
To alleviate this issue, an objective similar to the masked language model (MLM) loss \citep{bert} is adopted during the finetuning of the noise model.
With a given mask ratio $M$ (e.g., 0.4), \cite{he2021joint} randomly masks part of $x$, and trains the model to complete it:
\begin{displaymath}
\mathcal{L}_\text{MLM}=-\mathop{\mathbb{E}}\limits_{ x \sim p_\text{data},x^m \sim \text{mask}(x, M)} \log p_\text{N}(x|x^m)
\end{displaymath}
During noise sample generation, adopting the same mask ratio $M$, a masked $x^m$ is fed to $p_\text{N}$ ($x$ is from the training set), and the generated sample are used as the noise sample. In this way, the noise distribution is made closer to the data distribution.


\subsubsection{Evaluation of the Calibration performance}

To measure calibration performance, expected calibration error (ECE)\index{Expected calibration error (ECE)} metric is used, following \citep{pmlr-v70-guo17a,grathwohl2019your}.
Given an input sample $x$, for each label $y$, we say
that the model predicts that $x$ belongs to label $y$
with confidence $p_\theta(y|x)$. Assuming the test-set
contains $n$ samples, we will have $n \times |\mathcal{Y}|$ predictions.

ECE first partitions all predictions into $B$
equally-spaced bins by its confidence. Following
\citep{grathwohl2019your}, $B=20$ is used, which means the width of each bin is
$0.05$. For example, the first bin contains all predictions that have confidence in the range of $[0,0.05)$.
Then for each bin ECE computes how the average
of confidence is different from its actual accuracy:
\begin{equation}
    ECE=\frac{1}{|\mathcal{Y}|}\sum_{y=1}^{|\mathcal{Y}|}\sum_{b=1}^B \frac{|B_{yb}|}{n} |\text{acc}(B_{yb}) - \text{conf}(B_{yb})|
\end{equation}
where $n$ is the number of test samples. $\text{acc}(B_{yb})$ is the ratio of samples ($x$) whose true label is indeed $y$ in the bin $B_{yb}$, and $\text{conf}(B_{yb})$ is the average confidence in that bin. 

\begin{figure}[t]
	\centering
	\centerline{\includegraphics[width=0.8\linewidth]{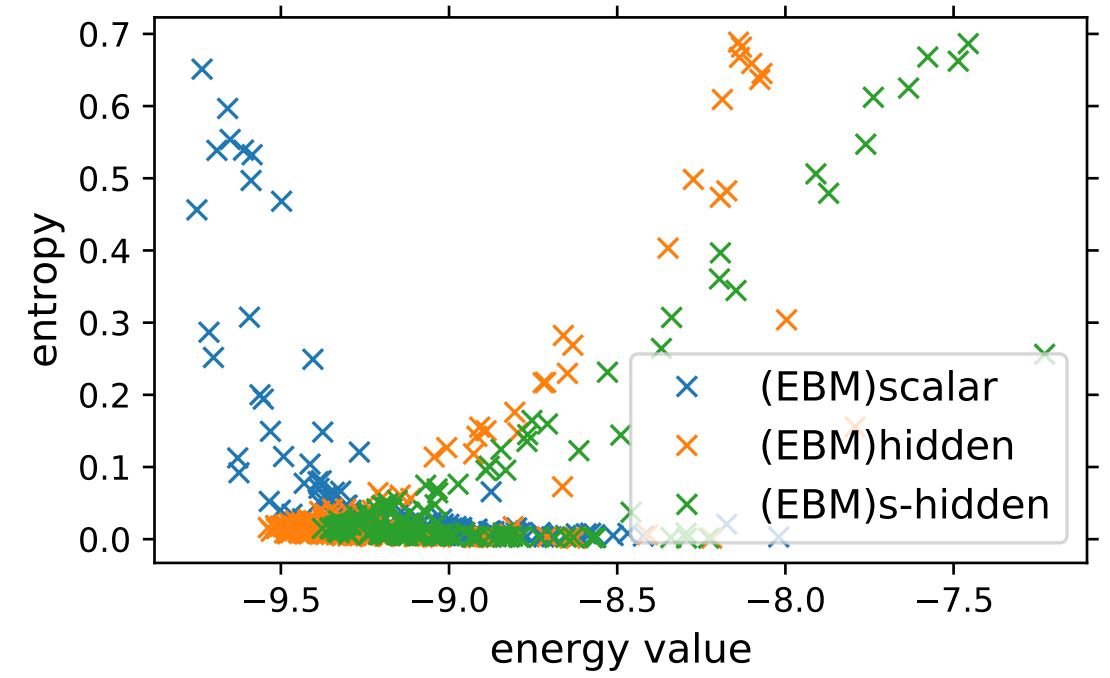}}
	\caption{The entropy of the posterior ($p_\theta(\cdot|x)$) versus energy value $\hat{E}_\theta(x)$ for SST-2 test-set samples. \citep{he2021joint}}
	\label{fig:calibrated_analysis}
\end{figure}

\cite{he2021joint} uses the RoBERTa-base model as the text encoder and finetune it on eight GLUE tasks \citep{wang2018glue}.
It is found that EBM training is able to reach a better trade-off between accuracy and calibration.
In most tasks, all three EBM variants get substantial improvement in ECE with little or no loss in accuracy comparing to the (strong) baseline methods. 

How does the model get better calibration?  \figref{fig:calibrated_analysis} from \cite{he2021joint} give some analysis.
\figref{fig:calibrated_analysis} shows the energy value $\hat{E}_\theta(x)$ versus the entropy\index{Entropy} of the posterior distribution
\begin{displaymath}
H(p_\theta(\cdot|x)) = - \sum_{y=1}^{|\mathcal{Y}|} p_\theta(y|x) \log p_\theta(\cdot|x)
\end{displaymath}
for samples in the SST-2 test set. 
It is shown that models trained with the \emph{hidden} and \emph{sharp-hidden} variants tend to assign more conservative
predictions (reflected by higher entropy) for higher-energy (less likely) samples. 
It is hypothesized that this is due to the strong coupling between the energy function and the classification logits. 
However, this interesting trend (\figref{fig:calibrated_analysis}) is not observed in all datasets (e.g., QNLI).

\chapter{Conclusion}
\label{ch:conclusion}

\section{Summary}
Energy-based models (EBMs) form an important aspect of modern artificial intelligence and signal processing.
Unlike most other probabilistic models which are self-normalized, EBMs specify probability density functions up to an unknown normalizing constant, and thus are also referred to as \emph{non-normalized probabilistic models}\index{Non-normalized probabilistic model}. Such model definition of not imposing normalization enables a great power and flexibility to the modeling process. One is generally free to choose any nonlinear regression function as the energy function, as long as it remains normalizable in principle. Accompanied with such flexibility, we have also shown the advantages of EBMs in naturally overcoming label bias and exposure bias suffered by locally-normalized models (\secref{sec:bias}) and in hybrid generative-discriminative and semi-supervised learning (\secref{ch:jem}).

On the other hand, although the flexibility of EBMs can provide significant modeling advantages, both computation of the exact likelihood and exact sampling from these models are generally intractable, which makes training of EBMs especially difficult and limits their applications.
Moreover, the sequential nature of speech and language also presents special challenges and needs treatment different from processing fix-dimensional data (e.g., images).

We are making progress. This monograph presents a systematic introduction to energy-based models, including both algorithmic progress and applications in speech and language processing, which is organized into four chapters.

In Chapter \ref{ch:basics}, we provide a self-contained introduction to basics for EBMs, including model definitions, learning algorithms, and sampling/generation algorithms.
There are two remarkable features in our introduction.
First, we starts with the framework of probabilistic graphical models (PGMs). The PGM framework enables readers to appreciate EBMs from the perspective of undirected graphical models and easily understand the differences between undirected graphical models and directed graphical models, and between globally-normalized and locally-normalized.
Graphical models provide a natural tool for formulating variations on classical methods, as well as for exploring entirely new models.
Second, our introduction to the stochastic approximation (SA) methodology is very useful for readers to develop new algorithms for learning EBMs, particularly learning with Monte Carlo sampling. The SA methodology is more general than the ordinary stochastic gradient descent (SGD) technique, while SGD is only one instance of SA.

The next three chapters are dedicated to present how to apply EBMs in three different scenarios, i.e., for modeling marginal, conditional and joint distributions, respectively.
As visualized in \figref{fig:outline}, our such organization is comprehensive and distinctive.
A wide range of applications in speech and language processing are covered, including language modeling, representation learning over text, speech recognition, sequence labeling in NLP, text generation, semi-supervised learning, and calibrated natural language understanding.

\section{Future challenges and directions}

EBM based methods represent an important class for the probabilistic approach to many fields. Despite the progress achieved in these years, much more work needs to be carried out.
\begin{itemize}
\item \textbf{Applications of EBMs are still mainly limited by lack of effective and efficient training techniques.}
Training techniques are crucial to problem solving with EBMs, and will remain an active direction for future research.
	EBMs have been especially difficult to train. 
	We mainly introduce maximum likelihood training with MCMC sampling, noise-contrastive estimation (NCE) and briefly covers score matching (SM).
	There has been persistent and ongoing interest in developing effective modeling and training techniques for learning EBMs (see \cite{zhang2023persistently} and its references), but those methods are mostly first studied for fix-dimensional data. There is vast room to improve training techniques for EBMs, especially for \emph{sequential data} such as speech and language.
	
	Moreover, as we have discussed, learning EBMs over \emph{discrete data} are more challenging than over continuous data, due to the combinatorial explosion of the state space. 
	Langevin sampling, a particular MCMC sampling method which utilizes gradients, has been dominantly used in training EBMs over continuous data. A recent progress in	\citep{grathwohl2021oops,qin2022cold} begins to use gradients
of the likelihood function with respect to its discrete inputs to propose updates in a Metropolis-Hastings sampler. More further research is needed.
\item \textbf{More downstream applications of EBMs are worthwhile for exploration.}
Considering the potential advantages of EBMs in modeling flexibility, overcoming label bias and exposure bias, and hybrid generative-discriminative and semi-supervised learning, the results reviewed in this monograph are only preliminary. There are many interesting tasks which could be approached by EBMs.

The applications of EBMs require more skill and experience, apart from using the mainstream deep learning toolkits.
To help the readers to get familiar with the techniques for developing and applying energy-based models, we summarize some open-source toolkits in \appref{sec:toolkit}.
\item \textbf{We look forward to more foundational and interdisciplinary research around EBMs.}
There is an opportunity for developing \emph{physics-based methods} that could address the difficulty of calculating or sampling the partition function of EBMs \citep{huembeli2022physics}. 
There are emergent areas of research at the interfaces of machine learning, quantum computing, many-body physics, and optimization (see references in \cite{huembeli2022physics}).

Many recent \emph{biologically plausible} algorithms utilize the framework of energy-based models \citep{scellier2017equilibrium,millidge2022backpropagation}. Biologically plausible algorithms compute gradients that approximate those computed by back-propagation (BP), and operate in ways that more closely satisfy the constraints imposed by neural circuitry.
We anticipate that progress in neuroscience and EBM based machine learning will benefit from an interplay between both fields.
\end{itemize}


\begin{acknowledgements}
Thanks to collaborators and students: Zhiqiang Tan, Bin Wang, Hongyu Xiang, Yunfu Song, Kai Hu,
Keyu An, Huahuan Zheng, Silin Gao, Hong Liu, Junlan Feng, and Yi Huang.

\noindent Thanks for funding supports from 
\begin{itemize}
	\item NSFC (National Science Foundation of China) through No. 60402029, No. 61075020, No. 61473168, and No. 61976122;
	\item Ministry of Education and China Mobile joint funding through No. MCM20170301;
	\item Joint Institute of Tsinghua University - China Mobile Communications Group Co. Ltd.;
	\item Beijing National Research Center for Information Science and Technology;
	\item Tsinghua Initiative through No. 20121088069 and No. 20141081250;
	\item Toshiba Corporation;
	\item Apple Corporation.
\end{itemize}
\end{acknowledgements}

\appendix

\chapter{Notations and definitions}
\vspace*{-1.2in}

\section{Notations}
\label{sec:notations}

\begin{center}
\begin{xltabular}{\textwidth}{l|X}
\hline
Example & Description \\
\hline\hline
$z_{i:j}$ & For any generic sequence $\left\lbrace z_n\right\rbrace $, we shall use $z_{i:j}$ to denote $z_i, z_{i+1}, \cdots\, z_j$.
Similarly, wherever a collection of indices appears in the subscript, we refer to the corresponding collection of indexed variables, e.g., $c_{l,1:H} \triangleq \left\lbrace  c_{l,1}, c_{l,2}, \cdots\, c_{l,H} \right\rbrace$. \\
\hline
$x$ & $x$ generally denotes a random variable, which can either be scalar- or vector-valued, and often denotes the observation variable. For simplicity, we also use the same notation $x$ to denote the values taken by the random variable $x$, e.g., in the argument of its density function, which should be clear from the context.\\ 
\hline
$h$ & The hidden variable. \\
\hline
$y$ & The class label, or the output variable. \\
\hline
$|\mathcal{B}|$ & The cardinality/size of a set $\mathcal{B}$ \\
\hline
$x^T, A^T$ & A superscript $T$ denotes the transpose of a vector $x$ or matrix $A$ \\
\hline
$\Delta^K$ & The $K$-dimensional probability simplex. \\
\hline
$\sum_x f(x)$ & The summation over $x$ is a shorthand, which should be an appropriate combination of summation and integration, depending on the components of $x$ being discrete variables, continuous variables, or a combination of the two.\\
\hline
$p_\text{ora}(\cdot)$ & The (unknown) oracle density, sometimes also known as the data distribution and denoted as $p_\text{data}(\cdot)$.\\
\hline
$p_{\text{emp}}(\cdot)$ & The empirical density. For a training dataset consisting of $n$ independent and identically distributed (IID) data points $\left\lbrace x_1, \cdots, x_N \right\rbrace $, we have 
\[p_{\text{emp}}(x) \triangleq \frac{1}{N} \sum_{i=1}^{N} \delta(x-x_i)\]\\
\hline
$p_{\theta}(\cdot), p(\cdot;\theta)$ & The (target) model density, parameterized by $\theta$. \\
\hline
$q_{\phi}(\cdot), q(\cdot;\phi)$ & The auxiliary density introduced in training, parameterized by $\phi$. \\
\hline
$\text{Uni}[a,b]$ & Uniform distribution for a continuous variable over interval $[a,b]$, or for a discrete variable over integers from $a$ to $b$.\\
\hline
 \hline
\end{xltabular}
\end{center}

\section{Definitions}
\label{sec:definitions}

\begin{center}
\begin{xltabular}{\textwidth}{l|X}
 \hline
 Term & Description \\
 \hline\hline
$\sigma(v)$ & The sigmoid function, $\sigma(v) \triangleq \frac{1}{1+e^{-v}}$, also called the logistic sigmoid function. It is also known as a squashing function, since it maps the whole real line to [0, 1], which is necessary for the output to be interpreted as a probability.\\
\hline
$\text{logit}(\sigma)$ & The logit function, $\text{logit}(\sigma) \triangleq \log(\frac{\sigma}{1-\sigma})$ for $0 < \sigma <1$, also known as the inverse of the sigmoid function. It represents the log of the ratio of probabilities for two classes, also known as the log odds.\\
\hline
$\text{softmax}(z_{1:K})$ & The softmax function, $\text{softmax}(z_{1:K})_k \triangleq \frac{\exp(z_k)}{\sum_{j=1}^K \exp(z_j)}$, which realizes normalization from $\mathbb{R}^K$ to $\Delta^K$ (the $K$-dimensional probability simplex). It is also known as the normalized exponential and can be regarded as a multiclass generalization of the logistic sigmoid.  \\
\hline$\delta(x=a)$ & An indicator function of $x$ which takes the value 1 when $x=a$ and 0 otherwise.\\
\hline$H[q]$ & The entropy is defined as $H[q] \triangleq - \int q log q$. \\
\hline
$KL[p||q]$ & 
The inclusive KL-divergence between two distributions $p(\cdot)$ and $q(\cdot)$ is defined as $KL[p||q] \triangleq \int p log \left( \frac{p}{q} \right) $, which by default is called the KL-divergence, and is sometimes also referred to as the forward KL-divergence, relative entropy. \\
\hline
$KL[q||p]$ & The exclusive KL-divergence is defined as $KL[q||p] \triangleq \int q log \left( \frac{q}{p} \right) $, which is also referred to as the reverse KL-divergence. \\
\hline
\end{xltabular}
\end{center}

\chapter{Background material}
\vspace*{-1.2in}

\section{Maximum entropy models}
\label{sec:maxent}

\begin{theorem}\label{th:maxent}
When confronted by a probability distribution $p(x)$ about which only a
few facts are known, the maximum entropy principle (\emph{maxent}) offers a
rule for choosing a distribution that satisfies those constraints \citep{cover1999elements,mackay2003information}. According to maxent, one should select the $p(x)$ that maximizes the entropy
\begin{equation}\label{eq:entropy}
    H(p) = - \sum_x p(x) \log p(x)
\end{equation}
subject to the constraints. When there is a reference distribution $q(x)$, one should select the $p(x)$ that minimizes the relative entropy or Kullback-Leibler divergence\footnote{When $q(x)$ is uniform, this is the same as maximizing the entropy.}
\begin{equation}\label{eq:relative_entropy}
    KL(p||q) = \sum_x p(x) \log \frac{p(x)}{q(x)}
\end{equation}
Assuming the constraints assert that the averages of certain functions $f_k(x)$ are known, i.e.,
\begin{equation} \label{eq:maxent_constraints}
    E_{p(x)} \left[ f_k(x) \right] = F_k, k=1,2,\cdots
\end{equation}
Then, it can be shown that by introducing Lagrange multipliers (one for each constraint, including normalization), 
\begin{itemize}
    \item The distribution that
maximizes the entropy has the following form
\begin{equation} \label{eq:maxent}
 p^*(x) = \frac{1}{Z} \exp \left(\sum_k w_k f_k(x) \right)
\end{equation}
\item The distribution that minimizing relative entropy relative to $q(x)$, has the following form
\begin{equation} \label{eq:maxent_relative}
 p^*(x) = \frac{1}{Z} q(x) \exp \left(\sum_k w_k f_k(x) \right)
\end{equation}
\end{itemize}
where $\{w_k\}$ are set such that the constraints \eqref{eq:maxent_constraints} are satisfied, and $Z$ is the normalizing constant. The two forms in \eqref{eq:maxent} and \eqref {eq:maxent_relative} are often collectively referred to as \emph{maximum entropy distributions}.
\end{theorem}

\thref{th:maxent} gives the form of maximum entropy distributions that satisfy certain moment constraints.
In an opposite way, when given that a distribution satisfies the form of \eqref{eq:maxent} or \eqref {eq:maxent_relative}, the following theorem establish the connection between the maximum entropy distribution and the maximum likelihood distribution.

\begin{theorem}\label{th:maxent_ML}
Assume that a variable $x$ comes from a probability distribution of the form in \eqref{eq:maxent} or \eqref {eq:maxent_relative}, where the functions $f_k(x)$ are given, and the parameters $\{w_k\}$ are not known. A dataset $\{x^{(n)}\}$ is supplied. Then, it can be shown that by differentiating the log likelihood, the maximum-likelihood (ML) parameters $w_{\text{ML}}$ satisfy
\begin{equation} \label{eq:model_empirical_equal}
\begin{split}
    E_{p(x)} \left[ f_k(x) \right] &= \frac{1}{N} \sum_n f_k(x^{(n)}), k=1,2,\cdots\\
    &= E_{p_{\text{emp}}(x)} \left[ f_k(x) \right]
\end{split}
\end{equation}
where the left-hand is the model average under the fitted model, the right-hand the empirical average over the training data points, and $p_{\text{emp}}(\cdot)$ denotes the empirical density over the training data points.
\end{theorem}

Combining the above two theorems, we can easily see that maximum entropy fitting with $F_k$'s being set as the empirical averages is equivalent to maximum likelihood fitting of a log-linear distribution \citep{mackay2003information,inducingfeatures}.


\section{Fisher equality}
\label{sec:fisher_eq}

Formally, for any density function $p_\theta(x)$, the partial derivative w.r.t. $\theta$ of the log density function, $\frac{\partial}{\partial \theta} log p_\theta(x)$, is called the ``score''. Under certain regularity conditions, the expectation of the score w.r.t. the density itself is 0. This formula is often referred in presenting Fisher information\footnote{\url{https://en.wikipedia.org/wiki/Fisher_information}}, so we call it \emph{Fisher equality}\index{Fisher equality}, which, is frequently used in this monograph.
\begin{equation}
\label{eq:fisher_eq_1}
E_{p_\theta(x)}\left[  \frac{\partial}{\partial \theta} \log p_\theta(x) \right] = 0.
\end{equation}

Further, based on the above basic Fisher equality, we have the following very useful theorem.
\begin{theorem}
Consider any latent-variable model $p_\theta(x,h)$, which consisting of observation $x$ and latent variable $h$, then we have
\begin{equation}
\label{eq:fisher_eq_2}
\frac{\partial}{\partial \theta} \log p_\theta(x)
=E_{p_\theta(h|x)}\left[ \frac{\partial}{\partial \theta} \log p_\theta(x,h)\right]
\end{equation}
which means that the gradient of the log marginal likelihood is equal to the expected log joint likelihood, where the expectation is taken over the posteriori distribution.
\end{theorem}
\begin{proof}
\begin{align*}
\label{eq:fisher_eq_2}
\frac{\partial}{\partial \theta} \log p_\theta(x)
&=E_{p_\theta(h|x)}\left[ \frac{\partial}{\partial \theta} \log p_\theta(x)\right]
\\&=E_{p_\theta(h|x)}\left[ \frac{\partial}{\partial \theta} \log p_\theta(x,h) - \frac{\partial}{\partial \theta} \log p_\theta(h|x) \right]\\
&=E_{p_\theta(h|x)}\left[ \frac{\partial}{\partial \theta} \log p_\theta(x,h)\right]
\end{align*}
where in the second line, according to Fisher equality, we have
\begin{displaymath}
E_{p_\theta(h|x)}\left[ \frac{\partial}{\partial \theta} \log p_\theta(h|x) \right] = 0,
\end{displaymath}
and thus we obtain the final line.
For simplicity, \eqref{eq:fisher_eq_2} is also referred to as Fisher equality.
\end{proof}

\chapter{Open-source toolkits related to EBMs}
\vspace*{-1.2in}
\label{sec:toolkit}

\begin{itemize}
\item Trans-dimensional random field (TRF) LMs: \url{https://github.com/thu-spmi/SPMILM}
\item Energy-based cloze models for representation learning over text (Electric): \url{https://github.com/google-research/electra}
\item CRF-based ASR Toolkit (CAT): \url{https://github.com/thu-spmi/CAT}
\item Neural CRF Transducers for Sequence Labeling: \url{https://github.com/thu-spmi/SPMISeq}
\item Controlled text generation from pre-trained language models (mix-and-match): \url{https://github.com/mireshghallah/mixmatch}
\item Learning neural random fields with inclusive auxiliary generators: \url{https://github.com/thu-spmi/Inclusive-NRF}
\item JEMs and JRFs for semi-supervised learning: \url{https://github.com/thu-spmi/semi-EBM}
\end{itemize}

%
%


\backmatter  

\printbibliography

\printindex

\end{document}